%% file: main.tex
\documentclass{article}

\usepackage[utf8]{inputenc} %
\usepackage[T1]{fontenc}    %
\usepackage{hyperref}       %
\usepackage{url}            %
\usepackage{booktabs}       %
\usepackage{amsfonts}       %
\usepackage{nicefrac}       %
\usepackage{microtype}      %
\usepackage{xcolor}         %
\usepackage{amsmath}
\usepackage{amsthm}
\usepackage{algorithm}
\usepackage{algpseudocode}
\algrenewcommand\algorithmiccomment[1]{\hfill \textcolor{blue}{//} #1}
\usepackage{cleveref}
\usepackage{comment}
\usepackage{longtable}
\usepackage{array}
\usepackage{pdflscape}
\usepackage{microtype}
\usepackage{graphicx}
\usepackage{afterpage}
\usepackage{subcaption}
\usepackage{multirow}
\usepackage{placeins}
\usepackage{tcolorbox}
\usepackage{minitoc}

\usepackage{bm}
\usepackage{bbm}

\usepackage{dsfont}
\usepackage{calc}

\usepackage{selectp}

\usepackage{iclr2026_conference,times}

\usepackage{tikz}
\usepackage{fontawesome5}
\usepackage{float}

\usetikzlibrary{
    arrows.meta,
    calc,
    fit,
    positioning,
    shadows,
    shadows.blur
}

\usepackage{aliascnt}

\newtheorem{thm}{Theorem}[section]

\newaliascnt{lemma}{thm}
\newtheorem{lemma}[lemma]{Lemma}
\aliascntresetthe{lemma}
\crefname{lemma}{Lemma}{Lemmas}
\Crefname{lemma}{Lemma}{Lemmas}

\newaliascnt{proposition}{thm}
\newtheorem{proposition}[proposition]{Proposition}
\aliascntresetthe{proposition}
\crefname{proposition}{Prop.}{Prop.}
\Crefname{proposition}{Prop.}{Prop.}

\newaliascnt{definition}{thm}
\newtheorem{definition}[definition]{Definition}
\aliascntresetthe{definition}
\crefname{definition}{Def.}{Def.}
\Crefname{definition}{Def.}{Def.}

\newaliascnt{assumption}{thm}
\newtheorem{assumption}[assumption]{Assumption}
\aliascntresetthe{assumption}
\crefname{assumption}{Asm.}{Assumptions}
\Crefname{assumption}{Asm.}{Assumptions}

\newaliascnt{corollary}{thm}
\newtheorem{corollary}[corollary]{Corollary}
\aliascntresetthe{corollary}
\crefname{corollary}{Cor.}{Cor.}
\Crefname{corollary}{Cor.}{Cor.}

\newcolumntype{P}[1]{>{\centering\arraybackslash}p{#1}}
\crefname{ineq}{ineq.}{ineq.}
\crefname{equation}{Eq.}{Eqs.}
\crefname{theorem}{Thm.}{Thm.}
\crefname{claim}{Claim}{Claims}
\crefname{algorithm}{Alg.}{Alg.}
\crefname{appendix}{Appx.}{Appx.}
\crefname{figure}{Fig.}{Figures}
\crefname{table}{Tab.}{Tables}
\crefname{section}{Sec.}{Sec.}

\title{Inducing Uncertainty on Open-Weight Models for Test-Time Privacy in Image Recognition}

\author{Muhammad H. Ashiq,$^{\input{logos/minilogos/logowisc}}$ Peter Triantafillou,$^{\input{logos/minilogos/logowarwick}}$ Hung Yun Tseng,$^{\input{logos/minilogos/logowisc}}$ Grigoris G. Chrysos$^{\input{logos/minilogos/logowisc}}$\\
        \raisebox{-0.5\height+0.5ex}{\includegraphics[width=0.35cm]{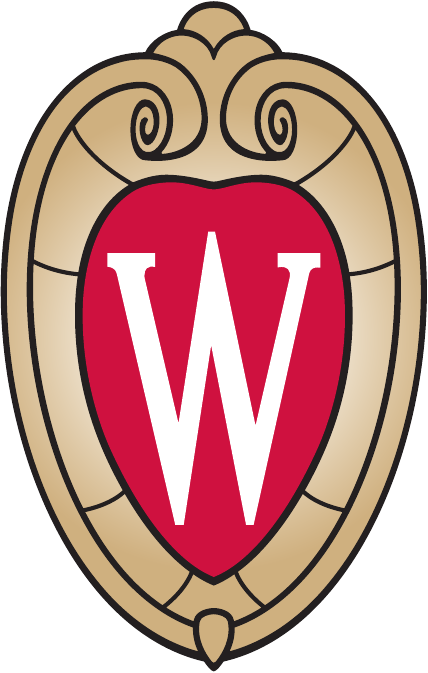}}: University of Wisconsin-Madison, USA\\
        \raisebox{-0.5\height+0.5ex}{\includegraphics[width=0.4cm]{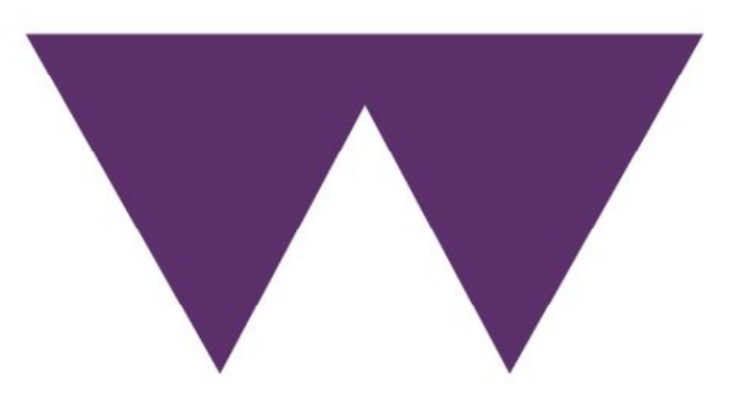}}: University of Warwick, UK\\
\footnotesize{\texttt{\{ashiq, htseng23, chrysos\}@wisc.edu}, \texttt{p.triantafillou@warwick.ac.uk}}\\
}

\iclrfinalcopy
\begin{document}
\doparttoc %
\faketableofcontents %

\part{} 

\maketitle

\input{sections/abstract}

\input{sections/intro_new}

\input{sections/related_work}

\input{sections/problem_formulation}

\input{sections/algorithm}

\input{sections/algorithm_study}

\input{sections/experiments}

\input{sections/discussion}

\input{sections/reproducibility}

\input{sections/acknowledgements}

\bibliographystyle{plainnat}
\bibliography{main}

\input{sections/appendix}

\end{document}

%% file: logos/minilogos/logowisc.tex
\includegraphics[width=0.3cm]{logos/figs/logowisc.pdf}

%% file: logos/minilogos/logowarwick.tex
\includegraphics[width=0.3cm]{logos/figs/logowarwick.pdf}

%% file: sections/abstract.tex
\begin{abstract}

A key concern for AI safety remains understudied in the machine learning (ML) literature: how can we ensure users of ML models do not leverage predictions on incorrect personal data to harm others? This is particularly pertinent given the rise of open-weight models, where simply masking model outputs does not suffice to prevent adversaries from recovering harmful predictions. To address this threat, which we call \textit{test-time privacy}, we induce maximal uncertainty on protected instances while preserving accuracy on all other instances. Our proposed algorithm uses a Pareto optimal objective that explicitly balances test-time privacy against utility. We also provide a certifiable approximation algorithm which achieves $(\varepsilon, \delta)$ guarantees without convexity assumptions. We then prove a tight bound that characterizes the privacy-utility tradeoff that our algorithms incur. Empirically, our method obtains at least $>3\times$ stronger uncertainty than pretraining with marginal drops in accuracy on various image recognition benchmarks. Altogether, this framework provides a tool to guarantee additional protection to end users.

\end{abstract}

%% file: sections/intro_new.tex
\input{sections/motivational_figure.tex}

\section{Introduction}\label{section:intro}

Data privacy is increasingly important for large-scale machine learning (ML), where models are often trained on sensitive user instances \citep{GDPR2016}. Furthermore, open-weight image recognition models, where users have access to the model parameters and architecture, have proliferated  \citep{torchvision2016,googlevit,microsoftresnet50}. 

Yet, there has been little work done to address privacy threats to ML models due to incorrect personal data, especially data which are public such as images posted to public forums. Concretely, suppose a model provider trains an open-weight medical imaging model $f$ which classifies skin images as harmless ailments like ``Benign Keratosis'' or serious diseases like ``Melanoma'' \citep{sun2016benchmark}. Next, a health insurance company scrapes images from public forums to build risk profiles. Then, this health insurance company downloads the open-weight model $f$ to automatically screen images for potential health liabilities. In particular, a person $p$ posts a photo of a harmless birthmark to a public health forum to ask a question. During the upload, a compression error causes the image file to become corrupted, severely distorting the birthmark. This results in an image $\bm{x}_p$. When the health insurance company feeds $\bm{x}_p$, scraped from the public forum, into $f$, it confidently classifies $\bm{x}_p$ as ``Melanoma''. This erroneous classification is then automatically added to person $p$'s risk profile, resulting in person $p$ being unfairly denied coverage.\footnote{Recently, there has been significant progress in building effective image recognition models for skin disease, making this problem pertinent \citep{yang2018clinical}; \citep{liu2025skin}.} We call this threat model \textit{test-time privacy} (TTP), and make this concrete in \cref{fig:motivation}.\footnote{We provide additional test-time privacy examples beyond medical classifiers in \cref{appendix:additional_ttp_examples}.} This privacy threat model is inspired by definitions privacy which corresponds to protecting a user from unfair interference or intrusion \citep{privacydefn}. This differs from settings in privacy which mainly protect sensitive information.

We discuss why existing solutions, like unlearning \citep{sekhari2021remember} or differential privacy \citep{dwork2006calibrating} do not suffice to solve this problem in \cref{section:related_work}. Furthermore, naive solutions like masking model outputs do not work for open-weight image recognition models, since the model parameters and architecture are available to the model user. An adversary could simply remove such a mask. To make this clear, we comprehensively detail our threat model as a security game in \cref{appendix:threat_model} and provide various motivating attacks in \cref{appendix:attacks_definition}. Therefore, we ask the following research question:

\begin{center}
    \vspace{-2mm}
    \textit{Can we ensure test-time privacy against adversaries with access to an open-weight model?}
    \vspace{-2mm}
\end{center}

To do so, we argue it suffices to have uniform model outputs over the protected instances. That way, a data controller can only guess at the prediction. Thus, we revisit inducing maximal uncertainty over a dataset \citep{pereyra2017regularizingneuralnetworkspenalizing}. Furthermore, we want to obtain high performance on all other instances as well. In particular, we answer our research question affirmatively, providing:

\begin{itemize}
\vspace{-2mm} 
    \item A method to finetune a pretrained model with a Pareto optimal objective, rendering the model maximally uncertain over protected instances while preserving accuracy on others.
    \item Several principled $(\varepsilon, \delta)$-certified full-batch and sequential algorithms which approximate the Pareto objective, derived without assumptions of convexity.
    \vspace{-1mm}
    \item A theoretical analysis of the privacy-utility tradeoff that our algorithms incur, establishing a tight, non-vacuous bound.
    \vspace{-1mm}
    \item Empirical studies on image recognition models like ResNet50 \citep{he2016deep} trained on datasets like CIFAR100 \citep{krizhevsky2009learning}, observing that our algorithms maintain high uniformity on protected instances while guaranteeing excellent utility on the rest. 
\end{itemize}
 
Following the literature on privacy, we focus on protecting a subset of the training data. However, as detailed in \cref{section:formulation} and \cref{appendix:threat_model}, our setup and algorithms can also work for corrupted test instances. The code for our experiments is available for reproducibility at  \url{https://github.com/ashiqwisc/test_time_privacy/blob/main/README.md}.

%% file: sections/motivational_figure.tex
    \definecolor{bgcolor}{RGB}{245, 245, 245}
    \definecolor{accentblue}{RGB}{0, 123, 255}
    \definecolor{accentred}{RGB}{220, 53, 69}
    \definecolor{accentgreen}{RGB}{40, 167, 69}
    \definecolor{accentgray}{RGB}{108, 117, 125}
    \definecolor{textdark}{RGB}{33, 37, 41}
    \definecolor{datagray}{RGB}{150, 150, 150}
    \definecolor{accentpurple}{RGB}{102, 51, 153} %

\afterpage{
\begin{figure}[t]
    \centering
    \resizebox{0.8\textwidth}{!}{%
    \begin{tikzpicture}[
        font=\sffamily,
        node distance=0.8cm and 1.5cm, %
        actor/.style={
            font=\Large,
            text=textdark
        },
        model/.style={
            font=\Huge,
            text=accentblue
        },
        data/.style={
            draw=accentgray,
            thick,
            rounded corners=3pt,
            minimum width=0.5cm,
            minimum height=0.3cm,
            align=center,
            fill=white,
            drop shadow
        },
        output/.style={
            draw=accentred,
            thick,
            rounded corners=3pt,
            text width=3.5cm,
            minimum height=0.4cm,
            font=\bfseries,
            align=center,
            fill=accentred!10,
            drop shadow
        },
        process/.style={
            ->,
            line width=1.5pt,
            rounded corners=5pt
        },
        emphasis/.style={
            rectangle,
            fill=accentred!90,
            text=white,
            font=\bfseries,
            rounded corners=3pt,
            inner sep=1mm,
            drop shadow
        },
    ]

    \node[actor, text=accentpurple] (model_provider) at (0,0) {\LARGE\faBuilding};
    \node[above=0.1cm of model_provider, text=accentpurple, align=center, font=\bfseries] {Model Provider};

    \node[model, text=black, right=of model_provider] (f) {\faBrain};
    \draw[process, accentpurple] (model_provider) -- (f) node[midway, above] {Trains $f$};

    \node[below=0.1cm of f, font=\large\bfseries] (f_label) {Pretrained Model $f$};

    \node[actor, text=accentred, above=1.0cm of f] (health_provider) {\LARGE\faBuilding};
    \node[above=0.01cm of health_provider, text=accentred, align=center, font=\bfseries] {Insurance Provider};
    \draw[process, accentred] (health_provider) -- (f) node[midway, right, align=left] {Queries $f$\\with corrupted $\bm{x}_p$};

    \node[output, right=of f] (pred1) {$f(\bm{x}_p) \rightarrow$ ``Melanoma''};
    \draw[process, accentred, shorten >= 2pt] (f) -- (pred1);

    \node[anchor=east, color=textdark, font=\large, above=0.8cm of pred1] (high_label) {\textbf{High}};
    \draw[fill=accentblue!70, draw=datagray, line width=1pt] ([xshift=0.1cm, yshift=-0.15cm]high_label.east) rectangle ++(2, 0.4) coordinate (bar1_right_edge);
    \node[anchor=east, color=textdark, font=\large, below=0.1cm of high_label] (low_label) {\textbf{Low}};
    \draw[fill=bgcolor, draw=datagray, line width=1pt] ([xshift=0.15cm, yshift=-0.2cm]low_label.east) rectangle ++(0.2, 0.4);

    \node[model, text=accentgreen, below=1.5cm of f] (fu) {\faBrain};
    \node[below=0.1cm of fu, text=accentgreen!70!black, font=\large\bfseries] (fu_label) {Uniform Classifier $f_U$};

    \node[output, right=of fu, draw=accentgreen, fill=accentgreen!30] (pred2) {$f_U(\bm{x}_p) \rightarrow$ ???};
    \draw[process, accentgreen, shorten >= 2pt] (fu) -- (pred2);

    \node[anchor=east, color=textdark, font=\large, above=0.8cm of pred2] (high_label2) {\textbf{High}};
    \draw[fill=accentblue!70, draw=datagray, line width=1pt] ([xshift=0.1cm, yshift=-0.15cm]high_label2.east) rectangle ++(1, 0.4);
    \node[anchor=east, color=textdark, font=\large, below=0.1cm of high_label2] (low_label2) {\textbf{Low}};
    \draw[fill=bgcolor, draw=datagray, line width=1pt] ([xshift=0.15cm, yshift=-0.2cm]low_label2.east) rectangle ++(1, 0.4);

    \path (bar1_right_edge) -- (bar1_right_edge |- pred2.center) coordinate[midway] (right_anchor);

    \node[rectangle, draw=accentgreen, fill=accentgreen!30, thick, rounded corners=8pt,
        inner sep=2mm, font=\bfseries, align=center, text width=2.0cm,
            right=0.5cm of right_anchor, anchor=west] (message) {
        Health insurance companies can \emph{no longer} confidently classify a person as having melanoma to deny coverage.
    };

    \end{tikzpicture}
    }
    \caption{An adversary, like a health insurance company (\textcolor{red}{\faBuilding}), can query a pretrained model $f$ (\textcolor{black}{\faBrain}) and use its outputs to make harmful decisions. However, after running our algorithm, the new model $f_U$ (\textcolor{accentgreen}{\faBrain}) provides maximal uncertainty, protecting against such an adversary.}
    \label{fig:motivation}
\end{figure}
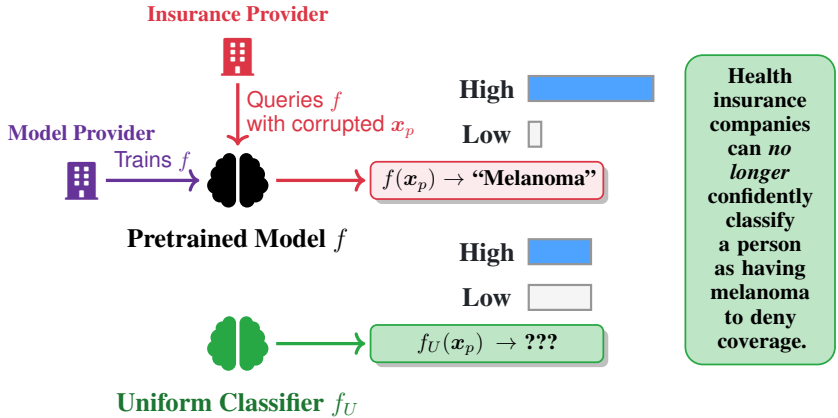
\FloatBarrier
}

%% file: sections/related_work.tex
\section{Related Work}\label{section:related_work}

\textbf{Data Privacy}: Data leakage is a persistent danger for large information systems \citep{al2019privacy}. In the context of ML, data privacy is ubiquitous \citep{fredrikson2015model}; \citep{song2017machine}; \citep{yeom2018privacy}. Approaches to privacy-preserving ML include differential privacy (DP) \citep{dwork2006calibrating}; \citep{improvedgaussian2018}; \citep{amin2024practical}, homomorphic encryption \citep{brakerski2014leveled}; \citep{lee2022privacy}; \citep{aono2017privacy}, and model obfuscation \citep{zhou2023modelobfuscator}. Notably, these methods protect against various privacy violations, like reconstruction attacks \citep{dinur2003revealing} due to failures of anonymization \citep{li2012sampling}. However, existing methods do not prevent confidently correct classification, and thus fail to protect against the attacks we consider in our setting. For example, if $\bm{x}_p$ is a corrupted medical image record, an adversary may not be able to use a DP model $f$ to recover the record $\bm{x}_p$ exactly, but they can still produce a confident prediction of e.g. ``Melanoma" to harm person $p$.

A dominant viewpoint in the privacy community is that a model $f$ working as expected does not constitute a privacy violation, e.g. correctly predicting ``Melanoma" for the corrupted medical image $\bm{x}_p$, as it has learned something underlying about nature \citep{mcsherry2016}; \citep{bun2021notaprivviolation}. Furthermore, the privacy leakage occurred when $\bm{x}_p$ became public \citep{gautamlecture}. This view misses the point: ML models are often trained and applied on freely available data. For example, training data could be scraped from the web or social media platforms. Subsets of this data can be obsolete, corrupted, or confidential. With such data as input, model $f$ presents a clear and present danger for AI safety which differential privacy falls short of addressing, as the above example showed. In parallel, as humans, we learn to not act upon certain kinds of knowledge. For example, when we read confidential documents or learn that previously obtained knowledge is incorrect, we are not allowed to share or act on this knowledge.

\textbf{Unlearning}: A related subfield is machine unlearning, which is inspired by the right to be forgotten (RTBF), mandating that ML model providers delete user data upon request \citep{GDPR2016}. In practice, model providers must remove user data and its effects from trained models and algorithms. Unlearning methods usually do so by approximating (and evaluating performance against) the model retrained from scratch without the protected user data \citep{sekhari2021remember}; \citep{bourtoule2021machine}; \citep{kurmanji2023towards}. 

However, while unlearning helps model providers comply with the RTBF, it cannot protect against attacks within our threat model. Specifically, recent unlearning research has established that data in the support of the training distribution will likely still be confidently predicted with the same prediction as before, even after using state-of-the-art algorithms (or even after applying exact retrain-from-scratch algorithms) to unlearn them \citep{zhao2024makes}. That is, denoting the pretrained model as $f$ and the unlearned model as $f_u$, for typical training instances, it holds that $f = f_u$.

To make clear why unlearning does not solve our problem, recall the example of a model $f$ trained on skin images to predict disease. This time, to remove the corrupted medical image $\bm{x}_p$ from $f$, person $p$ invokes the RTBF. Thus, the data controller for $f$ unlearns $\bm{x}_p$, yielding $f_u$. But, even after unlearning, any medical insurance company can still access the publicly available $\bm{x}_p$ and obtain $f_u(\bm{x}_p)$. But, $f_u(\bm{x}_p) = f(\bm{x}_p)$, and thus the medical insurance company \textit{incorrectly} labels person $p$ as high risk for medical coverage. Thus, the unfair and dangerous scenario for person $p$ remains. This holds similarly for unlearning methods which deal with corrupted or obsolete data, as they still do not aim to reduce confidence in the final prediction \citep{schoepf2025redirection}.

\emph{Differences from Unlearning}: Importantly, what we propose is \textbf{not an unlearning algorithm}, which would need to be aligned with the goals of unlearning (and indistinguishability from retrain-from-scratch). Instead, we aim to address an entirely new threat scenario—test time privacy—that unlearning cannot solve, which we detail in \cref{appendix:threat_model}. For example, indistinguishability from retrain is inconsequential in our threat model. Furthermore, we also consider corrupted test examples, unlike unlearning which focuses only on the training dataset. Finally, what may constitute a privacy violation in unlearning, e.g. revealing that an instance is in the forget set via a membership inference attack \citep{shokri2017membership} does \textit{not} constitute a violation in our threat model. This holds similarly for other threat models; for example, reconstruction attacks which lead to recovery of $\bm{x}_p$ \citep{dinur2003revealing} or  adversarial attacks which lead to misclassification of $\bm{x}_p$ \citep{goodfellow2014explaining} are not violations in our threat model, as explained in \cref{appendix:threat_model}. 

Still, the privacy guarantees that we provide in the threat model of test-time privacy are complementary to the guarantees that unlearning can provide. However, related work has focused heavily on unlearning--we fill this gap by presenting a framework for test-time privacy.

\textbf{Additional Related Works}: Due to space constraints, we provide an additional related works section in \cref{appendix:additional_related_work}. Critically, we describe how differentially privacy methods like private aggregation of teacher ensembles (PATE) \citep{papernot2018scalable} or label differential privacy (LabelDP) \citep{ghazi2021deep} differ from our setting, how label model inversion attacks \citep{zhu2019deep} relate to our threat model, and also why misclassification-based methods for unlearning \citep{cha2024learning} are suboptimal for addressing our threat model.

%% file: sections/problem_formulation.tex
\section{Approaches and Algorithms} \label{section:formulation}

\input{sections/notation}

In order to prevent test-time privacy violations, it suffices to have the model output a uniform distribution over the forget set, rendering the model maximally uncertain. Then, an adversary can only guess at the original sensitive prediction.\footnote{Please see \cref{appendix:threat_model} for a clear characterization of why this is optimal within our threat model, and why we can consider $\mathcal{D}_f \subset \mathcal{D}$, where $\mathcal{D}$ is a training dataset, without loss of generality.} We also would like to preserve retain set accuracy; to that end, we present an algorithm that finetunes the pretrained model with a Pareto optimal objective. To make this algorithm concrete, we define a \textit{uniform learner}, which we prove to exist in many common hypothesis classes. Then, we use this concept in order to construct a Pareto objective.

\subsection{The Exact Pareto Learner}\label{section:exact_pareto_learner}

For a dataset $\mathcal{D} \subset \mathcal{Z}^n$, we denote the pretrained model as $\mathcal{A}(\mathcal{D})$. Then, to make $\mathcal{A}(\mathcal{D})$ uniform over the forget set, we introduce our core concept of a \textit{uniform learner}:

 \begin{definition}[Uniform learner] Suppose we have a (randomized) learning algorithm $\mathcal{K} : \mathcal{Z}^n \to \mathcal{W}$ that, given $\mathcal{D} \subset \mathcal{Z}^n$, yields the parameter $\mathcal{K}(D) = \bm{w}_{\mathcal{D}}$. We say $\mathcal{K}$ is a uniform learner if $\; \forall \mathcal{D} \in \mathcal{Z}^n, \; \bm{w}_{\mathcal{D}}$ parametrizes $f_{\bm{w}_\mathcal{D}} \in \mathcal{H}_{\mathcal{W}}$ and satisfies:
\begin{equation}
    ||f_{\bm{w}_{\mathcal{D}}} - \underbrace{(\frac{1}{|\mathcal{Y}|}, ..., \frac{1}{|\mathcal{Y}|})}_{|\mathcal{Y}| \;  \text{times}}||_{\infty, \mathcal{X}} = 0.
\end{equation}
\vspace{-4mm}
 \label{defn:uniform_learner}
 \end{definition}
 
That is, $\mathcal{K}$ is a uniform learner if its parameterized outputs yield the uniform distribution 
$U[0, |\mathcal{Y}]]$ for all inputs across all datasets. We define this as a learning algorithm for full generality to handle e.g. neural networks with nonlinear transformations in their last layer. Furthermore, $\mathcal{K}$ exists in many common hypothesis spaces, proved in \cref{section:proof_of_unif_learner_exists}:

\begin{proposition}
    Suppose we have a hypothesis space $\mathcal{H}_{\mathcal{W}}$ consisting of functions where the ultimate layer is an affine transformation and the outputs are passed through a softmax. Let $\mathcal{K}$ be a uniform learner. Then, $f_{\mathcal{K}(\mathcal{D})} \in \mathcal{H}_{\mathcal{W}} \; \forall \mathcal{D} \subset \mathcal{Z}^n$.

\label{prop:uniform_learner_exists}
\end{proposition}

\textit{Proof Sketch: } Setting the weights in the ultimate affine layer to 0 yields uniform outputs. 

Most classifiers built and deployed in ML recently, including multilayer perceptrons (MLP), residual networks (ResNets) \citep{he2016deep}, and transformers \citep{vaswani2017attention}  satisfy the premise of \cref{prop:uniform_learner_exists}, making it widely applicable.

Next, we assume $\mathcal{A}$ and $\mathcal{K}$ are obtained through empirical risk minimization (ERM) to a local or global minima. That is, let $\mathcal{A}(\mathcal{D}) = \text{argmin}_{\bm{w} \in \mathcal{W}} \; \mathcal{L}_{\mathcal{A}}(\bm{w}, \mathcal{D})$ and $\mathcal{K}(\mathcal{D}) = \text{argmin}_{\bm{w} \in \mathcal{W}}\; \mathcal{L}_{\mathcal{K}}(\bm{w}, \mathcal{D})$, where $\mathcal{L}_{\mathcal{A}}$ penalizes incorrect classification and $\mathcal{L}_{\mathcal{K}}$ penalizes a lack of uniformity in model outputs. 
One choice of $\mathcal{L}_\mathcal{K}$ is the KL divergence \citep{kullback1951information} between the softmax outputs and the uniform distribution. This loss has been previously used to penalize highly confident classifier predictions \citep{pereyra2017regularizingneuralnetworkspenalizing}; we thus adapt this loss to our setting. Furthermore, by \cref{prop:uniform_learner_exists}, this loss can be completely minimized over $\mathcal{W}$. 

Critically, we seek the optimal tradeoff between uniformity over the forget set and utility over the retain set. That is, we should produce a learner that is Pareto optimal with respect to $\mathcal{L}_\mathcal{K}$ and $\mathcal{L}_\mathcal{A}$. This learner can be characterized as follows: 

\begin{proposition}
    Let $\theta \in (0,1)$. Fix $\mathcal{D} \subset \mathcal{Z}^n$ and consider the forget set $\mathcal{D}_f \subset \mathcal{D}$ and the retain set $\mathcal{D}_r = \mathcal{D} \setminus \mathcal{D}_f$. Then, if $\mathcal{M}_\theta(\mathcal{D}) = \text{argmin}_{\bm{w} \in \mathcal{W}} \; \theta\mathcal{L}_{\mathcal{K}}(\bm{w}, \mathcal{D}_f)+ (1-\theta) \mathcal{L}_\mathcal{A}(\bm{w}, \mathcal{D}_r)$ is a global minimizer, it is globally Pareto optimal with respect to $\mathcal{L}_{\mathcal{K}}(\bm{w}, \mathcal{D}_f)$ and $\mathcal{L}_\mathcal{A}(\bm{w}, \mathcal{D}_r)$. Similarly, if $\mathcal{M}_\theta(\mathcal{D})$ is a local minimizer, it is locally Pareto optimal. 
\label{prop:pareto_optimal}
\end{proposition}

\textit{Proof Sketch: } This holds by contradiction. If the solution to the minimized objective was not globally (locally) Pareto optimal, since $\theta \in (0,1)$, it could not be the global (local) minimizer. See \cref{section:proof_of_pareto_optimality} for a full proof. Definitions of Pareto optimality are included in \cref{section:proof_of_pareto_optimality} as well.

As shown in \cref{prop:pareto_optimal}, $\mathcal{M}_\theta$ yields a parameter that, given $\theta \in (0,1)$, presents a Pareto optimal tradeoff between uniformity over the forget set and utility over the retain set. One can adjust $\theta$ to vary over many Pareto optimal solutions, yielding different tradeoffs between uniformity and utility. This yields \cref{algo:finetuning_algo}, in which we finetune a pretrained model by using it as initialization for $\mathcal{M}_\theta(\mathcal{D})$.

\subsection{The Certified Pareto Learner}

While the aforementioned algorithm in \ref{section:exact_pareto_learner} provably guarantees an optimal tradeoff, so long as its objective is minimized, we would also like to make it \textit{certified}, obtaining a certificate that a third party can inspect to verify test-time privacy. Thus, to design a certified approximation algorithm, we take inspiration from certified unlearning \citep{zhang2024certifiedunlearningDNN}, which aims to add a small amount of structured noise such that the pretrained model becomes indistinguishable from the retrained model. In our setting, we would like to make the pretrained model indistinguishable from the solution to the Pareto objective. To do so, we define a new notion of $(\varepsilon, \delta)$-indistinguishability and use this definition to design a certifiable algorithm, with results in \cref{section:experiments}. 

Firstly, to motivate our definition, recall the definition of differential privacy \citep{dwork2006calibrating}: 

\begin{definition}[$(\varepsilon, \delta)$-differential privacy]
    Suppose we have privacy budgets $\varepsilon \in (0,1)$ and $\delta > 0$. A randomized algorithm $\mathcal{M} : \mathcal{Z}^n \to \mathcal{W}$ satisfies $(\epsilon, \delta)$-differential privacy if $\forall \mathcal{T} \subset \mathcal{W}$ and $\forall \mathcal{D}, \mathcal{D}^\prime \in \mathcal{Z}^n$ s.t. $||\mathcal{D} - \mathcal{D}^\prime||_1 \leq 1$: 
    
        \vspace{-6mm}
    \begin{equation}
         \mathcal{M}(\mathcal{D}) \approx_{\varepsilon, \delta, \mathcal{T}} \mathcal{M}(\mathcal{D}^\prime).
     \end{equation}
    \vspace{-4mm}
    \label{defn:diff_privacy}
\end{definition}

This guarantees that the algorithm $\mathcal{M}$ applied on a dataset is statistically indistinguishable from the same algorithm applied on all datasets different by one instance. One can leverage this definition to formalize certified unlearning \citep{sekhari2021remember}:

\begin{definition}[$(\varepsilon, \delta)$-certified unlearning] Suppose we have privacy budgets $\varepsilon \in (0,1)$ and $\delta > 0$. Consider $\mathcal{D} \subset \mathcal{Z}^n$ and let $\mathcal{D}_f \subset \mathcal{D}$ be the forget set to be unlearned, and $\mathcal{D}_r = \mathcal{D} \setminus \mathcal{D}_f$ be the retain set. $\mathcal{U} : \mathcal{Z}^n \times \mathcal{Z}^n \times \mathcal{W} \to \mathcal{W}$ is an $(\varepsilon, \delta)$-certified unlearning algorithm if $ \; \forall \mathcal{T} \subset \mathcal{W}$, we have: 
\begin{align}
    \mathcal{U}(\mathcal{D}, \mathcal{D}_f, \mathcal{A}(\mathcal{D})) \approx_{\varepsilon, \delta, \mathcal{T}} \mathcal{A}(\mathcal{D}_r). 
\end{align}
\vspace{-4mm}
\label{defn:certified_unlearning}
\end{definition}

This formalizes making $\mathcal{A}(\mathcal{D})$ indistinguishable from $\mathcal{A}(\mathcal{D}_r)$. In light of \cref{defn:certified_unlearning}, we seek to make $\mathcal{A}(\mathcal{D})$ indistinguishable from $\mathcal{M}_\theta(\mathcal{D})$. Thus, we provide the following new definition: 

\begin{definition}[$(\varepsilon,\delta,\theta)$-certified Pareto learner] Suppose we have privacy budgets $\varepsilon \in (0,1)$ and $\delta > 0$ with $\mathcal{D} \subset \mathcal{Z}^n$. Let $\mathcal{D}_f \subset \mathcal{D}$ be the forget set and let $\mathcal{D}_r  = \mathcal{D} \setminus \mathcal{D}_f$ be the retain set. Suppose we have $\theta \in (0,1)$ and $\mathcal{M}_\theta(\mathcal{D}) = \text{argmin}_{\bm{w} \in \mathcal{W}} \;  \theta \mathcal{L}_{\mathcal{K}}(\bm{w}, \mathcal{D}_f)+ (1-\theta)\mathcal{L}_\mathcal{A}(\bm{w}, \mathcal{D}_r)$, where $\mathcal{K}$ is a uniform learner and $\mathcal{A}$ is a learning algorithm both obtained through ERM.  An algorithm $ \mathcal{G} : \mathcal{Z}^n \times \mathcal{Z}^n \times \mathcal{W} \to \mathcal{W}$ is a  $(\varepsilon,\delta,\theta)$-certified Pareto learner if $\; \forall \mathcal{T} \subset \mathcal{W}$: 

  \vspace{-5mm}
    \begin{align}
    \mathcal{G}(\mathcal{D}, \mathcal{D}_f, \mathcal{A}(\mathcal{D})) \approx_{\varepsilon, \delta, \mathcal{T}} \mathcal{M}_\theta(\mathcal{D}).
    \end{align}

  \vspace{-2mm}
\label{defn:certified_uniformity}
\end{definition} 

\textbf{Discussion:} Qualitatively, the conditions in \cref{defn:certified_uniformity} mean that the model obtained by algorithm $\mathcal{G}$ is statistically indistinguishable from a model that is Pareto optimal between utility over the retain set and uniformity over the forget set. Here, we consider the classical setting of $\varepsilon \in (0,1)$.\footnote{Notably, \citet{improvedgaussian2018} provide a way to achieve $(\varepsilon, \delta)$-indistinguishability for $\varepsilon > 1$, and their technique can be adapted without loss of generality to our setting.} Finally, note that satisfying \cref{defn:certified_unlearning} and \cref{defn:certified_uniformity} together is not possible for forget sets which overlap; thus, a model provider should adopt whichever approach corresponds to their threat model.

One way we can design an algorithm which satisfies \cref{defn:certified_uniformity} is by taking a Newton step towards the Pareto model and applying structured Gaussian noise; this yields \cref{algo:hess_exact_algo}, which is certifiable as proved in \cref{section:certifying_algorithm}. Using local convex approximation \citep{nocedal1999numerical}, in which we add a regularization term to the objective of the Pareto learner, we design \cref{algo:hess_exact_algo} without any assumptions of convexity on the component loss functions. 

In addition, \cref{algo:hess_exact_algo} requires inverting a Hessian, which is computationally infeasible for practical neural networks e.g. ResNets, even after employing conjugate gradient methods \citep{nocedal1999numerical} and Hessian vector product techniques \citep{pearlmutter1994fast}. To resolve this issue, we also propose a derived \cref{algo:hess_estimator_algo} in \cref{section:certifying_algorithm}, which computes an efficient estimator for the inverse Hessian \citep{agarwal2016second}.  Furthermore, this algorithm does not assume convergence to a local minima for $\mathcal{A}(\mathcal{D})$, handling e.g. early stopping. An online version is presented in \cref{section:online_algorithm} as \cref{algo:sequential_algorithm}. While \cref{algo:hess_exact_algo,algo:hess_estimator_algo,algo:sequential_algorithm} have more hyperparameters than \cref{algo:finetuning_algo}, they offer a certificate which can be used to verify use of our method by a third party; we present ways to reduce hyperparameters in \cref{section:appendix_eliminate_hyperparams}.

%% file: sections/notation.tex
\textbf{Notation: } Let $\mathcal{X} \subset \mathbb{R}^d$ be a sample space and let $\mathcal{Y} \subset \mathbb{R}^o$ be a label space. Denote $\mathcal{Z} = \mathcal{X} \times \mathcal{Y}$ as the space of feature-label pairs. Let $\mathcal{Z}^n$ be the $n$-fold Cartesian product of $\mathcal{Z}$ such that a dataset $\mathcal{D} \subset \mathcal{Z}^n$ is a collection of $n$ feature-label pairs. Then, the $i$th instance is denoted as $\mathcal{D}^{(i)}$  with its feature in $\mathcal{X}$ being $\mathcal{D}^{(i, \mathcal{X})}$ and label being $\mathcal{D}^{(i, \mathcal{Y})}$. Following the unlearning literature, we subset $\mathcal{D}$ as a ``forget set" $\mathcal{D}_f$, containing instances to protect, and a retain set $\mathcal{D}_r = \mathcal{D} \setminus \mathcal{D}_f$. Then, suppose we have a (randomized) learning algorithm $\mathcal{A} : \mathcal{Z}^n \to \mathcal{W}$, where $\mathcal{W} \subset \mathbb{R}^z$ is a parameter space. Let the set of hypotheses parameterized with respect to this parameter space be $\mathcal{H}_{\mathcal{W}}$. Let $f_{\bm{w}} \in \mathcal{H}_{\mathcal{W}}$ be the hypothesis parameterized by $\bm{w} \in \mathcal{W}$, defined as $f_{\bm{w}}: \mathcal{X} \to \Delta_{|\mathcal{Y}|}$, where $\Delta_{|\mathcal{Y}|}$ is the probability simplex $\{p_1, ..., p_{\gamma} : p_i \geq 0, \sum_{i = 1}^{|\mathcal{Y}|} p_{i} = 1\}$. When $\bm{A}$ is a matrix, $||\bm{A}||_2$ is the $2$ operator norm. When $\bm{v}$ is a vector, $||\bm{v}||_2$ is the $\ell_2$ norm. Furthermore, let $\lambda_{\min}(\bm{A})$ denote the minimum eigenvalue of $\bm{A}$. If we have an objective $f_{A} + f_{B}$, we denote its gradient evaluated at $\bm{w}$ as $\nabla_{\bm{w}, A, B}$ and Hessian as $\bm{H}_{\bm{w}, A, B}$. Finally, when for sets $\mathcal{S}, \mathcal{F} \subset \mathcal{N}$ and $\mathcal{R}$ and mechanisms $\mathcal{M}, \mathcal{M}^\prime : \mathcal{N} \to \mathcal{R}$, we have $\Pr(\mathcal{M}(\mathcal{S}) \in \mathcal{R}) \leq e^{\varepsilon}\Pr(\mathcal{M}^\prime(\mathcal{F}) \in \mathcal{R}) + \delta$ and $\Pr(\mathcal{M}^\prime(\mathcal{S}) \in \mathcal{R}) \leq e^{\varepsilon}\Pr(\mathcal{M}(\mathcal{F}) \in \mathcal{R}) + \delta$, we will denote $\mathcal{M}(\mathcal{S}) \approx_{\varepsilon, \delta, \mathcal{R}} \mathcal{M}^\prime(\mathcal{F})$. We provide a symbol table in \cref{appendix:symbol_table}.

%% file: sections/algorithm.tex
\begin{algorithm}[tb]
\caption{$\mathcal{M}_\theta$ Finetuning}

\begin{algorithmic}

\Require Dataset $\mathcal{D}$; forget set $\mathcal{D}_f$; pretrained model $\bm{w}^* = \mathcal{A}(\mathcal{D})$; uniformity-utility tradeoff coefficient $\theta$; $e$ epochs.

\State Use $\bm{w}^*$ as initialization for the Pareto learner $\mathcal{M}_\theta(\mathcal{D})$.

\State Optimize the Pareto learner $\mathcal{M}_\theta(\mathcal{D})$  for $e$ epochs with e.g. SGD to yield $\bm{w}^-$.

\State \Return $\bm{w}^-$.

\end{algorithmic}

\label{algo:finetuning_algo}

\end{algorithm}

\begin{algorithm}[tb]
\caption{$(\epsilon, \delta, \theta)$-Certified Uniformity with Exact Inverse Hessian}

\begin{algorithmic}
\Require Dataset $\mathcal{D}$; forget set $\mathcal{D}_f$; pretrained model $\bm{w}^* = \mathcal{A}(\mathcal{D})$; privacy budgets $\varepsilon$ and $\delta$; uniformity-utility tradeoff coefficient $\theta$; local convex coefficient $\lambda$; norm upper bound $C$.

\State $\bm{\tilde{w}} \gets \bm{w}^* - (\bm{H}_{\bm{w}^*, \mathcal{K}, \mathcal{A}} + \lambda \bm{I})^{-1} \nabla_{\bm{w}^*, \mathcal{K}, \mathcal{A}}$. \Comment{\textcolor{blue}{Derived in \cref{section:certifying_algorithm}.}}

\State Compute $\Delta$ as the bound in \cref{eq:algo_exact_hessian_eq}.

\State $\sigma = \frac{\Delta}{\epsilon}\sqrt{2\ln(1.25/\delta)}$.

\State $\bm{w}^- \gets \bm{\tilde{w}} + Y$ where $Y \sim \mathcal{N}(\bm{0}, \sigma^2\bm{I})$.
\State \Return $\bm{w}^-$.

\end{algorithmic}

\label{algo:hess_exact_algo}

\end{algorithm}

%% file: sections/algorithm_study.tex
\section{Theoretical Analysis}\label{section:algorithm_study}

In what follows, we aim to analyze various properties of \cref{algo:finetuning_algo} and \cref{algo:hess_exact_algo} to understand how to appropriately choose $\theta$ and the privacy-utility tradeoffs these algorithms incur. To clarify the notation used in this section, we include a symbol table in \cref{appendix:symbol_table}. Firstly, we seek to understand how we can choose $\theta$ to guarantee uniformity over the forget set. To do so, we provide a constraint to be satisfied to ensure uniformity. Then, we provide an appropriate lower bound on $\theta$ to ensure the constraint is satisfied. In doing so, one obtains a bound on the privacy of our algorithm. We next want to obtain a bound on the utility of our algorithm. To that end, we upper bound the difference between the retain loss of the locally optimal learned model $\mathcal{A}(\mathcal{D}_r)$ and the locally optimal solution to the Pareto objective $\mathcal{M}_\theta(\mathcal{D})$. We obtain a tight, non-vacuous bound with respect to $\theta$ and characterize it asymptotically. Incorporating the two bounds provides a concrete characterization of the privacy-utility tradeoffs that occur when our algorithms are used.

In particular, across all algorithms, we make the pretrained model indistinguishable from: 
\vspace{-2mm}
\begin{align}
    \mathcal{M}_\theta(\mathcal{D}) &= \arg\min_{||\bm{w}||_2 \leq C, \bm{w} \in \mathcal{W}} \theta \mathcal{L}_{\mathcal{K}}(\bm{w}, \mathcal{D}_f) + (1-\theta)\mathcal{L}_{\mathcal{A}}(\bm{w}, \mathcal{D}_r) + \frac{\lambda}{2}||\bm{w}||_2^2 \\
    &=  \arg\min_{||\bm{w}||_2 \leq C, \bm{w} \in \mathcal{W}} \theta \sum_{i = 1}^{|\mathcal{D}_f|} \ell_{\mathcal{K}}(\bm{w}, \mathcal{D}_f^{(i)}) + (1-\theta)\sum_{j = 1}^{|\mathcal{D}_r|} \ell_{\mathcal{A}}(\bm{w}, \mathcal{D}_r^{(j)}) + \frac{\lambda}{2}||\bm{w}||_2^2,
    \label{eq:pareto_optimal_with_weight_bound_and_reg}
\end{align}
\vspace{-4mm}

where $\mathcal{\ell}_\mathcal{K}$ and  $\mathcal{\ell}_\mathcal{A}$ are component loss functions corresponding to individual data instances in the forget and retain sets, respectively. Note that $\lambda$ is present either as weight decay in the Pareto learning in \cref{algo:finetuning_algo} or as part of the local convex approximation in \cref{algo:hess_exact_algo}. Furthermore, note that the objective is constrained by $||\bm{w}||_2 \leq C$; we use this as a part of our local convex approximation when deriving \cref{algo:hess_exact_algo}; it is however unecessary for \cref{algo:finetuning_algo}. Similarly, unlearning methods assume this either implicitly or explicitly \citep{zhang2024certifiedunlearningDNN}. One can use projected gradient descent \citep{nocedal1999numerical} during pretraining to satisfy this constraint.

Note that \cref{algo:finetuning_algo} has $\lambda \approx 0$. For \cref{algo:hess_exact_algo}, by \cref{lemma:post_processing_of_uniformity}, our models $f_{\bm{w}^-}$ and $f_{\mathcal{M}_\theta(\mathcal{D})}$ have approximately the same outputs over $\mathcal{D}_f$, where $\bm{w}^-$ are the weights after applying one of our TTP algorithms. Hence, for any of our algorithms, to ensure indistinguishability from uniformity over $\mathcal{D}_f$, it suffices to ensure that $\mathcal{M}_\theta$ satisfies the following constraint: 
\begin{equation}
    ||f_{\mathcal{M}_\theta(\mathcal{D})}(\mathcal{D}_f) - U[0, |\mathcal{Y}|]||_\infty \leq \varepsilon.
    \label{eq:hard_constraint}
\end{equation}
\vspace{-4mm} 

We then have the following bound on \cref{eq:hard_constraint} with respect to $\theta$, the proof of which is in  \cref{section:proof_of_hard_constraint_bound}: 

\begin{proposition}
Let $\mathcal{M}_\theta(\mathcal{D})$ be the global solution to the Pareto objective. Choose, as surrogate losses, $\mathcal{\ell}_\mathcal{K}(\bm{w}, \mathcal{D}_f^{(i)}) =  D_{KL}(f_{\bm{w}}(\mathcal{D}_f^{(i)}) ||U[0, |\mathcal{Y}|)$, the KL divergence between the model outputs over the forget set and the uniform distribution, and $\ell_\mathcal{A}(\bm{w}, \mathcal{D}_r^{(j)}) = \mathbb{H}_{CE}(\mathcal{D}_r^{(j,\mathcal{Y})}, f_{\bm{w}}(\mathcal{D}_r^{(j,\mathcal{X})}))$, the cross entropy between model predictions and labels over the retain set. Then, $||f_{\mathcal{M}_\theta(\mathcal{D})}(\mathcal{D}_f) - U[0, |\mathcal{Y}|]||_\infty \leq \sqrt{2(\frac{1-\theta}{\theta} |\mathcal{D}_r|\ln|\mathcal{Y}|)}$.
\label{prop:bound_on_pop_distance_between_output_and_unif}
\end{proposition}

\textit{Proof Sketch}: By using \cref{prop:uniform_learner_exists} and the fact that $\mathcal{M}_\theta(\mathcal{D})$ is a global minimizer, we can yield a bound on $\mathcal{L}_\mathcal{K}$. Then, standard inequalities yield our result. %

Then, we can choose $\theta$ as follows to guarantee \cref{eq:hard_constraint}, the proof of which is in \cref{section:proof_of_lambda_large}:

\begin{corollary}
    Choosing $\theta \geq \frac{2|\mathcal{D}_r|\ln|\mathcal{Y}|}{\varepsilon^2 +2|\mathcal{D}_r|\ln|\mathcal{Y}|}$ guarantees that \cref{eq:hard_constraint} holds for any $\varepsilon > 0$.  %
    \label{corollary:lambda_large}
\end{corollary}

\textbf{Discussion}:  Note that \cref{corollary:lambda_large} is well-defined in that $\theta \in (0,1)$ for any choice of $|\mathcal{D}_r|, |\mathcal{Y}|$ and $\varepsilon$. Furthermore, \cref{corollary:lambda_large} restricts $\mathcal{M}_\theta(\mathcal{D})$ to a subset of Pareto optimal solutions, but this does not render it no longer Pareto optimal; thus, our formulation as in \cref{section:certifying_algorithm} still holds in its entirety. Importantly, this is a sufficient but not necessary condition to satisfy \cref{eq:hard_constraint}.

Similarly, by \cref{lemma:post_processing_of_uniformity}, we can study the affect of $\theta$ in $\mathcal{M}_\theta(\mathcal{D})$ on the (empirical) retain error on $\mathcal{D}_r$, after our algorithms are applied. To provide this bound, we require two key assumptions: 

\begin{assumption}   \label{assumption1}
    The gradients of  $\ell_\mathcal{K}$ and $\ell_\mathcal{A}$ are Lipschitz in $\bm{w}$ with constants $\frac{P_\mathcal{K}}{|\mathcal{D}_f|}$ and $\frac{P_\mathcal{A}}{|\mathcal{D}_r|}$. 
\end{assumption}

\begin{assumption}    \label{assumption2}
    The Hessians of  $\ell_\mathcal{K}$ and $\ell_\mathcal{A}$ are Lipschitz in $\bm{w}$ with constants $\frac{F_{\mathcal{K}}}{|\mathcal{D}_f|}$ and $\frac{F_{\mathcal{A}}}{|\mathcal{D}_r|}$.
\end{assumption}

\textbf{Discussion: } Note that these assumptions are only used to prove Thm. \labelcref{thm:retain_accuracy_bound} and in \cref{section:certifying_algorithm}; we do not require them to prove all previously mentioned theorems. These assumptions, similar to those studied by \citet{zhang2024certifiedunlearningDNN}, are less restrictive than those typically studied in certified unlearning \citet{sekhari2021remember}; importantly, we do not assume (strong) convexity of the losses.

We then present a tight, non-vacuous bound on the retain error after applying any of our algorithms: 

\begin{thm} 
     Suppose \cref{assumption1,assumption2} hold, and let $P_{\mathcal{K}}, P_{\mathcal{K}}, F_{\mathcal{K}}, F_{\mathcal{A}}$ be as defined in \cref{assumption1,assumption2}. Let $\alpha^* := \mathcal{L}_\mathcal{A}(\mathcal{A}(\mathcal{D}_r), \mathcal{D}_r)$ be the locally optimal (empirical) retain loss, achieved by $\mathcal{M}_\theta(\mathcal{D})$ when $\theta = 0$. Let $\alpha(\theta) := \mathcal{L}_{\mathcal{A}}(\mathcal{M}_\theta(\mathcal{D}), \mathcal{D}_r)$ be the locally optimal retain loss obtained by $\mathcal{M}_\theta(\mathcal{D})$ when $\theta \in (0,1)$. Suppose all weights used throughout are bounded by $||\bm{w}||_2 \leq C$. Additionally, denote by $F := \theta M_\mathcal{K} + (1-\theta)F_{\mathcal{A}}$ and $P := \theta P_{\mathcal{K}} + (1-\theta)P_{\mathcal{A}}$. Consider regularization coefficient $\lambda \geq L + 2\theta CF + \sqrt{2\theta C F (P + 2\theta C F + 8P_{\mathcal{K}}})$. Then, we have the following bound: %
      \begin{align}
         |\alpha^* - \alpha(\theta)| &\leq \mathcal{O}(\lambda C^2\theta + C^2\theta^2)\,.
          \label{eq:retain_accuracy_bound_equation_main}
      \end{align}
      \vspace{-4mm}
     \label{thm:retain_accuracy_bound}
\end{thm}

\textit{Proof Sketch}: We subtract the first order conditions, by definition of $\alpha^*$ and $\alpha(\theta)$, to get an expression with respect to the gradients; plugging in an equivalent path integral expression and applying \cref{lemma:descent_lemma} yields our desired result, with a full proof (including the full bound) in  \cref{section:proof_of_retain_accuracy_bound}.

\textbf{Discussion}: Three key hyperparameters should be kept small to ensure high retain accuracy: the $\ell_2$ regularization coefficient $\lambda$, the max model weight magnitude $C$, and the Pareto frontier hyperparameter $\theta$. In particular, large regularization coefficients take the model off the Pareto frontier. However, smaller or sparser weights are preferred, since the bound grows quadratically in $C$. 

In addition, that when $\theta = 0$, the bound simplifies to 0, indicating that it is tight near 0. We demonstrate that it is tight near 1 in \cref{appendix:additional_experiments}. Furthermore, in the case of \cref{algo:finetuning_algo}, since $\lambda \approx 0$, we do not need the condition on $\lambda$ and obtain a clearer characterization. Furthermore, we can obtain a more concise bound with simpler techniques, but such a bound is vacuous and does not incorporate information about $\theta$; we elaborate on this in \cref{section:proof_of_retain_accuracy_bound}.

%% file: sections/experiments.tex
\section{Empirical Analysis}\label{section:experiments}

Below, we provide empirical results for \cref{algo:finetuning_algo} and \cref{algo:hess_exact_algo}. We firstly discuss our experimental setup, baselines, and define our uniformity metric. Next, we provide our core results across \cref{algo:finetuning_algo}, \cref{algo:hess_exact_algo}, and our baselines for several architectures and benchmarks. We also comment on the Pareto frontier of \cref{algo:finetuning_algo} and \cref{algo:hess_exact_algo}, providing additional insight into the structure of our problem.

\textbf{Setup and Baselines}. Our primary results on \cref{algo:finetuning_algo} are for ResNet50 \citep{he2016deep} trained on SVHN, CIFAR10, and CIFAR100. We also provide results for logistic regression on MNIST to evaluate \cref{algo:hess_exact_algo}. We then include additional experiments with more complicated datasets and models, such as ViT \citep{dosovitskiy2021an} and TinyImageNet \citep{Le2015tinyimagenet}, in \cref{table:imagenet_vit}. We compare results with the pretrained model and the model retrained without the forget set, which constitutes exact unlearning \citep{bourtoule2021machine}. We also compare our methods to LabelDP \citep{ghazi2021deep} and a synthetic baseline that assigns random labels to instances neighboring the forget set. Across methods, we compare retain accuracy, test accuracy, and forget uniformity. We provide more details and the rationale for our baselines in \cref{appendix:experimental_details}.

\textbf{Providing a Uniformity Metric}: We require a metric to compare uniformity over the forget set in an interpretable manner. Thus, we define the ``confidence distance" as $\max\{0, f(\bm{x})_t  - \frac{1}{|\mathcal{Y}|}\}$ for $\bm{x} \in \mathcal{D}_f$, where $f(\bm{x})_t $ is the max confidence score. In our experiments, we use this as the primary metric for uniformity, reporting the average confidence distance over the forget set. We discuss why this is reasonable in \cref{appendix:uniformity_metric} and compare it to alternative metrics in \cref{appendix:additional_experiments}.

\begin{figure}[tb]
    \centering
    \begin{subfigure}[tb]{0.49\linewidth}
        \centering
        \includegraphics[width=\linewidth]{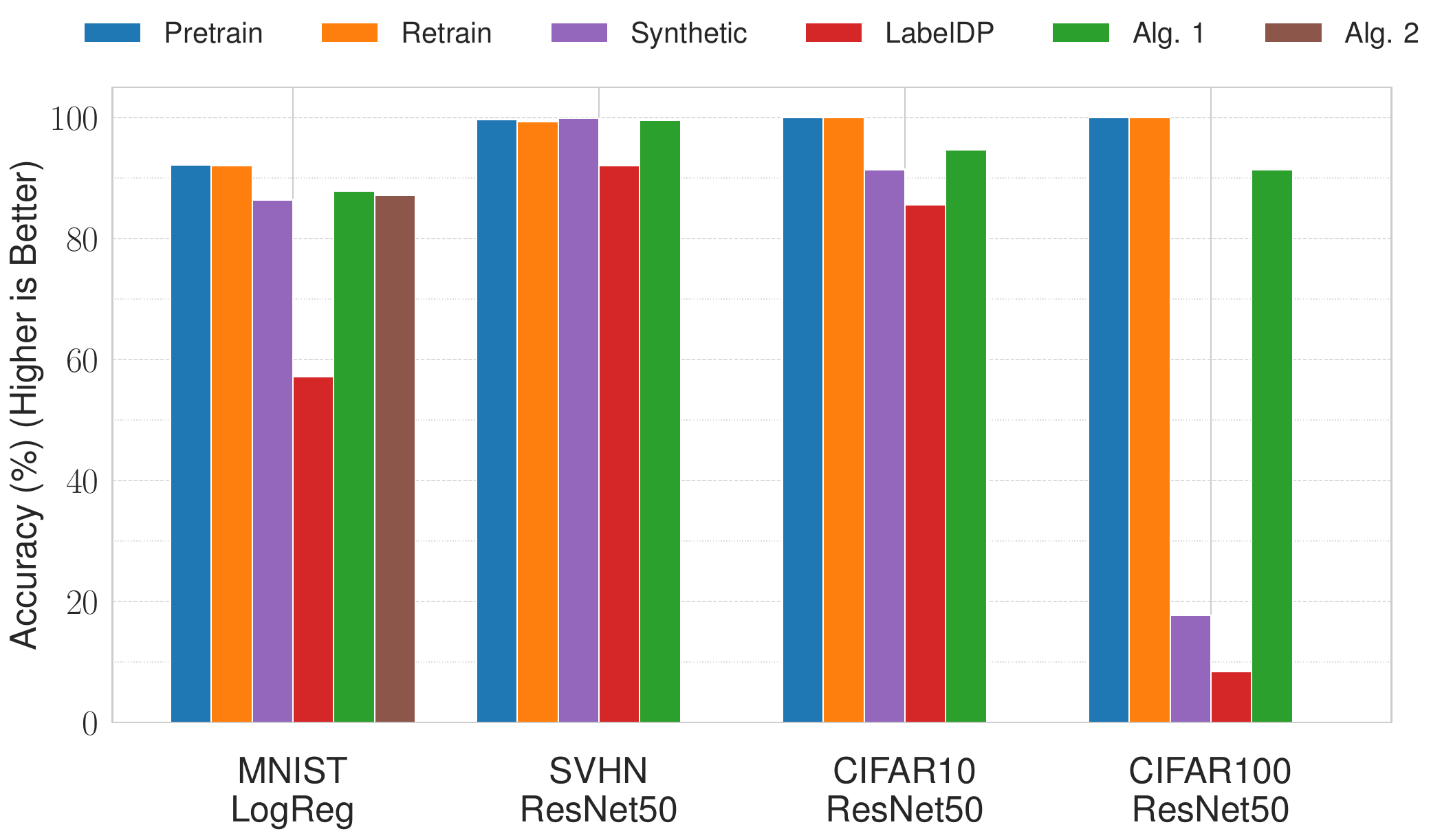}
        \caption{Accuracy on retain set for baselines as well as \cref{algo:finetuning_algo} and \cref{algo:hess_exact_algo} with $\theta  = 0.75$.}
        \label{fig:retain_accuracy_comparison}
    \end{subfigure} \hspace{1mm} 
    \begin{subfigure}[tb]{0.49\linewidth}
        \centering
        \includegraphics[width=\linewidth]{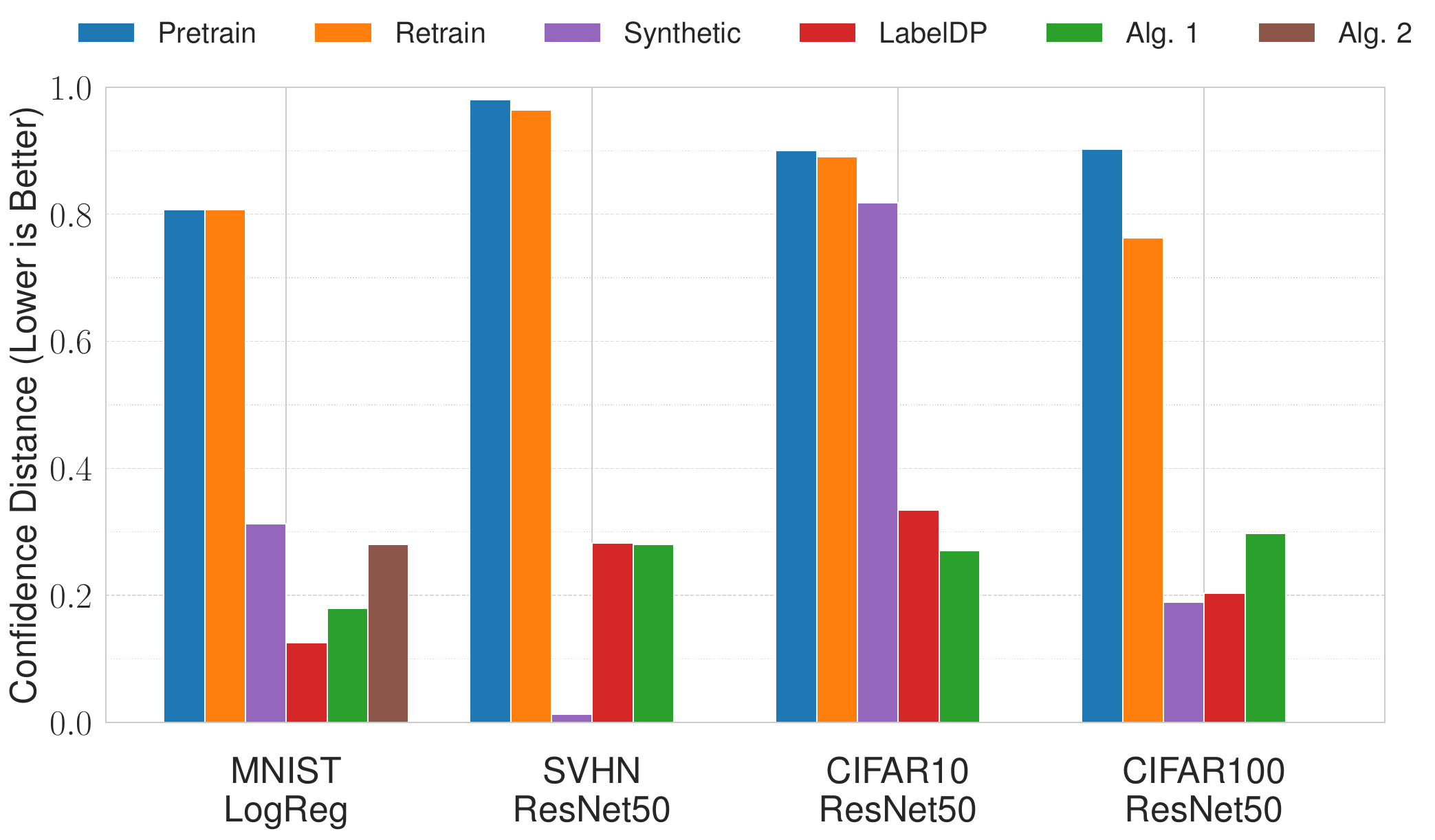}
        \caption{Confidence distance on forget set for baselines as well as \cref{algo:finetuning_algo} and \cref{algo:hess_exact_algo} with $\theta  = 0.75$.}
        \label{fig:confidence_distance_comparison}
    \end{subfigure}
    \vspace{1em}    
    \caption{Across datasets, observe a significant drop in confidence distance, where lower is better, for both our algorithms. We also observe that both algorithms provide strong accuracy on the retain set. We observe similar behavior for the test set in \cref{appendix:additional_experiments}, while the baselines are inconsistent. Variance is negligible for all metrics. }  
    \vspace{-4mm}
    \label{fig:alg1_alg2_results}
\end{figure}

\textbf{Overall Results: } The results for \cref{algo:finetuning_algo} are presented in \cref{fig:alg1_alg2_results}, in which we were able to achieve a $>3\times$ decrease in confidence distance with only a 0.01\% and 0.04\% decrease in retain and test accuracy, respectively, for a ResNet50 pretrained on SVHN. We obtain similar results for MNIST, CIFAR10, and CIFAR100: retain and test set accuracies remain high, while forget confidence distance is significantly reduced. Results for the test set are deferred to \cref{appendix:additional_experiments}. We additionally find that the synthetic baseline can induce uniformity well for SVHN, but can either fail to induce uncertainty entirely (CIFAR10) or induce uncertainty at great cost to retain and test accuracy (CIFAR100). We observe similar behavior for TinyImageNet in \cref{appendix:imagenet_vit}. This holds similarly for LabelDP, which furthermore undesirably reduces the confidence distance on retain and test sets, while our method does not, as demonstrated in \cref{table:label_dp_pareto_conf_dist}. Furthermore, our observations coincide with \citet{zhao2024makes}, observing that unlearned models still produce confident predictions on deleted instances. 

Furthermore, as illustrated by \cref{fig:alg1_alg2_results}, we find that \cref{algo:hess_exact_algo} also induces uniformity well, while marginally reducing retain and test accuracy. Thus, this algorithm produces a certificate through which test-time privacy can be verified while still obtaining a good privacy-utility tradeoff. For both algorithms, tables are included in \cref{appendix:additional_experiments} for completeness.

\begin{figure}[tb]
    \centering
    \begin{subfigure}[tb]{0.47\linewidth}
        \centering
        \includegraphics[width=\linewidth]{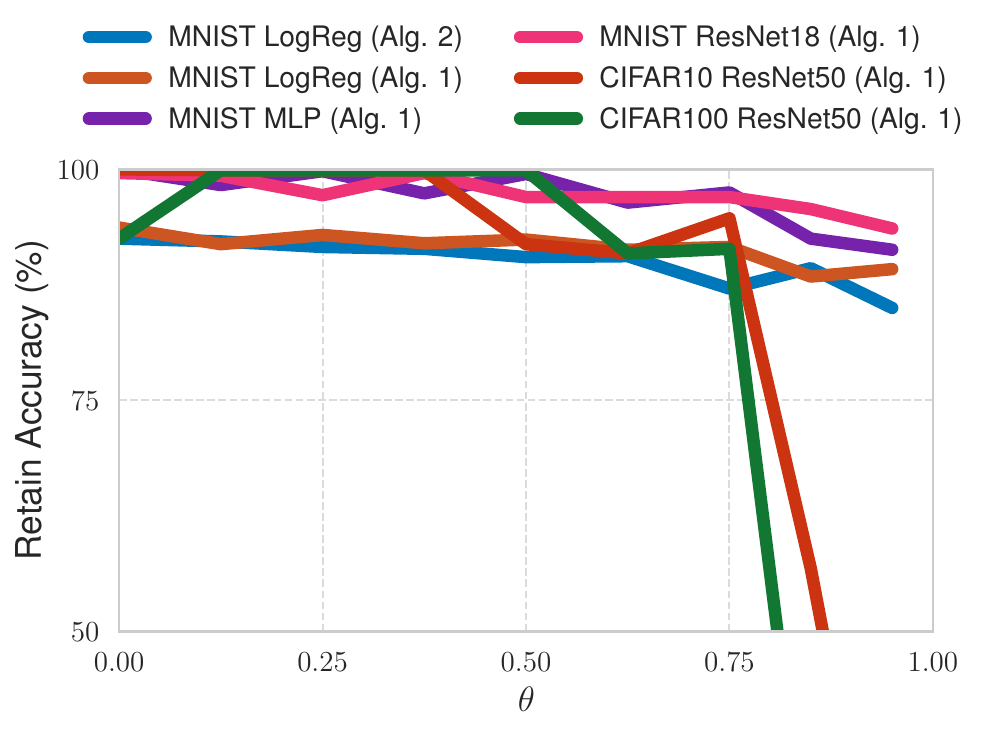}
        \caption{Retain Accuracy vs. $\theta$, MNIST}
        \label{fig:ret_pareto}
    \end{subfigure}
    \hfill 
    \begin{subfigure}[tb]{0.47\linewidth}
        \centering
        \includegraphics[width=\linewidth]{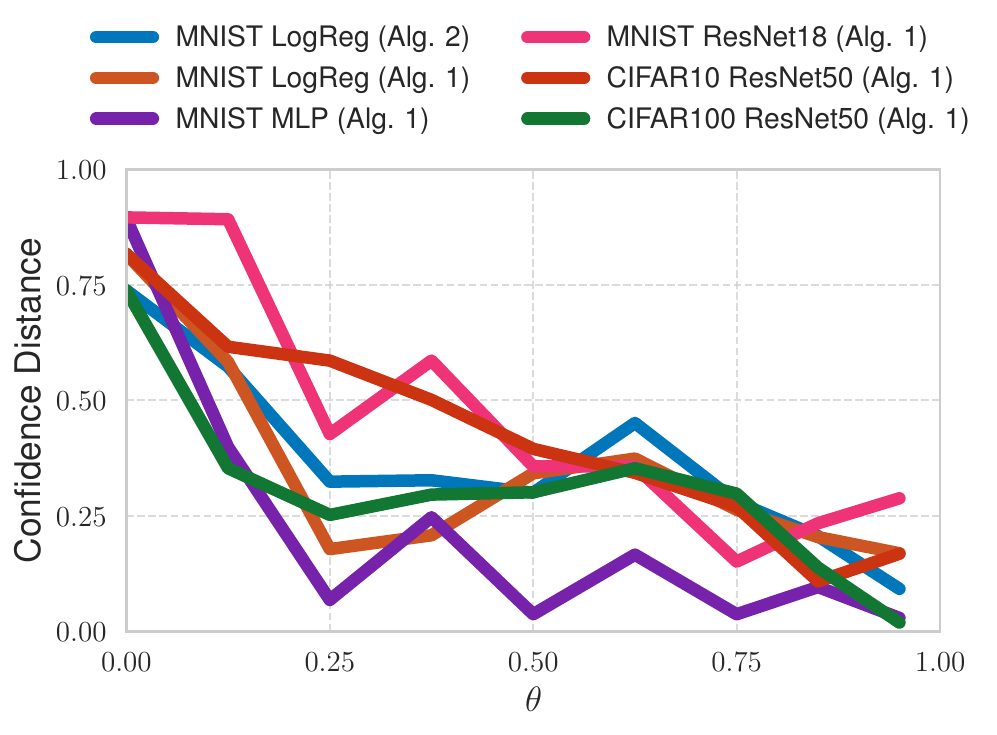}
        \caption{Confidence Distance vs. $\theta$, MNIST}
        \label{fig:conf_pareto}
    \end{subfigure}
    
    \caption{From \cref{fig:ret_pareto}, we observe that for simple datasets, the retain accuracy decreases smoothly. However, for larger datasets like CIFAR10 and CIFAR100 as one passes $\theta \approx 0.75$, retain accuracy drops significantly. This motivates our choice of $\theta = 0.75$ used throughout our experiments. In \cref{fig:conf_pareto} we observe that the confidence distance decreases roughly linearly as $\theta$ increases.}
    \vspace{-4mm}
    \label{fig:conf_ret_pareto}
\end{figure}

\textbf{Pareto Frontiers: } To better understand the structure of our problem, we explore the Pareto frontier in \cref{fig:conf_ret_pareto}. We observe that for MNIST, CIFAR10, and CIFAR100, various $\theta$ can provide good retain accuracy, albeit at the cost of uniformity. In general, we find that $\theta \approx 0.75$ offers a solid privacy-utility tradeoff. Thus, the $\varepsilon$ in  \cref{prop:bound_on_pop_distance_between_output_and_unif} can be chosen fairly large while ensuring low confidence distance.

\textbf{Additional Experiments}: We conduct various additional experiments in in \cref{appendix:additional_experiments} and briefly comment about them here. Firstly, we obtain excellent performance for TinyImageNet and ViT in \cref{appendix:imagenet_vit}. Secondly, as desired, we obtain obtain high confidence distances on the retain and test sets in \cref{appendix:labeldp_main_pareto_conf_dist}. Thirdly, we study the optimization dynamics of \cref{algo:finetuning_algo} in \cref{appendix:optimization_dynamics}, providing mathematical and empirical evidence for the necessity of early stopping in large models when using \cref{algo:finetuning_algo}. Fourthly, we evaluate our method on several strong TTP attacks, demonstrating that we can still offer effective defense, especially when compared to pretraining or retraining, in \cref{appendix:ttp_attacks}. Fifthly, in \cref{appendix:nearest_neighbors}, we find that we preserve strong accuracy and high confidence, as desired, on  test instances which are nearest neighbors to the forget set instances. Thus, an adversary querying nearby instances outside of the forget set does not suffice to circumvent our algorithms. Sixthly, we find that we can induce uncertainty on forget instances which were not part of the original training dataset, while still preserving retain and test accuracies, in \cref{appendix:test_finetuning}. Seventhly, we provide ablations on the size of our forget set in \cref{section:forget_set_size}. Finally, we compare our confidence metric to an $\ell_2$ uniformity metric, finding that they highly correlate, in \cref{appendix:l2_metric}.

%% file: sections/discussion.tex
\section{Discussion} \label{section:discussion}

We present \textit{test-time privacy}, a threat model in which an adversary seeks to directly use a confident prediction for harm. This contrasts with existing work like PATE and LabelDP, which focus on protecting against model inversion and leakage of ground truth labels. To protect against a test-time privacy adversary, we present multiple algorithms to induce uniformity on a known corrupted subset while preserving utility on the rest of the data instances. This can be used to prevent adversaries from taking advantage of model outputs. Furthermore, we prove a privacy-utility tradeoff for our algorithms, providing a tight bound which is empirically verified. We hope our test-time privacy can further inspire the community to explore different threat models for sensitive data. Limitations and future directions are provided in \cref{appendix:lims_future_dirs}.

%% file: sections/reproducibility.tex
\section{Reproducibility Statement}\label{section:reproducibility}

In order to ensure reproducibility of results throughout the paper, the code for all experiments is available for reproducibility at \url{https://tinyurl.com/testtimeprivacy}. Hyperparameters used throughout experiments are carefully detailed in \cref{appendix:experimental_details}. Psuedocode is additionally included for all algorithms and attacks used throughout the paper in either \cref{section:formulation} (\cref{algo:finetuning_algo}; \cref{algo:hess_exact_algo}), \cref{section:certifying_algorithm} (\cref{algo:hess_estimator_algo}), \cref{section:online_algorithm} (\cref{algo:sequential_algorithm}), or \cref{appendix:threat_model} (\cref{algo:TTP_attack_Gauss}; \cref{algo:TTP_attack_Gauss}; \cref{algo:TTP_attack_PGD}). Proofs of all theorems and otherwise formal statements made throughout the paper can be found in \cref{section:appendix_proofs}, with a symbol table in \cref{appendix:symbol_table}.

%% file: sections/acknowledgements.tex
\section*{Acknowledgements}

The authors would like to thank Kangwook Lee for feedback on the early idea and proposing the synthetic and GaussianUniform baselines. %

%% file: sections/appendix.tex
\newpage
\appendix

\label{section:appendix}

\renewcommand{\thefigure}{\thesection.\arabic{figure}} %
\renewcommand{\thetable}{\thesection.\arabic{table}} %
\renewcommand{\theequation}{\thesection.\arabic{equation}}

\renewcommand{\contentsname}{Contents of the Appendix}

\addcontentsline{toc}{section}{Appendix} %
\part{\Large Appendix} %
\parttoc %

\newpage

\input{sections/threat_model}

\section{Definining Test-Time Privacy Attacks}\label{appendix:attacks_definition}

In what follows, in light of our threat model provided in \cref{appendix:threat_model}, we design some simple test-time privacy attacks to motivate our problem. We also include experiments on these attacks, and how \cref{algo:finetuning_algo} performs against them, in \cref{appendix:ttp_attacks}. 

Our first simple algorithm is to add a small amount of uniformly sampled Gaussian noise, presented in \cref{algo:TTP_attack_Gauss}. We find that this is not very effective in increasing confidence distance, as demonstrated in \cref{table:attack_results_MNIST_KMNIST} and \cref{table:attack_results_SVHN_CIFAR}. When it brings the confidence distances from low to moderate, the model is usually confidently wrong, as demonstrated in \cref{table:attack_forget_accuracies}. 

One way to more optimally attack the TTP of a pretrained model is by finding instances in a $\phi$-ball around the forget set instances that maximize the prediction confidence. To design such an attack, suppose we have a pretrained classifier $f_{\bm{w^*}}: \mathcal{X} \to \Delta_{|\mathcal{Y}|}$. Here, $f_{\bm{w}}(\bm{x}) = \text{softmax}(\bm{z}(\bm{x}))$, for $\bm{x} \in \mathcal{X}$, where $\bm{z}$ is a vector of logits. For a forget set instance, we begin by adding a small amount of uniform noise to break symmetry and obtain a nonzero gradient. We then want to obtain the worst-case perturbation over the logits by solving the optimization problem: 

\begin{align}
    \max_\phi \max_{j} \bm{z}_{\bm{w}^*}^{(j)}(\bm{x} + \phi), \\ 
    \text{s.t. } ||\phi||_\infty \leq \gamma.
\end{align}

Since the $\max$ function is not differentiable everywhere, we use LogSumExp to approximate it. Denote $\rho(f_{\bm{w}}(\bm{x + \phi})) = \log \sum_{j = 1}^{|\mathcal{Y}|} \exp{(\bm{z}_{\bm{w}}^{j}(\bm{x} + \phi))}$. This yields the optimization problem: 

\begin{align}
    \max_\phi  \rho(f_{\bm{w}^*}(\bm{x} + \delta)), \\ 
    \text{s.t. } ||\phi||_\infty \leq \gamma.
\end{align}

Following the Fast Gradient Sign Method (FGSM), a simple attack used to generate adversarial examples \citep{goodfellow2014explaining}, we design an attack as \cref{algo:TTP_attack_FGSM}. Intuitively, we take a single linear step towards maximizing the function. We design also design stronger attack based on Projected Gradient Descent (PGD)  \citep{madry2017towards} as \cref{algo:TTP_attack_PGD}, taking 40 steps while incrementally maximizing the confidence function while projecting back to the ball around the original instance. Empirical results are in \cref{appendix:ttp_attacks}.

\section{Additional Related Work}\label{appendix:additional_related_work}

\input{sections/additional_related_work}

\section{Uniformity Metric}\label{appendix:uniformity_metric}

The confidence distance quantifies the adversary’s confidence in their final prediction, i.e. the difference between the argmax softmax score and the uniform softmax score. Importantly, our method aims to have the adversary lack confidence in their final prediction. Thus, our metric captures what we aim to measure and is interpretable, since it is minimized at 0. 
  
Furthermore, confidence distance allows us to quantify how uncertain the model is without relying on accuracy, since a drop in forget set accuracy is not the goal of our formulation. Next, if the maximum confidence score is very close to the uniform distribution, the probability mass of the output distribution must be distributed over the other softmax outputs, clearly yielding that the higher our uniformity metric, the more confident our model is, and the lower our uniformity metric, the less confident our model is. Additionally, it takes the dataset into account; for example, in CIFAR10, one would expect a uniformity score of $\approx0.2$ to be reasonable, as then the adversary can only be $\approx30\%$ confident that they have a useful prediction. However, for CIFAR100, a uniformity score of $\approx0.2$ is much better, as it implies that an adversary can only be $\approx21\%$ confident that they have a useful prediction. 

One objection to the use of this metric may be that it does not indicate uniformity if it is low. For example, on CIFAR10, one could have a confidence score of $0.2$, which yields that the max softmax output is $0.35$. There could be three other nonzero softmax outputs of $0.3$, $0.3$, $0.1$, $0.05$; this clearly is not uniform. However, this ensures test-time privacy; a test-time privacy adversary now has little confidence in their prediction, even if they choose the first one, rendering their warrant for misuse of sensitive data useless. 

We empirically compare our confidence distance metric to other similar metrics in \cref{appendix:additional_experiments}, finding that when our confidence metric is minimized, other metrics are minimized.  

\section{Test-Time Privacy Examples}\label{appendix:additional_ttp_examples}

Here, we provide a set of examples of the TTP threat model: 

\textbf{Health Insurance}: Suppose an open-weight medical imaging model $f$ is released, designed to perform multiclass classification of skin photos into categories like ``Dysplastic Nevus'', ``Benign Keratosis'', which are usually harmless, or serious classes like ``Melanoma'' \citep{sun2016benchmark}. A person $p$ posts a photo of a harmless birthmark on his arm to a public health forum to ask a question. During the upload, an e.g. server error or compression issue causes the image file to become corrupted, severely distorting the birthmark. This results in a photo $\bm{x}_p$. Next, a health insurance startup decides to build risk profiles by scraping these public forums. They download the open-weight model $f$ to automatically screen images for potential health liabilities. When they feed $\bm{x}_p$ into $f$, it confidently classifies $\bm{x}_p$ as ``Melanoma''. This erroneous classification is then automatically added to person $p$'s risk profile, resulting person $p$ being unfairly denied coverage. 

\textbf{Criminal Records}: Suppose a model $f$ is trained on criminal records to predict individual crime likelihood. Additionally, suppose the criminal record $\bm{x}_p$ of a person $p$ is corrupted and publicly available. Then, $f(\bm{x}_p)$ predicts that person $p$ is highly likely to commit crime. An adversarial law enforcement agency, or even a prospective employer, may ignore or be unaware of warnings about the data being corrupted, rendering a dangerous scenario for person $p$.\footnote{Recently, ML model providers have been involved in privacy cases involving criminal records \citep{iappgooglegdpr2020}, making this threat pertinent.} To make this clear, provide a figure similar to that of \cref{fig:motivation} at \cref{fig:motivation_police}. 

\input{sections/motivational_fig_duplicate}

\textbf{Mortgage Loans}: Suppose a model $f$ is trained on various items relevant to whether one receives a mortgage loan or not, like bank statements and past rent payments. Person $p$ has corrupted rent payment history $\bm{x}_p$. Then, the bank runs model $f$ and obtains $f(\bm{x}_p)$, which confidently says that $\bm{x}_p$ is undeserving of a loan. 

\textbf{Car Insurance}: Suppose a model $f$ is trained on one's history of car accidents. Person $p$ has corrupted car accident history $\bm{x}_p$. Then, when applying for car insurance, the provider runs model $f$ and obtains $f(\bm{x}_p)$, which confidently says that $\bm{x}_p$ is undeserving of a loan.\footnote{Note that recent, the Department of Motor Vehicles in America has been selling driving records, making this threat pertinent \citep{dmvselling2020}.}

We provide an additional example in the generative setting as well: 

\textbf{News Articles}: Consider a text-to-image generative model trained on a large dataset, including web data, which has web articles and associated images. A popular news site publishes an article about a businessperson, but mistakenly uses a picture of an unrelated individual $p$, $\bm{x}_p$, as the header image. This creates a strong, albeit false, association between this person's likeness and the (perhaps negative) content of the article. When prompted with a string similar to the headline of the news article, the model generates an image (or a similar image) of person $p$, algorithmically cementing a false narrative about person $p$. 

\subsection{Limitations and Future Directions} \label{appendix:lims_future_dirs}
Notably, our presented method only applies to classification. Extending this to generative models e.g. diffusion models for image generation \citep{song2020score} or autoregressive transformers for sequence-to-sequence generation \citep{vaswani2017attention} remains as future work. Furthermore, even in the discriminative setting, we focus our method on image classification. Extending our methods to the text setting, which is nontrivial due to discrete inputs, remains as future work. 

From an algorithmic perspective, in \cref{algo:finetuning_algo}, we use linear scalarization to design our objective \citep{hwang2012multiple}. One can instead design an objective using $\varepsilon$-constraints \citep{miettinen1999nonlinear}, which can then be solved by an augmented Lagrangian method \citep{nocedal1999numerical}.

\input{sections/designing_algo}

\section{Proofs}\label{section:appendix_proofs}

\subsection{Helpful Lemmas}\label{section:proof_of_helpful_lemma}

\begin{lemma}
    Given \cref{assumption1}, the gradients $\nabla_{\bm{w}, \mathcal{K}}$ and $\nabla_{\bm{w}, \mathcal{A}}$ exist and are Lipschitz with constants $P_{\mathcal{K}}$ and $P_\mathcal{A}$, respectively. Furthermore, given \cref{assumption2}, the Hessians  $\bm{H}_{\bm{w}, \mathcal{K}}$ and $\bm{H}_{\bm{w}, \mathcal{A}}$ exist and are Lipschitz with constants $F_{\mathcal{K}}$ and $F_{\mathcal{A}}$, respectively. 
    \label{lemma:lipschitzness_of_hessians_and_grads}
\end{lemma}

\begin{proof}
    \begin{align}
    ||\nabla_{\bm{w}_1, \mathcal{K}} - \nabla_{\bm{w}_2, \mathcal{K}}||_2  &= ||\sum_{i = 1}^{|\mathcal{D}_f|} \nabla^2 \ell_{\mathcal{K}}^{(i)}(\bm{w}_1, \mathcal{D}_f^{(i)}) - \sum_{i = 1}^{|\mathcal{D}_f|} \nabla^2 \ell_{\mathcal{K}}^{(i)}(\bm{w}_2, \mathcal{D}_f^{(i)})||_2 \\
    &\leq \sum_{i = 1}^{|\mathcal{D}_f|} ||\nabla^2 \ell_{\mathcal{K}}^{(i)}(\bm{w}_1, \mathcal{D}_f^{(i)}) - \nabla^2 \ell_{\mathcal{K}}^{(i)}(\bm{w}_2, \mathcal{D}_f^{(i)})||_2, \; \text{triangle inequality} \\ 
    &\leq \sum_{i = 1}^{|\mathcal{D}_f|} \frac{P_{\mathcal{K}}}{|\mathcal{D}_f|}||\bm{w}_1 - \bm{w}_2||_2, \text{\cref{assumption1}} \\ 
    &= P_{\mathcal{K}}||\bm{w}_1 - \bm{w}_2||_2
\end{align}

This follows similarly for $\nabla_{\bm{w}, \mathcal{A}}$,  $\bm{H}_{\bm{w}, \mathcal{K}}$, and $\bm{H}_{\bm{w}, \mathcal{A}}$. 
\end{proof}

\begin{lemma} Given \cref{assumption1}, for any dataset $\mathcal{D} \subset \mathcal{Z}^n$, $\mathcal{L}_\mathcal{A}$ satisfies: 

    \begin{equation}
        |\mathcal{L}_A(\bm{w}_1, \mathcal{D}) - \mathcal{L}_\mathcal{A}(\bm{w}_2, \mathcal{D})| \leq \frac{P_{\mathcal{K}}}{2}||\bm{w}_1 - \bm{w}_2||_2^2 + ||\nabla_{\bm{w}_2, \mathcal{A}}||_2||\bm{w}_1 - \bm{w}_2||_2
    \end{equation}
    \label{lemma:descent_lemma}
\end{lemma}

\begin{proof}
    By the fundamental theorem of calculus, we have the path integral: 

    \begin{equation}
        \int_0^1 \langle \nabla_{\bm{w}_2 + t(\bm{w}_1 -\bm{w}_2), \mathcal{A}}, \bm{w}_1 - \bm{w}_2 \rangle dt = \mathcal{L}_A(\bm{w}_1, \mathcal{D}) - \mathcal{L}_\mathcal{A}(\bm{w}_2, \mathcal{D})
        \label{eq:path_integral_L_A}
    \end{equation}

    We have that: 

    \begin{align}
        \int_0^1 \langle \nabla_{\bm{w}_2 + t(\bm{w}_1 -\bm{w}_2), \mathcal{A}}, \bm{w}_1 - \bm{w}_2 \rangle dt &= \int_0^1 \langle \nabla_{\bm{w}_2, \mathcal{A}} - \nabla_{\bm{w}_2, \mathcal{A}} + \nabla_{\bm{w}_2 + t(\bm{w}_1 -\bm{w}_2), \mathcal{A}}, \bm{w}_1 - \bm{w}_2 \rangle dt \\ 
        &= \int_0^1 \langle \nabla_{\bm{w}_2, \mathcal{A}}, \bm{w}_1 - \bm{w}_2 \rangle dt \\&+ \int_0^1 \langle \nabla_{\bm{w}_2 + t(\bm{w}_1 -\bm{w}_2), \mathcal{A}}  - \nabla_{\bm{w}_2, \mathcal{A}}, \bm{w}_1 - \bm{w}_2 \rangle dt
        \label{eq:descent_lemma_to_bound}
    \end{align}

    The first term can be bounded by Cauchy-Schwarz as: 

    \begin{equation}
        \int_0^1 \langle \nabla_{\bm{w}_2, \mathcal{A}}, \bm{w}_1 - \bm{w}_2 \rangle dt \leq ||\nabla_{\bm{w}_2, \mathcal{A}}||_2||\bm{w}_1 - \bm{w}_2||_2
    \end{equation}

    and similarly the second term can be bounded by Cauchy-Schwarz as: 

    \begin{align}
        \int_0^1 \langle \nabla_{\bm{w}_2 + t(\bm{w}_1 -\bm{w}_2), \mathcal{A}}  - \nabla_{\bm{w}_2, \mathcal{A}}, \bm{w}_1 - \bm{w}_2 \rangle dt &\leq \int_0^1 ||\nabla_{\bm{w}_2 + t(\bm{w}_1 -\bm{w}_2), \mathcal{A}}  - \nabla_{\bm{w}_2, \mathcal{A}}||_2 ||\bm{w}_1 - \bm{w}_2||_2 dt \\ 
        &\leq \int_0^1 P_{\mathcal{K}}t ||\bm{w}_1 - \bm{w}_2||_2, \text{by \cref{lemma:lipschitzness_of_hessians_and_grads}} \\ 
        &\leq \frac{P_{\mathcal{A}}}{2}||\bm{w}_1 - \bm{w}_2||_2
    \end{align}

    Incorporating these bounds into \cref{eq:descent_lemma_to_bound} and \cref{eq:path_integral_L_A}, upon applying the triangle inequality, yields:

    \begin{equation}
         |\mathcal{L}_A(\bm{w}_1, \mathcal{D}) - \mathcal{L}_\mathcal{A}(\bm{w}_2, \mathcal{D})| \leq \frac{P_{\mathcal{K}}}{2}||\bm{w}_1 - \bm{w}_2||_2^2 + ||\nabla_{\bm{w}_2, \mathcal{A}}||_2||\bm{w}_1 - \bm{w}_2||_2
    \end{equation}

    as desired. 
    
\end{proof}

\begin{lemma}
    Given \cref{assumption2}, the Hessians $\bm{H}_{\bm{w}, \mathcal{K}}$,  $\bm{H}_{\bm{w}, \mathcal{A}}$, and $\bm{H}_{\bm{w}, \mathcal{K}, \mathcal{A}}$ are symmetric. 
    \label{lemma:symmetricity_of_hessians}
\end{lemma}

\begin{proof}
    By \cref{lemma:lipschitzness_of_hessians_and_grads}, the Hessians $\bm{H}_{\bm{w}, \mathcal{K}}$ and  $\bm{H}_{\bm{w}, \mathcal{A}}$ are continuous, and thus $\bm{H}_{\bm{w}, \mathcal{K}, \mathcal{A}}$ is continuous by linearity. Hence, all second-order partial derivatives contained in the Hessians are continuous, so by Schwartz's theorem all Hessians are symmetric. Importantly, for e.g. $\bm{H}_{\bm{w}, \mathcal{K}, \mathcal{A}}$, $||\bm{H}_{\bm{w}, \mathcal{K}, \mathcal{A}}||_2 = \max_i |\lambda_{i}(\bm{H}_{\bm{w}, \mathcal{K}, \mathcal{A}})|$, where $\lambda_i$ denotes the $i$th eigenvalue. 
\end{proof}

\begin{lemma}
    (Corollary of Theorem A.1 in \citep{dwork2014algorithmic}) Let $X \sim \mathcal{N}(\lambda, \sigma^2 \bm{I})$ and $\mathcal{Y} \sim \mathcal{N}(\lambda', \sigma^2 \bm{I})$. Suppose $||\lambda - \lambda'||_2 \leq \Delta$. Then for any $\delta > 0$, $X$ and $Y$ are $(\varepsilon, \delta)$-indistinguishable if $\sigma \geq \frac{\Delta}{\varepsilon}\sqrt{2\ln(1.25/\delta)}$. 
    \label{lemma:gaussian_mech_lemma}
\end{lemma}

\begin{lemma} 
Suppose we have $n$ i.i.d. data samples $(X_1, ..., X_n)$ drawn from $D_f$ and $D_r$ with probabilities $\theta$ and $1-\theta$ respectively. For $t = 1, ..., n$, if $X_t \sim D_f$ let $\bm{H}_{t, \lambda} = \bm{H}_{w^*, \mathcal{K}, t} + \frac{\lambda I}{2\theta}$ and if $X_t \sim D_r$ let $\bm{H}_{t, \lambda} = \bm{H}_{w^*, \mathcal{A}, t} + \frac{\lambda I}{2(1-\theta)}$. Then, $\mathbb{E}[\bm{H}_{t, \lambda}] = \bm{H}_{\bm{w}^*, \mathcal{K}, \mathcal{A}} + \lambda \bm{I}$ i.e. $\bm{H}_{t, \lambda}$ is an unbiased estimator of our Hessian of interest. 
    \label{lemma:unbiased_estimator_of_H_w_star}
\end{lemma}

\begin{proof}
    At time $t$, we have sample $X_t$ s.t. $X_t \sim D_r$ or $X_t \sim D_f$. Note that $\mathbb{E}_{X_t \sim D_f}[\bm{H}_{\bm{w}^*, \mathcal{K}, t} + \frac{\lambda \bm{I}}{2\theta}] = \bm{H}_{\bm{w}^*, \mathcal{K}} + \frac{\lambda \bm{I}}{2\theta}$, and likewise $\mathbb{E}_{X_t \sim D_r}[\bm{H}_{\bm{w}^*, \mathcal{A}, t} + \frac{\lambda \bm{I}}{2(1-\theta)}] = \bm{H}_{\bm{w}^*, \mathcal{A}} + \frac{\lambda \bm{I}}{2(1-\theta)}$. 

    By the law of iterated expectation, we have that: 

    \begin{align*}
        \mathbb{E}[\bm{H}_{t, \lambda}] &= \mathbb{E}[\bm{H}_{t, \lambda} | X_t \sim \mathcal{D}_f]\Pr(X_t \sim \mathcal{D}_f) + \mathbb{E}[\bm{H}_{t, \lambda} | X_t \sim \mathcal{D}_r]\Pr(X_t \sim \mathcal{D}_r) \\ 
        &= \theta \mathbb{E}_{X_t \sim D_f}[\bm{H}_{\bm{w}^*, \mathcal{K}, t}  + \frac{\lambda I}{2\theta}) ] + (1-\theta)\mathbb{E}_{X_t \sim \mathcal{D}_r}[\bm{H}_{\bm{w}^*, \mathcal{A}, t} + \frac{\lambda I}{2(1-\theta)}] \\
        &= \theta(\bm{H}_{\bm{w}^*, \mathcal{K}} + \frac{\lambda I}{2\theta}) + (1-\theta)(\bm{H}_{\bm{w}^*, \mathcal{A}} + \frac{\lambda I}{2(1-\theta)}) \\
        &= \theta \bm{H}_{\bm{w}^*, \mathcal{K}} + (1-\theta)\bm{H}_{\bm{w}^*, \mathcal{A}} + \lambda I \\
        &= \bm{H}_{\bm{w}^*, \mathcal{K}, \mathcal{A}} + \lambda \bm{I}
    \end{align*} as desired. 
\end{proof}

\begin{lemma}
    Suppose \cref{assumption1,assumption2} hold. Let local condition number $\hat{\kappa}_l$ and maximum local condition number $\hat{\kappa}_l^{\max}$ correspond to the definitions of \citet{agarwal2016second} with respect to the Hessian of the loss of $\mathcal{M}_\theta$ after local convex approximation. Then, $\hat{\kappa}_l \leq \frac{B}{\lambda + \lambda_{\min}}$ and $\hat{\kappa}_l^{\max} \leq \frac{B}{\zeta_{\min}}$ where $B = \max\{\frac{\theta P_{\mathcal{K}} + \lambda}{|\mathcal{D}_f|}, \frac{(1-\theta)P_{\mathcal{A}} + \lambda}{|\mathcal{D}_r|}\}$ and where $\zeta_{\min} \geq \min_i \lambda_{\min}(\nabla^2_{\bm{w}} \tilde{\ell}_{\mathcal{K},\mathcal{A}}^{(i)}(\bm{w}, \mathcal{D}^{(i)}))$.
    \label{lemma:condition_nums_bound}
\end{lemma}

\begin{proof}

By \cref{eq:pareto_optimal_with_weight_bound_and_reg} and our local convex approximation technique, we have that: 

\begin{align}
    M_{\theta}(D) &= \arg\min_{\bm{w} \in \mathcal{W}} \;\theta \sum_{i = 1}^{|\mathcal{D}_f|} \ell^{(i)}_{\mathcal{K}}(\bm{w}, \mathcal{D}_f^{(i)}) + (1-\theta)\sum_{i = 1}^{|\mathcal{D}_r|} \ell^{(i)}_{\mathcal{A}}(\bm{w}, \mathcal{D}_r^{(i)}) + \frac{\lambda}{2}||\bm{w}||_2^2 \\
    &= \arg\min_{w \in \mathcal{W}} \; \sum_{i = 1}^{|D_f|} (\theta \ell^{(i)}_{\mathcal{K}}(\bm{w}, \mathcal{D}_f^{(i)}) + \frac{\lambda}{2|\mathcal{D}_f|}||\bm{w}||_2^2) +  \sum_{i = 1}^{|D_r|} (\theta \ell^{(i)}_{\mathcal{A}}(\bm{w}, \mathcal{D}_r^{(i)}) + \frac{\lambda}{2|\mathcal{D}_r|}||\bm{w}||_2^2) \\
    &= \arg\min_{w \in \mathcal{W}} \; \sum_{i = 1}^{|D|} \tilde{\ell}^{(i)}_{\mathcal{K},\mathcal{A}}(\bm{w}, \mathcal{D}^{(i)})
\end{align}

where 

\begin{align}
    \tilde{\ell}^{(i)}_{\mathcal{K},\mathcal{A}}(\bm{w}, \mathcal{D}^{(i)}) = 
    \begin{cases}
        \theta \ell_{\mathcal{K}}^{(i)}(\bm{w}, \mathcal{D}_f^{(i)}) + \frac{\lambda}{2|\mathcal{D}_f|}||\bm{w}||_2^2, \; 1 \leq i \leq |\mathcal{D}_f|\\ 
        (1-\theta) \ell^{(i-|\mathcal{D}_f|)}_{\mathcal{K}}(\bm{w}, \mathcal{D}_r^{(i-|\mathcal{D}_f|)}) + \frac{\lambda}{2|\mathcal{D}_r|}||\bm{w}||_2^2, \; |\mathcal{D}_f| + 1 \leq i \leq |\mathcal{D}|
    \end{cases}
\end{align}

By the definitions provided in \citet{agarwal2016second}, we have: 

\begin{equation}
\hat{\kappa}_l = \max_{\bm{w} \in \mathcal{W}} \frac{\max_i \lambda_{\max}(\nabla^2_{\bm{w}} \tilde{\ell}_{\mathcal{K},\mathcal{A}}(\bm{w}, \mathcal{D}^{(i)}))}{\lambda_{\min}(\bm{H}_{\bm{w}, \mathcal{K}, \mathcal{A}} + \lambda\bm{I})}
    \label{eq:defn_kappa_hat}
\end{equation}

and

\begin{equation}
\hat{\kappa}_l^{\max} = \max_{\bm{w} \in \mathcal{W}} \frac{\max_i \lambda_{\max}(\nabla^2_{\bm{w}} \tilde{\ell}_{\mathcal{K},\mathcal{A}}(\bm{w}, \mathcal{D}^{(i)}))}{\min_i \lambda_{\min}(\nabla^2_{\bm{w}} \tilde{\ell}_{\mathcal{K},\mathcal{A}}(\bm{w}, \mathcal{D}^{(i)}))}
    \label{eq:defn_kappa_hat_max}
\end{equation}

We then have that, for any $i$, 

\begin{align}
    \lambda_{\max}(\nabla^2_{\bm{w}} \tilde{\ell}^{(i)}_{\mathcal{K},\mathcal{A}}) &\leq ||\nabla^2_{\bm{w}} \tilde{\ell}_{\mathcal{K},\mathcal{A}}^{(i)}||_2, \; \text{by \cref{lemma:symmetricity_of_hessians}} \\
    &= \max\{||\theta \nabla^2_{\bm{w}} \ell_{\mathcal{K}}^{(i)} + \frac{\lambda\bm{I}}{|\mathcal{D}_f|}||_2 , ||(1-\theta)\nabla^2_{\bm{w}} \ell_{\mathcal{A}}^{(i)} + \frac{\lambda \bm{I}}{|\mathcal{D}_r|}||_2\} \\
\end{align}

Furthermore, by \cref{assumption1} and the triangle inequality:

\begin{align}
    ||\theta \nabla^2_{\bm{w}} \ell_{\mathcal{K}}^{(i)} + \frac{\lambda\bm{I}}{|\mathcal{D}_f|}||_2 &\leq \frac{\theta P_{\mathcal{K}} + \lambda}{|\mathcal{D}_f|}
\end{align}

and \begin{align}
    ||(1-\theta) \nabla^2_{\bm{w}} \ell_{\mathcal{A}}^{(i)} + \frac{\lambda\bm{I}}{|\mathcal{D}_r|}||_2 &\leq \frac{(1-\theta) P_{\mathcal{A}} + \lambda}{|\mathcal{D}_r|}
\end{align}

Taking max over all $i$, we obtain that \begin{equation}
    \max_i \lambda_{\max}(\nabla^2_{\bm{w}} \tilde{\ell}_{\mathcal{K},\mathcal{A}}(\bm{w}, \mathcal{D}^{(i)})) \leq \max\{\frac{\theta P_{\mathcal{K}} + \lambda}{|\mathcal{D}_f|},  \frac{(1-\theta) P_{\mathcal{A}} + \lambda}{|\mathcal{D}_r|}\}
\end{equation}

which we denote by $B$. 

Then, we obtain that  $\hat{\kappa}_l \leq \frac{B}{\lambda + \lambda_{\min}}$ and $\hat{\kappa}_l^{\max} \leq \frac{B}{\zeta_{\min}}$ as desired, since

\end{proof}

\begin{lemma} (Lemma 3.6 adapted from \citet{agarwal2016second}) Suppose \cref{assumption1} and \cref{assumption2} hold. Consider the estimator $\frac{\bm{\tilde{H}^{-1}_{n, \lambda}}}{H}$ in \cref{thm:asymptotically_unbiased_estimator}. Let $b$ be the number of inverse Hessian estimators we obtain. Suppose $n \geq 2\frac{B}{\lambda + \lambda_{\min}}\ln(\frac{B}{\lambda + \lambda_{\min}}b)$, where $B = \max\{\frac{\theta P_{\mathcal{K}} + \lambda}{|\mathcal{D}_f|}, \frac{(1-\theta)P_{\mathcal{A}} + \lambda}{|\mathcal{D}_r|}\}$. Then, we have that: 

\begin{equation}
\Pr[|| \bm{H}_{\bm{w}^*, \mathcal{K}, \mathcal{A}} +\lambda \bm{I})^{-1} - \frac{\bm{\tilde{H}}_{n, \lambda}^{-1}}{H}||_2 \leq 16\frac{B}{\zeta_{\min}}\sqrt{  \frac{\ln(\frac{d}{\rho})}{b}} + \frac{1}{16}] \geq 1 -\rho
    \label{eq:prob_bound}
\end{equation}

 where $\zeta_{\min} \geq \min_i \lambda_{\min}(\nabla^2_{\bm{w}} \tilde{\ell}_{\mathcal{K},\mathcal{A}}^{(i)}(\bm{w}, \mathcal{D}^{(i)}))$.
 
    \label{lemma:agrawal_et_al_bound}
\end{lemma}

\begin{proof}
     Note that $b = S_1$ in our setting. In our setting, following the subsequent steps of the proof in \citet{agarwal2016second} after plugging in the bounds in \cref{lemma:condition_nums_bound} in place of $\hat{\kappa_l}, \hat{\kappa}_l^{\max}$, noting that we choose $n = S_2 \geq 2\frac{B}{\lambda + \lambda_{\min}}\ln(\frac{B}{\lambda + \lambda_{\min}}b)$, we obtain the exact same result for the Neumann series bound of $\frac{1}{16}$. Using the fact that $\frac{B}{\zeta_{\min}}$ is an upper bound on $\hat{\kappa}_l^{\max}$ by \cref{lemma:condition_nums_bound}, the rest of the proof follows similarly. 
\end{proof}

\begin{lemma}
    (Proposition 2.1 in \citet{dwork2014algorithmic}) Let $\mathcal{M} : \mathbb{N}^{|\mathcal{X}|} \to R$ be a randomized algorithm that is $(\varepsilon, \delta)$-differentially private. Let $f: R \to R'$ be an arbitrary mapping. Then, $f \circ \mathcal{M} : \mathbb{N}^{|\mathcal{X}|} \to R'$ is $(\varepsilon, \delta)$-differentially private. 
    \label{lemma:diff_privacy_post_processing}
\end{lemma}

Note that, in the proof of \cref{lemma:diff_privacy_post_processing}, one proves this fact for deterministic mappings, so this holds for both randomized and deterministic $f$. 

\begin{lemma}
    Consider the mapping $\mathcal{J}: \mathcal{W} \to R$, and suppose $\mathcal{G} : \mathcal{Z}^n \times \mathcal{Z}^n \times \mathcal{W} \to \mathcal{W}$ satisfies \cref{defn:certified_uniformity}. Then, $\forall C \subset R$:

    \begin{align}
    \Pr(\mathcal{J}(\mathcal{G}(\mathcal{D}, \mathcal{D}_f, \mathcal{A}(\mathcal{D}))) \in C) \leq e^{\varepsilon}\Pr(\mathcal{J}(\mathcal{M}_\theta(\mathcal{D})) \in C) + \delta  \\ 
\Pr(\mathcal{J}(\mathcal{M}_\theta(\mathcal{D})) \in C) \leq e^{\varepsilon}\Pr(\mathcal{J}(\mathcal{G}(\mathcal{D}, \mathcal{D}_f, \mathcal{A}(\mathcal{D}))) \in C) + \delta 
\end{align}

\label{lemma:post_processing_of_uniformity}
\end{lemma}

\begin{proof}
    Immediate from \cref{lemma:diff_privacy_post_processing}. 
\end{proof}

\subsection{Proof of Proposition \labelcref{prop:uniform_learner_exists}}\label{section:proof_of_unif_learner_exists}

\begin{proof}
    Fix a dataset $\mathcal{D} \subset \mathcal{Z}^n$. 

    Suppose we have an $K-$layer function $f_{\bm{w}} : \mathbb{R}^d \to \mathbb{R}^o$ parameterized by $\bm{w} \in \mathcal{W}$ of the form $f(\bm{x}) = L_1 \circ ... \circ L_{K-1} \circ L_{K}$ where $L_{K-1}(\bm{x}) = \bm{W}^T_{K-1}\bm{x} + \bm{b}_{K-1}$ and $L_K(\bm{x}) = \text{softmax}(\bm{x})$, i.e. $L_{K-1}(\bm{x})_i = \frac{e^{x_i}}{\sum_{j = 1}^{|\mathcal{Y}|} e^{x_j}}$. Thus, $f_{\bm{w}} \in \mathcal{H}_{\mathcal{W}}$. Then, let $\bm{W}_{K-1} = \bm{0}$ and $\bm{b}_{K-1} = \bm{0}$. 

    Fix $\bm{z} \in \mathcal{D}$. This yields, for $j = 1, ..., |\mathcal{Y}|$, $f(\bm{z})_j = \frac{e^{0}}{\sum_{j = 1}^{|\mathcal{Y}|} e^0} = \frac{e^{0}}{|\mathcal{Y}| e^0} = \frac{1}{|\mathcal{Y}|}$. Hence, since $\bm{z}$ was arbitrary, $f_{\bm{w}}(\bm{z}) = \underbrace{(\frac{1}{|\mathcal{Y}|}, ..., \frac{1}{|\mathcal{Y}|})}_{|\mathcal{Y}| \; \text{times}}$ $\; \forall \bm{z} \in \mathcal{D}$. Since $\mathcal{D}$ was arbitrary, by definition of a uniform learner over $\mathcal{D}$, $f_{\mathcal{K}(\mathcal{D})} \in \mathcal{H}_{\mathcal{W}} \; \forall \mathcal{D} \subset \mathcal{Z}^n$ as desired.

\end{proof}

\subsection{Proof of Proposition \labelcref{prop:pareto_optimal}}\label{section:proof_of_pareto_optimality}

We use the following definition of global Pareto optimality: 

\begin{definition} (Chapter 1 of \citet{pardalos2017multiobjective}) Suppose we have a multiobjective optimization problem $\min \bm{f}(\bm{x}) \; \text{s.t.} \; \bm{x} \in A$, where $\bm{f}(\bm{x}) = (f_1(\bm{x}), f_2(\bm{x}), ..., f_m(\bm{x}))$. $\bm{x^*} \in A$ with $\bm{f}(\bm{x^*})$ is called globally Pareto optimal if and only if there exists no $\bm{x} \in A$ such that $f_i(\bm{x}) \leq f_i(\bm{x^*)}$ for all $i = 1, 2, ..., m$ and $f_j(\bm{x}) < f_j(\bm{x^*})$ for at least one $j \in \{1, ..., m\}$. 

\label{defn:pareto_optimal}
\end{definition}

We can then prove the statement:

\begin{proof}
Let $\theta \in (0,1)$. Fix $\mathcal{D} \subset \mathcal{Z}^n$, $\mathcal{D}_f \subset \mathcal{D}$, and $\mathcal{D}_r = \mathcal{D} \setminus \mathcal{D}_f$. 

Suppose, for the sake of contradiction, that $\bm{\tilde{w}^*} = \mathcal{M}_\theta(\mathcal{D}) = \text{argmin}_{\bm{w}} \theta\mathcal{L}_{\mathcal{K}}(\bm{w}, \mathcal{D}_f)+ (1-\theta) \mathcal{L}_\mathcal{A}(\bm{w}, \mathcal{D}_r)$, a global minimizer, is not globally Pareto optimal with respect to $\mathcal{L}_{\mathcal{K}}(\bm{w}, \mathcal{D}_f)$ and $\mathcal{L}_\mathcal{A}(\bm{w}, \mathcal{D}_r)$. Then, exists $\bm{w}'$ s.t. $\mathcal{L}_\mathcal{K}(\bm{w}', \mathcal{D}_f) \leq \mathcal{L}_\mathcal{K}(\bm{\tilde{w}^*}, \mathcal{D}_f)$ and $\mathcal{L}_\mathcal{A}(\bm{w}', \mathcal{D}_r) \leq \mathcal{L}_\mathcal{A}(\bm{\tilde{w}^*}, \mathcal{D}_r)$, with at least one of these inequalities being strict. 

Then, since $\theta \in (0,1)$ and $(1-\theta) \in (0,1)$, we have that $\theta\mathcal{L}_{\mathcal{K}}(\bm{w}', \mathcal{D}_f)+ (1-\theta) \mathcal{L}_\mathcal{A}(\bm{w}', \mathcal{D}_r) < \theta\mathcal{L}_{\mathcal{K}}(\bm{\tilde{w}^*}, \mathcal{D}_f)+ (1-\theta) \mathcal{L}_\mathcal{A}(\bm{\tilde{w}^*}, \mathcal{D}_r)$, contradicting optimality of $\bm{\tilde{w}^*}$. As such, $\mathcal{M}_\theta(\mathcal{D})$ is globally Pareto optimal respect to $\mathcal{L}_{\mathcal{K}}(\bm{w}, \mathcal{D}_f)$ and $\mathcal{L}_\mathcal{A}(\bm{w}, \mathcal{D}_r)$ as desired. 

This holds similarly for a local minimizer $\bm{\tilde{w}^*}$, where Pareto optimality similarly holds only locally in a neighborhood around the minima. 

\end{proof}

\subsection{Proof of Theorem \labelcref{thm:unlearning_guarantee_for_pareto}}\label{section:proof_of_cert_guarantee}

\begin{proof}
    The proof follows similarly to Lemma 10 in \citet{sekhari2021remember}; for completeness, we adapt their proof to our setting. 

    Let $\bm{w}^* := A(D), \bm{w}^- := \mathcal{G}(\mathcal{D}, \mathcal{D}_f, \bm{w}^*), \bm{\tilde{w}} := \mathcal{F}(\mathcal{D}, \mathcal{D}_f, \bm{w}^*)$. Departing from the notation of the theorem for clarity, let $\hat{\bm{w}^*} := \mathcal{M}_\theta(\mathcal{D}), \hat{\bm{w}^-} := \mathcal{G}(\mathcal{D}, \emptyset, \hat{\bm{w}^*}), \hat{\bm{\tilde{w}}} := \mathcal{F}(\mathcal{D}, \emptyset, \hat{\bm{w}^*})$. 

    Note that $\hat{\bm{\tilde{w}}} = \hat{\bm{w}^*}$. We then have that $||\bm{\tilde{w}} - \hat{\bm{\tilde{w}}}||_2 = ||\bm{\tilde{w}} - \hat{\bm{w}^*}||_2 \leq \Delta$, by definition of $\Delta$. 

    By definition of $\mathcal{G}$, we have that $\bm{w}^- = \bm{\tilde{w}} + Y$ and $\hat{\bm{w}^-} = \hat{\bm{\tilde{w}}} + Y$, where $Y \sim \mathcal{N}(0, \sigma^2 \bm{I})$ s.t. $\sigma \geq \frac{\Delta}{\varepsilon}\sqrt{2\ln(1.25/\delta)}$. 

    As such, $\bm{w}^- = \mathcal{N}(\bm{\tilde{w}}, \sigma^2 \bm{I})$ and $\hat{\bm{w}^-} \sim \mathcal{N}(\hat{\bm{\tilde{w}}}, \sigma^2 \bm{I})$. 

    Thus, by \cref{lemma:gaussian_mech_lemma}, $\bm{w}^-, \hat{\bm{w}^-}$ are $(\varepsilon, \delta)$-indistinguishable. In particular, since $\hat{\bm{w}^-} = \hat{\bm{w}^*}$ by construction, $\mathcal{G}(\mathcal{D}, \mathcal{D}_f, \mathcal{A}(\mathcal{D}))$ and $\mathcal{M}_\theta(\mathcal{D})$ are $(\varepsilon, \delta)$-indistinguishable, as desired.

\end{proof}

\subsection{Proof of Proposition \labelcref{prop:approx_general_form}} \label{section:proof_of_approx_general_form}

\begin{proof}

By the same token as Lemma 3.3 in \citep{zhang2024certifiedunlearningDNN}, we have that: 

\begin{equation}
    ||\bm{\tilde{w}} - \bm{\tilde{w}}^*||_2 \leq ||\bm{H}_{\bm{w}^*, \mathcal{K}, \mathcal{A}}^{-1}||_2 \int_0^1 ||\bm{H}_{\bm{w}^*, \mathcal{K}, \mathcal{A}} - \bm{H}_{\bm{w}^* + t(\bm{\tilde{w}}^* - \bm{w}^*), \mathcal{K}, \mathcal{A}}||_2 ||\bm{w}^* - \bm{\tilde{w}}^*||_2 dt 
\end{equation}. 

Let $\bm{w}' = \bm{w}^* + t(\bm{\tilde{w}} - \bm{w}^*)$. We have that $||\bm{w}^* - \bm{w}'||_2 = ||\bm{w}^* - \bm{w}^* + t(\bm{\tilde{w}}^* - \bm{w}^*)||_2 = t||\bm{w}^* - \bm{\tilde{w}}^*||_2$. 

Furthermore, by linearity of $\bm{H}_{\bm{w}}$ and the triangle inequality, we have that: 

\begin{align}
||\bm{H}_{\bm{w}^*, \mathcal{K}, \mathcal{A}} - \bm{H}_{\bm{w}', \mathcal{K}, \mathcal{A}}||_2  &= ||\theta \bm{H}_{\bm{w}^*, \mathcal{K}} + (1-\theta)\bm{H}_{\bm{w}^*, \mathcal{A}} - \theta \bm{H}_{\bm{w}', \mathcal{K}} - (1-\theta)\bm{H}_{\bm{w}', \mathcal{A}}||_2 \\ 
&\leq \theta ||\bm{H}_{\bm{w}^*, \mathcal{K}} - \theta \bm{H}_{\bm{w}', \mathcal{K}}||_2 + (1-\theta)||\bm{H}_{\bm{w}^*, \mathcal{A}} - \bm{H}_{\bm{w}', \mathcal{A}}||_2 \\
&= \theta F_{\mathcal{K}}||\bm{w}^* - \bm{w}'||_2 + (1-\theta)F_{\mathcal{A}}||\bm{w}^* - \bm{w}'||_2, \; \text{by \cref{lemma:lipschitzness_of_hessians_and_grads}} \\
&= \theta t F_{\mathcal{K}}||\bm{w}^* - \bm{\tilde{w}}^*||_2 + (1-\theta) tF_{\mathcal{A}}||\bm{w}^* - \bm{\tilde{w}}^*||_2,
\label{align:form_of_m}
\end{align}

This yields that: 

\begin{align}
    ||\bm{\tilde{w}} - \bm{\tilde{w}}^*||_2 &\leq ||\bm{H}_{\bm{w}^*, \mathcal{K}, \mathcal{A}}^{-1}||_2 \int_0^1 ||\bm{H}_{\bm{w}^*, \mathcal{K}, \mathcal{A}} - \bm{H}_{\bm{w}^* + t(\bm{\tilde{w}}^* - \bm{w}^*), \mathcal{K}, \mathcal{A}}||_2 ||\bm{w}^* - \bm{\tilde{w}}^*||_2 dt \\
    &\leq  ||\bm{H}_{\bm{w}^*, \mathcal{K}, \mathcal{A}}^{-1}||_2\int_0^1(\theta t F_{\mathcal{K}} + (1-\theta) tF_{\mathcal{A}})||\bm{w}^* - \bm{\tilde{w}}^*||_2^2 \\ 
    &= \frac{\theta F_{\mathcal{K}} + (1-\theta)F_{\mathcal{A}}}{2}||\bm{H}^{-1}_{\bm{w}^*, 
        \mathcal{K}, \mathcal{A}}||_2 ||\bm{w}^* - \bm{\tilde{w}}^*||_2^2
\end{align} as desired. 

\end{proof}

\subsection{Proof of Proposition \labelcref{prop:bound_after_cvx_approx_and_C}}\label{section:proof_of_prop_bound_after_cvx_approx_and_C}

\begin{proof}
    See the proof of theorem 3.4 in \citep{zhang2024certifiedunlearningDNN}, noting that in our setting $M = F = \theta F_{\mathcal{K}} + (1-\theta)F_{\mathcal{A}}$ by \cref{align:form_of_m}.
\end{proof}

\subsection{Proof of Theorem \labelcref{thm:asymptotically_unbiased_estimator}}\label{section:proof_of_asymptotically_unbiased_estimator}

\begin{proof}
    First, we have that: 

    \begin{align}
        \mathbb{E}[\bm{\tilde{H}}_{t, \lambda}^{-1}] &= \mathbb{E}[\bm{I} + \bm{\tilde{H}}^{-1}_{t-1, \lambda} - \frac{1}{J}\bm{H}_{t, \lambda} \bm{\tilde{H}}_{t-1, \lambda}^{-1}], \; \text{by definition} \\ 
        &= \bm{I} + \mathbb{E}[\bm{\tilde{H}}_{t - 1, \lambda}^{-1}] - \frac{1}{J}\mathbb{E}[\bm{H}_{t, \lambda}\bm{\tilde{H}}_{t - 1, \lambda}], \; \text{linearity of expectation} \\ 
        &= \bm{I} + \mathbb{E}[\bm{\tilde{H}}_{t - 1, \lambda}^{-1}] - \frac{1}{J}\mathbb{E}[\bm{H}_{t, \lambda}]\mathbb{E}[\bm{\tilde{H}}_{t - 1, \lambda}], \text{i.i.d. samples} \\ 
        &= \bm{I} + \mathbb{E}[\bm{\tilde{H}}_{t - 1, \lambda}^{-1}] - \frac{\bm{H}_{\bm{w}^*, \mathcal{K}, \mathcal{A}} + \lambda\bm{I}}{J}\mathbb{E}[\bm{\tilde{H}}_{t - 1, \lambda}], \; \text{by  \cref{lemma:unbiased_estimator_of_H_w_star}}
    \end{align}

    Denote $\bm{H}_* := \bm{H}_{\bm{w}^*, \mathcal{K}, \mathcal{A}}$ and $\bm{E}_t := \mathbb{E}[\bm{\tilde{H}}_{t, \lambda}^{-1}]$. We thus have that: 

    \begin{align}
        \bm{E}_{t} &= \bm{I} + \bm{E}_{t-1} - \frac{\bm{H}_*}{J}\bm{E}_{t-1} \\ 
        &= \bm{I}  +\bm{E}_{t-1}(\bm{I} - \frac{\bm{H}_*}{J}) \\
        &= \bm{I} + (\bm{I} - \bm{M})\bm{E}_{t-1}, \; \text{letting $\bm{M} := \frac{\bm{H}_*}{J}$}
    \end{align}

    We then know that, by assumption, $\lambda > ||\bm{H}_*||_2$, where $\bm{H}_*$ is a symmetric Hessian by \cref{lemma:symmetricity_of_hessians}; as such, $\bm{H}_* + \lambda\bm{I}$ is positive definite and has all positive eigenvalues. We also know that $||\bm{H}_*||_2 < J \implies ||\bm{M}||_2 < 1$, so we have that $0 < \lambda_i(\bm{M}) < 1$ for all eigenvalues $\lambda_i$. Furthermore, $\bm{I} - \bm{M}$ has eigenvalues $1 - \lambda_i(\bm{M})$, so we have that $0 <\lambda_i(\bm{I} - \bm{M}) < 1$, so $||\bm{I} - \bm{M}||_2 < 1$, since $\bm{I} - \bm{M}$ is symmetric. Since $\bm{I} - \bm{M}$ has spectral radius less than 1, the Neumann series $\sum_{k = 0}^\infty (\bm{I} - \bm{M})^k$ converges. \citep{mayer1985convergence}. Thus, the Neumann series is Cauchy.

    Fix $\varepsilon > 0$. Let $s_n = \sum_{k = 0}^n (\bm{I} - \bm{M})^k$. We know that $\exists N \in \mathbb{N}$ s.t. $m > n \geq N \implies ||s_m - s_n||_2 = ||\sum_{k = n+1}^m (\bm{I} - \bm{M})^k||_2 < \varepsilon$. For $m > n \geq N$, we have that $||\bm{E}_m - \bm{E}_n||_2 = ||\sum_{k = n+1}^m (\bm{I} - \bm{M})^k||_2 < \varepsilon$. As such, $\{\bm{E}_n\}$ is Cauchy; since it is real, it converges. As such, $\bm{E}_\infty = \lim_{t \to \infty} \bm{E}_n$ exists. 

    Taking limits on both sides, we then have: 
    
    \begin{align}
         \mathbb{E}[\bm{\tilde{H}}_{\infty, \lambda}^{-1}] = \bm{I} + \mathbb{E}[\bm{\tilde{H}}_{\infty, \lambda}^{-1}] + \frac{\bm{H}_{\bm{w}^*, \mathcal{K}, \mathcal{A}} + \lambda\bm{I}}{J}\mathbb{E}[\bm{\tilde{H}}_{\infty, \lambda}] \\ \iff  \mathbb{E}[\frac{\bm{\tilde{H}}_{\infty, \lambda}^{-1}}{J}] = (\bm{H}_{\bm{w}^*, \mathcal{K}, \mathcal{A}} + \lambda\bm{I})^{-1}
    \end{align} 
    
    rearranging using linearity of expectation and noting that $\lambda$ was chosen such that $\bm{H}_{\bm{w}^*, \mathcal{K}, \mathcal{A}} + \lambda\bm{I}$ is invertible, as desired. 
\end{proof}

\subsection{Proof of Theorem \labelcref{thm:bound_with_hessian_estimator}}\label{section:proof_of_bound_with_hessian_estimator}

This follows similarly to theorem 3.6 and proposition 4.1 in \citet{zhang2024certifiedunlearningDNN}, noting that we apply \cref{lemma:agrawal_et_al_bound} instead of applying lemma 3.6 from \citet{agarwal2016second}. Furthermore, note that $L = \theta P_{\mathcal{K}} + (1-\theta)P_{\mathcal{A}}$ in our setting. For completeness, we provide the full proof below.

\begin{proof}
    \begin{align}
        \bm{\tilde{w}} - \bm{\tilde{w}^*} &= \bm{w}^* - \frac{\bm{\tilde{H}}_{n, \lambda}^{-1}}{H}\nabla_{\bm{w}^*, \mathcal{K}, \mathcal{A}} - \bm{\tilde{w}^*} \\
        &= \bm{w^*} - \bm{\tilde{w}^*} - \frac{\bm{\tilde{H}}_{n, \lambda}^{-1}}{H}(\nabla_{\bm{w^*}, \mathcal{K}, \mathcal{A}} - \nabla_{\bm{\tilde{w}^*}, \mathcal{K}, \mathcal{A}}) - \frac{\bm{\tilde{H}}_{n, \lambda}^{-1}}{H}\nabla_{\bm{\tilde{w}^*}, \mathcal{K}, \mathcal{A}}
    \end{align}

    By the triangle inequality, this yields: 

    \begin{equation}
        ||\bm{\tilde{w}} - \bm{\tilde{w}^*}||_2 \leq ||\bm{w^*} - \bm{\tilde{w}^*} - \frac{\bm{\tilde{H}}_{n, \lambda}^{-1}}{H}(\nabla_{\bm{w^*}, \mathcal{K}, \mathcal{A}} - \nabla_{\bm{\tilde{w}^*}, \mathcal{K}, \mathcal{A}})||_2 +|| \frac{\bm{\tilde{H}}_{n, \lambda}^{-1}}{H}\nabla_{\bm{\tilde{w}^*}, \mathcal{K}, \mathcal{A}}||_2
        \label{eq:initial_bound_4_7}
    \end{equation}

    The first term in \cref{eq:initial_bound_4_7} can be bounded by the triangle inequality as: 
    \begin{align}
        ||\bm{w^*} - \bm{\tilde{w}^*} - \frac{\bm{\tilde{H}}_{n, \lambda}^{-1}}{H}(\nabla_{\bm{w^*}, \mathcal{K}, \mathcal{A}} - \nabla_{\bm{\tilde{w}^*}, \mathcal{K}, \mathcal{A}})||_2 \\ = ||\bm{w^*} - \bm{\tilde{w}^*} - ((\bm{H}_{\bm{w}^*, \mathcal{K}, \mathcal{A}} + \lambda \bm{I})^{-1} + \frac{\bm{\tilde{H}}_{n, \lambda}^{-1}}{H} - (\bm{H}_{\bm{w}^*, \mathcal{K}, \mathcal{A}} + \lambda \bm{I})^{-1})(\nabla_{\bm{w^*}, \mathcal{K}, \mathcal{A}} - \nabla_{\bm{\tilde{w}^*}, \mathcal{K}, \mathcal{A}})||_2 \\ 
        \leq ||\bm{w^*} - \bm{\tilde{w}^*} - (\bm{H}_{\bm{w}^*, \mathcal{K}, \mathcal{A}} + \lambda\bm{I})^{-1}(\nabla_{\bm{w^*}, \mathcal{K}, \mathcal{A}} - \nabla_{\bm{\tilde{w}^*}, \mathcal{K}, \mathcal{A}})||_2 \\ + ||((\bm{H}_{\bm{w}^*, \mathcal{K}, \mathcal{A}} + \lambda\bm{I})^{-1} - \frac{\bm{\tilde{H}}_{n, \lambda}^{-1}}{H})(\nabla_{\bm{w^*}, \mathcal{K}, \mathcal{A}} - \nabla_{\bm{\tilde{w}^*}, \mathcal{K}, \mathcal{A}})||_2
        \label{eq:star_1}
    \end{align}

    In the setting of \cref{prop:bound_after_cvx_approx_and_C}, we have that $\bm{\tilde{w}} - \bm{\tilde{w}^*} = \bm{w^*} - \bm{\tilde{w}^*} - (\bm{H}_{\bm{w}^*, \mathcal{K}, \mathcal{A}} + \lambda\bm{I})^{-1}(\nabla_{\bm{w^*}, \mathcal{K}, \mathcal{A}} - \nabla_{\bm{\tilde{w}^*}, \mathcal{K}, \mathcal{A}})$. Hence, by \cref{prop:bound_after_cvx_approx_and_C}, we have that:

    \begin{align}
        ||\bm{w^*} - \bm{\tilde{w}^*} - (\bm{H}_{\bm{w}^*, \mathcal{K}, \mathcal{A}} + \lambda\bm{I})^{-1}(\nabla_{\bm{w^*}, \mathcal{K}, \mathcal{A}} - \nabla_{\bm{\tilde{w}^*}, \mathcal{K}, \mathcal{A}})||_2  \\ \leq \frac{2C((\theta F_{\mathcal{K}} + (1-\theta)F_{\mathcal{A}})C + \lambda)}{\lambda + \lambda_{\min}}
    \end{align}

    Furthermore, we have: 

    \begin{align}
         ||((\bm{H}_{\bm{w}^*, \mathcal{K}, \mathcal{A}} + \lambda\bm{I})^{-1} - \frac{\bm{\tilde{H}}_{n, \lambda}^{-1}}{H})(\nabla_{\bm{w^*}, \mathcal{K}, \mathcal{A}} - \nabla_{\bm{\tilde{w}^*}, \mathcal{K}, \mathcal{A}})||_2 \\
         \leq ||((\bm{H}_{\bm{w}^*, \mathcal{K}, \mathcal{A}} + \lambda\bm{I})^{-1} - \frac{\bm{\tilde{H}}_{n, \lambda}^{-1}}{H})||_2||(\nabla_{\bm{w^*}, \mathcal{K}, \mathcal{A}} - \nabla_{\bm{\tilde{w}^*}, \mathcal{K}, \mathcal{A}})||_2, \text{property of op norm} \\ 
        \leq (16\frac{B}{\zeta_{\min}}\sqrt{\frac{\ln\frac{d}{\rho}}{b}} + \frac{1}{16})||(\nabla_{\bm{w^*}, \mathcal{K}, \mathcal{A}} - \nabla_{\bm{\tilde{w}^*}, \mathcal{K}, \mathcal{A}})||_2, \; \text{\cref{lemma:agrawal_et_al_bound}} \\ 
        \leq (16\frac{B}{\zeta_{\min}}\sqrt{\frac{\ln\frac{d}{\rho}}{b}} + \frac{1}{16})2C(\theta P_{\mathcal{K}} + (1-\theta)P_{\mathcal{A}}, \; \text{\cref{lemma:lipschitzness_of_hessians_and_grads}}
    \end{align}

    with probability at least 1 - $\rho$. Incorporating this into equation \cref{eq:star_1}, we have that: 

    \begin{align}
        ||\bm{w^*} - \bm{\tilde{w}^*} - \frac{\bm{\tilde{H}}_{n, \lambda}^{-1}}{H}(\nabla_{\bm{w^*}, \mathcal{K}, \mathcal{A}} - \nabla_{\bm{\tilde{w}^*}, \mathcal{K}, \mathcal{A}})||_2 \leq \frac{2C((\theta F_{\mathcal{K}} + (1-\theta)F_{\mathcal{A}})C + \lambda)}{\lambda + \lambda_{\min}} \\ 
        + (32\frac{B}{\zeta_{\min}}\sqrt{\frac{\ln\frac{d}{\rho}}{b}} + \frac{1}{8})C(\theta P_{\mathcal{K}} + (1-\theta)P_{\mathcal{A}})
    \end{align}

    It then suffices to bound the second term in \cref{eq:initial_bound_4_7}. We have that: 
 \begin{align}
\bigl\|\tfrac{\bm{\tilde H}_{n,\lambda}^{-1}}{H}\,
       \nabla_{\bm{\tilde w}^*,\mathcal K,\mathcal A}\bigr\|_2
&=
\Bigl\|\bigl[(\bm H_{w^*,\mathcal K,\mathcal A}+\lambda \bm I)^{-1}
       -(\bm H_{w^*,\mathcal K,\mathcal A}+\lambda \bm I)^{-1}
\nonumber\\[-0.5ex]
&\qquad\quad
       +\,\tfrac{\bm{\tilde H}_{n,\lambda}^{-1}}{H}\bigr]
       \,\nabla_{\bm{\tilde w}^*,\mathcal K,\mathcal A}\Bigr\|_2
\label{eq:step1}\\
&=
\Bigl\|(\bm H_{w^*,\mathcal K,\mathcal A}+\lambda \bm I)^{-1}
       \,\nabla_{\bm{\tilde w}^*,\mathcal K,\mathcal A}
\nonumber\\[-0.5ex]
&\qquad
   +\bigl(\tfrac{\bm{\tilde H}_{n,\lambda}^{-1}}{H}
          -(\bm H_{w^*,\mathcal K,\mathcal A}+\lambda \bm I)^{-1}\bigr)
       \,\nabla_{\bm{\tilde w}^*,\mathcal K,\mathcal A}\Bigr\|_2
\label{eq:step2}\\
&\le
\|(\bm H_{w^*,\mathcal K,\mathcal A}+\lambda \bm I)^{-1}\|_2
            \,\|\nabla_{\bm{\tilde w}^*,\mathcal K,\mathcal A}\|_2
\nonumber\\[-0.5ex]
&\quad
+\bigl\|\tfrac{\bm{\tilde H}_{n,\lambda}^{-1}}{H}
                 -(\bm H_{w^*,\mathcal K,\mathcal A}+\lambda \bm I)^{-1}\bigr\|_2
            \,\|\nabla_{\bm{\tilde w}^*,\mathcal K,\mathcal A}\|_2
\label{eq:step3}\\
&\le \frac{G}{\lambda + \lambda_{\min}}
     +\Bigl(16\frac{B}{\zeta_{\min}}
                \sqrt{\tfrac{\ln(d/\rho)}{b}}
          + \tfrac1{16}\Bigr)\,G.
\label{eq:final}
\end{align}

    by definition of $\lambda$, \cref{lemma:agrawal_et_al_bound}, and that $||\nabla_{\bm{\tilde{w}^*}, \mathcal{K}, \mathcal{A}}||_2 \leq G$. 

    Incorporating the above into \cref{eq:initial_bound_4_7}, this yields that: 

   \begin{align}
\|\bm{\tilde w} - \bm{\tilde w}^*\|_2
&\le
\frac{2C\bigl((\theta F_{\mathcal{K}} + (1-\theta)F_{\mathcal{A}})\,C + \lambda\bigr)}
     {\lambda + \lambda_{\min}}
\label{eq:bound-part1} \\[-0.5ex]
&\quad
+ \Bigl(32\frac{B}{\zeta_{\min}}
           \sqrt{\frac{\ln(d/\rho)}{b}}
       + \tfrac18\Bigr)\,
  C\Bigl(\theta P_{\mathcal{K}} + (1-\theta)P_{\mathcal{A}}
     + \frac{G}{\lambda + \lambda_{\min}}
\nonumber\\[-0.5ex]
&\qquad\qquad\quad
     + \Bigl(16\frac{B}{\zeta_{\min}}
                \sqrt{\frac{\ln(d/\rho)}{b}}
            + \tfrac1{16}\Bigr)\,G\Bigr)
\label{eq:bound-part1b} \\[1ex]
&=
\frac{2C\bigl((\theta F_{\mathcal{K}} + (1-\theta)F_{\mathcal{A}})\,C + \mu\bigr) \;+\; G}
     {\mu + \mu_{\min}}
\label{eq:bound-part2} \\[-0.5ex]
&\quad
+ \Bigl(16\frac{B}{\zeta_{\min}}
           \sqrt{\frac{\ln(d/\rho)}{b}}
       + \tfrac1{16}\Bigr)\,
  \bigl(2C(\theta P_{\mathcal{K}} + (1-\theta)P_{\mathcal{A}} + G)\bigr).
\nonumber
\end{align}

    as desired.

\end{proof}

\subsection{Proof of Proposition \labelcref{prop:bound_sequential}}\label{section:proof_of_sequential_bound}

\begin{proof}
    By the same token as proposition 4.2 in \citet{zhang2024certifiedunlearningDNN}, follow the proof of \cref{thm:bound_with_hessian_estimator}. 
\end{proof}

\subsection{Proof of Proposition \labelcref{prop:bound_on_pop_distance_between_output_and_unif}} \label{section:proof_of_hard_constraint_bound}

\begin{proof}
    Fix any sampled $\mathcal{D}$. Since $\mathcal{M}_\theta(\mathcal{D})$ is taken to be the global risk minimizer, we have that: 

    \begin{align}
        \theta \mathcal{L}_\mathcal{K}(\mathcal{M}_{\theta}(\mathcal{D}), \mathcal{D}_f) + (1-\theta)\mathcal{L}_\mathcal{A}(\mathcal{M}_\theta(\mathcal{D}), \mathcal{D}_r) \\ 
        \leq \theta \mathcal{L}_\mathcal{K}(\bm{w}, \mathcal{D}_f) + (1-\theta)\mathcal{L}_\mathcal{A}(\bm{w}, \mathcal{D}_r) \; \forall \bm{w} \in \mathcal{W}
        \label{eq:pop_risk_min_bound}
    \end{align}

    subtracting $\frac{\lambda}{2}||\bm{w}||_2$ from both sides. 

    Let $\bm{w}_U$ be the parameter that results in a parameterized model $f_{\bm{w}_U}$ which outputs a uniform distribution; by \cref{prop:uniform_learner_exists}, such a parameter exists. We then have that: 

    \begin{align}
        \mathcal{L}_\mathcal{K}(\bm{w}_U, \mathcal{D}_f) = \sum_{i = 1}^{|\mathcal{D}_f|} D_{KL}(U[0, |\mathcal{Y}|]) || U[0, |\mathcal{Y}|] = 0
    \end{align}

    and 

    \begin{align}
        \mathcal{L}_\mathcal{A}(\bm{w}_U, \mathcal{D}_r) = \sum_{i = 1}^{|\mathcal{D}_r|} \mathbb{H}_{CE}(\bm{y}^{(i)}, U[0, |\mathcal{Y}|]) = - \sum_{i = 1}^{|\mathcal{D}_r|} \sum_{j = 1}^{|\mathcal{Y}|} \bm{y}^{(i)}_j \ln \frac{1}{|\mathcal{Y}|} = |\mathcal{D}_r| \ln |\mathcal{Y}|
    \end{align}

    where $\bm{y}$ is a one hot vector of length $\mathcal{Y}$ such that for $\bm{y}^{(i)}_j$, $j = 1, ..., |\mathcal{Y}|$,

    \begin{equation}
        \bm{y}_j^{(i)} = \begin{cases}
      1 \quad \text{instance $i$ is labeled class $j$}   \\ 
     0 \quad   \text{instance $i$ is not labeled class $j$}
    \end{cases}
    \end{equation}

    Incorporating the above into \cref{eq:pop_risk_min_bound} yields: 

    \begin{align}
        \theta \mathcal{L}_\mathcal{K}(\mathcal{M}_{\theta}(\mathcal{D}), \mathcal{D}_f) + (1-\theta)\mathcal{L}_\mathcal{A}(\mathcal{M}_\theta(\mathcal{D}), \mathcal{D}_r) 
        &\leq \theta \mathcal{L}_\mathcal{K}(\bm{w}_U, \mathcal{D}_f) + (1-\theta)\mathcal{L}_\mathcal{A}(\bm{w}_U, \mathcal{D}_r)  \\ 
        &\leq \theta(0) + (1-\theta)|\mathcal{D}_r|\ln|\mathcal{Y}| \\ 
        &= |\mathcal{D}_r|(1-\theta)\ln|\mathcal{Y}|
    \end{align}

    This then yields that: 

    \begin{align}
        \mathcal{L}_\mathcal{K}(\mathcal{M}_{\theta}(\mathcal{D}), \mathcal{D}_f) &\leq \frac{1 - \theta}{\theta}(|\mathcal{D}_r|\ln|\mathcal{Y}| - \mathcal{L}_\mathcal{A}(\mathcal{M}_\theta(\mathcal{D}), \mathcal{D}_r) \\
        &\leq \frac{1-\theta}{\theta}|\mathcal{D}_r|\ln|\mathcal{Y}|
    \end{align}

    since the cross entropy is nonnegative, yielding that $-\mathcal{L}_\mathcal{A}(\mathcal{M}_\theta(\mathcal{D}), \mathcal{D}_r) \leq 0$. 
    
    Then, we have:  

    \begin{align}
        ||f_{\mathcal{M}_\theta(\mathcal{D})}(\mathcal{D}_f) - U[0, |\mathcal{Y}|]||_\infty &\leq   ||f_{\mathcal{M}_\theta(\mathcal{D})}(\mathcal{D}_f) - U[0, |\mathcal{Y}|]||_1 \\
        &\leq 2TV(f_{\mathcal{M}_\theta(\mathcal{D})}(\mathcal{D}_f) - U[0, |\mathcal{Y}|]) \\ 
        &\leq 2 \sqrt{\frac{1}{2}D_{KL}(f_{\mathcal{M}_\theta(\mathcal{D})},||U[0, |\mathcal{Y}|])}, \; \text{Pinsker's inequality \citep{pinsker1964information}} \\ 
        &= \sqrt{2D_{KL}(f_{\mathcal{M}_\theta(\mathcal{D})},||U[0, |\mathcal{Y}|])} \\
        &= \sqrt{2\mathcal{L}_\mathcal{K}(\mathcal{M}_{\theta}(\mathcal{D}), \mathcal{D}_f)} \\ 
        &\leq \sqrt{2|\mathcal{D}_r|(\frac{1-\theta}{\theta})\ln|\mathcal{Y}|}
    \end{align}

    by the above bound on $\mathcal{L}_\mathcal{K}$, as desired. 
\end{proof}

\subsection{Proof of Corollary \labelcref{corollary:lambda_large} }\label{section:proof_of_lambda_large}

\begin{proof}
    To have \cref{eq:hard_constraint}, by \cref{prop:bound_on_pop_distance_between_output_and_unif}, it suffices to solve for $\theta$ in the bound obtained. This results in: 
    
    \begin{align}
        \sqrt{2(\frac{1-\theta}{\theta})|\mathcal{D}_r|\ln|\mathcal{Y}|} \leq \varepsilon &\iff \frac{1-\theta}{\theta}|\mathcal{D}_r|\ln|\mathcal{Y} \leq \frac{\varepsilon^2}{2} \\
        &\iff \frac{|\mathcal{D}_r|\ln|\mathcal{Y}|}{\theta} - \frac{|\mathcal{D}_r|\ln|\mathcal{Y}|\theta}{\theta} \leq \frac{\varepsilon^2}{2} \\ 
        &\iff \frac{|\mathcal{D}_r|\ln|\mathcal{Y}|}{\theta} \leq \frac{\varepsilon^2}{2} + |\mathcal{D}_r|\ln|\mathcal{Y}| = \frac{\varepsilon^2 + 2|\mathcal{D}_r|\ln|\mathcal{Y}|}{2} \\ 
        &\iff \frac{\theta}{|\mathcal{D}_r| \ln|\mathcal{Y}|} \geq \frac{2}{\varepsilon^2 + 2|\mathcal{D}_r|\ln|\mathcal{Y}|} \\ 
        &\iff \theta \geq \frac{2|\mathcal{D}_r| \ln|\mathcal{Y}}{\varepsilon^2 + 2|\mathcal{D}_r|\ln|\mathcal{Y}|}
    \end{align}

    as desired.
    
\end{proof}

\subsection{Proof of Theorem \labelcref{thm:retain_accuracy_bound}} \label{section:proof_of_retain_accuracy_bound}

First, before we prove \cref{thm:retain_accuracy_bound}, we note that we can use \cref{lemma:descent_lemma} and that $||\bm{w}|| \leq 2$ to obtain a simple bound. Let: 

\begin{align}
    |\alpha^* - \alpha(\theta)| &= |\mathcal{L}_\mathcal{A}(\mathcal{A}(\mathcal{D}_r), \mathcal{D}_r) - \mathcal{L}_\mathcal{A}(\mathcal{M}_\theta(\mathcal{D}), \mathcal{D}_r)| \\
    &\leq \frac{P_{\mathcal{A}}}{2}|| \mathcal{M}_\theta(\mathcal{D}) - \mathcal{A}(\mathcal{D}_r) ||_2^2 + ||\nabla_{\mathcal{A}(\mathcal{D}_r), \mathcal{A}}||_2 || \mathcal{M}_\theta(\mathcal{D}) - \mathcal{A}(\mathcal{D}_r) ||_2 \\
    &\leq \frac{C^2P_{\mathcal{A}}}{2} + \lambda C^2
\end{align}

after applying the triangle inequality and rearranging the first order condition on $\mathcal{A}(\mathcal{D}_r)$. 

However, this bound is vacuous and not tight; it does not incorporate any information about $\theta$ or most of the constants that appear in \cref{assumption1} and \cref{assumption2}. Given this, we seek to construct a tighter, non-vacuous bound. We first restate the proof without any asymptotic characterizations: 

\begin{thm} 
     Suppose \cref{assumption1,assumption2} hold, and let $P_{\mathcal{K}}, P_{\mathcal{K}}, F_{\mathcal{K}}, F_{\mathcal{A}}$ be as defined in \cref{assumption1,assumption2}. Let $\alpha^* := \mathcal{L}_\mathcal{A}(\mathcal{A}(\mathcal{D}_r), \mathcal{D}_r)$ be the locally optimal (empirical) retain loss, achieved by $\mathcal{M}_\theta(\mathcal{D})$ when $\theta = 0$. Let $\alpha(\theta) := \mathcal{L}_{\mathcal{A}}(\mathcal{M}_\theta(\mathcal{D}), \mathcal{D}_r)$ be the locally optimal retain loss obtained by $\mathcal{M}_\theta(\mathcal{D})$ when $\theta \in (0,1)$. Suppose all weights used throughout are bounded by $||\bm{w}||_2 \leq C$. Additionally, denote by $F := \theta M_\mathcal{K} + (1-\theta)F_{\mathcal{A}}$ and $P := \theta P_{\mathcal{K}} + (1-\theta)P_{\mathcal{A}}$. Consider regularization coefficient $\lambda \geq L + 2\theta CF + \sqrt{2\theta C F (P + 2\theta C F + 8P_{\mathcal{K}}})$. Then, we have the following bound: 

     \begin{align}
         |\alpha^* - \alpha(\theta)| &\leq \frac{P_{\mathcal{K}}}{2}(\frac{\lambda - P - \sqrt{(\lambda - P)^2 - 4\theta C F (2P_{\mathcal{K}} + \lambda)}}{2F})^2 + \\
         & \lambda C(\frac{\lambda - P - \sqrt{(\lambda - P)^2 - 4\theta C F (2P_{\mathcal{K}} + \lambda)}}{2F}).
         \label{eq:retain_accuracy_bound_equation}
     \end{align}
     \vspace{-4mm}
     \label{thm:retain_accuracy_bound_appendix}
\end{thm}

\begin{proof}
    First, when $\theta = 0$, we have that: 

    \begin{equation}
        \bm{w}_{\alpha^*} := \arg\min_{\bm{w} \in \mathcal{W}, ||\bm{w}||_2 \leq C} \mathcal{L}_\mathcal{A}(\bm{w}, \mathcal{D}_r) + \frac{\lambda}{2}||\bm{w}||_2^2
    \end{equation}

    which yields the first order condition: 

    \begin{equation}
        \nabla_{\bm{w}_{\alpha^*}, \mathcal{A}} + \lambda \bm{w}_{\alpha^*} = 0 
    \end{equation}

    which, upon multiplying $1-\theta$ on both sides, yields: 

    \begin{equation}
        (1-\theta)\nabla_{\bm{w}_{\alpha^*}, \mathcal{A}} + (1-\theta)\lambda \bm{w}_{\alpha^*} = 0 
        \label{eq:first_order_condition_1}
    \end{equation}

    Then, when $\theta \in (0,1)$, we have: 

    \begin{equation}
        \bm{w}_{\alpha(\theta)} := \arg\min_{\bm{w} \in \mathcal{W}, ||\bm{w}||_2 \leq C} \theta \mathcal{L}_\mathcal{K}(\bm{w}, \mathcal{D}_f) + (1-\theta)\mathcal{L}_\mathcal{A}(\bm{w}, \mathcal{D}_r) + \frac{\lambda}{2}||\bm{w}||_2^2
    \end{equation}

    which yields the first order condition: 

    \begin{equation}
        \theta \nabla_{\bm{w}_{\alpha(\theta)}, \mathcal{K}}  + (1-\theta)\nabla_{\bm{w}_{\alpha(\theta)}, \mathcal{A}} + \lambda \bm{w}_{\alpha(\theta)} = 0
        \label{eq:first_order_condition_2}
    \end{equation}

    Subtracting \cref{eq:first_order_condition_1} from \cref{eq:first_order_condition_2} yields: 

    \begin{equation}
       \theta \nabla_{\bm{w}_{\alpha(\theta)}, \mathcal{K}}  + (1-\theta)\nabla_{\bm{w}_{\alpha(\theta)}, \mathcal{A}} + \lambda \bm{w}_{\alpha(\theta)} -  (1-\theta)\nabla_{\bm{w}_{\alpha^*}, \mathcal{A}} -(1-\theta)\lambda \bm{w}_{\alpha^*} = 0
    \end{equation}

    which simplifies to: 

    \begin{equation}
        \theta \nabla_{\bm{w}_{\alpha(\theta)}, \mathcal{K}}  + (1-\theta)(\nabla_{\bm{w}_{\alpha(\theta)}, \mathcal{A}} - \nabla_{\bm{w}_{\alpha^*}, \mathcal{A}}) + \lambda(\bm{w}_{\alpha(\theta)} - \bm{w}_{\alpha^*}) = -\theta\lambda \bm{w}_{\alpha^*}
        \label{eq:first_form_to_bound}
    \end{equation}

    The fundamental theorem of calculus then yields: 

    \begin{equation}
        \int_0^1 \bm{H}_{\bm{w}_{\alpha^*} + t(\bm{w}_{\alpha(\theta)} + \bm{w}_{\alpha^*}), \mathcal{A}}(\bm{w}_{\alpha(\theta} - \bm{w}_{\alpha^*}) dt = \nabla_{\bm{w}_{\alpha(\theta)}, \mathcal{A}} - \nabla_{\bm{w}_{\alpha^*}, \mathcal{A}}
        \label{eq:ftc_A}
    \end{equation}

    and 

    \begin{equation}
        \int_0^1 \bm{H}_{\bm{w}_{\alpha^*} + t(\bm{w}_{\alpha(\theta)} + \bm{w}_{\alpha^*}), \mathcal{K}}(\bm{w}_{\alpha(\theta)} - \bm{w}_{\alpha^*}) dt = \nabla_{\bm{w}_{\alpha(\theta)}, \mathcal{K}} - \nabla_{\bm{w}_{\alpha^*}, \mathcal{K}}
        \label{eq:ftc_K}
    \end{equation}

    We thus denote: 

    \begin{align}
        \bm{\bar{H}}_\mathcal{K} := \int_0^1 \bm{H}_{\bm{w}_{\alpha^*} + t(\bm{w}_{\alpha(\theta)} + \bm{w}_{\alpha^*}), \mathcal{K}} dt \\ 
        \bm{\bar{H}}_\mathcal{A} := \int_0^1 \bm{H}_{\bm{w}_{\alpha^*} + t(\bm{w}_{\alpha(\theta)} + \bm{w}_{\alpha^*}), \mathcal{A}} dt \\
        \Delta \bm{w} := \bm{w}_{\alpha(\theta)} - \bm{w}_{\alpha^*}
    \end{align}

    Incorporating \cref{eq:ftc_A} and \cref{eq:ftc_K} into \cref{eq:first_form_to_bound} then yields: 

    \begin{align}
&\theta(\nabla_{\bm{w}_{\alpha^*}, \mathcal{K}} + \bm{\bar{H}}_{\mathcal{K}}\Delta \bm{w}) 
+ (1-\theta)(\nabla_{\bm{w}_{\alpha^*}, \mathcal{A}} + \bm{\bar{H}}_{\mathcal{A}}\Delta \bm{w} 
- \nabla_{\bm{w}_{\alpha^*}, \mathcal{A}}) 
+ \lambda \Delta \bm{w} = - \theta \lambda \bm{w}_{\alpha^*} \nonumber \\
\iff\quad 
&\theta \nabla_{\bm{w}_{\alpha^*}, \mathcal{K}} 
+ \theta \bm{\bar{H}}_{\mathcal{K}} \Delta \bm{w} 
+ (1-\theta)\bm{\bar{H}}_{\mathcal{A}} \Delta \bm{w} 
+ \lambda \Delta \bm{w} 
+ \theta \lambda \bm{w}_{\alpha^*} = 0 \nonumber \\
\iff\quad 
&\big(\theta \bm{\bar{H}}_{\mathcal{K}} 
+ (1-\theta)\bm{\bar{H}}_{\mathcal{A}} 
+ \lambda \bm{I}\big)\Delta \bm{w} 
= - \theta\big(\nabla_{\bm{w}_{\alpha^*}, \mathcal{K}} + \lambda \bm{w}_{\alpha^*}\big) \nonumber \\
\iff\quad 
&\big(\bm{H}_{\bm{w}_{\alpha^*}, \mathcal{K}, \mathcal{A}} + \lambda \bm{I} 
+ \theta(\bm{\bar{H}}_{\mathcal{K}} - \bm{H}_{\bm{w}_{\alpha^*}, \mathcal{K}}) 
+ (1-\theta)(\bm{\bar{H}}_{\mathcal{A}} - \bm{H}_{\bm{w}_{\alpha^*}, \mathcal{A}})\big)\Delta \bm{w} \nonumber \\
&= - \theta\big(\nabla_{\bm{w}_{\alpha^*}, \mathcal{K}} + \lambda \bm{w}_{\alpha^*}\big) \nonumber \\
\iff\quad 
&\big(\bm{H}_{\bm{w}_{\alpha^*}, \mathcal{K}, \mathcal{A}} + \lambda \bm{I}\big)\Delta \bm{w} 
= - \big(\theta(\bm{\bar{H}}_{\mathcal{K}} - \bm{H}_{\bm{w}_{\alpha^*}, \mathcal{K}}) 
+ (1-\theta)(\bm{\bar{H}}_{\mathcal{A}} - \bm{H}_{\bm{w}_{\alpha^*}, \mathcal{A}})\big)\Delta \bm{w} \nonumber \\
&\quad - \theta\big(\nabla_{\bm{w}_{\alpha^*}, \mathcal{K}} + \lambda \bm{w}_{\alpha^*}\big).
\label{eq:to_bound_12}
\end{align}

    Then, note that: 

    \begin{align}
        ||\bm{\bar{H}}_{\mathcal{K}} - \bm{H}_{\bm{w}_{\alpha^*}, \mathcal{K}}||_2 &= ||\int_0^1 \bm{H}_{\bm{w}_{\alpha^*} + t\Delta \bm{w}, \mathcal{K}} dt - \bm{H}_{\bm{w}_{\alpha^*}, \mathcal{K}}||_2 \\ 
        &\leq \int_0^1 ||\bm{H}_{\bm{w}_{\alpha^*} + t \Delta \bm{w}, \mathcal{K}} - \bm{H}_{\bm{w}_{\alpha^*}, \mathcal{K}}||_2 \\ 
        &\leq \frac{F_{\mathcal{K}}}{2} ||\Delta \bm{w}||_2
        \label{eq:use_bound_15}
    \end{align}

    by the same token as in \cref{prop:bound_after_cvx_approx_and_C}. 

    Similarly: 

    \begin{equation}
        ||\bm{\bar{H}}_{\mathcal{A}} - \bm{H}_{\bm{w}_{\alpha^*, \mathcal{A}}}||_2 \leq \frac{F_{\mathcal{A}}}{2} ||\Delta \bm{w}||_2 
        \label{eq:use_bound_16}
    \end{equation}

    Also: 

    \begin{align}
        ||\nabla_{\bm{w}_{\alpha^*, \mathcal{K}}}||_2 &= ||\nabla_{\bm{w}_{\alpha^*, \mathcal{K}}} - \nabla_{\mathcal{K}(\mathcal{D}_f), \mathcal{K}}||_2 \\ 
        &\leq P_{\mathcal{K}}||\bm{w}_{\alpha^*} - \bm{w}_{\mathcal{K}(\mathcal{D}_f)}||_2 \\
        &\leq 2 P_{\mathcal{K}}C
        \label{eq:use_bound_17}
    \end{align}

    by definition of the uniform learner $\mathcal{K}$, \cref{lemma:lipschitzness_of_hessians_and_grads}, and the triangle inequality. 

    Additionally: 

    \begin{equation}
        ||\lambda \bm{w}_{\alpha^*}||_2 \leq \lambda C
        \label{eq:use_bound_18}
    \end{equation}

    By the triangle inequality, incorporating \cref{eq:use_bound_15}, \cref{eq:use_bound_16}, \cref{eq:use_bound_17}, and \cref{eq:use_bound_18} into \cref{eq:to_bound_12}, we have that: 

    \begin{align}
        ||(\bm{H}_{\bm{w}_{\alpha^*}, \mathcal{K}, \mathcal{A}} + \lambda \bm{I}) \Delta \bm{w}||_2 &\leq  (\theta||\bm{\bar{H}}_{\mathcal{K}} - \bm{H}_{\bm{w}_{\alpha^*, \mathcal{K}}}||_2 + (1-\theta)||\bm{\bar{H}}_{\mathcal{A}} - \\ & \bm{H}_{\bm{w}_{\alpha^*}, \mathcal{A}}||_2)||\Delta \bm{w}||_2 + \theta||\nabla_{\bm{w}_{\alpha^*}, \mathcal{K}}||_2 + \theta||\lambda \bm{w}_{\alpha^*}||_2 \\ 
        &\leq (\theta F_{\mathcal{K}} + (1-\theta)F_{\mathcal{A}})||\Delta \bm{w}||_2^2 + \theta C(2P_{\mathcal{K}} + \lambda)
        \label{eq:20_to_incorporate}
    \end{align}

    Note that we have, where $\sigma_{\min}(\cdot)$ denotes the minimum singular value: 

    \begin{align}
        ||(\bm{H}_{\bm{w}_{\alpha^*, \mathcal{K}, \mathcal{A}}} + \lambda \bm{I})\Delta \bm{w}||_2 &\geq \sigma_{\min}(\bm{H}_{\bm{w}_{\alpha^*, \mathcal{K}, \mathcal{A}}} + \lambda \bm{I})||\Delta \bm{w}||_2 \; \text{by property of op. norm} \\ 
        &= \lambda_{\min}(\bm{H}_{\bm{w}_{\alpha^*, \mathcal{K}, \mathcal{A}}} + \lambda \bm{I})||\Delta \bm{w}||_2 \; \text{by \cref{lemma:symmetricity_of_hessians}}
        \label{eq:23_mu_min}
    \end{align}

    Furthermore, by \cref{lemma:lipschitzness_of_hessians_and_grads}, we have that $||\bm{H}_{\bm{w}_{\alpha^*, \mathcal{K}}}||_2 \leq P_{\mathcal{K}}$ and $||\bm{H}_{\bm{w}_{\alpha^*, \mathcal{A}}}||_2 \leq P_{\mathcal{A}}$, which yields:

    \begin{equation}
        ||\bm{H}_{\bm{w}_{\alpha^*}, \mathcal{K}, \mathcal{A}}||_2 \leq \theta P_{\mathcal{K}} + (1-\theta)P_{\mathcal{A}}
    \end{equation}

    which by \cref{lemma:symmetricity_of_hessians} yields: 

    \begin{align}
        \lambda_{\min}(\bm{H}_{\alpha^*, \mathcal{K}, \mathcal{A}} + \lambda \bm{I}) \in [\lambda - \theta P_{\mathcal{K}} - (1-\theta)P_{\mathcal{A}}, \mu + \theta P_{\mathcal{K}} + (1-\theta)P_{\mathcal{A}}]
    \end{align}
    
    With \cref{eq:23_mu_min}, this yields that: 

    \begin{equation}
        ||(\bm{H}_{\bm{w}_{\alpha^*}, \mathcal{K}, \mathcal{A}} + \lambda \bm{I}) \Delta \bm{w}||_2 \geq (\lambda - \theta P_{\mathcal{K}} - (1-\theta)P_{\mathcal{A}})||\Delta \bm{w}||_2  
    \end{equation}

    Incorporating this into \cref{eq:20_to_incorporate} yields: 

    \begin{equation}
        (\lambda - \theta P_{\mathcal{K}} - (1-\theta)P_{\mathcal{A}})||\Delta \bm{w}||_2 \leq (\theta F_{\mathcal{K}} + (1-\theta)F_{\mathcal{A}})||\Delta \bm{w}||_2^2 + \theta C(2P_{\mathcal{K}} + \lambda)
    \end{equation}

    Simplifying yields the quadratic inequality: 

    \begin{equation}
        (\theta F_{\mathcal{K}} + (1-\theta)F_{\mathcal{A}})||\Delta \bm{w}||_2^2 - (\lambda - \theta P_{\mathcal{K}} - (1-\theta)P_{\mathcal{A}})||\Delta \bm{w}||_2 + \theta C(2P_{\mathcal{K}} + \lambda) \geq 0
    \end{equation}

    This then yields that: 

    \begin{align}
        ||\Delta \bm{w}||_2 \leq \frac{\lambda - P - \sqrt{(\lambda - P)^2 - 4\theta C F (2P_{\mathcal{K}} + \lambda)}}{2F}
        \label{eq:bound_on_delta_w_to_incorporate}
    \end{align}

    This is only valid when: 
    \begin{align}
        (\lambda - P)^2 - 4\theta C F(2P_{\mathcal{K}} + \lambda) \geq 0 \\ 
        \iff \lambda^2 - 2L\lambda + P^2 - 4\theta C 2P_{\mathcal{K}}F - 4\theta C \lambda F \geq 0 \\ 
        \iff \lambda^2 - 2P \lambda - 4 \theta C F \lambda + P^2 - 8 \theta C P_{\mathcal{K}} F \geq 0 \\ 
        \iff \lambda^2 - (2P + 4\theta C F)\lambda + (P^2 - 8\theta  C P_{\mathcal{K}}F) \geq 0 \\
        \impliedby \lambda \geq P + 2\theta C F+\sqrt{2\theta C F(P + 2\theta C F + 8P_{\mathcal{K}}}
    \end{align}

    which holds by assumption. Note that all components of $2\theta C F+\sqrt{2\theta C F(P + 2\theta C F + 8P_{\mathcal{K}}}$ are nonnegative, rendering this valid.  Incorporating \cref{eq:bound_on_delta_w_to_incorporate} into \cref{lemma:descent_lemma} yields the final bound as desired. 
\end{proof}

Then, \cref{thm:retain_accuracy_bound} follows as a corollary of \cref{thm:retain_accuracy_bound_appendix}: 

\begin{proof}
Note that we take care to ensure the bound holds for any choice of $\theta \in [0,1]$. Hence, fix $\theta \in [0,1]$. 

Let
\begin{equation}
a := \lambda - P > 0,\qquad
\varepsilon := 4\theta C F (2P_K+\lambda),\qquad
\Delta := \frac{a-\sqrt{a^2-\varepsilon}}{2F}.
\label{eq:defs}
\end{equation}

Theorem \labelcref{thm:retain_accuracy_bound_appendix} gives the inequality: 

\begin{equation}\label{eq:base-bound}
|\alpha^*-\alpha(\theta)| \le \frac{P_K}{2}\,\Delta^2 + \lambda C\,\Delta.
\end{equation}

By the condition on $\lambda$ in the theorem, the square root is real for every $\theta \in [0,1]$ (i.e.\ $a^2-\varepsilon\ge0$). This yields: 

\begin{equation}
a-\sqrt{a^2-\varepsilon}
= \frac{\varepsilon}{a+\sqrt{a^2-\varepsilon}}.
\label{eq:identity}
\end{equation}

Then, since $a+\sqrt{a^2-\varepsilon}\ge a>0$, \eqref{eq:identity} implies: 
\begin{equation}
a-\sqrt{a^2-\varepsilon} \le \frac{\varepsilon}{a}.
\label{eq:upper-diff}
\end{equation}

Dividing \eqref{eq:upper-diff} by $2F$ yields: 
\begin{equation}
\Delta \le \frac{\varepsilon}{2aF}.
\label{eq:Delta-bound}
\end{equation}

Then, substituting $\varepsilon=4\theta C F(2P_K+\lambda)$ from \eqref{eq:defs} yields: 
\begin{equation}\label{eq:Delta-explicit}
\Delta \le \frac{4\theta C F(2P_K+\lambda)}{2aF}
= \frac{2\theta C(2P_K+\lambda)}{a}.
\end{equation}

We now bound the two terms on the right-hand side of \eqref{eq:base-bound}. Using \eqref{eq:Delta-explicit}, we  have that: 
\begin{align}
\lambda C\,\Delta
&\le \lambda C\cdot\frac{2\theta C(2P_K+\lambda)}{a}
= \frac{2\lambda(2P_K+\lambda)}{a}\;C^2\theta
= \mathcal{O}(\lambda C^2\theta), \label{eq:linear-term} \\[6pt]
\frac{P_K}{2}\Delta^2
&\le \frac{P_K}{2}\left(\frac{2\theta C(2P_K+\lambda)}{a}\right)^2
= \frac{2P_K(2P_K+\lambda)^2}{a^2}\;C^2\theta^2
= \mathcal{O}(C^2\theta^2).
\label{eq:quadratic-term}
\end{align}

Combining \eqref{eq:base-bound}, \eqref{eq:linear-term} and \eqref{eq:quadratic-term} and absorbing constants (which are independent of $\theta\in[0,1]$) yields: 
\begin{equation}\label{eq:result_asymptotic}
|\alpha^*-\alpha(\theta)| = \mathcal{O}\big(\lambda C^2\theta + C^2\theta^2\big),
\qquad\text{for any }\theta\in[0,1].
\end{equation}
as desired. 
\end{proof}

\section{Online Algorithm}\label{section:online_algorithm}

We also consider the online setting, where users send requests in sequential order \citep{nguyen2022survey}. Here, we denote $\mathcal{D}_{f_k}$ as the forget set after the $k$-th request and the associated retain set as $\mathcal{D}_{r_k} = \mathcal{D} \setminus \cup_{i = 1}^k \mathcal{D}_{f_i}$. Letting $\bm{\tilde{w}}_0 = \mathcal{A}(\mathcal{D})$, we estimate $\tilde{\bm{w}}_k$ recursively as $\tilde{\bm{w}}_k = \tilde{\bm{w}}_{k-1} - \frac{\tilde{\bm{H}}^{-1}_{n, \lambda, k-1}}{H}\nabla_{\tilde{\bm{w}}_{k-1}, \mathcal{K}, \mathcal{A}}$, where $\tilde{\bm{H}}^{-1}_{n, \lambda, k-1}$ is an estimator for $(\bm{H}_{\tilde{\bm{w}}_{k-1}, \mathcal{K}, \mathcal{A}} + \lambda \bm{I})^{-1}$ with respect to $\mathcal{D}_{f_k}$ and $\mathcal{D}_{r_k}$. $\nabla_{\tilde{\bm{w}}_{k-1}, \mathcal{K}, \mathcal{A}}$ is also computed with respect to $\mathcal{D}_{f_k}$ and $\mathcal{D}_{r_k}$. Adding noise to $\tilde{\bm{w}}_k$ as stipulated in \cref{thm:unlearning_guarantee_for_pareto} yields a $\bm{w}_k^-$ satisfying \cref{defn:certified_uniformity}. Furthermore, we have that: 

\begin{proposition}
    Let $\lambda_{\min}$ be the smallest eigenvalue of $\bm{H}_{\tilde{\bm{w}}_{k-1}, \mathcal{K}, \mathcal{A}}$, $\lambda > ||\bm{H}_{\bm{w}^*, \mathcal{K}, \mathcal{A}}||_2$, and $||\nabla_{\bm{\tilde{w}}_k, \mathcal{K}, \mathcal{A}}||_2, ||\nabla_{\bm{\tilde{w}}_k, \mathcal{K}, \mathcal{A}}||_2 \leq G$ for all $k$, all evaluated with respect to $\mathcal{D}_{f_k}$ and $\mathcal{D}_{r_k}$. Then, the bound in \cref{thm:bound_with_hessian_estimator} is identical in the online setting. 
    \label{prop:bound_sequential}
\end{proposition}

\textit{Proof: } See \cref{section:proof_of_sequential_bound}. 

In the online setting, \cref{prop:bound_sequential} yields \cref{algo:sequential_algorithm}. 

\begin{algorithm}
\caption{Online $(\varepsilon, \delta, \theta)$-certified uniformity with DP}

\begin{algorithmic}
\Require Dataset $\mathcal{D}$; forget sets $\{\mathcal{D}_{f_1}, ..., \mathcal{D}_{f_k}\}$; pretrained model $\bm{w}^* = \mathcal{A}(\mathcal{D})$; privacy budgets $\varepsilon$ and $\delta$; ; privacy-utility tradeoff coefficient $\theta$; sample size $n$; local convex coefficient $\lambda$; norm upper bound $C$; cumulative Hessian upper bound $H$; individual Hessian minimum eigenvalue upper bound $\zeta_{\min}$; bound looseness probability $\rho$.

\State $\bm{\tilde{w}}_0 \gets \bm{w}^*$
\State $\mathcal{D}_{r_0} \gets \mathcal{D}$

\For{$i = 1, ..., k$}
\State $\mathcal{D}_{r_i} \gets \mathcal{D}_{r_{i-1}} \setminus \mathcal{D}_{f_i}$

\State $\bm{P}_{0, \lambda} \gets \nabla_{\bm{\tilde{w}_{i-1}}, \mathcal{K}, \mathcal{A}}$

\For{$t = 1, ..., n$}
    \State Sample $X_{t_i}$ from $D_{f_i}$ with probability $\theta$ or sample $X_{t_i}$ from $D_r$ with probability $1 - \theta$ 

    \If{$X_{t_i} \sim D_f$}
        \State $\bm{H}_{t, \lambda, i} \gets \nabla^2_{\bm{w}}\mathcal{L}_\mathcal{K}(\bm{w}^*, X_{t_i}) + \frac{\lambda \bm{I}}{2\theta}$
    \ElsIf{$X_i \sim D_r$}
        \State $\bm{H}_{t, \lambda} \gets \nabla^2_{\bm{w}}\mathcal{L}_\mathcal{A}(\bm{w}^*, X_{t_i}) + \frac{\lambda \bm{I}}{2(1-\theta)}$
    \EndIf

    $\bm{P}_{t, \lambda, i} = \bm{P}_{0, \lambda, i} + (\bm{I} - \frac{\bm{H}_{t, \lambda, i}}{H})\bm{P}_{t-1, \lambda, i}$. 
\EndFor
\State $\bm{\tilde{w}_i} \gets \bm{\tilde{w}}_{i-1} - \frac{\bm{P}_{n, \lambda, i}}{H}$

\EndFor 

\State Compute $\Delta$ as the bound in \cref{eq:hessian_estimator_bound}. 

\State $\sigma = \frac{\Delta}{k\varepsilon}\sqrt{2\ln(1.25/\delta)}$
\State $\bm{w}^- \gets \bm{\tilde{w}}_k + Y$ where $Y \sim \mathcal{N}(\bm{0}, \sigma^2\bm{I})$

\State \Return $\bm{w}^-$

\end{algorithmic}

\label{algo:sequential_algorithm}

\end{algorithm}

Note that, for simplicity, we set $b = 1$. However, they can be added similarly to \cref{algo:hess_estimator_algo} if the user desires. 

\section{Eliminating Hyperparameters in Certified Algorithms}\label{section:appendix_eliminate_hyperparams}

Here, we summarize how to eliminate hyperparameters in \cref{algo:hess_exact_algo}, \cref{algo:hess_estimator_algo}, and \cref{algo:sequential_algorithm}. 

\begin{itemize}
    \item $\lambda_{\min}$ can be chosen as 0 by convex approximation, or it can be estimated using simple algorithms like Gershgorin's circle theorem or inverse power iteration. 
    \item By \cref{lemma:lipschitzness_of_hessians_and_grads}, $\lambda$ can be chosen as $\theta P_{\mathcal{K}} + (1-\theta)P_{\mathcal{A}}$
    \item Similarly, by \cref{lemma:lipschitzness_of_hessians_and_grads}, $H$ can be chosen as 2$\lambda$
    \item In practice, since $\frac{B}{\lambda + \lambda_{\min}}$ offers a bound on $\hat{\kappa}_l$ by \cref{lemma:condition_nums_bound}, $\zeta_{\min}$ can be chosen as $\lambda + \lambda_{\min}$
    \item By \cref{lemma:agrawal_et_al_bound}, $n$ can be chosen as $n = 2 \frac{B}{\lambda + \lambda_{\min}}\ln(\frac{B}{\lambda + \lambda_{\min}}b)$
    \item $G$ can be approximated by computing $||\nabla_{\bm{w}^*, \mathcal{K}, \mathcal{A}}||_2$.
    \item $\theta$ can be chosen with \cref{corollary:lambda_large} to satisfy a particular closeness to uniformity. 
    \item In practice, we find that $C$ can be chosen as $10$, $20$, or $100$. 
    \item $b$ can be chosen to satisfy a particular concentration on the estimator, so we let it be free. However, one can set $b = 1$. 
    \item Common heuristics for $\varepsilon$ and $\delta$ are available in the differential privacy literature.
    \item In practice, following what is common in certified unlearning e.g. in \citep{zhang2024certifiedunlearningDNN}, the Lipchitz constants in \cref{assumption1} and \cref{assumption2} are treated as hyperparameters. However, in practice, they can all be set to 1.
\end{itemize}

\section{Experimental Details}\label{appendix:experimental_details}

\subsection{Dataset Details}

\textbf{MNIST: } The MNIST dataset contains 70k 28x28 greyscale images of hand-drawn digits in 10 classes \citep{mnist}. We conduct our experiments with 49k training images and 21k test images. The classes are mutually exclusive.

\textbf{Kuzushiji-MNIST: } The Kuzushiji-MNIST (KMNIST) dataset contains 70k 28x28 greyscale images of Japanese kanji in 10 classes \citep{clanuwat2018deep}.  We conduct our experiments with 49k training images and 21k test images. The classes are mutually exclusive. 

\textbf{CIFAR10: } The CIFAR10 dataset consists of 60k 32x32 color images in 10 classes. The classes are mutually exclusive and include airplanes, automobiles, birds, cats, deer, dogs, frogs, horses, ships, and trucks \citep{krizhevsky2009learning}.

\textbf{CIFAR-100: } The CIFAR-100 dataset is similar to CIFAR-10 but contains 100 classes, each with 600 images, making a total of 60k 32x32 color images. The 100 classes are grouped into 20 superclasses, and each image comes with a “fine” label (the class to which it belongs) and a “coarse” label (the superclass to which it belongs) \citep{krizhevsky2009learning}. 

\textbf{SVHN: } The Street View House Numbers (SVHN) dataset \citep{netzer2011reading} contains images of double-digit numbers on house walls as colored 32x32 images. We load SVHN with 100 classes, corresponding to 10 * 10 for each digit. There are 73k training images, 26k testing images. 

\subsection{Model Details}

\textbf{LogReg: } A logistic regression model that has a single linear layer between inputs and outputs, followed by a softmax output function. 

\textbf{MLP: } A two-layer ReLU feedforward neural network. 

\textbf{ResNet8: } A [1,1,1,0] residual network, , with standard convolutional blocks, as described in \citep{he2016deep}.

\textbf{ResNet18: } A [2,2,2,2] residual network, with standard convolutional blocks, as described in \citep{he2016deep}.

\textbf{ResNet50: } A [3, 4, 6, 3] residual network, with bottleneck convolutional blocks, as described in \citep{he2016deep}.

\textbf{ViT\_S\_16: } A vision transformer with $\approx20$ million parameters, as detailed in \citep{dosovitskiy2021an}. 

\textbf{ViT\_B\_16: } A vision transformer with $\approx80$ million parameters, as detailed in \citep{dosovitskiy2021an}. 

\subsection{Baseline Details}

We implement several baselines and provide the rationale for their use below: 

\textbf{Pretrained: } This is simply the pretrained model corresponding to whichever model and benchmark is specified. The rationale for using this is to demonstrate that we alter uniformity significantly from before without tarnishing accuracy for either the retain or test sets. 

\textbf{Retrained: } This is a model retrained over the retain set, performing exact unlearning. The rationale for using this is to demonstrate that our methods mimic unlearning in how we preserve accuracy, but induce uniformity in a way that unlearning does not.

\textbf{Synthetic: } This method proceeds as follows: for each instance in the forget set, sample $k$ instances from the $\varepsilon$-ball, with respect to the $\ell_2$ norm, around that instance. Then, assign these $k$ instances random labels from the label space, choosing labels uniformly at random. Do this for all forget set instances, yielding $|\mathcal{D}_f|k$ new instances. Then, append this to the retain set and retrain over this augmented dataset.  This provides strong accuracy for simple baselines like MNIST while also inducing uniformity, providing an alternative, simple algorithm to compare our method against in terms of time elapsed. 

\textbf{Label Differential Privacy: } Specifically, we use the multi-stage training method of \citet{ghazi2021deep} to obtain a model which is differentially private with respect to the labels--that is, an adversary cannot be sure whether the label they obtain is the true label. This is related to our work, albeit addresses a different threat model, as described in \cref{appendix:additional_related_work}. Still, we believe it is important to demonstrate that our method achieves privacy while not sacrificing utility to the extend that label differential privacy does, since it is a well-known method in the privacy literature that addresses a similar problem. For our experiments, we use the official repository with the hyperparameters reported in the paper: \url{https://github.com/google-research/label-dp}.

\subsection{Hyperparameter Details}

Please note that, throughout, we do not do extensive hyperparameter optimization, which may lead to improved performance.

We use a standard train-test split of 70-30 throughout. Pretraining and synthetic training have the same hyperparameters as pretraining, unless mentioned otherwise. We use ADAM, with standard PyTorch hyperparameters aside from learning rate and weight decay, throughout. A batch size of 128 is used for pretraining and also for the retain set in \cref{algo:finetuning_algo} throughout.  For all \cref{algo:finetuning_algo} experiments and \cref{algo:hess_exact_algo} experiments, we use a forget set size of 100 with a batch size (when loading the forget set into the finetuning in \cref{algo:finetuning_algo}) of 10. We perform \cref{algo:finetuning_algo} for 100 epochs and finetune \cref{algo:hess_exact_algo} for 50 epochs before running the certified Newton step. We generally use the forward KL divergence between model softmax outputs and the uniform distribution for $\mathcal{L}_\mathcal{K}$ and the cross entropy between model predictions and ground truth labels for $\mathcal{L}_\mathcal{A}$. For LogReg, we instead use the square loss between the uniform softmax probabilities and the model softmax outputs, since the forward KL is not necessarily convex in $\bm{w}$ in this case, while the square loss is; this allows us to use \cref{algo:hess_exact_algo} with small $\lambda$. For our synthetic baseline, we use $\varepsilon = 8/255$ throughout, where $\varepsilon$ is the size of the $\varepsilon$-ball where we sample instances to assign random labels for retraining. For the LabelDP baseline, we use the multi-stage training algorithm of \citet{ghazi2021deep} throughout. 

Early stopping is implemented by saving the model which first meets the early stopping conditions, and continuing to see if any model performs better in terms of confidence distance while still meeting the early stopping conditions specified below.

\textbf{Compute: } We use two RTX 6000 Ada Generation NVIDIA GPUs throughout. The most resource intensive experiments are the LabelDP experiments, which take up most of the memory on both GPUs. Besides those, the other experiments take up at most a fourth of the compute resources available on one GPU. No experiments ran required more compute than these two GPUs provide. 

\textbf{MNIST, LogReg Pretraining: }  Epochs: 25. Learning rate: 0.01. 

\textbf{MNIST, MLP Pretraining: } Epochs: 5. Learning rate: 0.01. 

\textbf{MNIST, ResNet18 Pretraining: } Epochs: 2. Learning rate: 0.001.

\textbf{MNIST, LogReg \cref{algo:finetuning_algo}: } Learning rate: 0.01

\textbf{MNIST, MLP \cref{algo:finetuning_algo}: } Learning rate: 0.01

\textbf{MNIST, ResNet18 \cref{algo:finetuning_algo}: } 
Learning rate: 0.001. Early stopping criterion of a confidence distance $<0.32$ and a retain accuracy of $>90\%$. %

\textbf{MNIST, LogReg \cref{algo:hess_exact_algo}: } $M = 1$. $C = 10$, pretrained with PGD with the same hyperparameters as the standard pretraining. Since the losses are convex in $\bm{w}$, $\lambda_{\min} = 0$. $\lambda = 0.0001$. Following \citet{zhang2024certifiedunlearningDNN}, we use the variance $\sigma^2$ as a hyperparameter, corresponding to a broad range of choices of $\varepsilon$ and $\delta$. We choose $\sigma = 0.001$. This results in large $\varepsilon$ and $\delta$, as typical in differential privacy \citep{dwork2014algorithmic} and certified unlearning \citep{qiao2025hessianfree}. However, we still observe good induced uniformity. 

\textbf{MNIST, LogReg Synthetic Baseline: } Sampled $k$ instances for each forget set instance: 5. 

\textbf{MNIST, ResNet18 Synthetic Baseline: }  Sampled $k$ instances for each forget set instance: 500. 

\textbf{MNIST, LogReg LabelDP Baseline: } Epochs: 200. Batch size: 256. Random flip, random left-right flip, and random cutout (8). SGD with learning rate 0.4 with momentum 0.9. $\varepsilon = 2.0$. Mixup for stage 1: 16. Mixup for stage 2: 8. Data split evenly between the two stages. Piecewise constant learning rate scheduler. These hyperparameters are chosen to match those in the best results of \citet{ghazi2021deep}. See \citet{ghazi2021deep} for more details on these hyperparameters. 

\textbf{MNIST, ResNet18 LabelDP Baseline: } Same as the MNIST LogReg LabelDP hyperpameters, except with a weight decay of 0.0005 throughout. 

\textbf{KMNIST, LogReg Pretraining: } Epochs: 100. Learning rate: 0.01. 

\textbf{KMNIST, MLP Pretraining: } Epochs: 100. Learning rate: 0.001. 

\textbf{KMNIST, ResNet18 Pretraining: } Epochs: 12 Learning rate: 0.002. 

\textbf{KMNIST, LogReg \cref{algo:finetuning_algo}: } Same as pretraining. 

\textbf{KMNIST, MLP \cref{algo:finetuning_algo}: } Learning rate: 0.01.

\textbf{KMNIST, ResNet18 \cref{algo:finetuning_algo}: } Learning rate: 0.002. Early stopping criterion of a confidence distance $<0.32$ and a retain accuracy of $>99\%$. 

\textbf{KMNIST, ResNet18 Synthetic Baseline: }  Sampled $k$ instances for each forget set instance: 500. 

\textbf{KMNIST, ResNet18 LabelDP Baseline: }  Same as the MNIST ResNet18 LabelDP hyperparameters. 

\textbf{SVHN, ResNet50 Pretraining: } Epochs: 150. Learning rate: 0.001. Weight decay: 0.00005.

\textbf{SVHN, ResNet50 \cref{algo:finetuning_algo}: } Same as pretraining. 

\textbf{SVHN, ResNet50 Synthetic Baseline: }  Sampled $k$ instances for each forget set instance: 500. 

\textbf{SVHN, ResNet50 LabelDP Baseline: } Same as MNIST ResNet18 LabelDP hyperparameters. 

\textbf{CIFAR10, ResNet18 Pretraining: } Epochs: 200 with SGD with a momentum of 0.9. Learning rate: 0.1. Weight decay: 0.0005.

\textbf{CIFAR10, ResNet18 \cref{algo:finetuning_algo}: } Same as pretraining. Early stopping criterion of a confidence distance $<0.42$ and a retain accuracy of $>87\%$. 
 
\textbf{CIFAR10, ResNet50 Pretraining: } Same as CIFAR10 ResNet18. 

\textbf{CIFAR10, ResNet50 \cref{algo:finetuning_algo}: } Same as pretraining. Early stopping criterion of a confidence distance $<0.42$ and a retain accuracy of $>87\%$. 

\textbf{CIFAR10, ResNet8 Pretraining: } Same as CIFAR10 ResNet18. 

\textbf{CIFAR10, ResNet8 \cref{algo:finetuning_algo}: } Same as CIFAR10 ResNet18. 

\textbf{CIFAR10, ResNet18 Synthetic Baseline: }  Sampled $k$ instances for each forget set instance: 5000. 

\textbf{CIFAR10, ResNet50 Synthetic Baseline: }  Sampled $k$ instances for each forget set instance: 5000. 

\textbf{CIFAR10, ResNet18 LabelDP Baseline: } Same as MNIST ResNet18 LabelDP hyperparameters, except with a batch size of 512. 

\textbf{CIFAR10, ResNet50 LabelDP Baseline: } Same as CIFAR10 ResNet18 LabelDP. 

\textbf{CIFAR10, ViT\_S\_16 Finetuning: }  8 epochs. Learning rate $0.0001$ with AdamW.

\textbf{CIFAR10, ViT\_B\_16 Finetuning: } 10 epochs. Learning rate $0.0001$ with AdamW.

\textbf{CIFAR100, ResNet50 Pretraining: } Same as CIFAR10 ResNet18. 

\textbf{CIFAR100, ResNet50 \cref{algo:finetuning_algo}: } Same as pretraining. Early stopping criterion of a confidence distance $<0.42$ and a retain accuracy of $>87\%$. 

\textbf{CIFAR100, ResNet8 Pretraining: } Same as CIFAR100 ResNet50. 

\textbf{CIFAR100, ResNet8 \cref{algo:finetuning_algo}: } Same as CIFAR100 ResNet50.  

\textbf{CIFAR100, ResNet50 Synthetic Baseline: } Sampled instances: 5000. 

\textbf{CIFAR100, ResNet50 LabelDP Baseline: } Same as CIFAR10 ResNet18 LabelDP. Please note that our results differ from the results reported in the original paper of \citet{ghazi2021deep}; however, we verified our results through several runs and used the official paper repository at \url{https://github.com/google-research/label-dp} with the hyperparameters reported in the paper. 

\textbf{CIFAR100, ViT\_S\_16 Finetuning: } 30 epochs. Learning rate $0.002$ with SGD with momentum 0.9. 500 warmup steps with cosine scheduler.

\textbf{CIFAR100, ViT\_S\_16 \cref{algo:finetuning_algo}: } 100 epochs. Learning rate $0.001$ with SGD. 

\textbf{CIFAR100, ViT\_B\_16 Finetuning: } 45 epochs. Learning rate $0.002$ with SGD with momentum 0.9. 500 warmup steps with cosine scheduler.

\textbf{CIFAR100, ViT\_B\_16 \cref{algo:finetuning_algo}: } Same as CIFAR100 ViT\_S\_16.

\textbf{TinyImageNet, ViT\_S\_16 Finetuning: } 30 epochs. Learning rate $0.0001$, momentum $0.9$, and weight decay $0.01$ with SGD.

\textbf{TinyImageNet, ViT\_S\_16 \cref{algo:finetuning_algo}: } Same as CIFAR100 ViT\_S\_16.

\textbf{TinyImageNet, ViT\_B\_16 Finetuning: } 50 epochs. Learning rate $0.0001$, momentum $0.9$, and weight decay $0.01$ with SGD.

\textbf{TinyImageNet, ViT\_B\_16 \cref{algo:finetuning_algo}: } Same as CIFAR100 ViT\_S\_16.

\textbf{Attacks}: $\alpha = 0.0001$. PGD learning rate: 0.001. PGD steps: 50. 

\textbf{Test-Set Finetuning}: Finetune pretrained model for 20 more epochs with the same hyperparameters as pretraining, then run \cref{algo:finetuning_algo} with the same hyperparameters as the original run of \cref{algo:finetuning_algo}. For CIFAR10/CIFAR100, finetune for 100 epochs.

\section{Additional Experiments}\label{appendix:additional_experiments}

\subsection{Test Set Accuracies for Main Paper Experiments}

We provide a figure, similar to \cref{fig:retain_accuracy_comparison}, for the test set in \cref{fig:test_set_comparison}. We observe similar results as one does on the retain set, with test accuracies preserved by \cref{algo:finetuning_algo} and \cref{algo:hess_exact_algo}.

\begin{figure}
    \centering
    \includegraphics[width=0.75\linewidth]{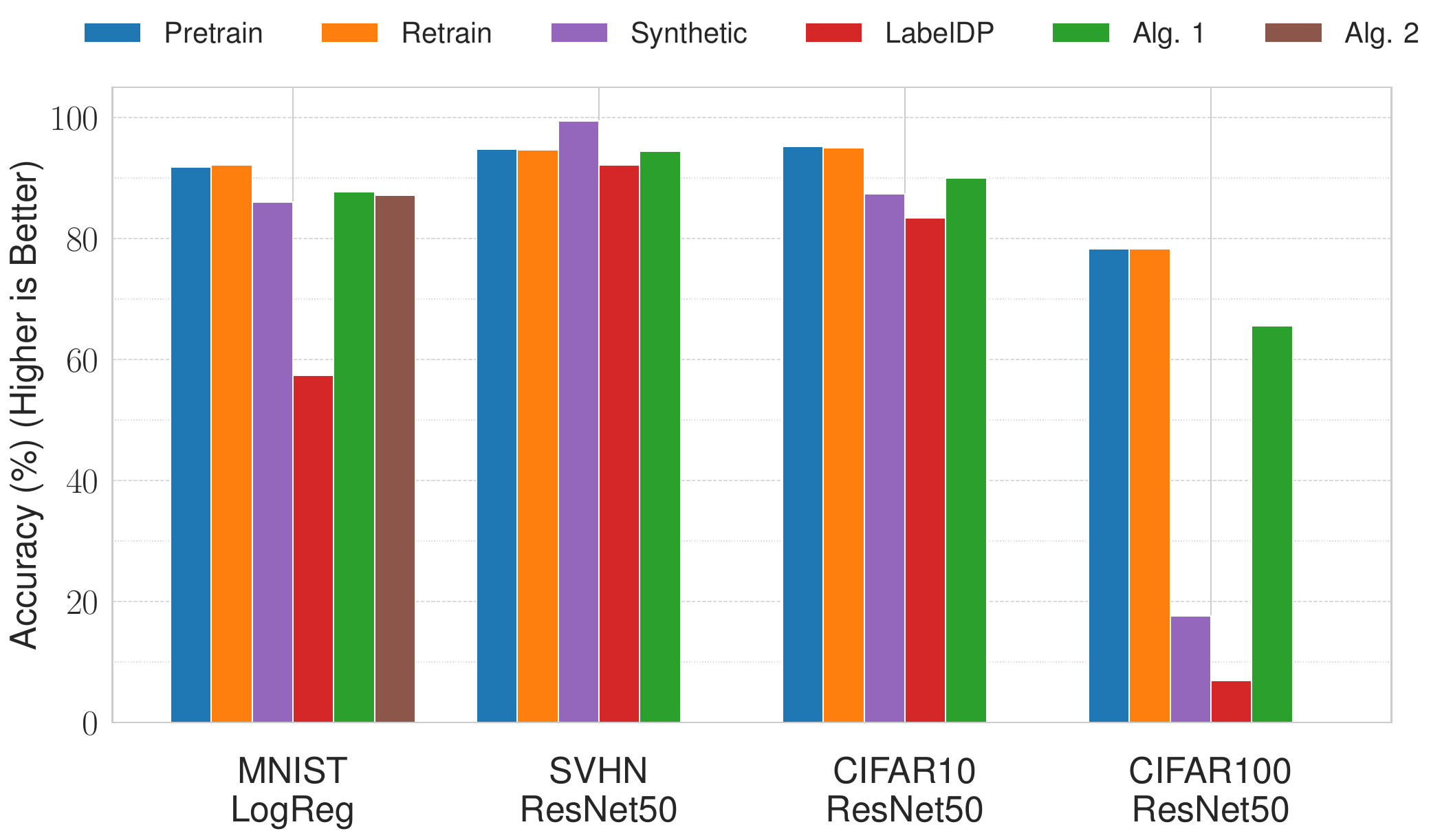}
    \caption{Accuracy on test set for baselines as well as \cref{algo:finetuning_algo} and \cref{algo:hess_exact_algo} with $\theta  = 0.75$.}
    \label{fig:test_set_comparison}
\end{figure}

\subsection{Tables for Main Paper Experiments}

The tables for \cref{fig:alg1_alg2_results} are in \cref{table:finetuning_main_combined} and \cref{table:hessian_exact_main}. We include results for MNIST and KMNIST ResNet18 as well. We see that \cref{algo:finetuning_algo} induces uniformity without great damage to utility, while all other baselines--including the synthetic baseline--fail to do so without critically harming utility. Furthermore, we observe that ResNet50 performs better than ResNet18, providing more credibility to the claim in \cref{section:experiments} that larger models tend to perform better when used in \cref{algo:finetuning_algo}. Next, we observe that \cref{algo:finetuning_algo} can actually provide better retain and test accuracy than the pretrained model, as observed for ResNet50 over CIFAR100; this is because we also minimize the retain accuracy during finetuning. We similarly have to use early stopping for \cref{algo:finetuning_algo}, as discussed in \cref{section:experiments}, since we use large models. Finally, we observe that LabelDP can induce uniformity, albeit at the cost of retain and test accuracy, but does so not only on the forget set but also the retain set; for a comparison of the confidence distances across the retain, test, and forget sets for LabelDP and our method, please see \cref{table:label_dp_pareto_conf_dist}. Additionally, for larger, more complex datasets like CIFAR100, LabelDP fails entirely. Please note that we do not perform extensive hyperparameter optimization during pretraining or retraining. 

We observe similar results for \cref{algo:hess_exact_algo} in \cref{table:hessian_exact_main}.

\begin{table}[tb] 
\centering 
\caption{Results for \cref{algo:finetuning_algo}, used in \cref{fig:alg1_alg2_results}. We find that we are able to induce uniformity while only slightly decreasing retain and test accuracy. $\theta = 0.75$ throughout.} 
\vspace{2mm} 
\setlength{\tabcolsep}{0.8\tabcolsep} 
\begin{tabular}{c|ccccc} 
\hline 
\textbf{Dataset} & \textbf{Model}    & \textbf{Method}                  & \textbf{Retain Acc.} & \textbf{Test Acc.}    & \begin{tabular}[c]{@{}c@{}}\textbf{Conf. Dist.}\\ \textbf{(Lower Better)}\end{tabular} \\ 
\hline 
\multirow{5}{*}{MNIST}  & \multirow{5}{*}{ResNet18}  
 & \textcolor{gray}{Pretrain}                     & \textcolor{gray}{98.0\%} & \textcolor{gray}{98.1\%} & \textcolor{gray}{0.877} \\ 
 &                              & \textcolor{gray}{Retrain}                      & \textcolor{gray}{97.3\%} & \textcolor{gray}{97.1\%} & \textcolor{gray}{0.876} \\ 
 &                              & \textcolor{gray}{Synthetic}                    & \textcolor{gray}{100.0\%} & \textcolor{gray}{99.1\%} & \textcolor{gray}{0.010} \\ 
 &                              & \textcolor{gray}{LabelDP}                      & \textcolor{gray}{98.8\%} & \textcolor{gray}{98.8\%} & \textcolor{gray}{0.593}  \\ 
 &                              & \cref{algo:finetuning_algo}                     & 99.6\% & 99.1\% & 0.070 \\ 
\hline 
\multirow{5}{*}{KMNIST} & \multirow{5}{*}{ResNet18}  
 & \textcolor{gray}{Pretrain}                     & \textcolor{gray}{98.2\%} & \textcolor{gray}{92.1\%} & \textcolor{gray}{0.880} \\ 
 &                              & \textcolor{gray}{Retrain}                      & \textcolor{gray}{98.4\%} & \textcolor{gray}{92.4\%} & \textcolor{gray}{0.884} \\ 
 &                              & \textcolor{gray}{Synthetic}                    & \textcolor{gray}{99.9\%}  & \textcolor{gray}{96.7\%} & \textcolor{gray}{0.019} \\ 
 &                              & \textcolor{gray}{LabelDP}                      & \textcolor{gray}{98.9\%}  & \textcolor{gray}{96.1\%} & \textcolor{gray}{0.530} \\ 
 &                              & \cref{algo:finetuning_algo}                     & 99.1\%  & 94.7\% & 0.257 \\ 
\hline 
\multirow{5}{*}{SVHN}   & \multirow{5}{*}{ResNet50}  
 & \textcolor{gray}{Pretrain}                     & \textcolor{gray}{99.6\%} & \textcolor{gray}{94.8\%} & \textcolor{gray}{0.980} \\ 
 &                              & \textcolor{gray}{Retrain}                      & \textcolor{gray}{99.3\%} & \textcolor{gray}{94.7\%} & \textcolor{gray}{0.964} \\ 
 &                              & \textcolor{gray}{Synthetic}                    & \textcolor{gray}{99.9\%} & \textcolor{gray}{99.4\%} & \textcolor{gray}{0.013} \\ 
 &                              & \textcolor{gray}{LabelDP}                      & \textcolor{gray}{92.0\%} & \textcolor{gray}{92.2\%} & \textcolor{gray}{0.282} \\ 
 &                              & \cref{algo:finetuning_algo}                     & 99.5\% & 94.4\% & 0.280 \\ 
\hline 
\multirow{10}{*}{CIFAR10} & \multirow{5}{*}{ResNet18} & \textcolor{gray}{Pretrain}      & \textcolor{gray}{100.0\%} & \textcolor{gray}{95.3\%} & \textcolor{gray}{0.898} \\ 
& & \textcolor{gray}{Retrain}      &  \textcolor{gray}{100.0\%} &  \textcolor{gray}{95.3\%} &  \textcolor{gray}{0.891} \\ 
& & \textcolor{gray}{Synthetic}    & \textcolor{gray}{94.0\%} & \textcolor{gray}{89.7\%} & \textcolor{gray}{0.844} \\ 
& & \textcolor{gray}{LabelDP}      & \textcolor{gray}{85.8\%} & \textcolor{gray}{83.6\%} & \textcolor{gray}{0.359} \\ 
& & \cref{algo:finetuning_algo} & 89.6\% & 83.1\% & 0.377 \\ 
\cline{2-6} 
& \multirow{5}{*}{ResNet50} & \textcolor{gray}{Pretrain}      & \textcolor{gray}{100.0\%} & \textcolor{gray}{95.2\%} & \textcolor{gray}{0.900} \\ 
& & \textcolor{gray}{Retrain}      &  \textcolor{gray}{100.0\%} &  \textcolor{gray}{95.0\%} &  \textcolor{gray}{0.891} \\ 
& & \textcolor{gray}{Synthetic}    & \textcolor{gray}{91.4\%} & \textcolor{gray}{87.4\%} & \textcolor{gray}{0.818} \\ 
& & \textcolor{gray}{LabelDP}      & \textcolor{gray}{85.5\%} & \textcolor{gray}{83.4\%} & \textcolor{gray}{0.334} \\ 
& & \cref{algo:finetuning_algo} & 94.7\% & 90.0\% & 0.270 \\ 
\hline 
\multirow{5}{*}{CIFAR100}  & \multirow{5}{*}{ResNet50} & \textcolor{gray}{Pretrain}      & \textcolor{gray}{100.0\%} & \textcolor{gray}{78.3\%} & \textcolor{gray}{0.902} \\ 
& & \textcolor{gray}{Retrain}      & \textcolor{gray}{100.0\%} & \textcolor{gray}{78.3\%} & \textcolor{gray}{0.765} \\ 
& & \textcolor{gray}{Synthetic}    & \textcolor{gray}{17.7\%} & \textcolor{gray}{17.6\%} & \textcolor{gray}{0.189} \\ 
& & \textcolor{gray}{LabelDP}      & \textcolor{gray}{8.41\%} & \textcolor{gray}{6.95\%} & \textcolor{gray}{0.203} \\ 
& & \cref{algo:finetuning_algo} & 91.4\% & 65.5\% & 0.298 \\

\end{tabular} 
\label{table:finetuning_main_combined} 
\end{table}

\begin{table}[tb]
\centering
\caption{Results for \cref{algo:hess_exact_algo} for logistic regression trained over MNIST, used in \cref{fig:alg1_alg2_results}. $\theta = 0.75$ throughout.}
\vspace{2mm}
\setlength{\tabcolsep}{0.8\tabcolsep}
\begin{tabular}{c|ccc}
\hline
\textbf{Method}                  & \textbf{Retain Acc.} & \textbf{Test Acc.} & \begin{tabular}[c]{@{}c@{}}\textbf{Conf. Dist.}\\ \textbf{(Lower Better)}\end{tabular} \\
\hline
\textcolor{gray}{Pretrain}      & \textcolor{gray}{92.1\%} & \textcolor{gray}{91.8\%} & \textcolor{gray}{0.807} \\
\textcolor{gray}{Retrain}       & \textcolor{gray}{92.0\%} & \textcolor{gray}{92.1\%} & \textcolor{gray}{0.807} \\
\textcolor{gray}{Synthetic}     & \textcolor{gray}{86.4\%} & \textcolor{gray}{86.0\%} & \textcolor{gray}{0.313} \\
\textcolor{gray}{LabelDP}       & \textcolor{gray}{$57.1\%$} & \textcolor{gray}{$57.4\%$} & \textcolor{gray}{0.125} \\
\cref{algo:finetuning_algo}                       & 87.8\%               & 87.7\%            & 0.180                   \\
\cref{algo:hess_exact_algo}                       & 87.1\%               & 87.2\%          & 0.280                    \\
\hline
\end{tabular}
\label{table:hessian_exact_main}
\end{table}

\subsection{Additional Experiments on TinyImageNet \& ViT}\label{appendix:imagenet_vit}

We provide experimental results for \cref{algo:finetuning_algo} for ViT trained on CIFAR100 and TinyImageNet in \cref{table:imagenet_vit}, observing similar behavior--in fact significantly lower confidence distance with little retain or test accuracy reduction--when compared to in \cref{table:finetuning_main_combined}. 

\begin{table}[tb]
  \centering
  \caption{Results for \cref{algo:finetuning_algo} for ViT trained on CIFAR100 and TinyImageNet.}
  \vspace{2mm}
  \begin{tabular}{c|ccccc}
    \hline 
    \textbf{Dataset}   & \textbf{Model}   & \textbf{Method} & \textbf{Retain Acc.} & \textbf{Test Acc.} &  \begin{tabular}[c]{@{}c@{}}\textbf{Conf. Dist.}\\ \textbf{(Lower Better)}\end{tabular} \\
    \hline 
\multirow{8}{*}{CIFAR100}  
& \multirow{5}{*}{ViT\_S\_16} & \textcolor{gray}{Pretrain}      & \textcolor{gray}{95.2\%} & \textcolor{gray}{90.1\%} & \textcolor{gray}{0.942} \\ 
& & \textcolor{gray}{Retrain}      & \textcolor{gray}{93.4\%} & \textcolor{gray}{89.1\%} & \textcolor{gray}{0.883} \\ 
& & \textcolor{gray}{Synthetic}    & \textcolor{gray}{86.36\%} & \textcolor{gray}{84.76\%} & \textcolor{gray}{0.682} \\ 
& & \cref{algo:finetuning_algo} & 91.6\% & 85.1\% & 0.036 \\
\cline{2-6}
& \multirow{5}{*}{ViT\_B\_16} & \textcolor{gray}{Pretrain}      & \textcolor{gray}{94.2\%} & \textcolor{gray}{91.2\%} & \textcolor{gray}{0.972} \\ 
& & \textcolor{gray}{Retrain}      & \textcolor{gray}{95.2\%} & \textcolor{gray}{91.0\%} & \textcolor{gray}{0.952} \\ 
& & \textcolor{gray}{Synthetic}    & \textcolor{gray}{17.7\%} & \textcolor{gray}{17.6\%} & \textcolor{gray}{0.189} \\ 
& & \cref{algo:finetuning_algo} & 91.6\% & 88.6\% & 0.074\\

\hline 
\multirow{12}{*}{TinyImageNet} & \multirow{5}{*}{ViT\_S\_16} & \textcolor{gray}{Pretrain}      & \textcolor{gray}{86.7\%} & \textcolor{gray}{84.4\%} & \textcolor{gray}{0.698} \\ 
& & \textcolor{gray}{Retrain}      &  \textcolor{gray}{87.2\%} &  \textcolor{gray}{84.4\%} &  \textcolor{gray}{0.742} \\ 
& & \textcolor{gray}{Synthetic}    & \textcolor{gray}{87.9\%} & \textcolor{gray}{86.3\%} & \textcolor{gray}{0.833} \\ 
& & \cref{algo:finetuning_algo} & 84.2\% & 81.4\% & 0.057\\ 
\cline{2-6}
& \multirow{5}{*}{ViT\_B\_16} & \textcolor{gray}{Pretrain}      & \textcolor{gray}{95.0\%} & \textcolor{gray}{91.7\%} & \textcolor{gray}{0.822} \\ 
& & \textcolor{gray}{Retrain}      &  \textcolor{gray}{96.8\%} &  \textcolor{gray}{90.6\%} &  \textcolor{gray}{0.826} \\ 
& & \textcolor{gray}{Synthetic}    & \textcolor{gray}{98.0\%} & \textcolor{gray}{84.4\%} & \textcolor{gray}{0.830} \\ 
& & \cref{algo:finetuning_algo} & 91.8\% & 88.3\% & 0.037 \\ 
\cline{2-6}
& \multirow{5}{*}{ResNet50} & \textcolor{gray}{Pretrain}      & \textcolor{gray}{91.2\%} & \textcolor{gray}{83.5\%} & \textcolor{gray}{0.812} \\ 
& & \textcolor{gray}{Retrain}      &  \textcolor{gray}{91.6\%} &  \textcolor{gray}{80.9\%} &  \textcolor{gray}{0.924} \\ 
& & \textcolor{gray}{Synthetic}    & \textcolor{gray}{92.1\%} & \textcolor{gray}{80.6\%} & \textcolor{gray}{0.569} \\ 
& & \cref{algo:finetuning_algo} & 92.1\% & 81.6\% & 0.197 \\ 
    
    \hline 
  \end{tabular}
  \label{table:imagenet_vit}
\end{table}

\subsection{LabelDP and \cref{algo:finetuning_algo} Confidence Distances for Retain, Test, and Forget Sets}\label{appendix:labeldp_main_pareto_conf_dist}

Here, we present \cref{table:label_dp_pareto_conf_dist}, which details the confidence distances for the retain, test, and forget sets of our method vs. LabelDP. Not only do we achieve better retain and test accuracy, but also we induce uniformity on \textit{only} the forget set, while LabelDP induces uniformity on the forget, retain, and test sets altogether, functionally the same as adjusting the temperature. This does not suffice for our threat model, since we want to preserve confident predictions on the retain and test sets. 

\begin{table}[tb]
    \centering
    \setlength{\tabcolsep}{0.8\tabcolsep}
    \caption{A comparison of the confidence distances on the retain, test, and forget sets between \cref{algo:finetuning_algo} and LabelDP. In general, we induce uniformity on only the forget set, while maintaining confidently correct predictions on the retain and test sets, while LabelDP falls short. Note that this is only for one experiment run.}
    \begin{tabular}{c|ccccc}
    \hline 
\textbf{Dataset} & \textbf{Model} & \textbf{Method}  & \begin{tabular}[c]{@{}c@{}}\textbf{Retain Conf. Dist.}\\\textbf{(Higher Better)} \end{tabular}   & \begin{tabular}[c]{@{}c@{}}\textbf{Test Conf. Dist.}\\\textbf{(Higher Better)} \end{tabular}       & \begin{tabular}[c]{@{}c@{}}\textbf{Forget Conf. Dist.}\\\textbf{(Lower Better)}\end{tabular} \\ 
\hline 
    \multirow{2}{*}{MNIST} &  \multirow{2}{*}{ResNet18} & LabelDP & 0.579 & 0.577 & 0.593 \\  
                            &                           & \cref{algo:finetuning_algo} & 0.503 & 0.509 & 0.070 \\ 
    \hline 
     \multirow{2}{*}{KMNIST} &  \multirow{2}{*}{ResNet18} & LabelDP & 0.495 & 0.466 & 0.530 \\  
                            &                           & \cref{algo:finetuning_algo} & 0.870 & 0.828 & 0.257  \\  
    \hline 
    \multirow{2}{*}{SVHN} &  \multirow{2}{*}{ResNet50} & LabelDP & 0.288 & 0.285 & 0.282 \\  
                            &                           & \cref{algo:finetuning_algo} & 0.888 & 0.855 & 0.280 \\ 
    \hline 
    \multirow{4}{*}{CIFAR10} &  \multirow{2}{*}{ResNet18} & LabelDP & 0.371 & 0.365 & 0.359 \\  
                            &                           & \cref{algo:finetuning_algo} & 0.724 & 0.690 & 0.377 \\ 
    \cline{2-6}
                            &  \multirow{2}{*}{ResNet50} & LabelDP & 0.366&  0.361 & 0.334 \\  
                            &                           & \cref{algo:finetuning_algo} & 0.725 & 0.701 & 0.270 \\ 
    \hline 
    \multirow{2}{*}{CIFAR100} &  \multirow{2}{*}{ResNet50} & LabelDP & 0.182 & 0.156 & 0.203 \\  
                            &                           & \cref{algo:finetuning_algo} & 0.576 & 0.470 & 0.298\\ 
    \end{tabular}
    \label{table:label_dp_pareto_conf_dist}
\end{table}

\subsection{Pareto Frontier Main Paper Table}

The results which correspond to \cref{fig:conf_pareto}, \cref{fig:ret_pareto}, and \cref{fig:te_pareto} are included in \cref{table:pareto_main} and \cref{table:pareto_cifar}. 

\begin{table}[tb]
\centering
\caption{Results for \cref{algo:finetuning_algo} and \cref{algo:hess_exact_algo} as we explore the Pareto frontier over MNIST, single run for \cref{fig:conf_pareto}, \cref{fig:ret_pareto}, and \cref{fig:te_pareto}.}
\vspace{2mm}
\setlength{\tabcolsep}{0.8\tabcolsep}
\begin{tabular}{c|cccc}
\hline
\textbf{Model} & \textbf{$\theta$} & \textbf{Retain Acc.} & \textbf{Test Acc.} & \begin{tabular}[c]{@{}c@{}}\textbf{Conf. Dist.}\\ \textbf{(Lower Better)}\end{tabular} \\ \hline
\multirow{9}{*}{LogReg} & 0.000 & 93.7\% & 92.9\% & 0.817 \\ 
 & 0.125 & 91.9\% & 91.2\% & 0.583 \\
 & 0.250 & 92.9\% & 92.4\% & 0.178 \\ 
 & 0.375 & 92.0\% & 92.3\% & 0.208 \\
 & 0.500 & 92.4\% & 92.1\% & 0.342 \\ 
 & 0.625 & 91.3\% & 91.1\% & 0.374 \\
 & 0.750 & 91.6\% & 91.2\% & 0.263 \\
 & 0.850 & 88.4\% & 87.5\% & 0.204 \\
 & 0.950 & 89.2\% & 89.1\% & 0.169 \\ 
\hline 
\multirow{9}{*}{Cert. LogReg} & 0.000 & 92.5\% & 92.2\% & 0.738\\ 
 & 0.125 & 92.2\% & 91.3\% & 0.573\\ 
 & 0.250 & 91.6\% & 91.5\% & 0.324 \\
 & 0.375 & 91.4\% & 90.5\% & 0.327\\
 & 0.500 & 90.5\% & 90.6\% & 0.300 \\ 
 & 0.625 & 90.6\% & 89.7\% & 0.451 \\ 
 & 0.750 & 87.1\% & 87.2\% & 0.280 \\
 & 0.850 & 89.3\% & 88.4\% & 0.206 \\ 
 & 0.950 & 85.0\% & 85.7\% & 0.092 \\ 
\hline 
\multirow{9}{*}{MLP} & 0.000 & 100.0\% & 98.1\% & 0.893 \\ 
 & 0.125 & 98.3\% & 97.6\% & 0.399 \\
 & 0.250 & 99.8\% & 97.4\% & 0.068 \\ 
 & 0.375 & 97.4\% & 96.6\% & 0.247 \\ 
 & 0.500 & 99.5\% & 97.1\% & 0.037\\
 & 0.625 & 96.4\% & 95.8\% & 0.166 \\ 
 & 0.750 & 97.5\% & 95.7\% & 0.037 \\ 
 & 0.850 & 92.5\% & 92.8\% & 0.096 \\
 & 0.950 & 91.3\% & 90.0\% & 0.029 \\ 
\hline 
\multirow{8}{*}{ResNet18} & 0.000 & 99.6\% & 99.4\% & 0.896 \\
 & 0.125 & 99.3\% & 99.1\% & 0.892 \\
 & 0.250 & 97.2\% & 97.5\% & 0.427 \\
 & 0.375 & 99.6\% & 99.3\% & 0.586 \\
 & 0.500 & 97.0\% & 97.0\% & 0.357\\ 
  & 0.625 & 97.0\% & 97.0\% & 0.357\\ 
 & 0.750 & 97.0\% & 97.0\% & 0.151 \\ 
 & 0.850 & 95.7\% & 95.4\% & 0.234 \\
 & 0.950 & 93.6\% & 94.0\% & 0.288 
\end{tabular}
\label{table:pareto_main}
\end{table}

\begin{table}[tb]
\centering
\caption{Results for \cref{algo:finetuning_algo} as we explore the Pareto frontier over CIFAR10 and CIFAR100 for ResNet50, single run for \cref{fig:conf_pareto}, \cref{fig:ret_pareto}, and \cref{fig:te_pareto}.}
\vspace{2mm}
\setlength{\tabcolsep}{0.8\tabcolsep}
\begin{tabular}{c|cccc}
\hline
\textbf{Dataset} & \textbf{$\theta$} & \textbf{Retain Acc.} & \textbf{Test Acc.} & \begin{tabular}[c]{@{}c@{}}\textbf{Conf. Dist.}\\ \textbf{(Lower Better)}\end{tabular}  \\ \hline
\multirow{5}{*}{CIFAR10} & 0.000 & 100.0\% &  92.9\%  & 0.817   \\
                        & 0.125 & 99.9\% & 94.1\% &  0.616 \\
                        & 0.250 & 99.9\% & 93.8\% &  0.586  \\
                        & 0.375 & 99.9\% & 94.6\% & 0.501 \\
                        & 0.500 & 91.9\% & 85.7\% &   0.395 \\ 
                        & 0.625 & 90.9\% & 86.9\% & 0.343 \\
                        & 0.750 & 94.7\% & 90.0\% &  0.270 \\ 
                        &0.850 & 56.8\% & 56.0\% & 0.108 \\
                        & 0.950 & 10.7\% &  10.4\% & 0.168  \\  
\hline 
\multirow{5}{*}{CIFAR100} &  0.000 & 92.5\% &  92.2\% &  0.738 \\ 
                        &0.125 &  99.9\% & 77.0\% & 0.353 \\
                             & 0.250 & 99.9\% & 77.5\% &  0.252 \\
                             &0.375 & 99.9\% & 77.0\% & 0.296 \\ 
                              &  0.500 & 99.9\% & 77.1\% & 0.301  \\& 0.625 & 90.9\% & 70.2\% & 0.353\\ 
                            &   0.750  & 91.4\% & 65.5\% & 0.298 \\ & 0.85 & 21.2\% & 20.6\% & 0.138\\ 
                     &    0.950   & 40.8\% &  30.3\% & 0.019 \\ 
\end{tabular}
\label{table:pareto_cifar}
\end{table}

\subsection{Optimization Dynamics}\label{appendix:optimization_dynamics}

\subsubsection{Empirical Results}
Upon using \cref{algo:finetuning_algo}, we observe that logistic regression fails to induce uniformity for more complex benchmarks than MNIST, e.g. KMNIST. Logistic regression has poor test accuracy; we thus conclude that a model must be large enough to generalize well in order to have uniformity induced over it without a large cost to retain accuracy. We discuss mathematical intuition for this in \cref{appendix:early_stopping}.

However, when using \cref{algo:finetuning_algo} on larger models, we observe that we need early stopping. After achieving a good uniformity-utility tradeoff, large models e.g. ResNet50 on CIFAR10 will powerfully increase accuracy at the cost of uniformity. This is undesired behavior when compared to, for example, a ResNet8 trained on CIFAR10, where we initially increase uniformity at the cost of accuracy but slowly regain accuracy without critically damaging uniformity. We illustrate this in \cref{fig:conf_ret_epoch}. Altogether, our method works best for large models with early stopping during finetuning. We characterize this mathematically in \cref{appendix:early_stopping} as well.

\begin{figure}[tb]
    \centering
    \begin{subfigure}[tb]{0.47\linewidth}
        \centering
        \includegraphics[width=\linewidth]{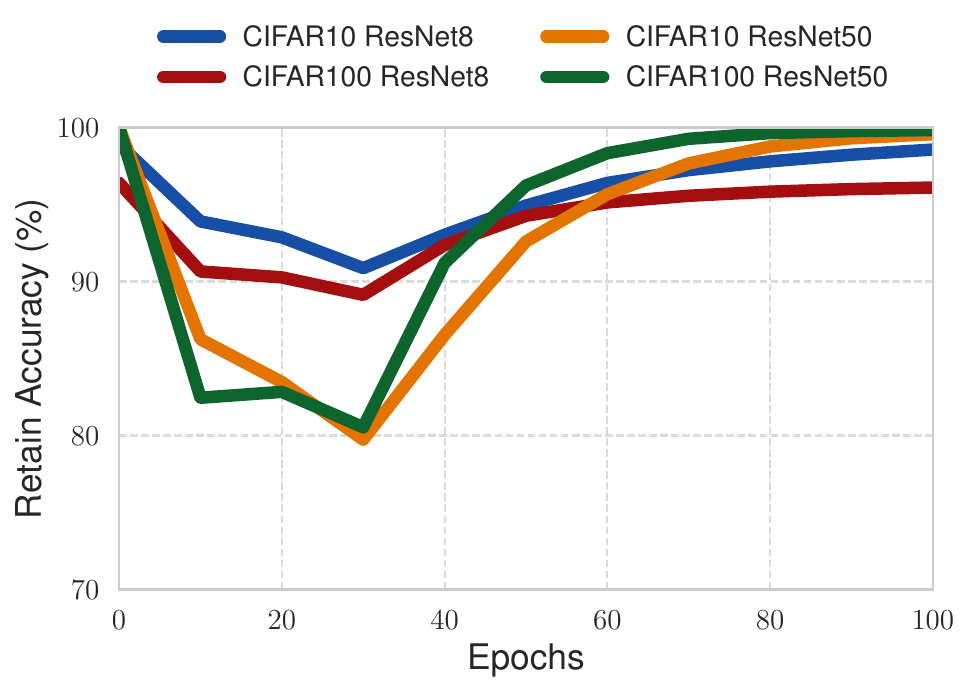}
        \caption{Retain Accuracy vs. Epochs, $\theta$ = 0.75}
        \label{fig:ret_vs_epoch}
    \end{subfigure}
    \hfill
    \begin{subfigure}[tb]{0.47\linewidth}
        \centering
        \includegraphics[width=\linewidth]{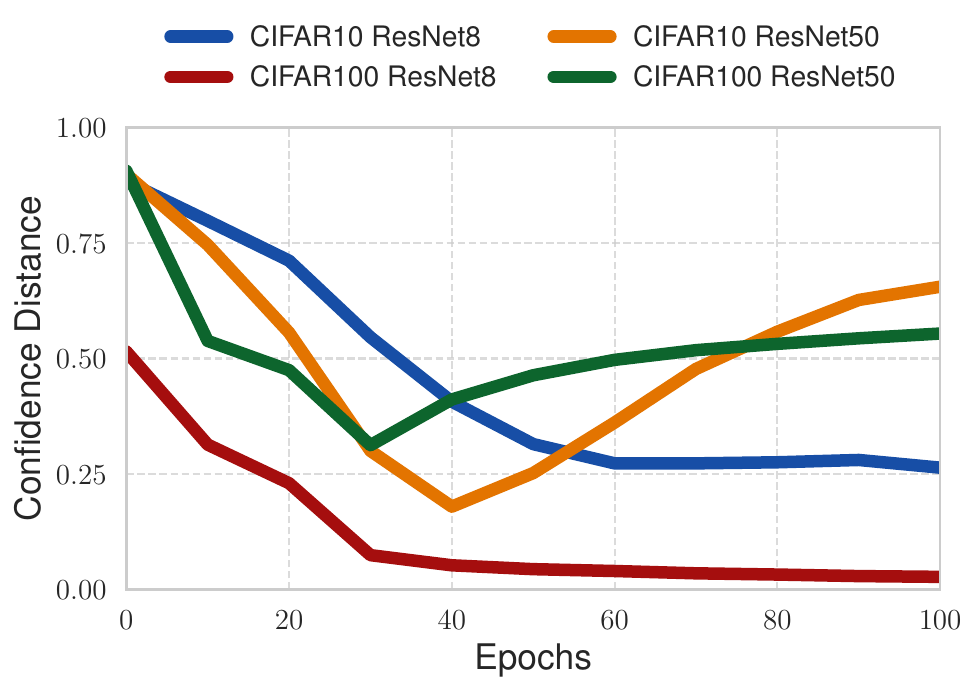}
        \caption{Confidence Distance vs. Epochs, $\theta$ = 0.75}
        \label{fig:conf_vs_epoch}
    \end{subfigure}

    \vspace{1em} %
    
    \caption{For CIFAR10 and CIFAR100 ResNet50, we observe a sharp drop in confidence distance followed by a sharp increase in \cref{fig:conf_vs_epoch}, in line with the drops and increases for retain accuracy in \cref{fig:ret_vs_epoch}. Test accuracy is similar. This highlights the need for early stopping when using \cref{algo:finetuning_algo} for large models, since otherwise one escapes from a good privacy-utility tradeoff. For smaller models, e.g. MNIST MLP, this issue does not persist--we obtain good uniformity after an initial drop in accuracy, but then increase accuracy and decrease confidence distance simultaneously.}
    \vspace{-4mm}
    \label{fig:conf_ret_epoch}
\end{figure}

We provide a plot characterizing how test accuracy for \cref{algo:finetuning_algo} and \cref{algo:hess_exact_algo} applied on CIFAR10 and CIFAR100 for $\theta = 0.75$ changes over 100 epochs in \cref{fig:te_vs_epoch} for ResNet8 and ResNet50. We observe similar behavior to retain accuracy. 

\begin{figure}
    \centering
    \includegraphics[width=0.5\linewidth]{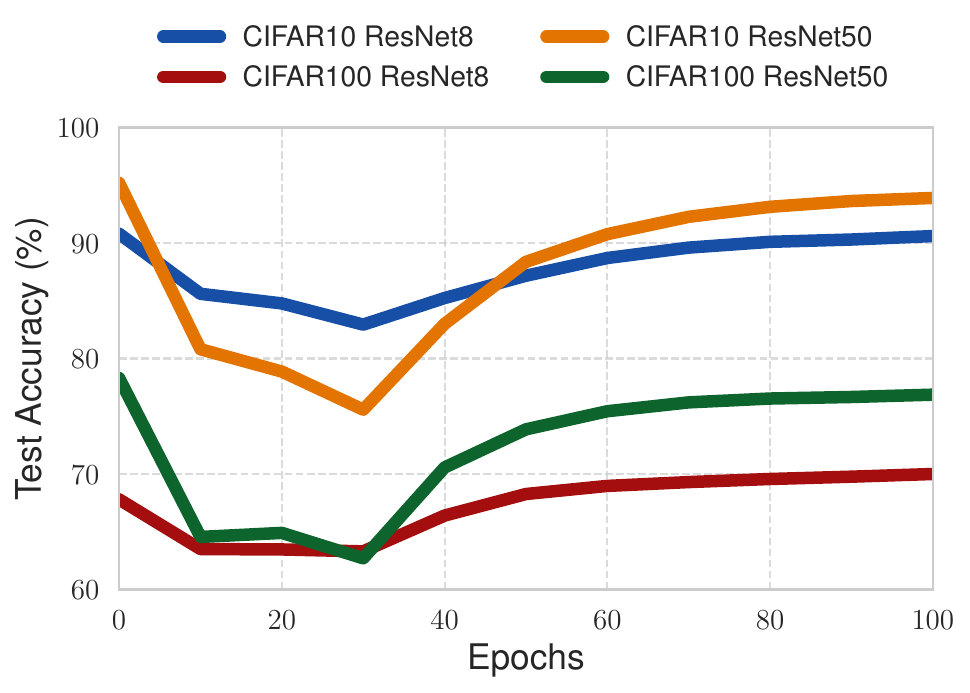}
    \caption{Test Accuracy vs. Epochs, $\theta = 0.75$, MNIST. This has similar behavior to \cref{fig:ret_vs_epoch}.}
    \label{fig:te_vs_epoch}
\end{figure}

Results as used in \cref{fig:ret_vs_epoch}, \cref{fig:conf_vs_epoch}, and \cref{fig:te_vs_epoch} are included in \cref{table:optimization_analysis_main}. 

\begin{table}
\caption{Studying the optimization dynamics of \cref{algo:finetuning_algo}. Used in \cref{fig:conf_ret_epoch}.}
\centering
\begin{tabular}{l|lcccc}
\hline
\textbf{Dataset} & \textbf{Model} & \textbf{Epoch} & \textbf{Retain Acc.} & \textbf{Test Acc.} & \begin{tabular}[c]{@{}c@{}}\textbf{Conf. Dist.}\\ \textbf{(Lower Better)}\end{tabular} \\ \hline
\multirow{11}{*}{CIFAR10} & \multirow{11}{*}{ResNet8} 
& 0 & 98.9\% & 90.8\% & 0.883 \\ 
& & 10 & 88.9\% & 80.4\% & 0.713 \\ 
& & 20 & 90.8\% & 83.1\% & 0.539 \\ 
& & 30 & 92.9\% & 85.3\% & 0.386 \\ 
& & 40 & 95.4\% & 87.3\% & 0.295 \\ 
& & 50 & 96.5\% & 88.9\% & 0.263 \\ 
& & 60 & 97.3\% & 89.9\% & 0.261 \\ 
& & 70 & 97.9\% & 90.0\% & 0.295 \\ 
& & 80 & 98.2\% & 90.4\% & 0.270 \\ 
& & 90 & 98.6\% & 90.5\% & 0.276 \\ 
& & 100 & 98.9\% & 90.9\% & 0.245 \\ 
\hline 
\multirow{11}{*}{CIFAR10} & \multirow{11}{*}{ResNet50} 
& 0   & 100.0\% & 95.2\% & 0.899 \\
& & 10  & 72.5\% & 66.4\% & 0.593 \\
& & 20  & 77.9\% & 75.0\% & 0.176 \\
& & 30  & 88.8\% & 85.3\% & 0.131 \\
& & 40  & 92.9\% & 88.7\% & 0.231 \\
& & 50  & 96.1\% & 91.1\% & 0.394 \\
& & 60  & 98.0\% & 92.5\% & 0.459 \\
& & 70  & 98.9\% & 93.2\% & 0.579 \\
& & 80  & 99.4\% & 93.7\% & 0.636 \\
& & 90  & 99.6\% & 94.0\% & 0.665 \\
& & 100 & 99.7\% & 94.0\% & 0.665 \\
\hline 
\multirow{11}{*}{CIFAR100} & \multirow{11}{*}{ResNet8} 
& 0 & 96.4\% & 67.8\% & 0.515 \\ 
& & 10 & 84.9\% & 59.2\% & 0.112 \\ 
& & 20 & 89.5\% & 63.4\% & 0.063 \\ 
& & 30 & 93.0\% & 67.3\% & 0.048 \\ 
& & 40 & 94.6\% & 68.5\% & 0.046 \\ 
& & 50 & 95.2\% & 69.0\% & 0.038 \\ 
& & 60 & 95.6\% & 69.4\% & 0.036 \\ 
& & 70 & 95.9\% & 69.5\% & 0.031 \\ 
& & 80 & 96.0\% & 69.8\% & 0.030 \\ 
& & 90 & 96.1\% & 70.0\% & 0.026 \\ 
& & 100 & 96.2\% & 70.2\% & 0.026 \\ 
\hline 
\multirow{11}{*}{CIFAR100} & \multirow{11}{*}{ResNet50} 
& 0   & 100.0\% & 78.3\% & 0.906 \\
& & 10  & 64.9\% & 50.8\% & 0.170 \\
& & 20  & 83.6\% & 65.6\% & 0.348 \\
& & 30  & 93.1\% & 71.7\% & 0.419 \\
& & 40  & 96.9\% & 74.4\% & 0.467 \\
& & 50  & 98.7\% & 75.5\% & 0.505 \\
& & 60  & 99.4\% & 76.4\% & 0.519 \\
& & 70  & 99.7\% & 76.7\% & 0.530 \\
& & 80  & 99.8\% & 76.5\% & 0.546 \\
& & 90  & 99.8\% & 76.8\% & 0.555 \\
& & 100 & 99.9\% & 77.3\% & 0.561 \\
\hline
\end{tabular}
\label{table:optimization_analysis_main}
\end{table}

\subsubsection{Intuition for Early Stopping}\label{appendix:early_stopping}

In what follows, we give mathematical justification for the behavior observed in \cref{fig:conf_ret_epoch} and \cref{table:optimization_analysis_main}. 

Firstly, recall that random vectors are nearly orthogonal in high dimensions \citep{vershynin2018high}. In particular, for larger models, the gradients will conflict i.e. point in opposite directions more strongly, since their parameter space is very large. Second, every gradient step of our Alg. 1 is given by $\theta \nabla_{\mathbf{w}} \mathcal{L}_{\mathcal{K}}(\mathbf{w}, \mathcal{D}_f) + (1-\theta) \nabla_{\mathbf{w}} \mathcal{L}_{\mathcal{A}}(\mathbf{w}, \mathcal{D}_r)$. 

In what follows, consider a large model e.g. CIFAR10 ResNet50. 

When we begin finetuning the pretrained model with Alg. 1, $\theta$ is large, $\mathcal{L}_\mathcal{A}$ is small, and $\mathcal{L}_\mathcal{K}$ is large. Thus, $ \nabla_{\mathbf{w}} \mathcal{L}_\mathcal{K}$ significantly dominates $ \nabla_{\mathbf{w}} \mathcal{L}_\mathcal{A}$. For large models, since the gradients are nearly orthogonal, we will move very fast in the direction of the forget gradient. 

As finetuning continues, we reach a point where $\theta$ is large, $\mathcal{L}_\mathcal{K}$ is small, and $\mathcal{L}_\mathcal{A}$ is large. At this point, the $ \nabla_{\mathbf{w}} \mathcal{L}_\mathcal{A}$  begins to dominate, albeit less significantly since $\theta$ is large. For large models, since the gradients are nearly orthogonal, we will move fast in the direction of the retain gradient. However, due to our choice of large $\theta$, the retain gradient will be reduced in magnitude. Thus, we will move more slowly at this stage. 

We must stop shortly after this, otherwise the forget loss will climb back up to a point where the model is no longer reasonably uniform. 

Importantly, since smaller models (e.g. CIFAR10 ResNet8) have smaller parameter spaces, the gradients do not conflict as much. Thus, when we begin finetuning the model, while we do increase in retain loss initially, after the uniform loss is minimized we can minimize the retain loss freely. As such, we can move in a direction that minimizes both the forget and retain gradients, and do not need to stop early. 

However, as noted in the main paper, the model needs to be sufficiently large to achieve strong test accuracy, e.g. logistic regression trained on KMNIST does not work well. We hypothesize that this is because a model with more parameters has many more subspaces where two task losses can coexist in a way that provides a good tradeoff. We leave studying this and the above more formally to future work.

\subsection{Evaluating Test-Time Privacy Attacks}\label{appendix:ttp_attacks}
In what follows, we assume a test-time privacy (TTP) adversary with open-weight model access. We describe our attacks, specifically in \cref{appendix:threat_model}. 

The results for the attacks with \cref{algo:TTP_attack_Gauss}, \cref{algo:TTP_attack_FGSM}, and \cref{algo:TTP_attack_PGD}  are in \cref{table:attack_results_MNIST_KMNIST} and \cref{table:attack_results_SVHN_CIFAR}. We see that our method effectively defends against \cref{algo:TTP_attack_Gauss}, \cref{algo:TTP_attack_FGSM}, and \cref{algo:TTP_attack_PGD} for various choices of $\gamma$ in most scenarios. In our choices of $\gamma$, we follow the adversarial robustness literature, choosing $\gamma$ sufficiently small such that the perturbation is invisible to the naked eye, but sufficiently large such that our attack is effective. However, in some cases, the attacks succeed despite the use of our method. Still, as demonstrated in \cref{table:attack_forget_accuracies}, we find that our algorithm renders the forget accuracy very low in these cases. As such, the adversary cannot be confidently correct--rather, in most cases, they can at best be confidently wrong. In particular, we have significantly better protection from attacks than in the pretrain, retrain, synthetic, or LabelDP cases. We provide visual intuition for our attacks in \cref{fig:attack_fig}. 

\input{sections/attack_figure}

\begin{algorithm}[tb]
\caption{Gaussian Noise Open-Weight Test Time Privacy Attack}

\begin{algorithmic}
\Require Forget set $\mathcal{D}_f$; pretrained model $\bm{w}^* = \mathcal{A}(\mathcal{D})$; adversarial $\gamma$

\State $\mathcal{D}_f^{adv} = []$

\For{$i = 1, ..., |\mathcal{D}_f|$}
\State $\bm{x}_{adv} = \mathcal{D}_f^{(i)} + \beta$, $\beta \sim \mathcal{N}(0, \gamma\bm{I})$
\State $\mathcal{D}_f^{adv, (i)} = \bm{x}_{\text{adv}}$

\EndFor

\Return $\mathcal{D}_f^{\text{adv}}$

\end{algorithmic}
\label{algo:TTP_attack_Gauss}
\end{algorithm}

\begin{algorithm}[tb]
\caption{FGSM-Style Open-Weight Test Time Privacy Attack}

\begin{algorithmic}

\Require Forget set $\mathcal{D}_f$; pretrained model $\bm{w}^* = \mathcal{A}(\mathcal{D})$; adversarial $\gamma$; symmetry breaking $\alpha$

\State $\mathcal{D}_f^{adv} = []$

\For{$i = 1, ..., |\mathcal{D}_f|$}
\State $\bm{x}_0 = \mathcal{D}_f^{(i)} + \beta$, $\beta \sim U([-\alpha\gamma, \alpha\gamma])$
\State $\bm{x}_{\text{adv}} = \mathcal{D}_f^{(i)} +\gamma \text{sign}(\nabla_{\bm{x}}\rho(f_{\bm{w}^*}(\bm{x}))\Big|_{\bm{x} = \bm{x}_0})$
\State $\mathcal{D}_f^{adv, (i)} = \bm{x}_{\text{adv}}$

\EndFor

\Return $\mathcal{D}_f^{\text{adv}}$
\end{algorithmic}
\label{algo:TTP_attack_FGSM}
\end{algorithm}

\begin{algorithm}[tb]
\caption{PGD-Style Open-Weight Test Time Privacy Attack}

\begin{algorithmic}
\Require Forget set $\mathcal{D}_f$; pretrained model $\bm{w}^* = \mathcal{A}(\mathcal{D})$; adversarial $\gamma$; symmetry breaking $\alpha$; step count $N$
\State $\mathcal{D}_f^{adv} = []$

\For{$i = 1, ..., |\mathcal{D}_f|$}
\State $\bm{x}_{adv}^0 = \mathcal{D}_f^{(i)} + \beta$, $\beta \sim U([-\alpha\gamma, \alpha\gamma])$

\For{$j = 1, ..., N$}
\State 
    $\bm{x}_{adv}^{j} = \bm{x}_{adv}^{j-1} + \beta \text{sign}(\nabla_{\bm{x}}\rho(f_{\bm{w}^*}(\bm{x}))\Big|_{\bm{x} = \bm{x}_0})$

\State $\bm{x}_{adv}^j = \prod_{\mathcal{B}_{\gamma}(\mathcal{D}_f^{(i)})}(\bm{x}_{adv})$

\EndFor 

\State $\mathcal{D}_f^{adv, (i)} = \bm{x}_{\text{adv}}^{N}$

\EndFor

\Return $\mathcal{D}_f^{\text{adv}}$
\end{algorithmic}
\label{algo:TTP_attack_PGD}
\end{algorithm}

\begin{table}[tb]
  \centering
  \caption{Confidence distances over the forget set after \cref{algo:TTP_attack_Gauss}, \cref{algo:TTP_attack_FGSM}, and \cref{algo:TTP_attack_PGD} are applied to pretrained models and models finetuned with \cref{algo:finetuning_algo} for MNIST and KMNIST. Lower is better.}
  \vspace{2mm}
  \begin{tabular}{c|cccccc}
    \hline
    \textbf{Dataset} & \textbf{Model} & $\bm{\gamma}$ & \textbf{Method} & \textbf{Attack} &\begin{tabular}[c]{@{}c@{}}\textbf{Prior Conf. Dist.}\\ \textbf{(Lower Better)}\end{tabular} & \begin{tabular}[c]{@{}c@{}}\textbf{Attack Conf. Dist.}\\ \textbf{(Lower Better)}\end{tabular} \\
    \hline
    \multirow{24}{*}{MNIST} & \multirow{12}{*}{MLP} & \multirow{4}{*}{$\tfrac{2}{255}$}
      & \textcolor{gray}{Pretrain}       & \textcolor{gray}{\cref{algo:TTP_attack_FGSM}} & \textcolor{gray}{0.879} & \textcolor{gray}{0.884} \\
      &                     &                & \multirow{3}{*}{\cref{algo:finetuning_algo}}  & \cref{algo:TTP_attack_Gauss} & 0.037 & 0.043   \\
      &                     &                &                                      & \cref{algo:TTP_attack_FGSM} & 0.037 & 0.054 \\
      &                     &                &                                      & \cref{algo:TTP_attack_PGD}  & 0.037 & 0.185 \\
    \cline{3-7}
      &                     & \multirow{4}{*}{$\tfrac{5}{255}$}
      & \textcolor{gray}{Pretrain}       & \textcolor{gray}{\cref{algo:TTP_attack_FGSM}} & \textcolor{gray}{0.879} & \textcolor{gray}{0.888} \\
      &                     &                & \multirow{3}{*}{\cref{algo:finetuning_algo}}  & \cref{algo:TTP_attack_Gauss} & 0.037 & 0.064    \\
      &                     &                &                                      & \cref{algo:TTP_attack_FGSM} & 0.037 & 0.089 \\
      &                     &                &                                      & \cref{algo:TTP_attack_PGD}  & 0.037 & 0.488 \\
    \cline{3-7}
      &                     & \multirow{4}{*}{$\tfrac{8}{255}$}
      & \textcolor{gray}{Pretrain}       & \textcolor{gray}{\cref{algo:TTP_attack_FGSM}} & \textcolor{gray}{0.879} & \textcolor{gray}{0.891} \\
      &                     &                & \multirow{3}{*}{\cref{algo:finetuning_algo}}  & \cref{algo:TTP_attack_Gauss} & 0.037 & 0.091   \\
      &                     &                &                                      & \cref{algo:TTP_attack_FGSM} & 0.037 & 0.125 \\
      &                     &                &                                      & \cref{algo:TTP_attack_PGD}  & 0.037 & 0.632 \\
    \cline{2-7}
    & \multirow{12}{*}{ResNet18} & \multirow{4}{*}{$\tfrac{2}{255}$}
      & \textcolor{gray}{Pretrain}       & \textcolor{gray}{\cref{algo:TTP_attack_FGSM}} & \textcolor{gray}{0.895} & \textcolor{gray}{0.896} \\
      &                          &                & \multirow{3}{*}{\cref{algo:finetuning_algo}}  & \cref{algo:TTP_attack_Gauss} & 0.070 & 0.075   \\
      &                          &                &                                      & \cref{algo:TTP_attack_FGSM} & 0.070 & 0.133 \\
      &                          &                &                                      & \cref{algo:TTP_attack_PGD}  & 0.070 & 0.133\\
    \cline{3-7}
      &                          & \multirow{4}{*}{$\tfrac{5}{255}$}
      & \textcolor{gray}{Pretrain}       & \textcolor{gray}{\cref{algo:TTP_attack_FGSM}} & \textcolor{gray}{0.895} & \textcolor{gray}{0.897} \\
      &                          &                & \multirow{3}{*}{\cref{algo:finetuning_algo}}  & \cref{algo:TTP_attack_Gauss} & 0.070 & 0.088    \\
      &                          &                &                                      & \cref{algo:TTP_attack_FGSM} & 0.070 & 0.164 \\
      &                          &                &                                      & \cref{algo:TTP_attack_PGD}  & 0.070 & 0.350 \\
    \cline{3-7}
      &                          & \multirow{4}{*}{$\tfrac{8}{255}$}
      & \textcolor{gray}{Pretrain}       & \textcolor{gray}{\cref{algo:TTP_attack_FGSM}} & \textcolor{gray}{0.895} & \textcolor{gray}{0.898} \\
      &                          &                & \multirow{3}{*}{\cref{algo:finetuning_algo}}  & \cref{algo:TTP_attack_Gauss} & 0.070 & 0.116 \\
      &                          &                &                                      & \cref{algo:TTP_attack_FGSM} & 0.070 & 0.248 \\
      &                          &                &                                      & \cref{algo:TTP_attack_PGD}  & 0.070 & 0.638 \\
    \hline
    \multirow{12}{*}{KMNIST} & \multirow{12}{*}{ResNet18} & \multirow{4}{*}{$\tfrac{2}{255}$}
      & \textcolor{gray}{Pretrain}       & \textcolor{gray}{\cref{algo:TTP_attack_FGSM}} & 0.858 & 0.865 \\
      &                          &                & \multirow{3}{*}{\cref{algo:finetuning_algo}}  & \cref{algo:TTP_attack_Gauss} & 0.257 & 0.258    \\
      &                          &                &                                      & \cref{algo:TTP_attack_FGSM} & 0.257 & 0.302 \\
      &                          &                &                                      & \cref{algo:TTP_attack_PGD}  & 0.257 & 0.317 \\
    \cline{3-7}
      &                          & \multirow{4}{*}{$\tfrac{5}{255}$}
      & \textcolor{gray}{Pretrain}       & \textcolor{gray}{\cref{algo:TTP_attack_FGSM}} & 0.858 & 0.872 \\
      &                          &                & \multirow{3}{*}{\cref{algo:finetuning_algo}}  & \cref{algo:TTP_attack_Gauss} & 0.257 & 0.259    \\
      &                          &                &                                      & \cref{algo:TTP_attack_FGSM} & 0.257 & 0.370 \\
      &                          &                &                                      & \cref{algo:TTP_attack_PGD}  & 0.257 & 0.420 \\
    \cline{3-7}
      &                          & \multirow{4}{*}{$\tfrac{8}{255}$}
      & \textcolor{gray}{Pretrain}       & \textcolor{gray}{\cref{algo:TTP_attack_FGSM}} & 0.858 & 0.878 \\
      &                          &                & \multirow{3}{*}{\cref{algo:finetuning_algo}}  & \cref{algo:TTP_attack_Gauss} & 0.257 & 0.259   \\
      &                          &                &                                      & \cref{algo:TTP_attack_FGSM} & 0.257 & 0.433 \\
      &                          &                &                                      & \cref{algo:TTP_attack_PGD}  & 0.257 & 0.520 \\
    \hline
  \end{tabular}
  \label{table:attack_results_MNIST_KMNIST}
\end{table}

\begin{table}[tb]
  \centering
  \caption{Confidence distances over the forget set after \cref{algo:TTP_attack_Gauss}, \cref{algo:TTP_attack_FGSM}, and \cref{algo:TTP_attack_PGD} are applied to pretrained models and models finetuned with \cref{algo:finetuning_algo} for SVHN, CIFAR10, and CIFAR100. Lower is better.}
  \vspace{2mm}
  \begin{tabular}{c|cccccc}
    \hline
    \textbf{Dataset}  & \textbf{Model}      & $\bm{\gamma}$         & \textbf{Method}                & \textbf{Attack}                        & \begin{tabular}[c]{@{}c@{}}\textbf{Prior Conf. Dist.}\\ \textbf{(Lower Better)}\end{tabular} & \begin{tabular}[c]{@{}c@{}}\textbf{Attack Conf. Dist.}\\ \textbf{(Lower Better)}\end{tabular} \\
    \hline
    \multirow{8}{*}{SVHN}  & \multirow{8}{*}{ResNet50} & \multirow{4}{*}{$\tfrac{1}{255}$} & \textcolor{gray}{Pretrain}       & \textcolor{gray}{\cref{algo:TTP_attack_FGSM}} & \textcolor{gray}{0.972} & \textcolor{gray}{0.886} \\
                          &                           &                & \multirow{3}{*}{\cref{algo:finetuning_algo}}  & \cref{algo:TTP_attack_Gauss} & 0.289 & 0.519    \\
                          &                           &                &                                      & \cref{algo:TTP_attack_FGSM} & 0.289 & 0.571 \\
                          &                           &                &                                      & \cref{algo:TTP_attack_PGD}  & 0.289 & 0.582 \\
    \cline{3-7}
                          &                           & \multirow{4}{*}{$\tfrac{2}{255}$} & \textcolor{gray}{Pretrain}       & \textcolor{gray}{\cref{algo:TTP_attack_FGSM}} & \textcolor{gray}{0.972} & \textcolor{gray}{0.904} \\
                          &                           &                & \multirow{3}{*}{\cref{algo:finetuning_algo}}  & \cref{algo:TTP_attack_Gauss} & 0.289 & 0.519    \\
                          &                           &                &                                      & \cref{algo:TTP_attack_FGSM} & 0.289 & 0.613 \\
                          &                           &                &                                      & \cref{algo:TTP_attack_PGD}  & 0.289 & 0.640 \\
    \hline
    \multirow{16}{*}{CIFAR10} & \multirow{8}{*}{ResNet18}  & \multirow{4}{*}{$\tfrac{1}{255}$} & \textcolor{gray}{Pretrain}       & \textcolor{gray}{\cref{algo:TTP_attack_FGSM}} & 0.898 & 0.898 \\
                              &                             &                & \multirow{3}{*}{\cref{algo:finetuning_algo}}  & \cref{algo:TTP_attack_Gauss} & 0.377 & 0.300   \\
                              &                             &                &                                      & \cref{algo:TTP_attack_FGSM} & 0.377 & 0.353 \\
                              &                             &                &                                      & \cref{algo:TTP_attack_PGD}  & 0.377 & 0.356 \\
    \cline{3-7}
                              &                             & \multirow{4}{*}{$\tfrac{2}{255}$} & \textcolor{gray}{Pretrain}       & \textcolor{gray}{\cref{algo:TTP_attack_FGSM}} & 0.898 & 0.822 \\
                              &                             &                & \multirow{3}{*}{\cref{algo:finetuning_algo}}  & \cref{algo:TTP_attack_Gauss} & 0.377 & 0.301    \\
                              &                             &                &                                      & \cref{algo:TTP_attack_FGSM} & 0.377  & 0.397 \\
                              &                             &                &                                      & \cref{algo:TTP_attack_PGD}  & 0.377  & 0.408 \\
    \cline{2-7}
                              & \multirow{8}{*}{ResNet50} & \multirow{4}{*}{$\tfrac{1}{255}$} & \textcolor{gray}{Pretrain}       & \textcolor{gray}{\cref{algo:TTP_attack_FGSM}} & \textcolor{gray}{0.900} & \textcolor{gray}{0.806} \\
                              &                             &                & \multirow{3}{*}{\cref{algo:finetuning_algo}}  & \cref{algo:TTP_attack_Gauss} & 0.270 & 0.336     \\
                              &                             &                &                                      & \cref{algo:TTP_attack_FGSM} & 0.270 & 0.374 \\
                              &                             &                &                                      & \cref{algo:TTP_attack_PGD}  & 0.270 & 0.378 \\
    \cline{3-7}
                              &                             & \multirow{4}{*}{$\tfrac{2}{255}$} & \textcolor{gray}{Pretrain}       & \textcolor{gray}{\cref{algo:TTP_attack_FGSM}} & \textcolor{gray}{0.900} & \textcolor{gray}{0.827} \\
                              &                             &                & \multirow{3}{*}{\cref{algo:finetuning_algo}}  & \cref{algo:TTP_attack_Gauss} & 0.270 & 0.336    \\
                              &                             &                &                                      & \cref{algo:TTP_attack_FGSM} &0.270 & 0.407 \\
                              &                             &                &                                      & \cref{algo:TTP_attack_PGD}  & 0.270 & 0.425 \\
    \hline
    \multirow{8}{*}{CIFAR100} & \multirow{8}{*}{ResNet50} & \multirow{4}{*}{$\tfrac{1}{255}$} & \textcolor{gray}{Pretrain}       & \textcolor{gray}{\cref{algo:TTP_attack_FGSM}} & \textcolor{gray}{0.906} & \textcolor{gray}{0.587} \\
                              &                          &                & \multirow{3}{*}{\cref{algo:finetuning_algo}}  & \cref{algo:TTP_attack_Gauss} & 0.298 & 0.275    \\
                              &                          &                &                                      & \cref{algo:TTP_attack_FGSM} & 0.298 & 0.360 \\
                              &                          &                &                                      & \cref{algo:TTP_attack_PGD}  & 0.298 & 0.373 \\
    \cline{3-7}
                              &                          & \multirow{4}{*}{$\tfrac{2}{255}$} & \textcolor{gray}{Pretrain}       & \textcolor{gray}{\cref{algo:TTP_attack_FGSM}} & \textcolor{gray}{0.906} & \textcolor{gray}{0.652} \\
                              &                          &                & \multirow{3}{*}{\cref{algo:finetuning_algo}}  & \cref{algo:TTP_attack_Gauss} & 0.298 & 0.276   \\
                              &                          &                &                                      & \cref{algo:TTP_attack_FGSM} & 0.298 & 0.421 \\
                              &                          &                &                                      & \cref{algo:TTP_attack_PGD}  & 0.298 & 0.468 \\
    \hline
  \end{tabular}
  \label{table:attack_results_SVHN_CIFAR}
\end{table}

\begin{table}[tb]
\centering
\caption{Accuracies over the forget set after \cref{algo:TTP_attack_FGSM} and \cref{algo:TTP_attack_PGD} are applied to pretrained models and models finetuned with \cref{algo:finetuning_algo} for SVHN, CIFAR10, and CIFAR100. We see that \cref{algo:finetuning_algo} significantly lowers forget set accuracy.}
\vspace{2mm}
\begin{tabular}{c|cccccc}
\hline
\textbf{Dataset} & \textbf{Model} & \textbf{$\gamma$} & \textbf{Method} & \textbf{Attack} & \begin{tabular}[c]{@{}c@{}}\textbf{Prior. Forget Acc.}\\ \textbf{(Lower Better)}\end{tabular} & \begin{tabular}[c]{@{}c@{}}\textbf{Atk. Forget Acc.}\\ \textbf{(Lower Better)}\end{tabular} \\
\hline
\multirow{6}{*}{MNIST} & \multirow{4}{*}{MLP}    & \multirow{2}{*}{$\tfrac{5}{255}$} & \textcolor{gray}{Pretrain} & \textcolor{gray}{\cref{algo:TTP_attack_PGD}} & \textcolor{gray}{97.0\%} & \textcolor{gray}{96.0\%} \\
                        &                         &                                   &       \cref{algo:finetuning_algo}            & \cref{algo:TTP_attack_PGD}  & 71.0\% & 43.0\% \\
                        \cline{3-7}
                        &                         & \multirow{2}{*}{$\tfrac{8}{255}$} & \textcolor{gray}{Pretrain} & \textcolor{gray}{\cref{algo:TTP_attack_PGD}} & \textcolor{gray}{98.3\%} & \textcolor{gray}{96.0\%} \\
                        &                         &                                   &     \cref{algo:finetuning_algo}            & \cref{algo:TTP_attack_PGD}  & 71.0\% & 41.0\% \\
\cline{2-7}
                        & \multirow{2}{*}{ResNet18}    & \multirow{2}{*}{$\tfrac{8}{255}$} & \textcolor{gray}{Pretrain} & \textcolor{gray}{\cref{algo:TTP_attack_PGD}} & \textcolor{gray}{98.9\%} & \textcolor{gray}{100.0\%} \\
                        &                         &                                   &     \cref{algo:finetuning_algo}            & \cref{algo:TTP_attack_PGD}  & 28.0\% & 54.0\% \\
\hline 
\multirow{2}{*}{KMNIST} & \multirow{2}{*}{ResNet18}    & \multirow{2}{*}{$\tfrac{8}{255}$} & \textcolor{gray}{Pretrain} & \textcolor{gray}{\cref{algo:TTP_attack_PGD}} & 96.0\% & 96.0\% \\
                        &                         &                                   &     \cref{algo:finetuning_algo}            & \cref{algo:TTP_attack_PGD}  & 50.0\% & 55.0\% \\
\hline 
    \multirow{6}{*}{SVHN} 
    & \multirow{6}{*}{ResNet50}
        & \multirow{4}{*}{$\tfrac{1}{255}$} & 
        \textcolor{gray}{Pretrain} & \textcolor{gray}{\cref{algo:TTP_attack_FGSM}}  & \textcolor{gray}{97.0\%}  & \textcolor{gray}{66.0\%} \\
    &                            &                & \multirow{3}{*}{\cref{algo:finetuning_algo}}
                                           & \cref{algo:TTP_attack_Gauss}      & 40.0\% & 53.0\%    \\
    &                            &                &                                     
                                           & \cref{algo:TTP_attack_FGSM}       & 40.0\% & 52.0\% \\
    &                            &                &                                     
                                           & \cref{algo:TTP_attack_PGD}        & 40.0\% & 52.0\% \\
    \cline{3-7}
    &                            & \multirow{4}{*}{$\tfrac{2}{255}$}
            & \textcolor{gray}{Pretrain}
                                           & \textcolor{gray}{\cref{algo:TTP_attack_FGSM}}  & \textcolor{gray}{97.0\%}  & \textcolor{gray}{66.0\%} \\
    &                            &                & \multirow{3}{*}{\cref{algo:finetuning_algo}}
                                           & \cref{algo:TTP_attack_Gauss}      & 40.0\% & 52.0\%   \\
    &                            &                &                                     
                                           & \cref{algo:TTP_attack_FGSM}       & 40.0\% & 53.0\% \\
    &                            &                &                                     
                                           & \cref{algo:TTP_attack_PGD}        & 40.0\% & 53.0\% \\
    \hline
    \multirow{12}{*}{CIFAR10} 
    & \multirow{6}{*}{ResNet18}
        & \multirow{4}{*}{$\tfrac{1}{255}$}
            & \textcolor{gray}{Pretrain}
                                           & \textcolor{gray}{\cref{algo:TTP_attack_FGSM}}  & \textcolor{gray}{100.0\%}  & \textcolor{gray}{61.0\%} \\
    &                            &                & \multirow{3}{*}{\cref{algo:finetuning_algo}}
                                           & \cref{algo:TTP_attack_Gauss}      & 60.0\% & 35.0\%  \\
    &                            &                &                                     
                                           & \cref{algo:TTP_attack_FGSM}       & 60.0\%  & 34.0\% \\
    &                            &                &                                     
                                           & \cref{algo:TTP_attack_PGD}        & 60.0\%  & 34.0\% \\
    \cline{3-7}
    &                            & \multirow{4}{*}{$\tfrac{2}{255}$}
            & \textcolor{gray}{Pretrain}
                                           & \textcolor{gray}{\cref{algo:TTP_attack_FGSM}}  & \textcolor{gray}{100.0\%}  & \textcolor{gray}{61.0\%} \\
    &                            &                & \multirow{3}{*}{\cref{algo:finetuning_algo}}
                                           & \cref{algo:TTP_attack_Gauss}      & 60.0\%  & 35.0\%   \\
    &                            &                &                                     
                                           & \cref{algo:TTP_attack_FGSM}       & 60.0\%  & 33.0\% \\
    &                            &                &                                     
                                           & \cref{algo:TTP_attack_PGD}        & 60.0\%  & 32.0\% \\
    \cline{2-7}
    & \multirow{6}{*}{ResNet50}
        & \multirow{4}{*}{$\tfrac{1}{255}$}
            & \textcolor{gray}{Pretrain}
                                           & \textcolor{gray}{\cref{algo:TTP_attack_FGSM}}  & \textcolor{gray}{100.0\%}  & \textcolor{gray}{56.0\%} \\
    &                            &                & \multirow{3}{*}{\cref{algo:finetuning_algo}}
                                           & \cref{algo:TTP_attack_Gauss}      &   55.0\% &   33.0\% \\
    &                            &                &                                     
                                           & \cref{algo:TTP_attack_FGSM}       & 55.0\% & 30.0\% \\
    &                            &                &                                     
                                           & \cref{algo:TTP_attack_PGD}        & 55.0\% & 30.0\% \\
    \cline{3-7}
    &                            & \multirow{4}{*}{$\tfrac{2}{255}$}
    & \textcolor{gray}{Pretrain} & \textcolor{gray}{\cref{algo:TTP_attack_FGSM}}  & \textcolor{gray}{100.0\%}   & \textcolor{gray}{56.0\%} \\
    &                            &                & \multirow{3}{*}{\cref{algo:finetuning_algo}}
                                           & \cref{algo:TTP_attack_Gauss}      & 55.0\% & 30.0\%   \\
    &                            &                &                                     
                                           & \cref{algo:TTP_attack_FGSM}       & 55.0\% & 30.0\% \\
    &                            &                &                                     
                                           & \cref{algo:TTP_attack_PGD}        & 55.0\% & 32.0\% \\
    \hline
    \multirow{6}{*}{CIFAR100} 
    & \multirow{6}{*}{ResNet50}
        & \multirow{4}{*}{$\tfrac{1}{255}$}
            & \textcolor{gray}{Pretrain}
                                           & \textcolor{gray}{\cref{algo:TTP_attack_FGSM}}  & \textcolor{gray}{100.0\%} & \textcolor{gray}{32.0\%} \\
    &                            &                & \multirow{3}{*}{\cref{algo:finetuning_algo}}
                                           & \cref{algo:TTP_attack_Gauss}      & 48.0\% & 11.0\%   \\
    &                            &                &                                     
                                           & \cref{algo:TTP_attack_FGSM}       & 48.0\% & 10.0\%  \\
    &                            &                &                                     
                                           & \cref{algo:TTP_attack_PGD}        & 48.0\% & 11.0\% \\
    \cline{3-7}
    &                            & \multirow{4}{*}{$\tfrac{2}{255}$}
    & \textcolor{gray}{Pretrain} & \textcolor{gray}{\cref{algo:TTP_attack_FGSM}}  & \textcolor{gray}{100.0\%} & \textcolor{gray}{32.0\%} \\
    &                            &                & \multirow{3}{*}{\cref{algo:finetuning_algo}}
                                           & \cref{algo:TTP_attack_Gauss}      & 48.0\% & 11.0\%  \\
    &                            &                &                                     
                                           & \cref{algo:TTP_attack_FGSM}       & 48.0\% & 10.0\% \\
    &                            &                &                                     
                                           & \cref{algo:TTP_attack_PGD}        & 48.0\% & 11.0\% \\
    \hline
  \end{tabular}
  \label{table:attack_forget_accuracies}
\end{table}

\subsection{Robustness of \cref{algo:finetuning_algo} Classifier on Neighboring Test Instances}\label{appendix:nearest_neighbors}
To illustrate what happens for the finetuned classifier on neighboring test instances, we run an experiment evaluating accuracy and average confidence distance on test instances which are nearest neighbors to forget instances. Specifically, for each forget instance $\mathbf{x} \in \mathcal{D}_f$, we found the nearest neighbor of $\mathbf{x}$ in the test set. We evaluated this nearest neighbor in pixel space with respect to the $\ell_2$ distance.

Results are in \cref{table:nearest_neighbor}. We notice that the classifier works as intended. That is, we obtain high accuracy as well as high confidence distance on the test set, including for nearby instances.  While we do observe a drop for CIFAR100 ResNet50, we also observe that the confidence distance is still much higher than the 0.298 confidence distance on the forget set. 

\begin{table}[tb]
  \centering
  \caption{Accuracies and confidence distances for test instances which are nearest neighbors (with respect to $\ell_2$ distance) of forget set instances. We observe that models continue to confidently and correctly classify nearby test instances after finetuning with \cref{algo:finetuning_algo}.}
  \vspace{2mm}
  \begin{tabular}{c|cccc}
    \hline 
    \textbf{Dataset}   & \textbf{Model}   & \textbf{Method} & \textbf{Acc.} & \begin{tabular}[c]{@{}c@{}}\textbf{Conf. Dist.}\\ \textbf{(Higher Better)}\end{tabular} \\
    \hline 
    \multirow{4}{*}{MNIST} & \multirow{2}{*}{MLP} & \textcolor{gray}{Pretrain} & \textcolor{gray}{100.0\%} & \textcolor{gray}{0.894} \\ 
                            &        & \cref{algo:finetuning_algo} & 93.0\% &  0.754 \\ 
    \cline{2-5}
                            & \multirow{2}{*}{ResNet18} & \textcolor{gray}{Pretrain} & \textcolor{gray}{100.0\%} & \textcolor{gray}{0.896} \\ 
                            &        & \cref{algo:finetuning_algo} & 100.0\% & 0.875\\ 
    \hline 
     \multirow{2}{*}{KMNIST} & \multirow{2}{*}{ResNet18} & \textcolor{gray}{Pretrain} & 99.0\% & 0.871 \\ 
                            &        & \cref{algo:finetuning_algo} & 99.0\% & 0.850 \\ 
    \hline 
     \multirow{2}{*}{SVHN} & \multirow{2}{*}{ResNet50} & \textcolor{gray}{Pretrain} & \textcolor{gray}{95.0\%} & \textcolor{gray}{0.982} \\ 
                            &        & \cref{algo:finetuning_algo} & 95.0\% & 0.939 \\ 
    \hline 
     \multirow{4}{*}{CIFAR10} & \multirow{2}{*}{ResNet18} & \textcolor{gray}{Pretrain} & 98.0\% & 0.876 \\ 
                            &                           & \cref{algo:finetuning_algo}& 85.0\%  & 0.538 \\ 
                            \cline{2-5}
                            & \multirow{2}{*}{ResNet50} & \textcolor{gray}{Pretrain} & 96.0\% & 0.884 \\ 
                            &        & \cref{algo:finetuning_algo}& 90.0\% & 0.700 \\  
    \hline 
    \multirow{2}{*}{CIFAR100} & \multirow{2}{*}{ResNet50} & \textcolor{gray}{Pretrain} & 78.0\%  & 0.778  \\ 
                            &        & \cref{algo:finetuning_algo}& 66.0\% & 0.484 \\ 
    
    \hline 
  \end{tabular}
  \label{table:nearest_neighbor}
\end{table}

\subsection{Ensuring Test-Time Privacy for Test Instances}\label{appendix:test_finetuning}

 In our paper, we focus on training data examples because this is the basis for scenarios addressed by the GDPR, HIPAA, etc. Providing the same guarantee for non-training (test) data is an equally important problem. However, the proposed method can be extended to cover this new case without loss of generality. One can just finetune on test instances highlighted to be corrupted and then run \cref{algo:finetuning_algo}. In \cref{table:test_finetuning}, we find that finetuning with test instances yields similar performance to using our algorithm over just the training instances. 
 
\begin{table}[tb]

  \centering

  \caption{Finetuning on test instances and running \cref{algo:finetuning_algo}. Here, ``pretrain'' denotes the initially pretrained model (no additional test instances), while other rows correspond to finetuning on the specified number of test instances.}

  \vspace{2mm}

  \begin{tabular}{c|cccccc}

   \hline

    \textbf{Dataset}   & \textbf{Model}  & \textbf{\% Forget Size} & \textbf{Method} & \textbf{Retain Acc.} & \textbf{Test Acc.} & \begin{tabular}[c]{@{}c@{}}\textbf{Conf. Dist.}\\ \textbf{(Lower Better)}\end{tabular} \\

    \hline

  \multirow{10}{*}{MNIST} & \multirow{5}{*}{MLP}

    & \multirow{2}{*}{0}   & \textcolor{gray}{Pretrain}   & \textcolor{gray}{98.3\%} & \textcolor{gray}{97.0\%} & \textcolor{gray}{0.879} \\

    &                     &                  &  \cref{algo:finetuning_algo} &        97.5\%               &                  95.7\%      &            0.037         \\

    \cline{3-7}

    &                       & \multirow{2}{*}{$1/5$}  & \textcolor{gray}{Pretrain}             & \textcolor{gray}{99.1\%} & \textcolor{gray}{97.3\%} & \textcolor{gray}{0.892} \\

    &                       &                     & \cref{algo:finetuning_algo} &       96.4\%              &      93.9\%               &          0.147           \\

    \cline{3-7}

    &                       & \multirow{2}{*}{$1/2$}  & \textcolor{gray}{Pretrain}             & \textcolor{gray}{99.0\%} & \textcolor{gray}{97.1\%} & \textcolor{gray}{0.896} \\

    &                       &                     & \cref{algo:finetuning_algo} &             96.2\%        &      93.4\%               &           0.157          \\

    \cline{2-7}

    & \multirow{5}{*}{ResNet18}

    & \multirow{2}{*}{0}   & \textcolor{gray}{Pretrain}             & \textcolor{gray}{99.2\%} & \textcolor{gray}{98.9\%} & \textcolor{gray}{0.879} \\

    &                       &                     & \cref{algo:finetuning_algo} &     99.6\%                &     99.1\%                &            0.070         \\

    \cline{3-7}

    &                       & \multirow{2}{*}{$1/5$}  & \textcolor{gray}{Pretrain}             & \textcolor{gray}{98.9\%} & \textcolor{gray}{98.5\%} & \textcolor{gray}{0.884} \\

    &                       &                     & \cref{algo:finetuning_algo} &    99.6\%                 &                 98.9\%    &             0.181        \\

    \cline{3-7}

    &                       & \multirow{2}{*}{$1/2$}  & \textcolor{gray}{Pretrain}             & \textcolor{gray}{99.4\%} & \textcolor{gray}{100.0\%} & \textcolor{gray}{0.897} \\

    &                       &                     & \cref{algo:finetuning_algo} &       99.5\%              &  98.7\%                   &  0.243                   \\

    \hline

  \multirow{6}{*}{KMNIST} & \multirow{6}{*}{ResNet18}

    & \multirow{2}{*}{0}   & \textcolor{gray}{Pretrain}             & \textcolor{gray}{98.2\%} & \textcolor{gray}{92.1\%} & \textcolor{gray}{0.880} \\

    &                       &                     & \cref{algo:finetuning_algo} &           99.1\%          &      94.7\%               &              0.257       \\

    \cline{3-7}

    &                       & \multirow{2}{*}{$1/5$}  & \textcolor{gray}{Pretrain}             & \textcolor{gray}{100.0\%} & \textcolor{gray}{97.5\%} & \textcolor{gray}{0.900} \\

    &                       &                     & \cref{algo:finetuning_algo} &     \textcolor{gray}{99.7\%}                &  \textcolor{gray}{96.5\%}               &                \textcolor{gray}{0.301}    \\

    \cline{3-7}
        &                       & \multirow{2}{*}{$1/2$}  & \textcolor{gray}{Pretrain}             & \textcolor{gray}{100.0\%} & \textcolor{gray}{97.7\%} & \textcolor{gray}{0.895} \\

    &                       &                     & \cref{algo:finetuning_algo} &       99.5\%          &  96.9\%             &              0.233     \\

    \hline

  \multirow{6}{*}{SVHN}   & \multirow{3}{*}{ResNet50}

    & \multirow{2}{*}{0}   & \textcolor{gray}{Pretrain}             & \textcolor{gray}{99.5\%} & \textcolor{gray}{94.4\%} & \textcolor{gray}{0.971} \\

    &                       &                     & \cref{algo:finetuning_algo} &          99.0\%           &  94.4\%                    &     0.280                \\

    \cline{3-7}

    &                       & \multirow{2}{*}{$1/5$}  & \textcolor{gray}{Pretrain}             & \textcolor{gray}{99.7\%} & \textcolor{gray}{94.6\%} & \textcolor{gray}{0.974} \\

    &                       &                     & \cref{algo:finetuning_algo} &       99.2\%              &     94.4\%                &       0.184              \\

    \cline{3-7}

    &                       & \multirow{2}{*}{$1/2$}  & \textcolor{gray}{Pretrain}             & \textcolor{gray}{99.5\%} & \textcolor{gray}{94.5\%} & \textcolor{gray}{0.986} \\

    &                       &                     & \cref{algo:finetuning_algo} &          99.8\%           &    95.1\%                 &             0.391        \\

    \hline

  \multirow{10}{*}{CIFAR10} & \multirow{5}{*}{ResNet18}

    & \multirow{2}{*}{0}   & \textcolor{gray}{Pretrain}             & \textcolor{gray}{100.0\%} & \textcolor{gray}{95.3\%} & \textcolor{gray}{0.898} \\

    &                       &                     & \cref{algo:finetuning_algo} &            89.6\%         &       83.1\%              &         0.377            \\

    \cline{3-7}

  	&                       & \multirow{2}{*}{$1/5$}  & \textcolor{gray}{Pretrain}             & \textcolor{gray}{100.0\%} &  \textcolor{gray}{98.4\%} & \textcolor{gray}{0.894} \\

  	&                       &                     & \cref{algo:finetuning_algo} &             89.4\%           &      85.6\%               &        0.358           \\

  	\cline{3-7}

  	&                       & \multirow{2}{*}{$1/2$}  & \textcolor{gray}{Pretrain}             & \textcolor{gray}{100.0\%}  & \textcolor{gray}{93.5\%}  & \textcolor{gray}{0.887}  \\

  	&                       &                     & \cref{algo:finetuning_algo} &            87.9\%        &       83.7\%             &         0.494           \\

  	\cline{2-7}

  	& \multirow{5}{*}{ResNet50}

  	& \multirow{2}{*}{0}   & \textcolor{gray}{Pretrain}             & \textcolor{gray}{100.0\%} & \textcolor{gray}{95.2\%} & \textcolor{gray}{0.900} \\

  	&                       &                     & \cref{algo:finetuning_algo} &     94.7\%               &     90.0\%               &               0.270    \\

  	\cline{3-7}

  	&                       & \multirow{2}{*}{$1/5$}  & \textcolor{gray}{Pretrain}             & \textcolor{gray}{99.9\%} & \textcolor{gray}{93.4\%} & \textcolor{gray}{0.896} \\

  	&                       &                     & \cref{algo:finetuning_algo} &            91.6\%         &       87.1\%              &         0.392           \\

  	\cline{3-7}

  	&                       & \multirow{2}{*}{$1/2$}  & \textcolor{gray}{Pretrain}             & \textcolor{gray}{99.8\%} & \textcolor{gray}{93.6\%} &  \textcolor{gray}{0.894} \\

  	&                       &                     & \cref{algo:finetuning_algo} &            89.1\%        &      86.7\%             &         0.390           \\

  	\hline

  \multirow{6}{*}{CIFAR100} & \multirow{3}{*}{ResNet50}

  	& \multirow{2}{*}{0}   & \textcolor{gray}{Pretrain}             & \textcolor{gray}{100.0\%} & \textcolor{gray}{78.3\%} & \textcolor{gray}{0.902} \\

  	&                       &                     & \cref{algo:finetuning_algo} &      91.4\%  &        65.5\%         &    0.298    \\

  	\cline{3-7}

  	&                       & \multirow{2}{*}{$1/5$}  & \textcolor{gray}{Pretrain}             & \textcolor{gray}{100.0\%} & \textcolor{gray}{79.0\%} & \textcolor{gray}{0.880}  \\

  	&                       &                     & \cref{algo:finetuning_algo}   &     93.5\% &      69.1\%      &     0.021             \\

  	\cline{3-7}

  	&                       & \multirow{2}{*}{$1/2$}  & \textcolor{gray}{Pretrain}             & \textcolor{gray}{99.9\%} & \textcolor{gray}{78.2\%} & \textcolor{gray}{0.860}\\

  	&                       &                     & \cref{algo:finetuning_algo} &       97.9\%    &    72.8\%     &              0.221    \\

  	\hline

  \end{tabular}

  \label{table:test_finetuning}

\end{table}

\subsection{Ablation Study on Forget Set Size}\label{section:forget_set_size}

\begin{figure}[tb]
    \centering
    \begin{subfigure}[tb]{0.47\linewidth}
        \centering
        \includegraphics[width=\linewidth]{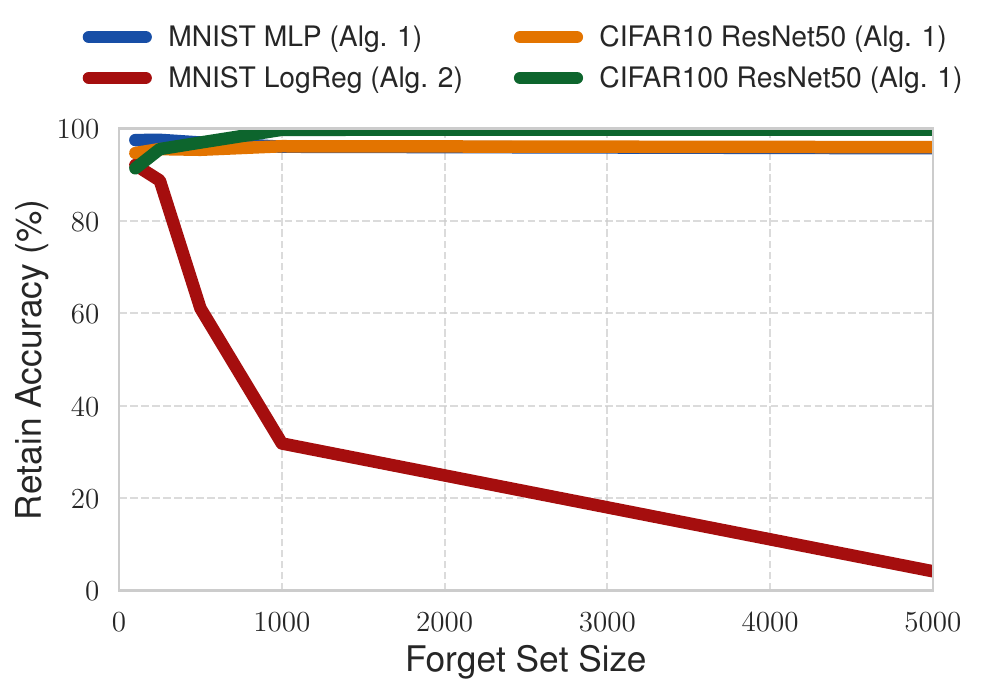}
        \caption{Retain Accuracy vs. Epochs, $\theta$ = 0.75}
        \label{fig:ret_v_size}
    \end{subfigure}
    \hfill
    \begin{subfigure}[tb]{0.47\linewidth}
        \centering
        \includegraphics[width=\linewidth]{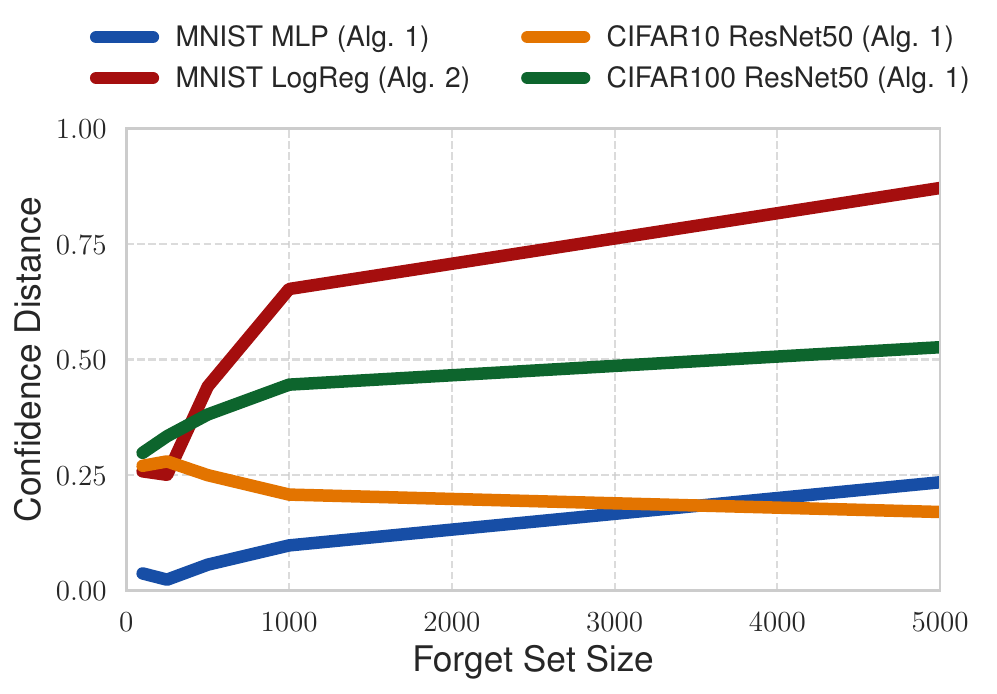}
        \caption{Confidence Distance vs. Epochs, $\theta$ = 0.75}
        \label{fig:conf_v_size}
    \end{subfigure}

    \vspace{1em} 
    
    \caption{We observe that retain accuracy stays fairly stable as the forget size increases, in \cref{fig:ret_v_size}, except for \cref{algo:hess_exact_algo} where it causes catastrophic failure due to the magnitude of the Newton step. Furthermore, we find that confidence distance slowly increases as the forget set size increases in \cref{fig:conf_v_size}.}
    \vspace{-4mm}
    \label{fig:conf_ret_size}
\end{figure}

\begin{figure}
    \centering
    \includegraphics[width=0.5\linewidth]{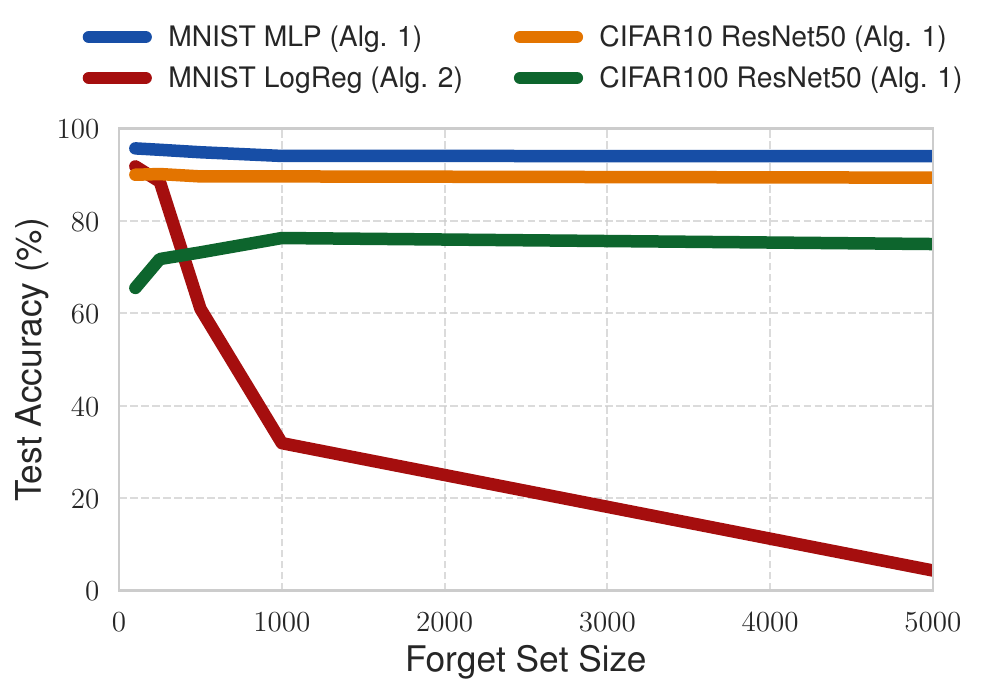}
    \caption{Test Accuracy vs. Forget Set Size, $\theta = 0.75$. This has similar behavior to \cref{fig:ret_v_size}.}
    \label{fig:te_v_size}
\end{figure}

Figures are provided in \cref{fig:conf_ret_size} and \cref{fig:te_v_size}. In \cref{table:forget_set_size_cifar}, we provide experiments for ResNet50 trained on CIFAR10 and CIFAR100 for \cref{algo:finetuning_algo}. In \cref{table:forget_set_size}, we provide experiments on MLP trained over MNIST for \cref{algo:finetuning_algo} and logistic regression trained over MNIST for \cref{algo:hess_exact_algo}. 

Throughout our experiments, we use a forget set size of 100. We do so because for our use case, it is likely that a data controller would want to induce uniformity only for a small number of instances. We observe that as one increases the forget set size, it becomes harder to induce uniformity with the same hyperparameters. Still, for \cref{algo:finetuning_algo}, we are able to obtain strong uniformity with good retain and test accuracy for significantly larger forget set sizes. Furthermore, we observe that \cref{algo:hess_exact_algo} fails for sufficiently large forget set size; this is likely because Hessian matrix is significantly larger (in norm) for a larger forget set, resulting in a catastrophically large Newton step. Mitigating this phenomenon is left to future work, where Hessian-free techniques like those of \citet{qiao2025hessianfree} may be advantageous. 

\begin{table}[tb]
\centering
\caption{Results on applying \cref{algo:finetuning_algo} on ResNet50 for various forget set sizes over CIFAR10 and CIFAR100. Please see \cref{section:forget_set_size} for a discussion on \cref{algo:hess_exact_algo}. $\theta = 0.75$ throughout.}
\vspace{2mm}
\setlength{\tabcolsep}{0.8\tabcolsep}
\begin{tabular}{c|ccccccc}
\hline
\textbf{Dataset} & \textbf{Model} & \textbf{Forget Size} & \textbf{Retain Acc.} & \textbf{Test Acc.} & \begin{tabular}[c]{@{}c@{}}\textbf{Conf. Dist.}\\ \textbf{(Lower Better)}\end{tabular}  \\ \hline
\multirow{5}{*}{CIFAR10} & \multirow{5}{*}{ResNet50} & 100 & 94.7\% & 90.0\% & 0.270 \\ 
                        &                    & 250 & 96.4\% & 90.3\% & 0.289 \\ 
                          &                    & 500 & 95.1\% & 88.6\% & 0.190 \\ 
                          &                    & 1000 & 97.0\% & 90.0\% &  0.144 \\ 
                      &                    & 5000 & 95.8\% & 89.5\% &  0.176 \\ 
\hline 
\multirow{5}{*}{CIFAR100} & \multirow{5}{*}{ResNet50} & 100 & 91.4\% & 65.5\% & 0.298 \\ 
                        &                    & 250 & 99.7\% & 78.0\%  & 0.370 \\ 
                          &                    & 500 & 99.8\% & 76.1\%  & 0.475 \\ 
                          &                    & 1000 & 99.7\% & 74.8\% & 0.492 \\ 
                      &                    & 5000 & 99.8\% & 74.1\% & 0.613 \\ 

\end{tabular}
\label{table:forget_set_size_cifar}
\end{table}

\begin{table}[tb]
\centering
\caption{Results on applying \cref{algo:finetuning_algo} and \cref{algo:hess_exact_algo} on various forget set sizes over MNIST. We observe that, while \cref{algo:finetuning_algo} still works well, confidence distance increases as forget set size does; please see \cref{section:forget_set_size} for a discussion on \cref{algo:hess_exact_algo}. $\theta = 0.75$ throughout.}
\vspace{2mm}
\setlength{\tabcolsep}{0.8\tabcolsep}
\begin{tabular}{c|ccccccc}
\hline
\textbf{Method} & \textbf{Model} & \textbf{Forget Size} & \textbf{Retain Acc.} & \textbf{Test Acc.} & \begin{tabular}[c]{@{}c@{}}\textbf{Conf. Dist.}\\ \textbf{(Lower Better)}\end{tabular}  \\ \hline
\multirow{5}{*}{\cref{algo:finetuning_algo}} & \multirow{5}{*}{MLP} & 100 & 97.5\% & 95.7\% & 0.037 \\ 
                        &                    & 250 & 97.6\% & 95.1\% & 0.010 \\ 
                          &                    & 500 & 95.8\% & 93.8\% & 0.121 \\ 
                          &                    & 1000 & 94.8\% & 93.3\% &  0.162 \\ 
                      &                    & 5000 & 96.6\% & 95.0\% &  0.420 \\ 
\hline 
\multirow{5}{*}{\cref{algo:hess_exact_algo}} & \multirow{5}{*}{LogReg} & 100 & 92.1\% & 91.8\% & 0.258 \\ 
                        &                    & 250 & 85.4\% & 85.2\%  & 0.243 \\ 
                          &                    & 500 & 5.7\% & 5.9\%  & 0.824 \\ 
                          &                    & 1000 & 4.4\% & 4.6\% & 0.891 \\ 
                      &                    & 5000 & 2.4\% & 2.5\% & 0.899 \\ 

\end{tabular}
\label{table:forget_set_size}
\end{table}

\subsection{Evaluating Confidence Distance as a TTP Metric}\label{appendix:l2_metric}

The $\ell_2$ metric $ || f(\mathbf{x}) - \vec{\frac{1}{K}} ||_2$ has similar utility to our presented metric. However, it is slightly less interpretable and may accidentally overpenalize uncertain outputs. For example, if one class has no probability but the other 9 classes are uniform. Still, as demonstrated in \cref{table:l2_metric}, we find that we minimize this metric as well for the same models and datasets. This holds similarly for other potential metrics e.g. the $\ell_1$ metric. 

\begin{table}[tb]

  \centering

  \caption{$\ell_2$ confidence distances for models finetuned with \cref{algo:finetuning_algo}.}

  \vspace{2mm}

  \begin{tabular}{c|cccc}

   \hline 

    \textbf{Dataset}   & \textbf{Model}   & \textbf{Conf. Dist. Type} & \begin{tabular}[c]{@{}c@{}}\textbf{Pretrained}\\ \textbf{(Lower Better)}\end{tabular} & \begin{tabular}[c]{@{}c@{}}\textbf{\cref{algo:finetuning_algo}}\\ \textbf{(Lower Better)}\end{tabular}\\

    \hline 

  \multirow{4}{*}{MNIST} & \multirow{2}{*}{MLP} & \textcolor{gray}{Paper} & \textcolor{gray}{0.879} & \textcolor{gray}{0.037} \\

                        &                       & $\ell_2$ & 0.930 &  0.053 \\

                        \cline{2-5}

                        & \multirow{2}{*}{ResNet18} & \textcolor{gray}{Paper}  & \textcolor{gray}{0.895} & \textcolor{gray}{0.070} \\

                        &                       & $\ell_2$ & 0.944 & 0.070  \\

    \hline 

    \multirow{2}{*}{KMNIST} & \multirow{2}{*}{ResNet18} & \textcolor{gray}{Paper} & \textcolor{gray}{0.880} & \textcolor{gray}{0.257} \\

                        &                       & $\ell_2$ & 0.911 & 0.302   \\

    \hline 

    \multirow{2}{*}{SVHN} & \multirow{2}{*}{ResNet50} & \textcolor{gray}{Paper} & \textcolor{gray}{0.972} & \textcolor{gray}{0.289} \\

                        &                       & $\ell_2$ & 0.979 &  0.298 \\

    \hline 

    \multirow{4}{*}{CIFAR10} & \multirow{2}{*}{ResNet18} & \textcolor{gray}{Paper} & \textcolor{gray}{0.898} & \textcolor{gray}{0.377} \\

                        &                       & $\ell_2$ & 0.947 &  0.435 \\

                        \cline{2-5}

                        & \multirow{2}{*}{ResNet50} & \textcolor{gray}{Paper} & \textcolor{gray}{0.900} & \textcolor{gray}{0.270} \\

                        &                       & $\ell_2$ & 0.948 &  0.323 \\

    \hline 

    \multirow{2}{*}{CIFAR100} & \multirow{2}{*}{ResNet50} & \textcolor{gray}{Paper} & 0.902 & 0.298 \\

                        &                       & $\ell_2$ & 0.911  & 0.311 \\

    \hline 

  \end{tabular}
  \label{table:l2_metric}
\end{table}

\subsection{An Additional Baseline with Randomly Sampled Labels: GaussianUniform}\label{appendix:kangwook_baseline_gaussianuniform} 

In what follows, we present the \textit{GaussianUniform} baseline, an alternative idea to our approach based on the notable work of \citet{zhang2017understanding}, which demonstrates that a neural network can fully minimize its loss over a training dataset where samples have labels sampled uniformly at random. The approach of GaussianUniform is as follows: 
\begin{enumerate}
    \item Begin with a training dataset $\mathcal{D} = \mathcal{D}_f \cup \mathcal{D}_r$. 
    \item Perturb all samples in $\mathcal{D}$ to yield $\mathcal{D}' = \mathcal{D}_f' \cup \mathcal{D}_r'$. We use mean zero Gaussian noise with $0.1$ variance, which adds a small amount of noise. 
    \item Sample all labels in $\mathcal{D}$ uniformly at random to yield $\tilde{\mathcal{D}} = \tilde{\mathcal{D}}_f \cup \tilde{\mathcal{D}}_r$. 
    \item Train $\mathcal{A}(\mathcal{D}' \cup  \tilde{\mathcal{D}})$. 
\end{enumerate}

In this scenario, inducing uncertainty may not be necessary, since the forget set would have very strong uniformity, with strong accuracy on the retain set available by slightly perturbing with Gaussian noise. We use the same hyperparameters as pretraining for the respective model and dataset tested, as reported in \cref{appendix:experimental_details}. We find that this is not the case for ResNet50 trained on SVHN, CIFAR10, and CIFAR100. However, it achieves very poor test accuracy compared to \cref{algo:finetuning_algo} on both normal $\mathcal{D}_{\text{test}}$ and perturbed test $\mathcal{D}_{\text{test}}'$ datasets; thus, we prefer \cref{algo:finetuning_algo} to this approach, since it generalizes well and also does not require retraining. Results are in \cref{table:gaussianuniform}. 

\begin{table}[t]
\centering
\caption{Results for the \textit{GaussianUniform} baseline described in \cref{appendix:kangwook_baseline_gaussianuniform}. This method results in significantly degraded accuracy on the test set compared to a model finetuned with \cref{algo:finetuning_algo}, despite achieving high accuracy on the perturbed retain set ($D_r'$). Note that performance on $\tilde{\mathcal{D}}_r$ is similar to $\mathcal{D}'_r$.}
\label{table:gaussianuniform}
\vspace{2mm}
\setlength{\tabcolsep}{5pt} %
\begin{tabular}{c|c|cc|ccc|c}
\hline
\multirow{2}{*}{\textbf{Dataset}} & \multirow{2}{*}{\textbf{Model}} & \multicolumn{2}{c|}{\textbf{Train Set Acc. (\%)}} & \multicolumn{3}{c|}{\textbf{Test Set Acc. (\%)}} & \textbf{Conf. Dist.} \\
\cline{3-8}
& & \textbf{$\mathcal{D}_r'$} & \textbf{$\mathcal{D}_f'$} & \textbf{$\mathcal{D}_{\text{test}}$} & \textbf{$\mathcal{D}_{\text{test}}'$} & \textbf{\cref{algo:finetuning_algo}} & \textbf{ $\mathcal{D}_f$} \\
\hline
SVHN     & ResNet50 & 82.4\% & 2.0\%  & 6.2\%  & 81.1\% & 94.4\% & 0.009 \\
CIFAR10  & ResNet50 & 100.0\% & 12.0\% & 11.1\% & 73.3\% & 90.0\% & 0.573 \\
CIFAR100 & ResNet50 & 99.9\% & 4.0\%  & 3.4\%  & 40.4\% & 65.5\% & 0.057 \\
\hline
\end{tabular}
\end{table}

\subsection{Tightness of Bound in Theorem \labelcref{thm:retain_accuracy_bound}}\label{appendix:bound_tightness}

To evaluate how tight our bound is, we run an experiment for MNIST logistic regression. We use the notation of \cref{thm:retain_accuracy_bound} in  \cref{table:bound_table}.  We find that our constant bound is fairly tight as $\theta \to 1$; we leave using more advanced techniques to ensure better tightness to future work. 

\begin{table}[tb]
  \centering
  \caption{Comparison of bounds}
  \label{table:alpha_bounds}
  \begin{tabular}{cc}
    \toprule
    $\alpha^* - \alpha(1)$ & \textbf{O(1) bound size} \\
    \midrule
    $1.678$                & $3.374$                 \\
    \bottomrule
  \end{tabular}
  \label{table:bound_table}
\end{table}

\subsection{Confidence Intervals for Main Paper Experiments}\label{appendix:confidence_intervals}

In what follows, we report confidence intervals for only ResNet50 trained on SVHN due to compute constraints. We find that variance is low in \cref{table:confidence_intervals}.

\begin{table}[tb] 
\centering 
\caption{Results for \cref{algo:finetuning_algo} for ResNet50 trained on SVHN over three runs. We find that we are able to induce uniformity while only slightly decreasing retain and test accuracy, $\theta = 0.75$ throughout, with minimal variance for all metrics.} 
\vspace{2mm} 
\setlength{\tabcolsep}{0.8\tabcolsep} 
\begin{tabular}{c|ccccc} 
\hline 
\textbf{Dataset} & \textbf{Model}    & \textbf{Method}                  & \textbf{Retain Acc.} & \textbf{Test Acc.}    & \begin{tabular}[c]{@{}c@{}}\textbf{Conf. Dist.}\\ \textbf{(Lower Better)}\end{tabular} \\ 
\hline 
\multirow{2}{*}{SVHN} & \multirow{2}{*}{ResNet50} & Pretrain & $99.9$\% $\pm$ $0.1$\% & $95.3$\% $\pm$ $0.5$\% & $0.987 \pm 0.001$ \\ 
& & \cref{algo:finetuning_algo} & $99.5$\% $\pm$ $0.2$\% & $94.8$\% $\pm$ $0.3$\% & $0.276 \pm 0.004$
\end{tabular} 
\label{table:confidence_intervals} 
\end{table}

\subsection{Proportions of Time Elapsed in \cref{algo:finetuning_algo}}\label{appendix:time_elapsed}

Results are reported in \cref{table:time_proportions}. 

\begin{table}[tb]
  \centering
  \caption{Time proportions of each step in \cref{algo:finetuning_algo}.}
  \vspace{2mm}
  \begin{tabular}{c|cccccc}
    \hline
    \textbf{Dataset}   & \textbf{Model}   & \textbf{Retain Grad.} & \textbf{Forget Grad.} & \textbf{Surgery} & \textbf{Reg. Grad.} & \textbf{Step} \\
    \hline
    \multirow{2}{*}{MNIST}   & MLP        &     0.967  &     0.033   &    0.000    &   0.000    &    0.000  \\
    \cline{2-7}
                             & ResNet18   &    0.984   &     0.016   &   0.0010    &  0.000      &  0.000   \\
    \hline
    KMNIST                   & ResNet18   &    0.989   &     0.011   &     0.000   &  0.000    &    0.000  \\
    \hline
    SVHN                     & ResNet50   &  0.980     & 0.020       &    0.000   &   0.000    &  0.000   \\
    \hline
    \multirow{2}{*}{CIFAR10} & ResNet18   &     0.989  &    0.011    &    0.000    &       0.000 & 0.000      \\
    \cline{2-7}
                             & ResNet50   &    0.992   &   0.008     &  0.000      &  0.000     &   0.000   \\
    \hline
    CIFAR100                 & ResNet50   &  0.993     &     0.006   &      0.000  &       0.000 &  0.000     \\
    \hline
  \end{tabular}
  \label{table:time_proportions}
\end{table}

\subsection{Warmup Values for MNIST LogReg}

Results are contained in \cref{table:hessian_exact_warmup}. We find that after applying the certified Newton step in \cref{algo:hess_exact_algo}, we obtain better retain and test accuracy, at small cost to uniformity. Thus, warming up is not the only component of achieving good results in \cref{algo:hess_exact_algo}. 
\begin{table}[tb]
\centering
\caption{Warmup values for \cref{algo:hess_exact_algo} for logistic regression trained over MNIST, contrasted with the values after \cref{algo:hess_exact_algo} is applied.}
\vspace{2mm}
\setlength{\tabcolsep}{0.8\tabcolsep}
\begin{tabular}{c|cccc}
\hline 
\textbf{$\theta$} & \textbf{Method} & \textbf{Retain Acc.} & \textbf{Test Acc.} & \begin{tabular}[c]{@{}c@{}}\textbf{Conf. Dist.}\\ \textbf{(Lower Better)}\end{tabular} \\ 
\hline 

$0.0$ & Warmup & 91.4\% & 91.6\% & 0.765 \\ 
$0.0$ & \cref{algo:hess_exact_algo} & 92.5\% & 92.2\% & 0.738 \\ 

$0.25$ & Warmup & 90.1\% & 90.4\% & 0.283 \\ 
$0.25$ & \cref{algo:hess_exact_algo} & 91.6\% & 91.5\% & 0.324 \\ 

$0.50$ & Warmup & 89.3\% & 89.7\% & 0.215 \\ 
$0.50$ & \cref{algo:hess_exact_algo} & 90.5\% & 90.6\% & 0.300 \\ 

$0.75$ & Warmup & 88.4\% & 88.6\% & 0.154 \\ 
$0.75$ & \cref{algo:hess_exact_algo} & 87.1\% & 87.2\% & 0.280 \\

$0.95$ & Warmup & 85.4\% & 86.0\% & 0.097 \\ 
$0.95$ & \cref{algo:hess_exact_algo} & 85.0\% & 85.7\% & 0.092 \\

\hline
\end{tabular}
\label{table:hessian_exact_warmup}
\end{table}

\subsection{Visualization of Softmax Outputs}\label{appendix:softmax_vis}

We provide a comparison of pretrained $f$ and \cref{algo:finetuning_algo} softmax probabilities across five different CIFAR10 forget set samples, demonstrating visually the effectiveness of \cref{algo:finetuning_algo} at inducing uniformity (and the relevance of our confidence distance metric) in \cref{fig:combined_barcharts}. 

\begin{figure}[t!]
    \centering
    
    \begin{subfigure}[b]{0.19\textwidth}
        \includegraphics[width=\textwidth]{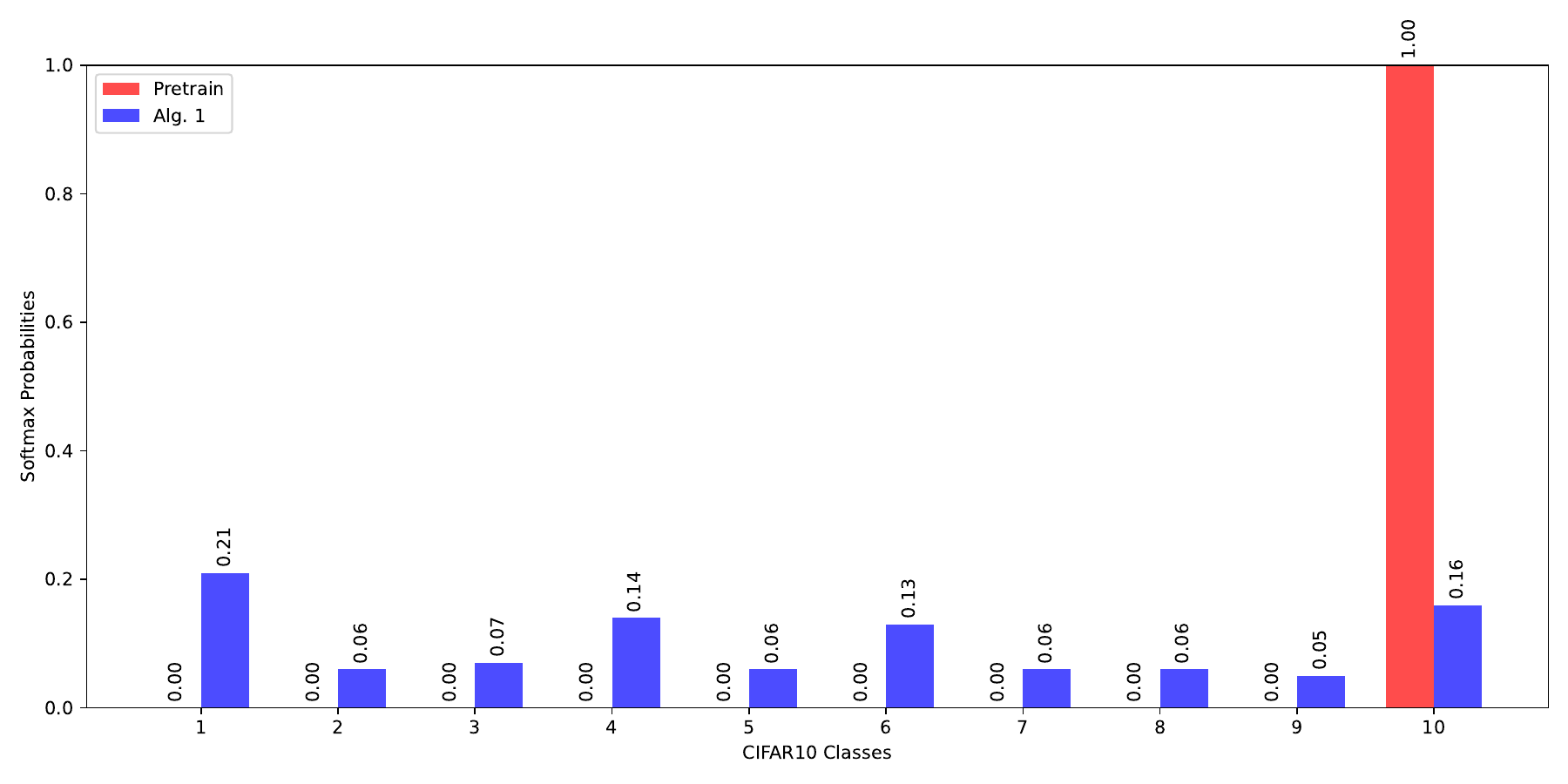}
        \caption{Sample 1}
        \label{fig:sample1}
    \end{subfigure}%
    \hfill
    \begin{subfigure}[b]{0.19\textwidth}
        \includegraphics[width=\textwidth]{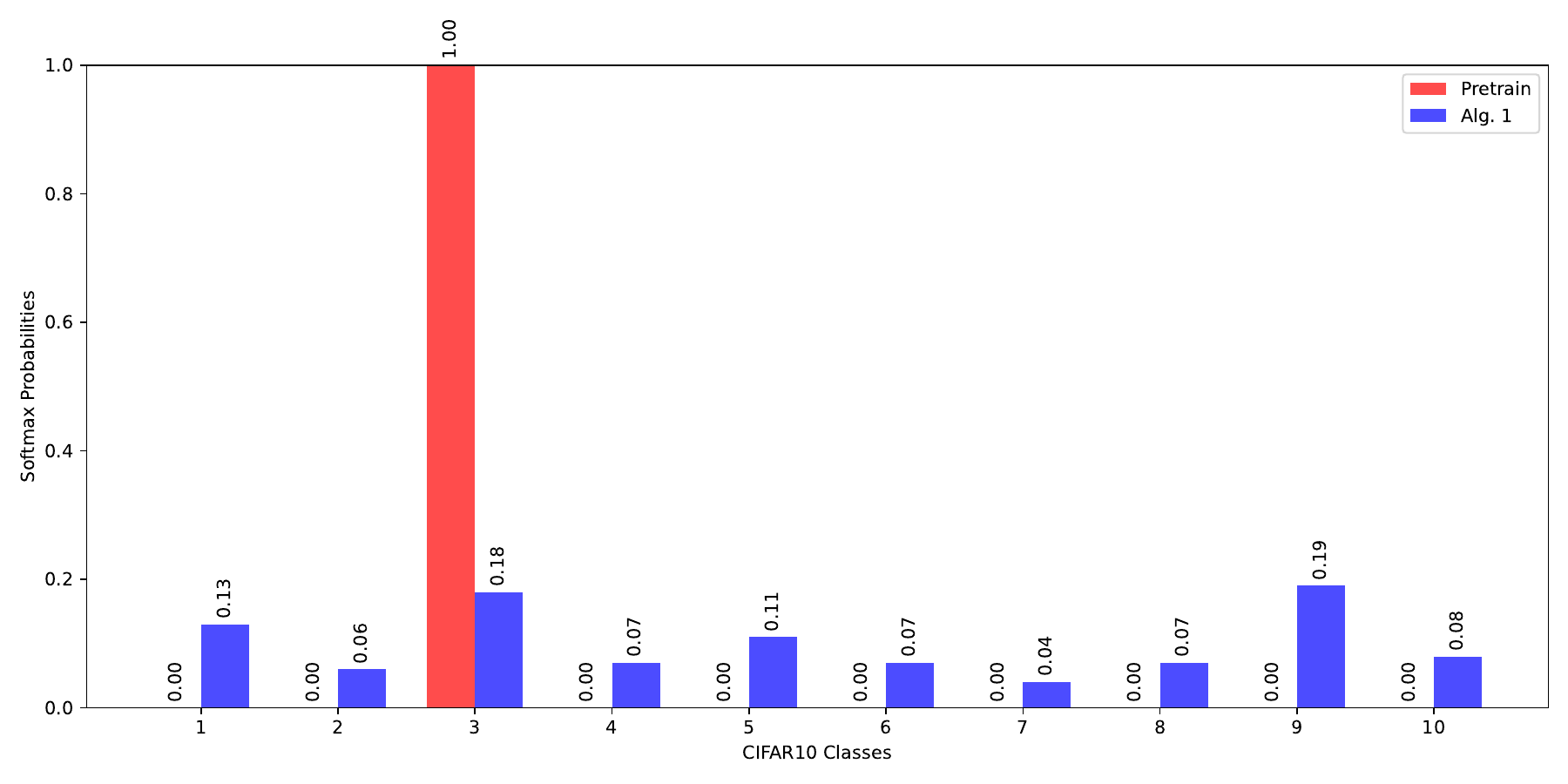}
        \caption{Sample 2}
        \label{fig:sample2}
    \end{subfigure}%
    \hfill
    \begin{subfigure}[b]{0.19\textwidth}
        \includegraphics[width=\textwidth]{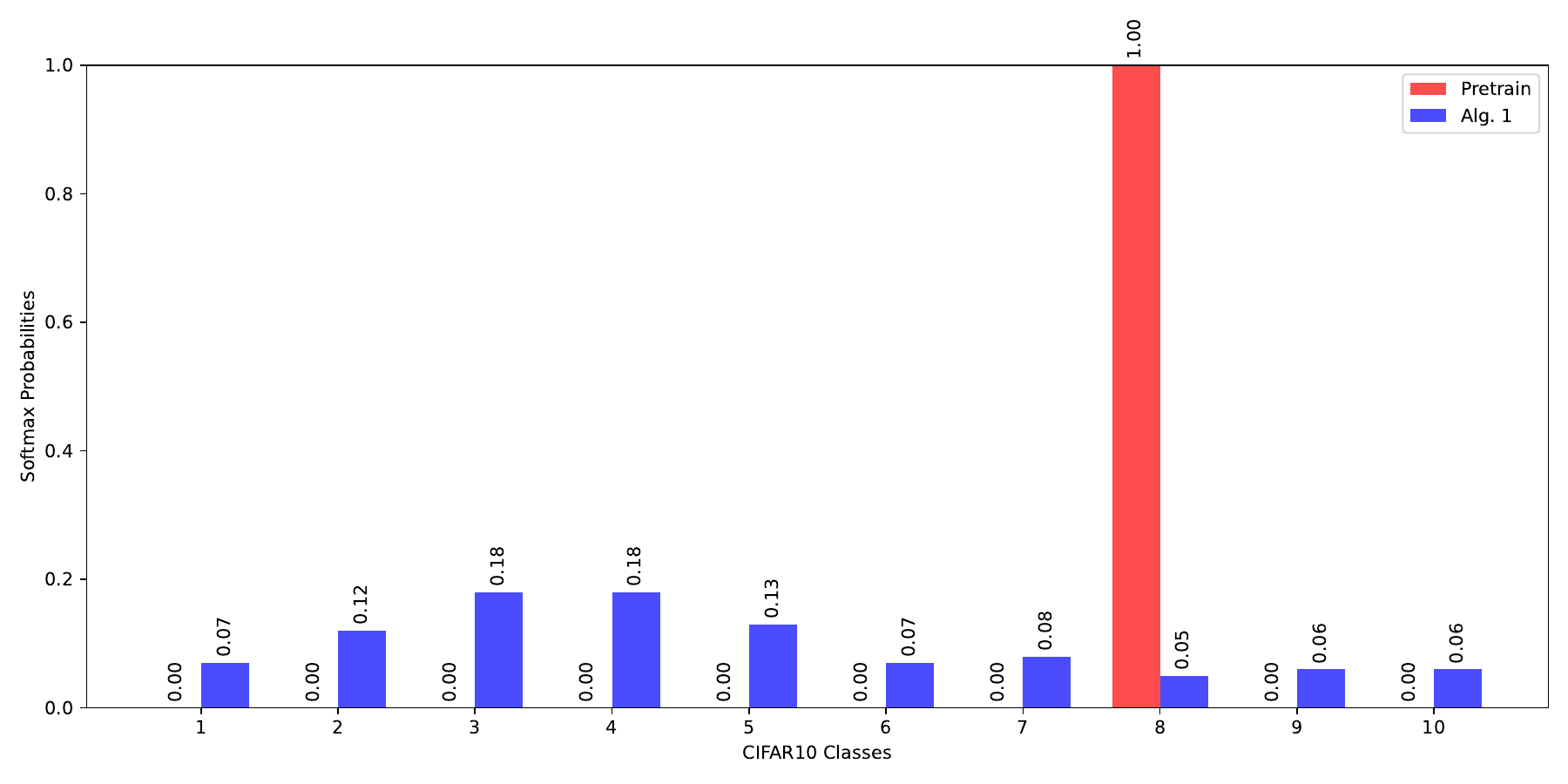}
        \caption{Sample 3}
        \label{fig:sample3}
    \end{subfigure}%
    \hfill
    \begin{subfigure}[b]{0.19\textwidth}
        \includegraphics[width=\textwidth]{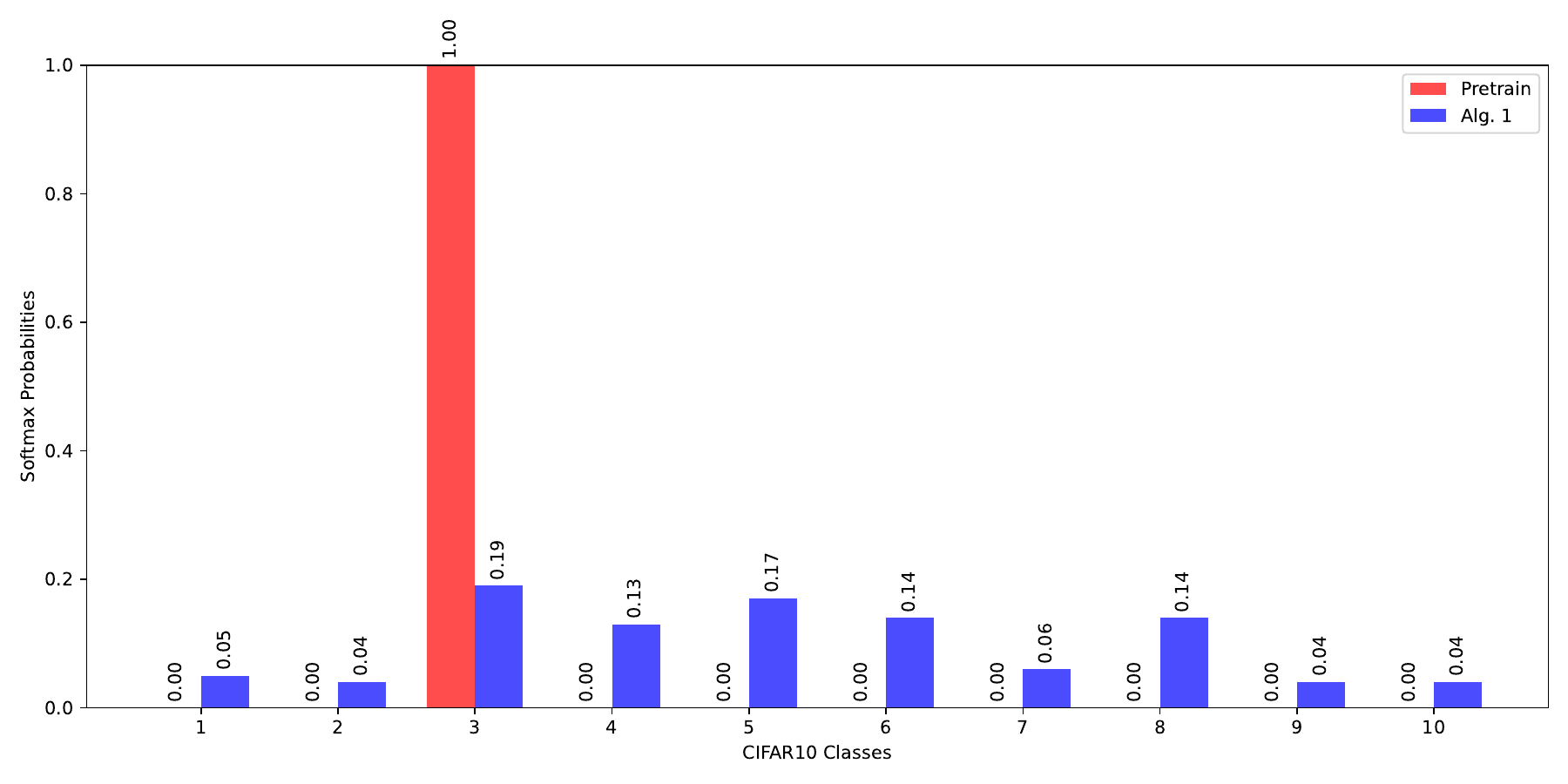}
        \caption{Sample 4}
        \label{fig:sample4}
    \end{subfigure}%
    \hfill
    \begin{subfigure}[b]{0.19\textwidth}
        \includegraphics[width=\textwidth]{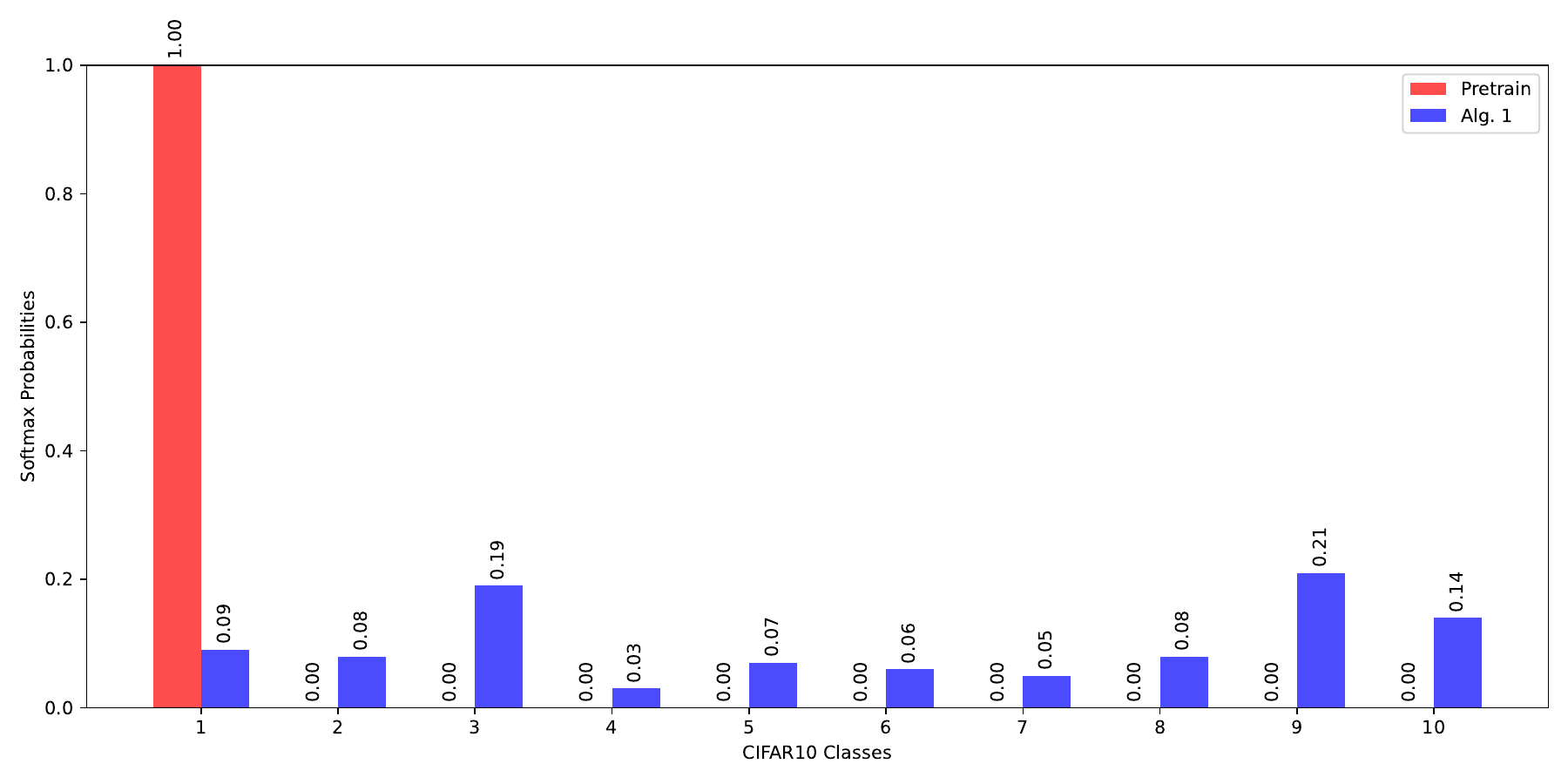}
        \caption{Sample 5}
        \label{fig:sample5}
    \end{subfigure}

    \caption{Comparison of Pretrain (red) and Alg. 1 (blue) softmax probabilities across five different CIFAR10 forget set samples.}
    \label{fig:combined_barcharts}
\end{figure}

\subsection{Results for KMNIST LogReg and MLP}

Results are contained in \cref{table:kmnist_results}. As mentioned in \cref{section:experiments}, one can see that a small model for a more complex benchmark, logistic regression on KMNIST, fails to induce uniformity, since the pretrained model is too small to generalize. However, on a bigger model i.e. an MLP trained over KMNIST, since it is large enough to generalize, one can induce uniformity over it. Thus, larger models which generalize well are preferred for our method, in line with the goals of ML. 

\begin{table}[tb]
\centering
\caption{Results for \cref{algo:finetuning_algo} applied to logistic regression and MLP trained over KMNIST, $\theta = 0.75$ for \cref{algo:finetuning_algo}.}
\vspace{2mm}
\setlength{\tabcolsep}{0.8\tabcolsep}
\begin{tabular}{c|cccccc}
\hline
\textbf{Model} & \textbf{Method} & \textbf{Retain Acc.} & \textbf{Test Acc.} & \begin{tabular}[c]{@{}c@{}}\textbf{Conf. Dist.}\\ \textbf{(Lower Better)}\end{tabular}  \\ \hline
\multirow{3}{*}{LogReg} & Pretrain &  81.4\% & 66.4\% &  0.775 \\ 
                        & Retrain & 80.9\% & 65.3\% &  0.770 \\ 
                         & \cref{algo:finetuning_algo} & 77.4\% & 63.4\% &  0.770 \\ 
\multirow{3}{*}{MLP} & Pretrain & 100\% & 88.4\% & 0.900  \\ 
                        & Retrain & 100\% & 88.5\% & 0.887 \\ 
                         & \cref{algo:finetuning_algo} & 92.8\% & 80.3\% &  0.039 \\ 
\end{tabular}
\label{table:kmnist_results}
\end{table}

\subsection{Ablation Study on Synthetic Baseline Sample Size}

Below, we study what we happen if we increase the number of samples sampled in the $\varepsilon$-ball in the synthetic baseline. For each forget set instance, we sample $k$ instances from the $\varepsilon$-ball around the forget set, and then assign random labels to these instances, yielding an additional $|\mathcal{D}_f|k$ instances in the training data. We then retrain the model over the retain set along with these new $|\mathcal{D}_f|k$ instances. For a MLP trained over MNIST, we observe better performance as we increase sample size. However, for a ResNet18 trained over CIFAR10, even if we have a very large sample size. This is presented in \cref{table:synthetic_baseline_k}. Thus, since \cref{algo:finetuning_algo} can induce uniformity as shown in \cref{table:finetuning_main_combined}, without great cost to retain or test accuracy, it is better than the synthetic baseline. 

\begin{table}[tb]
\centering
\caption{Results for the synthetic baseline applied with various sampled $k$ on a MLP over MNIST and a ResNet18 over CIFAR10. We observe that increasing $k$ yields better performance, but nevertheless even very large $k$ (an additional 50k instances, with a forget set size of 100) fails to induce uniformity for CIFAR10.}
\vspace{2mm}
\setlength{\tabcolsep}{0.8\tabcolsep}
\begin{tabular}{c|ccccccc}
\hline
\textbf{Dataset} & \textbf{Model} & \textbf{Sampled $k$} & \textbf{Retain Acc.} & \textbf{Test Acc.} & \begin{tabular}[c]{@{}c@{}}\textbf{Conf. Dist.}\\ \textbf{(Lower Better)}\end{tabular}  \\ \hline
\multirow{5}{*}{MNIST} & \multirow{5}{*}{MLP} & 5 & 99.4\% & 97.1\% & 0.541 \\ 
                        &                    & 25 & 99.0\% & 96.4\% & 0.183 \\ 
                          &                    & 125 & 99.4\% & 96.8\% & 0.105 \\ 
                          &                    & 250 & 98.6\% & 96.0\% & 0.066\\ 
                      &                    & 500 & 99.6\% & 96.3\% & 0.003 \\ 
\hline 
\multirow{6}{*}{CIFAR10} & \multirow{6}{*}{ResNet18} & 5 & 99.0\% & 91.0\% & 0.683 \\ 
                        &                    & 25 & 99.2\% & 91.3\% & 0.865 \\ 
                          &                    & 125 & 98.8\% & 90.9\% & 0.856 \\ 
                          &                    & 250 & 98.4\% & 90.7\% & 0.869 \\ 
                      &                    & 500 & 98.3\% & 91.1\% & 0.852 \\ 
                       &                    & 5000 & 94.0\% & 89.7\% & 0.844 \\

\end{tabular}
\label{table:synthetic_baseline_k}
\end{table}

\subsection{Test Accuracy Plot for Pareto Frontier Experiments}

We provide a plot characterizing test accuracy for \cref{algo:finetuning_algo} and \cref{algo:hess_exact_algo} applied on MNIST for various choices of $\theta$ in \cref{fig:te_pareto}.

\begin{figure}
    \centering
    \includegraphics[width=0.5\linewidth]{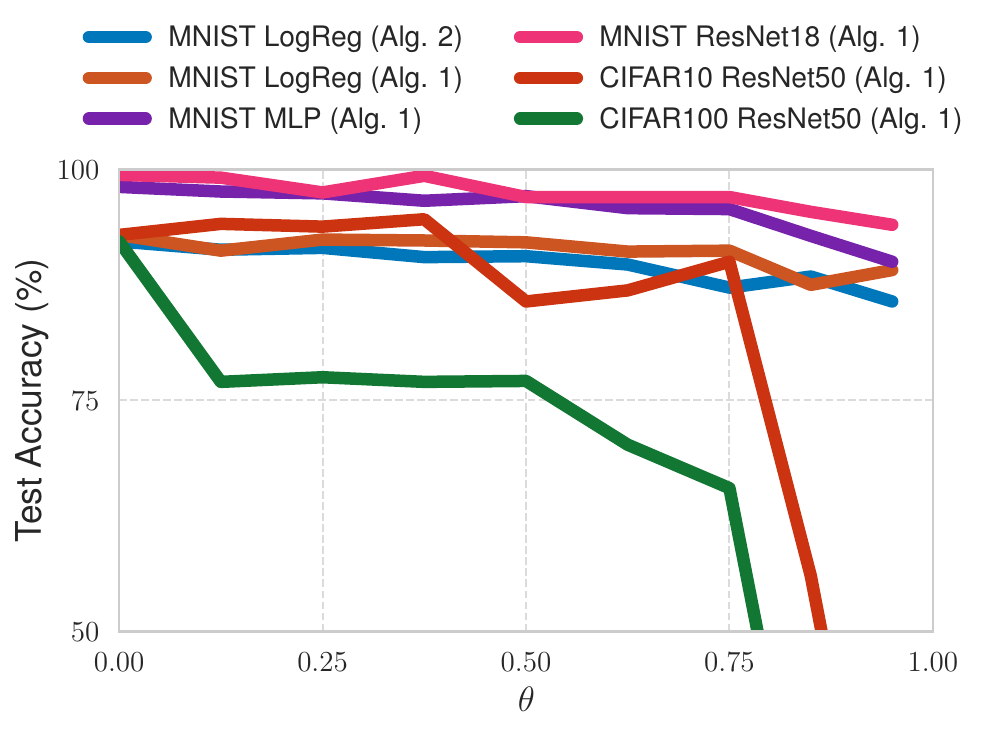}
    \caption{Test Accuracy vs. $\theta$, MNIST. This has similar behavior to \cref{fig:ret_pareto}.}
    \label{fig:te_pareto}
\end{figure}

\section{Broader Impacts}\label{appendix:broader_impacts}

 Potential positive impacts are motivated by the threat model as discussed in \cref{section:intro} and \cref{appendix:threat_model} ; per our example provided in the introduction, violations of test-time privacy constitute a real threat for ML safety. Hence, providing the defense that we do constitutes a positive societal impact. However, we acknowledge the potential danger in providing a new threat model--it is possible that potential adversaries had not thought of this before. Still, we provide a way to address this threat.

\clearpage 

\section{Symbol Table}\label{appendix:symbol_table}

\begin{tcolorbox}[colback=blue!4,colframe=green!60!blue]
\centerline{\bf Symbols}
\def\arraystretch{1.5}
\begin{tabular}{p{1in}p{3.25in}}
$\displaystyle f$ & A pretrained classifier. \\
$\displaystyle f_u$ & A pretrained classifier after unlearning has been conducted over $\bm{x}_p$. \\
$\displaystyle \bm{x}_p$ & A data instance corresponding to person $p$. \\
$\displaystyle \mathcal{X}$ & A sample space, subset of $\mathbb{R}^d$. \\
$\displaystyle \mathcal{Y}$ & A label space, subset of $\mathbb{R}^o$. \\
$\displaystyle \mathcal{Z}$ & The Cartesian product of a sample space and a label space. This is the space where a dataset is drawn from. \\
$\displaystyle \mathcal{D}$ & A dataset, subset of $\mathcal{Z}^n$, which is the n-fold Cartesian product of $\mathcal{Z}$. This represents a set of $n$ data instances.\\
$\displaystyle \mathcal{D}_f$ & A forget set, a subset of a dataset $\mathcal{D}$. \\
$\displaystyle \mathcal{D}_r$ & A retain set, the complement of the forget set in $\mathcal{D}$. \\
$\displaystyle \mathcal{W}$ & A space of parameters, subset of $\mathbb{R}^p$. \\
$\displaystyle \mathcal{A}$ & A function that maps datasets to parameters; this represents the learning algorithm that a ML model provider uses throughout our paper. \\
$\displaystyle \mathcal{H}_\mathcal{W}$ & A set of functions which map samples in $\mathcal{X}$ to the probability simplex $\Delta_{|\mathcal{Y}|}$, parameterized by a $\bm{w} \in \mathcal{W}$. \\
$\displaystyle ||\bm{w}||_2$ & The $\ell_2$ norm of a vector $\bm{w}$. \\
$\displaystyle ||A||_2$ & The 2-operator norm of a matrix $\bm{A}$. \\
$\displaystyle \lambda$ & An $\ell_2$ regularization coefficient used in \cref{algo:finetuning_algo} for regularization and the certified algorithms e.g. \cref{algo:hess_exact_algo} for local convex approximation. \\
$\displaystyle \lambda_{\min}(\bm{A})$ & The minimum eigenvalue of a matrix $\bm{A}$. \\
$\displaystyle \mathcal{K}$ & A uniform learner, which maps samples to parameters which, when one parametrizes a function by any such parameter, a uniform distribution over all possible labels, $U[0, |\mathcal{Y}|]$ is outputted. \\
$\displaystyle f_{\bm{w}}$ & A classifier parameterized by a parameter $\bm{w} \in \mathcal{W}$. \\
$\displaystyle \mathcal{L}_{\mathcal{A}}$ & A loss function to yield accurate predictions e.g. the cross entropy loss between model predictions and labels. \\ 
$\displaystyle \mathcal{L}_{\mathcal{K}}$ & A loss function to yield accurate uniformity e.g. the Kullback-Liebler divergence between softmax outputs and uniform distribution. \\ 
$\displaystyle \theta$ & A trade off parameter in $(0,1)$ between utility and uniformity. 
\end{tabular}
\end{tcolorbox}

\begin{tcolorbox}[colback=blue!4,colframe=green!60!blue]
\centerline{\bf Symbols}
\def\arraystretch{1.5}
\begin{tabular}{p{1in}p{3.25in}}
\\
$\displaystyle \mathcal{M}_\theta$ & A map between datasets and parameters that is the minimizer of a Pareto objective between $\mathcal{L}_\mathcal{A}$ and $\mathcal{L}_\mathcal{K}$, where $\theta$ spans the (convex) Pareto frontier. \\ 
$\displaystyle J$  & The bound on the norm of the Hessian at $\bm{w}^*$. \\ 
$\displaystyle \mathcal{D}_f^{(i)}$ & The $i$th instance of the forget set. \\ 
$\displaystyle \mathcal{D}_r^{(j)}$ & The $j$th instance of the retain set. \\ 
$\displaystyle \mathcal{D}_r^{(j, \mathcal{X})}$ & The feature of the $j$th instance of the retain set. \\ 
$\displaystyle \mathcal{D}_r^{(j, \mathcal{Y})}$ & The label of the $j$th instance of the retain set. \\ 
$\displaystyle \approx_{\varepsilon, \delta, \mathcal{T}}$ & Used to denote when two algorithms are $(\varepsilon, \delta)$ indistinguishable across all subsets $\mathcal{T} \subset \mathcal{W}$, i.e. $\mathcal{M}(\mathcal{D}) \approx_{\varepsilon, \delta, \mathcal{T}} \mathcal{M}^\prime(\mathcal{D}^\prime)$ means that $\Pr[\mathcal{M}(\mathcal{D}) \in \mathcal{T}] \leq e^{\varepsilon}\Pr[\mathcal{M}^\prime(\mathcal{D}^\prime) \in \mathcal{T}] + \delta$ and $\Pr[\mathcal{M}^\prime(\mathcal{D}^\prime) \in \mathcal{T}] \leq e^{\varepsilon}\Pr[\mathcal{M}(\mathcal{D}) \in \mathcal{T}] + \delta$. \\ 
$\displaystyle C$ & A bound on the model weights. \\
$\displaystyle P_{\mathcal{K}}$ & The Lipschitz constant for the gradients of $\ell_\mathcal{K}$, the component loss functions of $\mathcal{L}_\mathcal{K}$, from \cref{assumption1}. \\
$\displaystyle P_{\mathcal{A}}$ & The Lipschitz constant for the gradients of $\ell_\mathcal{A}$, the component loss functions of $\mathcal{L}_\mathcal{A}$, from \cref{assumption1}. \\
$\displaystyle F_{\mathcal{K}}$ & The Lipschitz constant for the Hessians of $\ell_\mathcal{K}$, the component loss functions of $\mathcal{L}_\mathcal{K}$, from \cref{assumption2}. \\
$\displaystyle F_{\mathcal{A}}$ & The Lipschitz constant for the Hessians of $\ell_\mathcal{A}$, the component loss functions of $\mathcal{L}_\mathcal{}$, from \cref{assumption2}. \\ 
$\displaystyle P $ & A convex combination of $P_{\mathcal{K}}$ and $P_{\mathcal{A}}$ with respect to $\theta$. \\
$\displaystyle F$ & A convex combination of $F_{\mathcal{K}}$ and $F_{\mathcal{A}}$ with respect to $\theta$. \\
$\displaystyle \mathcal{N}(0, \sigma^2\bm{I}) $ & The standard normal distribution with an isotropic covariance matrix. \\
$\displaystyle \nabla_{\bm{w}^*, \mathcal{K}, \mathcal{A}} $ & The gradient of the Pareto objective evaluated at $\bm{w}^*$, used in the main paper with regularization. \\ 
$\displaystyle \bm{H}_{\bm{w}^*, \mathcal{K}, \mathcal{A}} $ & The Hessian of the Pareto objective evaluated at $\bm{w}^*$, used in the main paper without regularization.

\end{tabular}
\end{tcolorbox}

%% file: sections/threat_model.tex
\section{Test-Time Privacy Threat Model as a Security Game}\label{appendix:threat_model}

Following recent works on privacy and cybersecurity \citep{baigsecurity2025}, we begin by making our threat model concrete as an informal security game. Broadly, we consider a \textit{test-time privacy (TTP) game} where a \textit{TTP adversary} aims use an open-weight ML model $f$ to produce a confident, harmful prediction $m$ for a specific set of corrupted inputs $\mathcal{D}_f$ drawn from a distribution $\mathcal{P}$. 

\textbf{Actors and Assets}: The game begins with three key actors: 

\begin{enumerate}
    \item The data corrupter $c$, an entity that either maliciously or erroneously creates a ``forget set" $\mathcal{D}_f$ of corrupted instances, e.g. a server which makes an error in compressing a medical image uploaded to an online forum.  
    \item The \textit{model provider} $\tau$, a benign challenger that uses a learning algorithm $\mathcal{A}$ and releases a model $f$. For example, $f$ can classify skin disease from skin images. They then seek to ensure TTP by running algorithm $\mathcal{G}$ to obtain model $\hat{f}$. 
    \item The TTP adversary $\nu$ e.g. a potential medical insurance provider who has access to the architecture and parameters of $\hat{f}$ and aims to obtain harmful prediction $m$ on $\mathcal{D}_f$ to e.g. use as a warrant to reject insurance applicants. 
\end{enumerate}

\textbf{Assumptions}: We operate under two core assumptions: 

\begin{itemize}
    \item \textbf{Open-Weight Access}: The adversary $\nu$ has complete access to the model's architecture and weights. This renders naive defenses, like obtaining $\hat{f}$ by masking softmax outputs of $f$, useless, as an adversary can simply move such a mask and recover prediction $m$. 
    \item \textbf{$\tau$-Limited Knowledge}: The model provider $\tau$ is notified about the existence of a corrupted forget set $\mathcal{D}_f$, but \textit{does not know the specific harmful label} $m$. Furthermore, they \textit{do not know the specific adversary}. To make this concrete, the model provider $\tau$ does not know whether e.g. $\nu$ is a medical insurance company aiming to obtain a prediction of ``Melanoma" to reject coverage or a defense attorney in a criminal case against a doctor aiming to obtain a prediction of ``Benign" to clear a doctor of accusations of medical malpractice.
    
\end{itemize}

\textbf{Game}: The game is then played in the first round. 

\textbf{Round 1}: The first round contains preliminary steps as follows:  
\begin{itemize}
    \item The corrupter $c$ corrupts the data and yields $\mathcal{D}_f$, which adversary $\nu$ gains access to e.g. through the public Internet. 
    \item The model provider $\tau$ trains a model $f$ over instances from $\mathcal{P}$. 
    \item The model provider $\tau$ is made aware that $\mathcal{D}_f$ contains corrupted instances, and seeks to protect them from a TTP adversary $\nu$. 
\end{itemize}

\textbf{Round 2}: The second round contains the following steps: 
\begin{enumerate}
    \item The model provider $\tau$, who is aware of TTP, aims to provide a model $\hat{f}$ to replace $f$ such that $\hat{f}(\bm{x}) \neq m$, where $m$ is the harmful prediction. However, they are \textit{unaware of which prediction $m$ is}. $\tau$ thus runs an algorithm $\mathcal{G}$ with respect to $f$, $\mathcal{D}_f$, and a training dataset $\mathcal{D}$, which yields a new model $\hat{f}$. 
    \item The TTP adversary takes model $\hat{f}$ and attempts to obtain a confident prediction $m$ which serves as a warrant to endanger individuals e.g. to reject individuals from a health insurance provider because their image was classified as a high risk disease like melanoma. 
\end{enumerate}

\textbf{Win Conditions}: The TTP adversary $\nu$ wins if it is clearly the case that $f(\bm{x}) = m$ for all $\bm{x} \in \mathcal{D}_f$, as they can then e.g. use this prediction as a warrant to reject people's insurance applications. The model provider $\tau$ wins if $\hat{f}$ leaves $\nu$ uncertain as to whether the prediction is $m$ or not. 

Given this win condition, an algorithm $\mathcal{G}$ satisfies \textit{test-time privacy} if the adversary can only guess at the model output for all instances in $\mathcal{D}_f$. Thus, it is optimal to induce maximal uncertainty over $\mathcal{D}_f$. In particular, in the discriminative setting--which our work focuses on--it is optimal for model $f$ to output uniform softmax outputs over $\mathcal{D}_f$, while maintaining strong accuracy on all other instances. Furthermore, to defend against such a adversary with open-weight model access, one must perturb the model weights in a non-invertible manner, motivating our approaches detailed in \cref{section:formulation}. b

\textbf{Importantly}, while in our formulation in \cref{section:formulation} we define the forget set of corrupted instances in terms of the training dataset, we do so \textbf{without loss of generality}. As detailed previously, we assume that the forget set contains \textit{all} corrupted instances, including instances outside of the training dataset that are known to be corrupted. Denoting the set of training forget set instances $\mathcal{D}_f^{\text{train}}$ and $\mathcal{D}_f^{\text{test}}$, we can thus let $\mathcal{D}_f = \mathcal{D}_f^{\text{train}} \cup \mathcal{D}_f^{\text{test}}$ and again consider $\mathcal{D} = \mathcal{D}_f \cup \mathcal{D}_r$; in this scenario, all formal definitions and statements throughout \cref{section:formulation}, \cref{section:algorithm_study}, and elsewhere follow in the exact same manner. A concrete example of when instances outside of a training dataset can become relevant is credit score classification; one's credit score report can become corrupted, even if they are not in the training dataset, and one should be able to ask a credit bureau to remedy this to ensure that e.g. a loan officer does not incorrectly estimate their credit score. 

In \cref{appendix:attacks_definition}, we present some simple TTP attacks on open-weight image classifiers to further motivate our threat model.

%% file: sections/additional_related_work.tex
\textbf{Differential Privacy: } Differential privacy has widely been studied in the ML community in order to ensure privacy-preservation \citep{chaudhuri2008privacy}; \citep{chaudhuri2011differentially}; \citep{abadi2016deep}; \citep{chua2024scalable}. There also exist methods to finetune pretrained models to satisfy differential privacy \citep{yu2021differentially}. Furthermore, there are also ways to aggregate label noise to preserve privacy \citep{papernot2018scalable}. 

However, differential privacy is designed to address an entirely different threat model than ours. In particular, in the threat model of differential privacy, an adversary seeks to use model outputs to recover private information about data instance $\bm{x}_p$ corresponding to person $p$ with e.g. a model inversion attack. A differentially private classifier generally results in confident, accurate predictions. This does not address our threat model, where an adversary may use confident model outputs to violate the privacy of person $p$ in a different manner, taking advantage of them directly to use as a warrant to cause harm to person $p$.

\textbf{Label Differential Privacy}: Similarly, our formulation differs from label differential privacy (LabelDP) \citep{ghazi2021deep}, which seeks to protect an adversary from learning the true labels of the instances in the training data. Given an instance, even after computing $f(\bm{x}_p)$, under LabelDP an adversary cannot be confident that $f(\bm{x}_p) = y$. However, LabelDP is applied to the entire dataset; our threat model involves only a particular subset of the training data. Furthermore, we do not need to protect the user's ground truth label, necessarily. In our law enforcement example in \cref{section:intro}, the agency does not care about the ground truth label. Instead, they want any confirmation such that they have a warrant to act adversarially towards person $p$; for this, a confident prediction by model $f$ suffices. Finally, LabelDP results in poor retain and test accuracy for larger datasets e.g. CIFAR100, as demonstrated in \cref{fig:alg1_alg2_results}.

Furthermore, from the perspective of protecting the privacy of the labels themselves, rather than protecting against any confident prediction, \citet{busa2021pitfalls} demonstrate that testing a model, trained with LabelDP, on the training dataset allows an adversary to recover the labels of the label-private data with high probability. Since our algorithms induce uniformity, an adversary cannot infer the correct forget set labels by testing the model on the training dataset; thus, we provide better privacy against this threat model than LabelDP as well. \citet{wu2022does} argue that, under this threat model where one seeks to protect the labels, any model that generalizes must leak the accurate labels when tested on the training data. However, as we demonstrate by inducing uniformity while maintaining high test accuracy, this only holds when the model is to be tested on the \textit{entire} training data, not a \textit{subset} of the training data (or other test instances which are known to be corrupted), as in our setting. 

\textbf{Label Model Inversion Attacks}: Related to LabelDP are model inversion attacks to recover the ground truth labels, like gradient inversion \citep{zhang2022survey}; \citep{zhu2019deep}; \citep{zhao2020idlg}. Yet, these methods do not report the confidence values for the recovered labels. Thus, they do not constitute test-time privacy attacks within our threat model. Furthermore, by the same token as above, an adversary seeks to recover a confident prediction to use as a warrant, not necessarily the ground truth labels. Still, these methods could potentially be extended to test-time privacy attacks by reporting a confidence score for the recovered labels. We leave this to future work.

\textbf{Other Paradigms in Privacy}: Other paradigms in the privacy literature correspond to a notion of ``test-time privacy" which differ from our threat model. For example, several works study defense against model inversion attacks as test-time privacy \citep{wang2019beyond}; \citep{xiao2020adversarial}; \citep{sun2021soteria}; \citep{tran2023personalized}. However, this is a separate threat model from ours; the adversary already has access to the instance $\bm{x}_p$ within our threat model. 

\textbf{Misclassification \& Relabeling in Machine Unlearning}: 
Recently, methods have emerged to finetune a model to misclassify rather than mimicking retraining from scratch \citep{cha2024learning}. There are other similar relabeling methods in the debiasing literature which could be used for this purpose \citep{angelopoulos2025gradient}. However, these methods often achieve poor performance on the remaining training data and fail to provide protection against our threat model in all cases. In particular, a purposefully incorrect classification can also be used to endanger an individual. For example, in the insurance example in \cref{section:intro}, it may still be problematic to classify the user as ``Benign" instead of ``Melanoma"; for example, the user of model $f$ could be a medical professional instead of an insurance provider. Furthermore, in the binary classification case, if an adversary knows that $\bm{x}_p$ is in the forget set, they can recover the true $f(\bm{x}_p)$ by taking complements, if an unlearning method which seeks to induce misclassification is used. They can also use the information that learned representations are markedly different than other similar examples to understand the method used. Additionally, in the multiclass setting, an adversary can still take complements of this class, yielding a probability of recovering the true class which is significantly better than choosing uniformly at random. Instead, it is fairer and more robust to have an output that is maximally uncertain.

\textbf{Model Calibration and Confidence}: In our setting, we use the model softmax outputs to represent the adversary's confidence in the final prediction. However, some argue that this type of interpretation is incorrect, i.e. ML models are poorly calibrated \citep{guo2017calibration}. Still, this interpretation is common \citep{pearce2021understanding}, and thus a model user would likely rely on the softmax outputs as the confidence scores. We leave to inducing uncertainty over the calibrated outputs to future work.

%% file: sections/motivational_fig_duplicate.tex
    \definecolor{bgcolor}{RGB}{245, 245, 245}
    \definecolor{accentblue}{RGB}{0, 123, 255}
    \definecolor{accentred}{RGB}{220, 53, 69}
    \definecolor{accentgreen}{RGB}{40, 167, 69}
    \definecolor{accentgray}{RGB}{108, 117, 125}
    \definecolor{textdark}{RGB}{33, 37, 41}
    \definecolor{datagray}{RGB}{150, 150, 150}
    \definecolor{accentpurple}{RGB}{102, 51, 153} %

\afterpage{
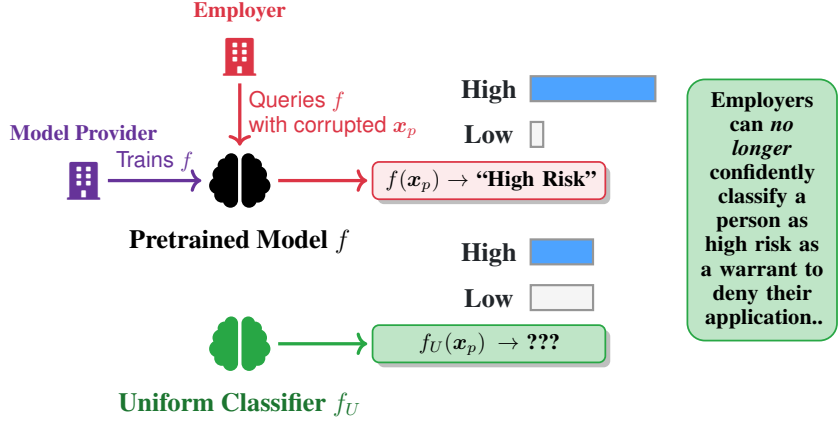
\begin{figure}[t]
    \centering
    \resizebox{0.8\textwidth}{!}{%
    \begin{tikzpicture}[
        font=\sffamily,
        node distance=0.8cm and 1.5cm, %
        actor/.style={
            font=\Large,
            text=textdark
        },
        model/.style={
            font=\Huge,
            text=accentblue
        },
        data/.style={
            draw=accentgray,
            thick,
            rounded corners=3pt,
            minimum width=0.5cm,
            minimum height=0.3cm,
            align=center,
            fill=white,
            drop shadow
        },
        output/.style={
            draw=accentred,
            thick,
            rounded corners=3pt,
            text width=3.5cm,
            minimum height=0.4cm,
            font=\bfseries,
            align=center,
            fill=accentred!10,
            drop shadow
        },
        process/.style={
            ->,
            line width=1.5pt,
            rounded corners=5pt
        },
        emphasis/.style={
            rectangle,
            fill=accentred!90,
            text=white,
            font=\bfseries,
            rounded corners=3pt,
            inner sep=1mm,
            drop shadow
        },
    ]

    \node[actor, text=accentpurple] (model_provider) at (0,0) {\LARGE\faBuilding};
    \node[above=0.1cm of model_provider, text=accentpurple, align=center, font=\bfseries] {Model Provider};

    \node[model, text=black, right=of model_provider] (f) {\faBrain};
    \draw[process, accentpurple] (model_provider) -- (f) node[midway, above] {Trains $f$};

    \node[below=0.1cm of f, font=\large\bfseries] (f_label) {Pretrained Model $f$};

    \node[actor, text=accentred, above=1.0cm of f] (health_provider) {\LARGE\faBuilding};
    \node[above=0.01cm of health_provider, text=accentred, align=center, font=\bfseries] {Employer};
    \draw[process, accentred] (health_provider) -- (f) node[midway, right, align=left] {Queries $f$\\with corrupted $\bm{x}_p$};

    \node[output, right=of f] (pred1) {$f(\bm{x}_p) \rightarrow$ ``High Risk''};
    \draw[process, accentred, shorten >= 2pt] (f) -- (pred1);

    \node[anchor=east, color=textdark, font=\large, above=0.8cm of pred1] (high_label) {\textbf{High}};
    \draw[fill=accentblue!70, draw=datagray, line width=1pt] ([xshift=0.1cm, yshift=-0.15cm]high_label.east) rectangle ++(2, 0.4) coordinate (bar1_right_edge);
    \node[anchor=east, color=textdark, font=\large, below=0.1cm of high_label] (low_label) {\textbf{Low}};
    \draw[fill=bgcolor, draw=datagray, line width=1pt] ([xshift=0.15cm, yshift=-0.2cm]low_label.east) rectangle ++(0.2, 0.4);

    \node[model, text=accentgreen, below=1.5cm of f] (fu) {\faBrain};
    \node[below=0.1cm of fu, text=accentgreen!70!black, font=\large\bfseries] (fu_label) {Uniform Classifier $f_U$};

    \node[output, right=of fu, draw=accentgreen, fill=accentgreen!30] (pred2) {$f_U(\bm{x}_p) \rightarrow$ ???};
    \draw[process, accentgreen, shorten >= 2pt] (fu) -- (pred2);

    \node[anchor=east, color=textdark, font=\large, above=0.8cm of pred2] (high_label2) {\textbf{High}};
    \draw[fill=accentblue!70, draw=datagray, line width=1pt] ([xshift=0.1cm, yshift=-0.15cm]high_label2.east) rectangle ++(1, 0.4);
    \node[anchor=east, color=textdark, font=\large, below=0.1cm of high_label2] (low_label2) {\textbf{Low}};
    \draw[fill=bgcolor, draw=datagray, line width=1pt] ([xshift=0.15cm, yshift=-0.2cm]low_label2.east) rectangle ++(1, 0.4);

    \path (bar1_right_edge) -- (bar1_right_edge |- pred2.center) coordinate[midway] (right_anchor);

    \node[rectangle, draw=accentgreen, fill=accentgreen!30, thick, rounded corners=8pt,
        inner sep=2mm, font=\bfseries, align=center, text width=2.0cm,
            right=0.5cm of right_anchor, anchor=west] (message) {
        Employers can \emph{no longer} confidently classify a person as high risk as a warrant to deny their application.. 
    };

    \end{tikzpicture}
    }
    \caption{An adversary, like an employer (\textcolor{red}{\faBuilding}), can query a pretrained model $f$ (\textcolor{black}{\faBrain}) and use its outputs to make harmful decisions. However, after running our algorithm, the new model $f_U$ (\textcolor{accentgreen}{\faBrain}) provides maximal uncertainty, protecting against such an adversary. This is a duplicate of \cref{fig:motivation}, to make clear how TTP extends to other settings.}
    \label{fig:motivation_police}
\end{figure}
\FloatBarrier
}

%% file: sections/designing_algo.tex
\section{Designing Certified Algorithms}\label{section:certifying_algorithm}

In what follows, we design $(\varepsilon, \delta, \theta)$-certified Pareto learners. A symbol table can be found at \cref{appendix:symbol_table}.  

In our setting, the original model is obtained using ERM over some loss function $\mathcal{L}_\mathcal{A}$, some dataset $\mathcal{D}$, and some parameter space $\mathcal{W}$. Furthermore, we consider the common scenario where the cumulative loss $\mathcal{L}_\mathcal{A}$ over the dataset is a finite sum of individual losses $\mathcal{\ell}_\mathcal{A}$. Thus, we denote the pretrained model as:

\begin{equation}
    \bm{w}^* = \mathcal{A}(D) := \arg\min_{\bm{w} \in \mathcal{W}} \; \mathcal{L}_{\mathcal{A}}(\bm{w}, \mathcal{D}) =  \arg\min_{\bm{w} \in \mathcal{W}} \sum_{i = 1}^{|\mathcal{D}|} \ell_{\mathcal{A}}(\bm{w}, \mathcal{D}^{(i)}).
    \label{eq:pretrained_model_erm}
\end{equation}

By \cref{prop:uniform_learner_exists}, we can similarly obtain a uniform learner through ERM with respect to some loss function $\mathcal{L}_\mathcal{K}$. Furthermore, in our setting, we have the forget set $\mathcal{D}_f$ and retain set $\mathcal{D}_r = \mathcal{D} \setminus \mathcal{D}_f$. Thus, the uniform learner over the forget set can be characterized as: 

\begin{equation}
    \mathcal{K}(\mathcal{D}_f) := \arg\min_{\bm{w} \in \mathcal{W}} \; \mathcal{L}_\mathcal{K}(\bm{w}, \mathcal{D}) = \sum_{i = 1}^{|\mathcal{D}_f|} \ell_{\mathcal{K}}(\bm{w}, \mathcal{D}_{f}^{(i)}).
     \label{eq:uniform_learner_erm}
\end{equation}

Let $\theta \in (0,1)$ be a tradeoff parameter between uniformity over the forget set and utility over the retain set. This yields a concrete characterization of $\mathcal{M}_\theta$ as: 

\begin{align}
    \bm{\tilde{w}^*} = \mathcal{M}_\theta(D) &:= \arg\min_{\bm{w} \in \mathcal{W}} \;\theta\mathcal{L}_{\mathcal{K}}(\bm{w}, \mathcal{D}_f) + (1-\theta)\mathcal{L}_{\mathcal{A}}(\bm{w}, \mathcal{D}_r),\\
    &= \arg\min_{\bm{w} \in \mathcal{W}} \theta \sum_{i = 1}^{|\mathcal{D}_f|} \ell_{\mathcal{K}}(\bm{w}, \mathcal{D}_f^{(i)}) + (1-\theta)\sum_{i = 1}^{|\mathcal{D}_r|} \ell_{\mathcal{A}}(\bm{w}, \mathcal{D}_r^{(i)}).
\end{align}

as in \cref{eq:pareto_optimal_with_weight_bound_and_reg}.

To design an algorithm which takes in $\mathcal{D}, \mathcal{D}_r,$ and $\bm{w}^*$ and outputs a parameter which satisfies \cref{defn:certified_uniformity}, we follow the methodology of certified unlearning \citet{zhang2024certifiedunlearningDNN},  which seeks to satisfy \cref{defn:certified_unlearning}.  

First, we simplify the problem of deriving a model that satisfies \cref{defn:certified_uniformity}: 

\begin{thm}
    (Certification Guarantee) Let $\bm{\tilde{w}} := \mathcal{F}(\mathcal{D}_, \mathcal{D}_f, \bm{w^*})$  be an approximation to $\bm{\tilde{w}^*}$. Suppose $||\bm{\tilde{w}} - \bm{\tilde{w}^*}||_2 \leq \Delta$. Then, $\mathcal{U}(\mathcal{D}, \mathcal{D}_f, A(D)) = \bm{w}^- = \bm{\tilde{w} + Y}$ is a $(\epsilon, \delta, \theta)$ certified uniformity algorithm, where $Y \sim \mathcal{N}(0, \sigma^2\bm{I})$ and $\sigma \geq \frac{\Delta}{\epsilon}\sqrt{2\ln(1.25/\delta)}$.
     \label{thm:unlearning_guarantee_for_pareto}
\end{thm}

\textit{Proof: } See \cref{section:proof_of_cert_guarantee}. 

Thus, it then suffices to find an approximation of $\bm{\tilde{w}^*}$, i.e. a form for $\mathcal{F}(\mathcal{D}, \mathcal{D}_f, \bm{w}^*)$ and its associated $\Delta$. To do so, we consider the two assumptions \cref{assumption1} and \cref{assumption2}.

For any $\bm{w} \in \mathcal{W}$, denote $\nabla_{\bm{w}, \mathcal{K}, \mathcal{A}} := \nabla_{\bm{w}} (\theta\mathcal{L}_\mathcal{K}(\bm{w}, \mathcal{D}_f) + (1-\theta)\mathcal{L}_\mathcal{A}(\bm{w}, \mathcal{D}_r))$, the gradient of the objective of $\mathcal{M}_\theta$ with respect to $\bm{w}$, and $\bm{H}_{\bm{w}, \mathcal{K}, \mathcal{A}} := \nabla_{w}^2 (\theta\mathcal{L}_\mathcal{K}(\bm{w}, \mathcal{D}_f) + (1-\theta)\mathcal{L}_\mathcal{A}(\bm{w}, \mathcal{D}_r))$, the Hessian of the objective of $\mathcal{M}_\theta$ with respect to $\bm{w}$. We thus have $\nabla_{\bm{w}, \mathcal{K}, \mathcal{A}} = \theta \nabla_{\bm{w}, \mathcal{K}} + (1-\theta)\nabla_{\bm{w}, \mathcal{A}}$, and similarly for the Hessian.

Next, letting $g(\bm{w}) := \nabla_{\bm{w}, \mathcal{K}, \mathcal{A}}$, by Taylor's theorem, expanding $g(\bm{\tilde{w}}^*)$ around $\bm{w}^*$, we have that: 

\begin{equation}
    g(\bm{\tilde{w}}^*) \approx g(\bm{w}^*) + Dg|_{\bm{w}^*}(\bm{\tilde{w}^*} - \bm{w}^*).
    \label{eq:taylor_expansion}
\end{equation}. 

Note that $g(\bm{\tilde{w}}^*) = 0$, since $\bm{\tilde{w}}^*$ is the minimizer of the objective in $\mathcal{M}_\theta$. Isolating $\bm{\tilde{w}}^*$ and using the definition of $g$, we then have that: 

\begin{equation}
    \bm{\tilde{w}^*} \approx \bm{w}^* - \bm{H}^{-1}_{\bm{w}^*, \mathcal{K}, \mathcal{A}} \nabla_{\bm{w}^*, \mathcal{K}, \mathcal{A}}.
    \label{eq:form_of_approx}
\end{equation}. 

Thus, we let $\bm{\tilde{w}} = \bm{w}^* - \bm{H}^{-1}_{\bm{w}^*, \mathcal{K}, \mathcal{A}} \nabla_{\bm{w}^*, \mathcal{K}, \mathcal{A}}$. This yields the following general form of $\Delta$: 

\begin{proposition}
    Suppose \cref{assumption1} and \cref{assumption2} hold. Suppose $\bm{\tilde{w}} = \bm{w}^* - \bm{H}^{-1}_{\bm{w}^*, \mathcal{K}, \mathcal{A}} \nabla_{\bm{w}^*, \mathcal{K}, \mathcal{A}}$. Then, 
    \begin{equation}
        ||\bm{\tilde{w}}^* - \bm{\tilde{w}}||_2 \leq \frac{\theta F_{\mathcal{K}} + (1-\theta)F_{\mathcal{A}}}{2}||\bm{H}^{-1}_{\bm{w}^*, 
        \mathcal{K}, \mathcal{A}}||_2 ||\bm{w}^* - \bm{\tilde{w}}^*||_2^2.
    \end{equation}
    \label{prop:approx_general_form}
\end{proposition}

\textit{Proof}: See \cref{section:proof_of_approx_general_form}. 

We then use local convex approximation \citep{nocedal1999numerical} to bound $||\bm{H}^{-1}_{\bm{w}^*, \mathcal{K}, \mathcal{A}}||_2$. To that end, we let the objective of $\mathcal{M}_\theta$ have a regularization term $\frac{\lambda}{2}||\bm{w}||_2^2$, yielding the inverse Hessian $||(\bm{H}_{\bm{w}^*, \mathcal{K}, \mathcal{A}} + \lambda \bm{I})^{-1}||_2$; thus, in \cref{prop:approx_general_form}, the norm of the inverse Hessian is replaced by $||(\bm{H}_{\bm{w}^*, \mathcal{K}, \mathcal{A}} + \lambda \bm{I})^{-1}||_2$. It then suffices to bound this term. 
 
Additionally, note that since the objective of $\mathcal{M}_\theta$ is nonconvex, the Hessian may not be invertible, i.e. $\lambda_{\min}(\bm{H}_{\bm{w}^*, \mathcal{K}, \mathcal{A}}) < 0$. However, $\lambda_{\min}(\bm{H}_{\bm{w}^*, \mathcal{K}, \mathcal{A}} + \lambda\bm{I}) = \lambda_{\min}(\bm{H}_{\bm{w}^*, \mathcal{K}, \mathcal{A}}) + \lambda$. Thus, for $\lambda$ sufficiently large, we can make $\bm{H}_{\bm{w}^*, \mathcal{K}, \mathcal{A}} + \lambda \bm{I}$ positive definite and hence invertible, resolving this issue. In particular, we can take $\lambda > ||\bm{H}_{\bm{w}^*, \mathcal{K}, \mathcal{A}}||_2$. 

Furthermore, we let $||\bm{w}||_2 \leq C$ in $\mathcal{M}_\theta$ and $\mathcal{A}$, i.e. $\mathcal{M}_\theta = \arg\min_{||\bm{w}||_2 \leq C, \bm{w} \in \mathcal{W}} \theta \mathcal{L}_{\mathcal{K}}(\bm{w}, \mathcal{D}_f) + (1-\theta)\mathcal{L}_{\mathcal{A}}(\bm{w}, \mathcal{D}_\mathcal{A}) + \frac{\lambda}{2}||\bm{w}||_2^2$ and $\mathcal{A}(\mathcal{D}) = \arg\min_{\bm{w} \in \mathcal{W}, ||\bm{w}|| \leq C} \mathcal{L}_\mathcal{A}(\bm{w}, \mathcal{D})$. Note that, as mentioned in \citep{zhang2024certifiedunlearningDNN}, unlearning methods implicitly assume this. 

Together, these two methods yield a tractable form of $\Delta$:  

\begin{proposition}
    Suppose \cref{assumption1} and \cref{assumption2} hold. Suppose $\bm{\tilde{w}} = \bm{w}^* - (\bm{H}_{\bm{w}^*, \mathcal{K}, \mathcal{A}} + \lambda \bm{I})^{-1}\nabla_{\bm{w}^*, \mathcal{K}, \mathcal{A}}$ where $||\bm{w}^*||_2, ||\bm{\tilde{w}^*}||_2 \leq C$. Let $\lambda_{\min} := \lambda_{\min}(\bm{H}_{\bm{w}^*, \mathcal{K}, \mathcal{A}})$. Suppose $\lambda > ||\bm{H}_{\bm{w}^*, \mathcal{K}, \mathcal{A}}||_2$. Then, 

    \begin{equation}
        ||\bm{\tilde{w}}^* - \bm{\tilde{w}}||_2 \leq \frac{2C((\theta M_K + (1-\theta)M_A)C + \lambda)}{\lambda + \lambda_{\min}}.
        \label{eq:algo_exact_hessian_eq}
    \end{equation}
    \label{prop:bound_after_cvx_approx_and_C}
\end{proposition}

\textit{Proof: } See \cref{section:proof_of_prop_bound_after_cvx_approx_and_C}.

While \cref{prop:bound_after_cvx_approx_and_C} does yield a form of $\mathcal{F}$ and $\Delta$, the computation of $\bm{\tilde{w}}$ requires obtaining the exact inverse Hessian, which has runtime $\mathcal{O}(dz^2 + z^3)$, where $z$ is the number of learnable parameters. Furthermore, computing the gradient product with the inverse Hessian is $\mathcal{O}(z^2)$. Finally, computing the gradient $\nabla_{\bm{w}^*, \mathcal{K}, \mathcal{A}}$ is $\mathcal{O}(|\mathcal{D}|z)$. Thus, the algorithm yielded by \cref{prop:bound_after_cvx_approx_and_C} has a runtime complexity of $\mathcal{O}(dz^2 + z^3 + z^2 + |\mathcal{D}|z)$. 

If we consider the additional assumption of convexity, we can take $\lambda$ very small to ensure the Hessian is invertible, since we have $\lambda_{\min} = 0$. Thus, for convex models e.g. logistic regression with a mean-square uniform loss, this is tractable. This yields \cref{algo:hess_exact_algo}. 

However, for nonconvex models e.g. large scale neural networks, this is computationally intractable. Thus, to provide better runtime, we derive an asymptotically unbiased estimator of the inverse Hessian. However, the estimator in \citet{zhang2024certifiedunlearningDNN} does not trivially extend to our case. In particular, we cannot glean Hessian samples using sampled i.i.d. data from the retain set, because the Hessian in our setting is defined over the forget set as well. Thus, we must derive an unbiased estimator while sampling Hessians from \textit{both} the retain and forget set. As such, following the techniques of \citep{agarwal2016second}, we design an unbiased estimator as follows: 

\begin{thm}
    Suppose we have $n$ i.i.d. data samples $(X_1, ..., X_n)$ drawn from $D_f$ and $D_r$, uniformly at random, with probabilities $\theta$ and $1-\theta$ respectively. Then, suppose $||\bm{H}_{w^*, \mathcal{K}, \mathcal{A}} + \lambda I||_2 \leq J$. For $t = 1, ..., n$, if $X_t \sim D_f$ let $\bm{H}_{t, \lambda} = \bm{H}_{w^*, \mathcal{K}, t} + \frac{\lambda I}{2\theta}$ and if $X_t \sim D_r$ let $\bm{H}_{t, \lambda} = \bm{H}_{w^*, \mathcal{A}, t} + \frac{\lambda I}{2(1-\theta)}$. Suppose $\lambda > ||\bm{H}_{\bm{w}^*, \mathcal{K}, \mathcal{A}}||_2$. Then, compute: 

    \begin{equation}
        \bm{\tilde{H}}_{t, \lambda}^{-1} = \bm{I} + (\bm{I} - \frac{\bm{H}_{t, \lambda}}{J})\bm{\tilde{H}}_{t-1, \lambda}^{-1}, \; \bm{\tilde{H}}_{0, \lambda} = \bm{I}.
    \end{equation}

    Then, $\frac{\bm{\tilde{H}}_{n, \lambda}^{-1}}{J}$ is an asymptotically unbiased estimator for $(\bm{H}_{\bm{w}^*, \mathcal{K}, \mathcal{A}} + \lambda \bm{I})^{-1}$

    \label{thm:asymptotically_unbiased_estimator}
\end{thm}

\textit{Proof: } See \cref{section:proof_of_asymptotically_unbiased_estimator}

One simple choice of $J$ is $J = 2\lambda$, by \cref{lemma:lipschitzness_of_hessians_and_grads}. However, we let $J$ be free. The computation of the estimator in \cref{thm:asymptotically_unbiased_estimator} has a runtime complexity of $\mathcal{O}(nz^2)$, a great speedup over the original $\mathcal{O}(dz^2 + z^3)$. Furthermore, with Hessian vector product (HVP) techniques \citep{pearlmutter1994fast}, we obtain a space complexity of $\mathcal{O}(z)$ instead of $\mathcal{O}(z^2)$, since we do not have to compute the sample Hessians explicitly.  Furthermore, computing $\tilde{\bm{H}}_{n, \lambda}^{-1}\nabla_{\bm{w}^, \mathcal{K}, \mathcal{A}}$ recursively reduces $\mathcal{O}(z^2)$ to $\mathcal{O}(nz)$.

Additionally, following \citet{agarwal2016second}, we can average $b$ unbiased estimators $\frac{\bm{\tilde{H}}_{t, \lambda}^{-1}}{J}$ as $\frac{1}{b}\sum_{i = 1}^b \frac{\bm{\tilde{H}}_{t, \lambda}^{-1, (i)}}{J}$ to achieve better concentration.  Altogether, we achieve a final runtime complexity of $\mathcal{O}(bnz^2 + bnz +|\mathcal{D}|z)$.

Furthermore, we relax the assumption that $\bm{w}^*$ and $\bm{\tilde{w}}^*$ are the global minimizers of $\mathcal{L}_\mathcal{A}$ and $\theta\mathcal{L}_{\mathcal{A}} + (1-\theta)\mathcal{L}_\mathcal{K}$. We do so because, in practice, it is possible that the data controller trained their model with early stopping, i.e. they did not reach the global minimizer. Altogether, this yields a final form of $\Delta$ as:

\begin{thm}
    Let $\bm{\tilde{w}^*}$ and $\bm{w}^*$ not be empirical risk minimizers of their respective losses, but rather approximations thereof. Suppose \cref{assumption1} and \cref{assumption2} hold. Suppose $||\bm{w}^*||_2, ||\bm{\tilde{w}^*}||_2 \leq C$. Let $\lambda_{\min} := \lambda_{\min}(\bm{H}_{\bm{w}^*, \mathcal{K}, \mathcal{A}})$. Suppose $\lambda > ||\bm{H}_{\bm{w}^*, \mathcal{K}, \mathcal{A}}||_2$.
    Let $\bm{\tilde{w}} = \bm{w}^* - \frac{\bm{\tilde{H}}_{t, \lambda}^{-1}}{J}\nabla_{\bm{w}^*, \mathcal{K}, \mathcal{A}}$. Let $b$ be the number of inverse Hessian estimators we average. Letting $n$ be the number of steps taken during unbiased estimation of the inverse Hessian, require $n \geq 2\frac{B}{\lambda + \lambda_{\min}}\ln(\frac{B}{\lambda + \lambda_{\min}}b)$ where $B = \max\{\frac{\theta P_{\mathcal{K}} + \lambda}{|\mathcal{D}_f|}, \frac{(1-\theta)P_{\mathcal{A}} + \lambda}{|\mathcal{D}_r|}\}$. Suppose $||\nabla_{\bm{w}^*, \mathcal{K}, \mathcal{A}}||_2, ||\nabla_{\bm{\tilde{w}}^*, \mathcal{K}, \mathcal{A}}||_2 \leq G$, With probability larger than $1 - \rho$, we have that: 

    \begin{align}
        ||\bm{\tilde{w}^*} - \bm{\tilde{w}}||_2 &\leq \frac{2C((\theta F_{\mathcal{K}} + (1-\theta)F_{\mathcal{A}})C + \lambda) + G}{\lambda + \lambda_{\min}} \\ &+ (16 \frac{B}{\zeta_{\min}} \sqrt{\frac{\ln(\frac{d}{\rho})}{b}} + \frac{1}{16})(2C(\theta P_{\mathcal{K}} + (1-\theta)P_{\mathcal{A}}) + G).
    \label{eq:hessian_estimator_bound}
    \end{align}

    where $\zeta_{\min} \geq \min_i \lambda_{\min}(\nabla^2_{\bm{w}} \tilde{\ell}_{\mathcal{K},\mathcal{A}}(\bm{w}, \mathcal{D}^{(i)}))$.
    
    \label{thm:bound_with_hessian_estimator}
\end{thm}

\textit{Proof: } See \cref{section:proof_of_bound_with_hessian_estimator}.

Note that if we let $\bm{w}^*$ be an ERM in \cref{thm:bound_with_hessian_estimator}, we can use $\nabla_{\bm{w}, \mathcal{K}, \mathcal{A}}$ and obtain the same result. Altogether, this yields \cref{algo:hess_estimator_algo}.

\begin{algorithm}[tb]
\caption{$(\epsilon, \delta, \theta)$-Certified Uniformity with Inverse Hessian Estimator}

\begin{algorithmic}
\Require Dataset $\mathcal{D}$; forget set $\mathcal{D}_f$; pretrained model $\bm{w}^* = \mathcal{A}(\mathcal{D})$; privacy budgets $\epsilon$ and $\delta$; uniformity-utility tradeoff coefficient $\theta$; estimator concentration $b$; sample size $n$; local convex coefficient $\lambda$; norm upper bound $C$; cumulative Hessian upper bound $H$; individual Hessian minimum eigenvalue upper bound $\zeta_{\min}$; gradient norm upper bound $G$; bound looseness probability $\rho$. 

\State $\bm{P}_{0, \lambda}^{(0)} \gets \nabla_{\bm{w}^*, \mathcal{K}, \mathcal{A}}$

\For{$j = 1, ..., b$}

\For{$t = 1, ..., n$}
    \State Sample $X_t$ from $D_f$ uniformly with probability $\theta$ or,
    \State sample $X_t$ from $D_r$ uniformly with probability $1 - \theta$ .

    \If{$X_t \sim D_f$}
        \State $\bm{H}_{t, \lambda}^{(j)} \gets \nabla^2_{\bm{w}}\mathcal{L}_\mathcal{K}(\bm{w}^*, X_i) + \frac{\lambda \bm{I}}{2\theta}$.
    \ElsIf{$X_t \sim D_r$}
        \State $\bm{H}_{t, \lambda}^{(j)} \gets \nabla^2_{\bm{w}}\mathcal{L}_\mathcal{A}(\bm{w}^*, X_i) + \frac{\lambda \bm{I}}{2(1-\theta)}$.
    \EndIf

    $\bm{P}_{t, \lambda}^{(j)} = \bm{P}_{0, \lambda}^{(0)} + (\bm{I} - \frac{\bm{H}_{t, \lambda}^{(j)}}{H})\bm{P}_{t-1, \lambda}^{(j)}$. 
\EndFor

\EndFor 

\State $\bm{P}_{n, \lambda} \gets \frac{1}{b} \sum_{j = 1}^b \bm{P}_{n, \lambda}^{(j)}$.

\State $\bm{\tilde{w}} \gets \bm{w}^* - \frac{\bm{P}_{n, \lambda}}{H}$.

\State Compute $\Delta$ as the bound in \cref{eq:hessian_estimator_bound}. 

\State $\sigma = \frac{\Delta}{\epsilon}\sqrt{2\ln(1.25/\delta)}$
\State $\bm{w}^- \gets \bm{\tilde{w}} + Y$ where $Y \sim \mathcal{N}(\bm{0}, \sigma^2\bm{I})$.

\State \Return $\bm{w}^-$.

\end{algorithmic}

\label{algo:hess_estimator_algo}

\end{algorithm}

%% file: sections/attack_figure.tex
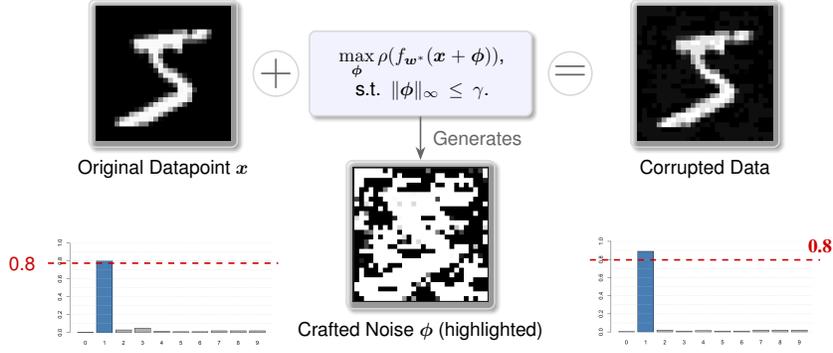
\begin{figure}[t]
    \centering %
    \resizebox{0.8\textwidth}{!}{%
    
\begin{tikzpicture}[font=\sffamily, scale=0.95, transform shape]
    \colorlet{primaryBlue}{blue!60!black}
    \colorlet{accentRed}{red!80!black}
    \colorlet{darkGray}{gray!80!black}
    \colorlet{lightGray}{gray!30}
    \tikzset{
      stage_title/.style={font=\large\bfseries, text=darkGray},
      op_symbol/.style={circle, draw=lightGray, thick, font=\Huge, text=darkGray, minimum size=0.8cm, inner sep=0pt},
      content_box/.style={draw=lightGray, thick, rounded corners=3pt, blur shadow={shadow blur steps=5, shadow xshift=1.5pt, shadow yshift=-1.5pt}},
      main_arrow/.style={->, >=Stealth, thick, darkGray},
      threshold_line/.style={dashed, thick}
    }

    \node[content_box] (img_original) at (0, 2.5) {\includegraphics[width=2.5cm]{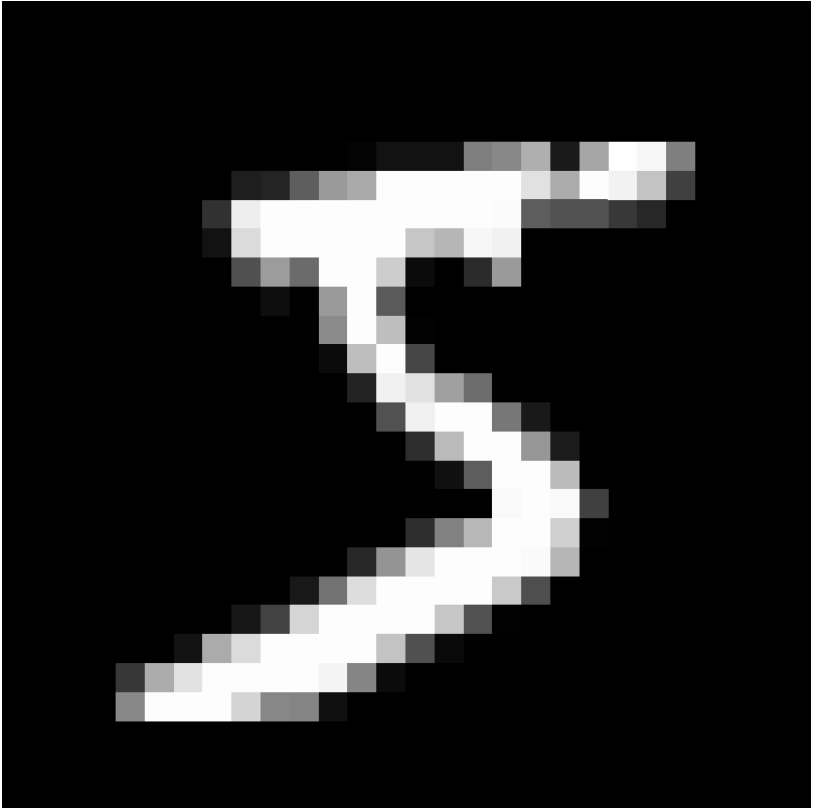}};
    \node (label_original) [below=0.1cm of img_original] {Original Datapoint $\bm{x}$};
    \node (chart1) [below=0.8cm of label_original] {\includegraphics[width=4cm]{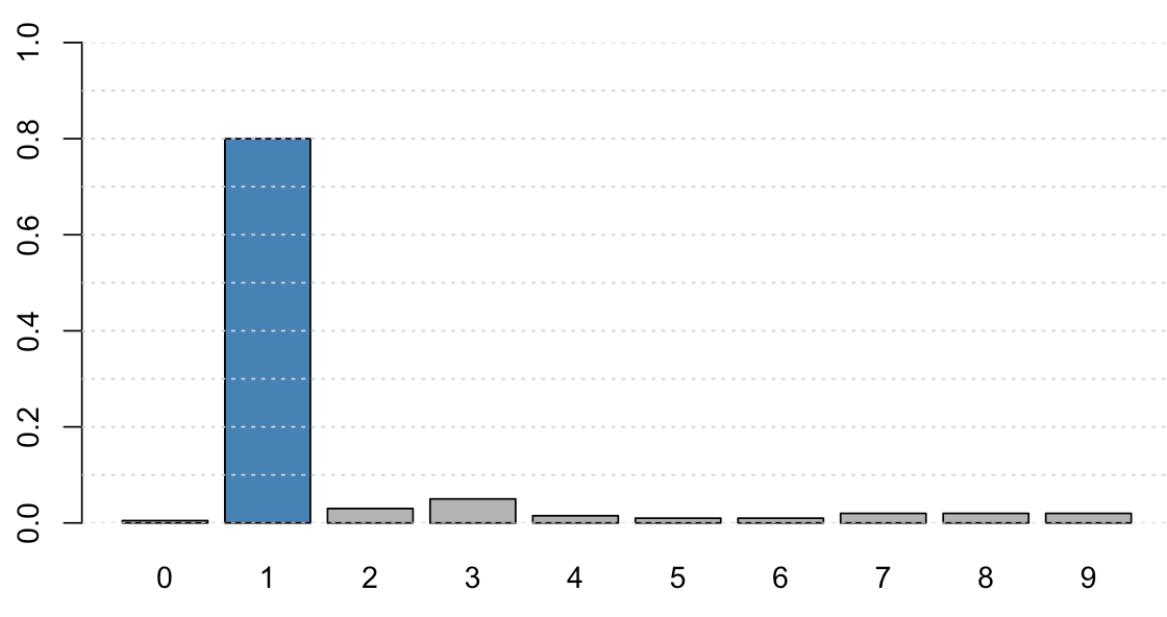}};
    \draw[threshold_line, accentRed] ([yshift=1.72cm]chart1.south west) node[left, text=accentRed, xshift=-0.1cm] {0.8} -- ([yshift=1.72cm]chart1.south east);

    \node[op_symbol] (plus_op) at ([xshift=0.7cm]img_original.east) {$+$};
    \node[content_box, fill=blue!5] (math_box) at (4.75, 2.5) [text width=3.5cm, align=center, inner sep=0.3cm] {
        $\max \limits_{\bm{\phi}} \rho( f_{\bm{w}^*}(\bm{x} + \bm{\phi}))$, \\
        s.t. $\|\bm{\phi}\|_\infty \leq \gamma$.
    };
    \node[content_box] (img_noise) [below=0.8cm of math_box] {\includegraphics[width=2.5cm]{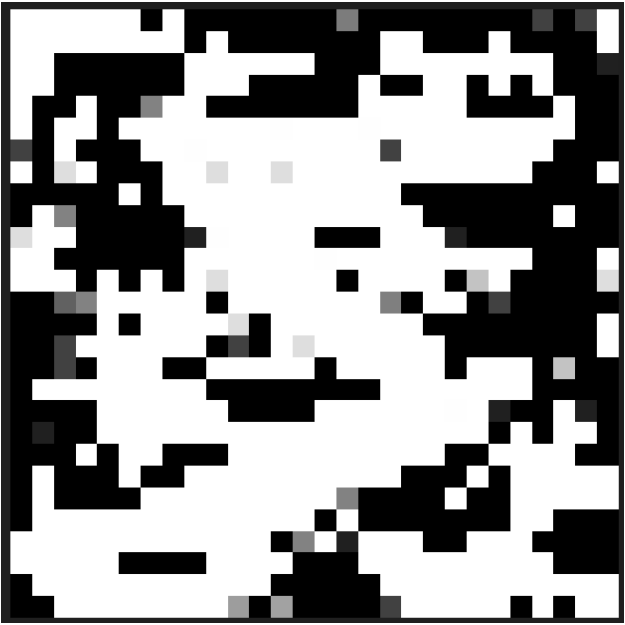}};
    \node (label_noise) [below=0.1cm of img_noise] {Crafted Noise $\bm{\phi}$ (highlighted)};
    \draw[main_arrow] (math_box) -- (img_noise) node[midway, right, xshift=0.1cm] {Generates};

    \node[content_box] (img_corrupted) at (10, 2.5) {\includegraphics[width=2.5cm]{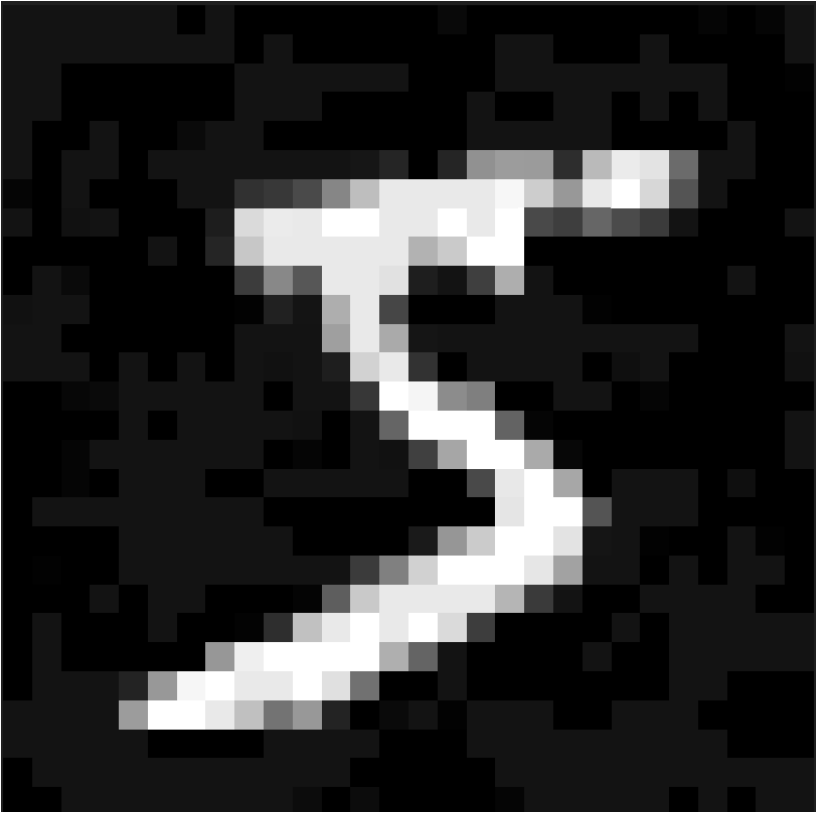}};
    \node[op_symbol] (equals_op) at ([xshift=0.7cm]math_box.east) {$=$};
    \node (label_corrupted) [below=0.1cm of img_corrupted] {Corrupted Data};
    \node (chart2) [below=0.8cm of label_corrupted] {\includegraphics[width=4cm]{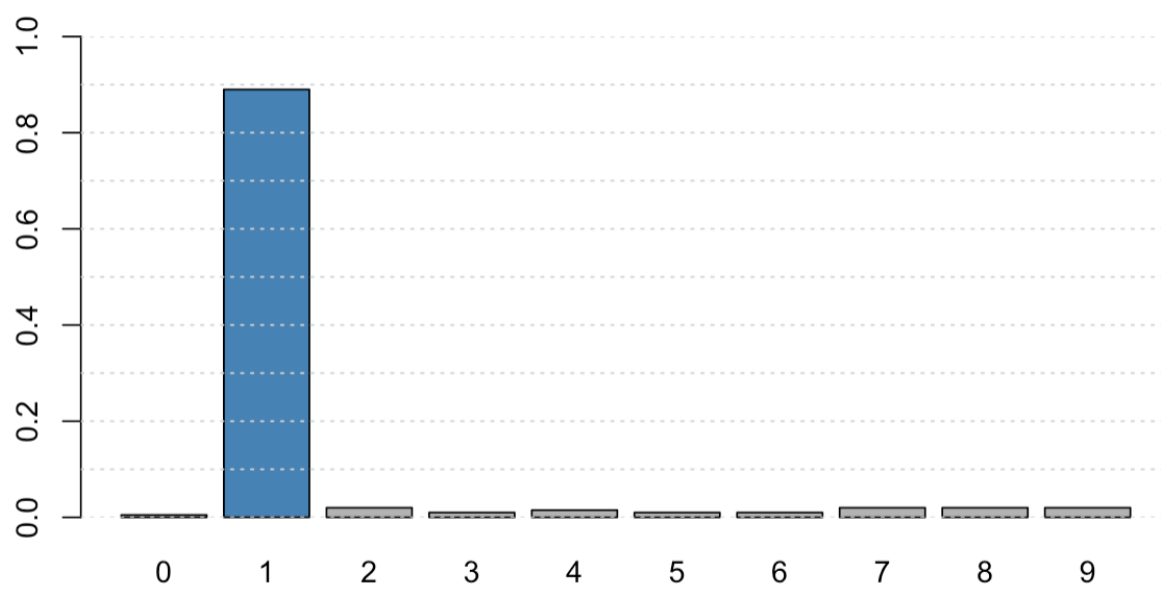}};
    \draw[threshold_line, accentRed] ([yshift=1.72cm]chart2.south west) -- ([yshift=1.72cm]chart2.south east) node[above, font=\bfseries] {0.8};
    
\end{tikzpicture}
} %

\caption{We corrupt an instance to increase the confidence of the final prediction. Noise is highlighted throughout for clarity.}
\label{fig:attack_fig}
\end{figure}

%% file: main.bbl
\begin{thebibliography}{80}
\providecommand{\natexlab}[1]{#1}
\providecommand{\url}[1]{\texttt{#1}}
\expandafter\ifx\csname urlstyle\endcsname\relax
  \providecommand{\doi}[1]{doi: #1}\else
  \providecommand{\doi}{doi: \begingroup \urlstyle{rm}\Url}\fi

\bibitem[Abadi et~al.(2016)Abadi, Chu, Goodfellow, McMahan, Mironov, Talwar,
  and Zhang]{abadi2016deep}
Martin Abadi, Andy Chu, Ian Goodfellow, H~Brendan McMahan, Ilya Mironov, Kunal
  Talwar, and Li~Zhang.
\newblock Deep learning with differential privacy.
\newblock \emph{Proceedings of the Conference on on Computer and
  Communiscations Security (CCS)}, pages 308--318, 2016.

\bibitem[Agarwal et~al.(2016)Agarwal, Bullins, and Hazan]{agarwal2016second}
Naman Agarwal, Brian Bullins, and Elad Hazan.
\newblock Second-order stochastic optimization in linear time.
\newblock \emph{stat}, 1050:\penalty0 15, 2016.

\bibitem[Al-Rubaie and Chang(2019)]{al2019privacy}
Mohammad Al-Rubaie and J~Morris Chang.
\newblock Privacy-preserving machine learning: threats and solutions.
\newblock \emph{IEEE Security \& Privacy}, 17\penalty0 (2):\penalty0 49--58,
  2019.

\bibitem[Amin et~al.(2024)Amin, Kulesza, and Vassilvitskii]{amin2024practical}
Kareem Amin, Alex Kulesza, and Sergei Vassilvitskii.
\newblock Practical considerations for differential privacy, 2024.
\newblock URL \url{https://arxiv.org/abs/2408.07614}.

\bibitem[Angelopoulos et~al.(2025)Angelopoulos, Jordan, and
  Tibshirani]{angelopoulos2025gradient}
Anastasios~N Angelopoulos, Michael~I Jordan, and Ryan~J Tibshirani.
\newblock Gradient euilibrium in online learning: theory and applications.
\newblock \emph{arXiv preprint arXiv:2501.08330}, 2025.

\bibitem[Aono et~al.(2017)Aono, Hayashi, Wang, Moriai, et~al.]{aono2017privacy}
Yoshinori Aono, Takuya Hayashi, Lihua Wang, Shiho Moriai, et~al.
\newblock Privacy-preserving deep learning via additively homomorphic
  encryption.
\newblock \emph{IEEE Transactions on Information Forensics and Security},
  13\penalty0 (5):\penalty0 1333--1345, 2017.

\bibitem[Baig and Pietrzak(2025)]{baigsecurity2025}
Mirza~Ahad Baig and Krzysztof Pietrzak.
\newblock On the (in)security of proofs-of-space based longest-chain
  blockchains.
\newblock Cryptology {ePrint} Archive, Paper 2025/942, 2025.
\newblock URL \url{https://eprint.iacr.org/2025/942}.

\bibitem[Balle and Wang(2018)]{improvedgaussian2018}
Borja Balle and Yu-Xiang Wang.
\newblock Improving the gaussian mechanism for differential privacy: analytical
  calibration and optimal denoising.
\newblock \emph{International Conference on Machine Learning (ICML)}, 2018.

\bibitem[Bourtoule et~al.(2021)Bourtoule, Chandrasekaran, Choquette-Choo, Jia,
  Travers, Zhang, Lie, and Papernot]{bourtoule2021machine}
Lucas Bourtoule, Varun Chandrasekaran, Christopher~A Choquette-Choo, Hengrui
  Jia, Adelin Travers, Baiwu Zhang, David Lie, and Nicolas Papernot.
\newblock Machine unlearning.
\newblock \emph{IEEE Symposium on Security and Privacy (SP)}, pages 141--159,
  2021.

\bibitem[Brakerski et~al.(2014)Brakerski, Gentry, and
  Vaikuntanathan]{brakerski2014leveled}
Zvika Brakerski, Craig Gentry, and Vinod Vaikuntanathan.
\newblock (leveled) fully homomorphic encryption without bootstrapping.
\newblock \emph{Transactions on Computation Theory (TOCT)}, 6\penalty0
  (3):\penalty0 1--36, 2014.

\bibitem[Bun et~al.(2021)Bun, Desfontaines, Dwork, Moni, Nissim, Roth, Smith,
  Steinke, Ullman, and Vadhan]{bun2021notaprivviolation}
Mark Bun, Damien Desfontaines, Cynthia Dwork, Naor Moni, Kobbi Nissim, Aaron
  Roth, Adam Smith, Thomas Steinke, Jonathan Ullman, and Salil Vadhan.
\newblock Statistical inference is not a privacy violation, 2021.

\bibitem[Busa-Fekete et~al.(2021)Busa-Fekete, Syed, Vassilvitskii,
  et~al.]{busa2021pitfalls}
Robert~Istvan Busa-Fekete, Umar Syed, Sergei Vassilvitskii, et~al.
\newblock On the pitfalls of label differential privacy.
\newblock In \emph{NeurIPS Workshops}, 2021.

\bibitem[Cha et~al.(2024)Cha, Cho, Hwang, Lee, Moon, and Lee]{cha2024learning}
Sungmin Cha, Sungjun Cho, Dasol Hwang, Honglak Lee, Taesup Moon, and Moontae
  Lee.
\newblock Learning to unlearn: instance-wise unlearning for pre-trained
  classifiers.
\newblock \emph{AAAI Conference on Artificial Intelligence}, 38\penalty0
  (10):\penalty0 11186--11194, 2024.

\bibitem[Chaudhuri and Monteleoni(2008)]{chaudhuri2008privacy}
Kamalika Chaudhuri and Claire Monteleoni.
\newblock Privacy-preserving logistic regression.
\newblock \emph{Advances in Neural Information Processing Systems (NeurIPS)},
  21, 2008.

\bibitem[Chaudhuri et~al.(2011)Chaudhuri, Monteleoni, and
  Sarwate]{chaudhuri2011differentially}
Kamalika Chaudhuri, Claire Monteleoni, and Anand~D Sarwate.
\newblock Differentially private empirical risk minimization.
\newblock \emph{Journal of Machine Learning Research}, 12\penalty0 (3), 2011.

\bibitem[Chua et~al.(2024)Chua, Ghazi, Kamath, Kumar, Manurangsi, Sinha, and
  Zhang]{chua2024scalable}
Lynn Chua, Badih Ghazi, Pritish Kamath, Ravi Kumar, Pasin Manurangsi, Amer
  Sinha, and Chiyuan Zhang.
\newblock Scalable dp-sgd: shuffling vs. poisson subsampling.
\newblock \emph{Advances in Neural Information Processing Systems (NeurIPS)},
  37:\penalty0 70026--70047, 2024.

\bibitem[Clanuwat et~al.(2018)Clanuwat, Bober-Irizar, Kitamoto, Lamb, Yamamoto,
  and Ha]{clanuwat2018deep}
Tarin Clanuwat, Mikel Bober-Irizar, Asanobu Kitamoto, Alex Lamb, Kazuaki
  Yamamoto, and David Ha.
\newblock Deep learning for classical japanese literature.
\newblock \emph{Advances in Neural Information Processing Systems (NeurIPS)},
  2018.

\bibitem[Deng(2012)]{mnist}
Li~Deng.
\newblock The mnist database of handwritten digit images for machine learning
  research.
\newblock \emph{IEEE Signal Processing}, 29\penalty0 (6):\penalty0 141--142,
  2012.

\bibitem[Dinur and Nissim(2003)]{dinur2003revealing}
Irit Dinur and Kobbi Nissim.
\newblock Revealing information while preserving privacy.
\newblock \emph{Proceedings of the Symposium on Principles of Databases
  (PODS)}, pages 202--210, 2003.

\bibitem[Dosovitskiy et~al.(2021)Dosovitskiy, Beyer, Kolesnikov, Weissenborn,
  Zhai, Unterthiner, Dehghani, Minderer, Heigold, Gelly, Uszkoreit, and
  Houlsby]{dosovitskiy2021an}
Alexey Dosovitskiy, Lucas Beyer, Alexander Kolesnikov, Dirk Weissenborn,
  Xiaohua Zhai, Thomas Unterthiner, Mostafa Dehghani, Matthias Minderer, Georg
  Heigold, Sylvain Gelly, Jakob Uszkoreit, and Neil Houlsby.
\newblock An image is worth 16x16 words: Transformers for image recognition at
  scale.
\newblock \emph{International Conference on Learning Representations (ICLR)},
  2021.

\bibitem[Dwork et~al.(2006)Dwork, McSherry, Nissim, and
  Smith]{dwork2006calibrating}
Cynthia Dwork, Frank McSherry, Kobbi Nissim, and Adam Smith.
\newblock Calibrating noise to sensitivity in private data analysis.
\newblock \emph{Proceedings of the Theory of Cryptography Conference (TCC)},
  pages 265--284, 2006.

\bibitem[Dwork et~al.(2014)Dwork, Roth, et~al.]{dwork2014algorithmic}
Cynthia Dwork, Aaron Roth, et~al.
\newblock The algorithmic foundations of differential privacy.
\newblock \emph{Foundations and Trends{\textregistered} in Theoretical Computer
  Science}, 9\penalty0 (3-4):\penalty0 211--407, 2014.

\bibitem[{European Parliament}(2016)]{GDPR2016}
{European Parliament}.
\newblock Regulation ({EU}) 2016/679 of the {European} {Parliament} and of the
  {Council}, 2016.
\newblock URL \url{https://data.europa.eu/eli/reg/2016/679/oj}.

\bibitem[Fredrikson et~al.(2015)Fredrikson, Jha, and
  Ristenpart]{fredrikson2015model}
Matt Fredrikson, Somesh Jha, and Thomas Ristenpart.
\newblock Model inversion attacks that exploit confidence information and basic
  countermeasures.
\newblock In \emph{Proceedings of the Conference on on Computer and
  Communiscations Security (CCS)}, pages 1322--1333, 2015.

\bibitem[Ghazi et~al.(2021)Ghazi, Golowich, Kumar, Manurangsi, and
  Zhang]{ghazi2021deep}
Badih Ghazi, Noah Golowich, Ravi Kumar, Pasin Manurangsi, and Chiyuan Zhang.
\newblock Deep learning with label differential privacy.
\newblock \emph{Advances in Neural Information Processing Systems (NeurIPS)},
  34:\penalty0 27131--27145, 2021.

\bibitem[Goodfellow et~al.(2015)Goodfellow, Shlens, and
  Szegedy]{goodfellow2014explaining}
Ian~J Goodfellow, Jonathon Shlens, and Christian Szegedy.
\newblock Explaining and harnessing adversarial examples.
\newblock \emph{International Conference on Learning Representations (ICLR)},
  2015.

\bibitem[Google(2023)]{googlevit}
Google.
\newblock Vision transformer pretrained on imagenet-21k, 2023.
\newblock URL \url{https://huggingface.co/google/vit-base-patch16-224}.

\bibitem[Guo et~al.(2017)Guo, Pleiss, Sun, and Weinberger]{guo2017calibration}
Chuan Guo, Geoff Pleiss, Yu~Sun, and Kilian~Q Weinberger.
\newblock On calibration of modern neural networks.
\newblock In \emph{International Conference on Machine Learning (ICML)}, 2017.

\bibitem[He et~al.(2016)He, Zhang, Ren, and Sun]{he2016deep}
Kaiming He, Xiangyu Zhang, Shaoqing Ren, and Jian Sun.
\newblock Deep residual learning for image recognition.
\newblock \emph{Conference on Computer Vision and Pattern Recognition (CVPR)},
  pages 770--778, 2016.

\bibitem[Hwang and Masud(2012)]{hwang2012multiple}
C-L Hwang and Abu Syed~Md Masud.
\newblock \emph{Multiple objective decision making—methods and applications:
  a state-of-the-art survey}, volume 164.
\newblock Springer Science \& Business Media, 2012.

\bibitem[Kamath(2020)]{gautamlecture}
Gautam Kamath.
\newblock Lecture 12: what is privacy?
\newblock \emph{Lectures on private ml and stats}, 2020.

\bibitem[Krizhevsky et~al.(2009)Krizhevsky, Hinton,
  et~al.]{krizhevsky2009learning}
Alex Krizhevsky, Geoffrey Hinton, et~al.
\newblock Learning multiple layers of features from tiny images.
\newblock \emph{Technical Report from the University of Toronto}, 2009.

\bibitem[Kullback and Leibler(1951)]{kullback1951information}
Solomon Kullback and Richard~A Leibler.
\newblock On information and sufficiency.
\newblock \emph{The Anals of Mathematical Statistics}, 22\penalty0
  (1):\penalty0 79--86, 1951.

\bibitem[Kurmanji et~al.(2023)Kurmanji, Triantafillou, Hayes, and
  Triantafillou]{kurmanji2023towards}
Meghdad Kurmanji, Peter Triantafillou, Jamie Hayes, and Eleni Triantafillou.
\newblock Towards unbounded machine unlearning.
\newblock \emph{Advances in Neural Information Processing Systems (NeurIPS)},
  2023.

\bibitem[Le and Yang(2015)]{Le2015tinyimagenet}
Ya~Le and Xuan~S. Yang.
\newblock Tiny imagenet visual recognition challenge, 2015.

\bibitem[Lee et~al.(2022)Lee, Kang, Lee, Choi, Eom, Deryabin, Lee, Lee, Yoo,
  Kim, et~al.]{lee2022privacy}
Joon-Woo Lee, HyungChul Kang, Yongwoo Lee, Woosuk Choi, Jieun Eom, Maxim
  Deryabin, Eunsang Lee, Junghyun Lee, Donghoon Yoo, Young-Sik Kim, et~al.
\newblock Privacy-preserving machine learning with fully homomorphic encryption
  for deep neural network.
\newblock \emph{IEEE Access}, 10:\penalty0 30039--30054, 2022.

\bibitem[Li et~al.(2012)Li, Qardaji, and Su]{li2012sampling}
Ninghui Li, Wahbeh Qardaji, and Dong Su.
\newblock On sampling, anonymization, and differential privacy or,
  k-anonymization meets differential privacy.
\newblock \emph{Proceedings of the Conference on on Computer and
  Communiscations Security (CCS)}, pages 32--33, 2012.

\bibitem[Liu et~al.(2025)Liu, Dou, Wang, Zou, Sen, Liu, and Li]{liu2025skin}
Hui Liu, Yibo Dou, Kai Wang, Yunmin Zou, Gan Sen, Xiangtao Liu, and Huling Li.
\newblock A skin disease classification model based on multi scale combined
  efficient channel attention module.
\newblock \emph{Scientific Reports}, 15\penalty0 (1):\penalty0 6116, 2025.

\bibitem[Madry et~al.(2018)Madry, Makelov, Schmidt, Tsipras, and
  Vladu]{madry2017towards}
Aleksander Madry, Aleksandar Makelov, Ludwig Schmidt, Dimitris Tsipras, and
  Adrian Vladu.
\newblock Towards deep learning models resistant to adversarial attacks.
\newblock \emph{International Conference on Learning Representations (ICLR)},
  2018.

\bibitem[Mayer(1985)]{mayer1985convergence}
G{\"u}nter Mayer.
\newblock On the convergence of the neumann series in interval analysis.
\newblock \emph{Linear algebra and its applications}, 65:\penalty0 63--70,
  1985.

\bibitem[Mcsherry(2016)]{mcsherry2016}
Frank Mcsherry.
\newblock Statistical inference considered harmful, 2016.
\newblock URL
  \url{https://github.com/frankmcsherry/blog/blob/master/posts/2016-06-14.md}.

\bibitem[Merriam-Webster(2022)]{privacydefn}
Merriam-Webster.
\newblock Privacy.
\newblock \emph{Merriam-Webster Dictionary}, 2022.
\newblock URL \url{https://www.merriam-webster.com/dictionary/privacy}.

\bibitem[Microsoft(2024)]{microsoftresnet50}
Microsoft.
\newblock Resnet50 pretrained on imagenet-21k, 2024.
\newblock URL \url{https://huggingface.co/microsoft/resnet-50}.

\bibitem[Miettinen(1999)]{miettinen1999nonlinear}
Kaisa Miettinen.
\newblock \emph{Nonlinear multiobjective optimization}, volume~12.
\newblock Springer Science \& Business Media, 1999.

\bibitem[Netzer et~al.(2011)Netzer, Wang, Coates, Bissacco, Wu, Ng,
  et~al.]{netzer2011reading}
Yuval Netzer, Tao Wang, Adam Coates, Alessandro Bissacco, Baolin Wu, Andrew~Y
  Ng, et~al.
\newblock Reading digits in natural images with unsupervised feature learning.
\newblock \emph{NeurIPS Workshops}, page~4, 2011.

\bibitem[Nguyen et~al.(2022)Nguyen, Huynh, Ren, Nguyen, Liew, Yin, and
  Nguyen]{nguyen2022survey}
Thanh~Tam Nguyen, Thanh~Trung Huynh, Zhao Ren, Phi~Le Nguyen, Alan Wee-Chung
  Liew, Hongzhi Yin, and Quoc Viet~Hung Nguyen.
\newblock A survey of machine unlearning.
\newblock \emph{arXiv preprint arXiv:2209.02299}, 2022.

\bibitem[Nocedal and Wright(1999)]{nocedal1999numerical}
Jorge Nocedal and Stephen~J Wright.
\newblock \emph{Numerical optimization}.
\newblock Springer, 1999.

\bibitem[of~Privacy~Professionals(2020)]{iappgooglegdpr2020}
International~Association of~Privacy~Professionals.
\newblock Swedish court rejects google's appeal in rtbf case, 2020.
\newblock URL
  \url{https://iapp.org/news/a/swedish-court-rejects-googles-appeal-in-rtbf-case}.

\bibitem[Papernot et~al.(2018)Papernot, Song, Mironov, Raghunathan, Talwar, and
  Erlingsson]{papernot2018scalable}
Nicolas Papernot, Shuang Song, Ilya Mironov, Ananth Raghunathan, Kunal Talwar,
  and {\'U}lfar Erlingsson.
\newblock Scalable private learning with pate.
\newblock \emph{International Conference on Learning Representations (ICLR)},
  2018.

\bibitem[Pardalos et~al.(2017)Pardalos, {\v{Z}}ilinskas, {\v{Z}}ilinskas,
  et~al.]{pardalos2017multiobjective}
Panos~M Pardalos, Antanas {\v{Z}}ilinskas, Julius {\v{Z}}ilinskas, et~al.
\newblock \emph{Non-convex multi-objective optimization}.
\newblock Springer, 2017.

\bibitem[Pearce et~al.(2021)Pearce, Brintrup, and Zhu]{pearce2021understanding}
Tim Pearce, Alexandra Brintrup, and Jun Zhu.
\newblock Understanding softmax confidence and uncertainty.
\newblock \emph{arXiv preprint arXiv:2106.04972}, 2021.

\bibitem[Pearlmutter(1994)]{pearlmutter1994fast}
Barak~A Pearlmutter.
\newblock Fast exact multiplication by the hessian.
\newblock \emph{Neural computation}, 6\penalty0 (1):\penalty0 147--160, 1994.

\bibitem[Pereyra et~al.(2017)Pereyra, Tucker, Chorowski, Łukasz Kaiser, and
  Hinton]{pereyra2017regularizingneuralnetworkspenalizing}
Gabriel Pereyra, George Tucker, Jan Chorowski, Łukasz Kaiser, and Geoffrey
  Hinton.
\newblock Regularizing neural networks by penalizing confident output
  distributions.
\newblock \emph{International Conference on Learning Representations (ICLR)},
  2017.

\bibitem[Pinsker(1964)]{pinsker1964information}
Mark~S Pinsker.
\newblock Information and information stability of random variables and
  processes.
\newblock \emph{Holden-Day}, 1964.

\bibitem[Qiao et~al.(2025)Qiao, Zhang, Tang, and Wei]{qiao2025hessianfree}
Xinbao Qiao, Meng Zhang, Ming Tang, and Ermin Wei.
\newblock Hessian-free online certified unlearning.
\newblock \emph{International Conference on Learning Representations (ICLR)},
  2025.

\bibitem[Schoepf et~al.(2025)Schoepf, Mozer, Mitchell, Brintrup, Kaissis,
  Kairouz, and Triantafillou]{schoepf2025redirection}
Stefan Schoepf, Michael~Curtis Mozer, Nicole~Elyse Mitchell, Alexandra
  Brintrup, Georgios Kaissis, Peter Kairouz, and Eleni Triantafillou.
\newblock Redirection for erasing memory (rem): Towards a universal unlearning
  method for corrupted data.
\newblock \emph{arXiv preprint arXiv:2505.17730}, 2025.

\bibitem[Sekhari et~al.(2021)Sekhari, Acharya, Kamath, and
  Suresh]{sekhari2021remember}
Ayush Sekhari, Jayadev Acharya, Gautam Kamath, and Ananda~Theertha Suresh.
\newblock Remember what you want to forget: algorithms for machine unlearning.
\newblock \emph{Advances in Neural Information Processing Systems (NeurIPS)},
  34:\penalty0 18075--18086, 2021.

\bibitem[Shokri et~al.(2017)Shokri, Stronati, Song, and
  Shmatikov]{shokri2017membership}
Reza Shokri, Marco Stronati, Congzheng Song, and Vitaly Shmatikov.
\newblock Membership inference attacks against machine learning models.
\newblock \emph{IEEE Symposium on Security and Privacy (SP)}, pages 3--18,
  2017.

\bibitem[Song et~al.(2017)Song, Ristenpart, and Shmatikov]{song2017machine}
Congzheng Song, Thomas Ristenpart, and Vitaly Shmatikov.
\newblock Machine learning models that remember too much.
\newblock \emph{Proceedings of the Conference on on Computer and
  Communiscations Security (CCS)}, pages 587--601, 2017.

\bibitem[Song et~al.(2021)Song, Sohl-Dickstein, Kingma, Kumar, Ermon, and
  Poole]{song2020score}
Yang Song, Jascha Sohl-Dickstein, Diederik~P Kingma, Abhishek Kumar, Stefano
  Ermon, and Ben Poole.
\newblock Score-based generative modeling through stochastic differential
  equations.
\newblock \emph{International Conference on Learning Representations (ICLR)},
  2021.

\bibitem[Sun et~al.(2021)Sun, Li, Wang, Yang, Li, and Chen]{sun2021soteria}
Jingwei Sun, Ang Li, Binghui Wang, Huanrui Yang, Hai Li, and Yiran Chen.
\newblock Soteria: Provable defense against privacy leakage in federated
  learning from representation perspective.
\newblock \emph{Conference on Computer Vision and Pattern Recognition (CVPR)},
  2021.

\bibitem[Sun et~al.(2016)Sun, Yang, Sun, and Wang]{sun2016benchmark}
Xiaoxiao Sun, Jufeng Yang, Ming Sun, and Kai Wang.
\newblock A benchmark for automatic visual classification of clinical skin
  disease images.
\newblock \emph{European Conference on Computer Vision (ECCV)}, 2016.

\bibitem[{TorchVision}(2016)]{torchvision2016}
{TorchVision}.
\newblock Torchvision: Pytorch's computer vision library, 2016.
\newblock URL \url{https://github.com/pytorch/vision}.

\bibitem[Tran and Fioretto(2023)]{tran2023personalized}
Cuong Tran and Ferdinando Fioretto.
\newblock Personalized privacy auditing and optimization at test time.
\newblock \emph{arXiv preprint arXiv:2302.00077}, 2023.

\bibitem[Vaswani et~al.(2017)Vaswani, Shazeer, Parmar, Uszkoreit, Jones, Gomez,
  Kaiser, and Polosukhin]{vaswani2017attention}
Ashish Vaswani, Noam Shazeer, Niki Parmar, Jakob Uszkoreit, Llion Jones,
  Aidan~N Gomez, {\L}ukasz Kaiser, and Illia Polosukhin.
\newblock Attention is all you need.
\newblock \emph{Advances in Neural Information Processing Systems (NeurIPS)},
  30, 2017.

\bibitem[Vershynin(2018)]{vershynin2018high}
Roman Vershynin.
\newblock \emph{High-dimensional probability: An introduction with applications
  in data science}, volume~47.
\newblock Cambridge university press, 2018.

\bibitem[Wang et~al.(2019)Wang, Song, Zhang, Song, Wang, and
  Qi]{wang2019beyond}
Zhibo Wang, Mengkai Song, Zhifei Zhang, Yang Song, Qian Wang, and Hairong Qi.
\newblock Beyond inferring class representatives: User-level privacy leakage
  from federated learning.
\newblock \emph{IEEE Conference on Computer Communications}, 2019.

\bibitem[White(2020)]{dmvselling2020}
Annie White.
\newblock Dmvs can (and do) collect and sell your personal data.
\newblock \emph{Car and Driver}, 2020.

\bibitem[Wu et~al.(2023)Wu, Zhou, Weinberger, and Guo]{wu2022does}
Ruihan Wu, Jin~Peng Zhou, Kilian~Q Weinberger, and Chuan Guo.
\newblock Does label differential privacy prevent label inference attacks?
\newblock \emph{International Conference on Artificial Intelligence and
  Statistics (AISTATS)}, 2023.

\bibitem[Xiao et~al.(2020)Xiao, Tsai, Sohn, Chandraker, and
  Yang]{xiao2020adversarial}
Taihong Xiao, Yi-Hsuan Tsai, Kihyuk Sohn, Manmohan Chandraker, and Ming-Hsuan
  Yang.
\newblock Adversarial learning of privacy-preserving and task-oriented
  representations.
\newblock \emph{AAAI Conference on Artificial Intelligence}, 2020.

\bibitem[Yang et~al.(2018)Yang, Sun, Liang, and Rosin]{yang2018clinical}
Jufeng Yang, Xiaoxiao Sun, Jie Liang, and Paul~L Rosin.
\newblock Clinical skin lesion diagnosis using representations inspired by
  dermatologist criteria.
\newblock \emph{Conference on Computer Vision and Pattern Recognition (CVPR)},
  2018.

\bibitem[Yeom et~al.(2018)Yeom, Giacomelli, Fredrikson, and
  Jha]{yeom2018privacy}
Samuel Yeom, Irene Giacomelli, Matt Fredrikson, and Somesh Jha.
\newblock Privacy risk in machine learning: Analyzing the connection to
  overfitting.
\newblock \emph{2018 IEEE Symposium on Computer Security Foundations (CSF)},
  pages 268--282, 2018.

\bibitem[Yu et~al.(2022)Yu, Naik, Backurs, Gopi, Inan, Kamath, Kulkarni, Lee,
  Manoel, Wutschitz, et~al.]{yu2021differentially}
Da~Yu, Saurabh Naik, Arturs Backurs, Sivakanth Gopi, Huseyin~A Inan, Gautam
  Kamath, Janardhan Kulkarni, Yin~Tat Lee, Andre Manoel, Lukas Wutschitz,
  et~al.
\newblock Differentially private fine-tuning of language models.
\newblock \emph{International Conference on Learning Representations (ICLR)},
  2022.

\bibitem[Zhang et~al.(2024)Zhang, Dong, Wang, and
  Li]{zhang2024certifiedunlearningDNN}
Binchi Zhang, Yushun Dong, Tianhao Wang, and Jundong Li.
\newblock Towards certified unlearning for deep neural networks.
\newblock \emph{International Conference on Machine Learning (ICML)}, 2024.

\bibitem[Zhang et~al.(2017)Zhang, Bengio, Hardt, Recht, and
  Vinyals]{zhang2017understanding}
Chiyuan Zhang, Samy Bengio, Moritz Hardt, Benjamin Recht, and Oriol Vinyals.
\newblock Understanding deep learning requires rethinking generalization.
\newblock In \emph{International Conference on Learning Representations
  (ICLR)}, 2017.

\bibitem[Zhang et~al.(2022)Zhang, Guo, Wang, Xie, and Tao]{zhang2022survey}
Rui Zhang, Song Guo, Junxiao Wang, Xin Xie, and Dacheng Tao.
\newblock A survey on gradient inversion: Attacks, defenses and future
  directions.
\newblock \emph{International Joint Conferences on Artificial Intelligence
  (IJCAI)}, 2022.

\bibitem[Zhao et~al.(2020)Zhao, Mopuri, and Bilen]{zhao2020idlg}
Bo~Zhao, Konda~Reddy Mopuri, and Hakan Bilen.
\newblock idlg: Improved deep leakage from gradients.
\newblock \emph{arXiv preprint arXiv:2001.02610}, 2020.

\bibitem[Zhao et~al.(2024)Zhao, Kurmanji, B{\u{a}}rbulescu, Triantafillou, and
  Triantafillou]{zhao2024makes}
Kairan Zhao, Meghdad Kurmanji, George-Octavian B{\u{a}}rbulescu, Eleni
  Triantafillou, and Peter Triantafillou.
\newblock What makes unlearning hard and what to do about it.
\newblock \emph{Advances in Neural Information Processing Systems (NeurIPS)},
  2024.

\bibitem[Zhou et~al.(2023)Zhou, Gao, Wu, Grundy, Chen, Chen, and
  Li]{zhou2023modelobfuscator}
Mingyi Zhou, Xiang Gao, Jing Wu, John Grundy, Xiao Chen, Chunyang Chen, and
  Li~Li.
\newblock Modelobfuscator: obfuscating model information to protect deployed
  ml-based systems.
\newblock \emph{Proceedings of the Symposium on Software Testing and Analysis},
  pages 1005--1017, 2023.

\bibitem[Zhu et~al.(2019)Zhu, Liu, and Han]{zhu2019deep}
Ligeng Zhu, Zhijian Liu, and Song Han.
\newblock Deep leakage from gradients.
\newblock \emph{Advances in Neural Information Processing Systems (NeurIPS)},
  2019.

\end{thebibliography}
